\documentclass{article}
\usepackage{float}
\usepackage[utf8]{inputenc} 
\usepackage[T1]{fontenc}    
\usepackage{microtype}
\usepackage{graphicx}
\usepackage{booktabs} 

\usepackage{hyperref}

\usepackage{url,cite}            
\usepackage{booktabs}       
\usepackage{amsmath,amsfonts}       
\usepackage{nicefrac}       
\usepackage{microtype}      
\usepackage{array}
\RequirePackage{etexcmds}
\usepackage{color}
\usepackage{paralist}
\usepackage[shortlabels]{enumitem}

\usepackage{graphicx}
\usepackage{epstopdf}
\usepackage{paralist}
\usepackage{subcaption}
\usepackage{multirow}
\usepackage[table,xcdraw]{xcolor}
\usepackage{amsmath,stackrel}
\usepackage{amsfonts}
\usepackage{mathalfa}
\usepackage{amssymb,mathtools}
\usepackage{amscd}
\usepackage{multirow}
\usepackage{makecell}
\usepackage{booktabs}
\usepackage[ruled,vlined,commentsnumbered,linesnumbered,lined,boxed,nofillcomment]{algorithm2e}

\usepackage[capitalise,noabbrev]{cleveref}
\SetKwProg{Fn}{Function}{}{}
\newcommand{\smartparagraph}[1]{\noindent{\bf #1}\ }
\usepackage{subcaption}
\usepackage{verbatim}
\usepackage{enumerate}                     
\usepackage{amsthm}
\usepackage{eucal}
\usepackage[nolist]{acronym}  
\usepackage[normalem]{ulem}
\usepackage{xpatch}

\usepackage[colorinlistoftodos,bordercolor=orange,backgroundcolor=orange!20,linecolor=orange,textsize=scriptsize]{todonotes}

\theoremstyle{plain}

\theoremstyle{plain}

\theoremstyle{plain}
\theoremstyle{plain}

\newtheorem{lemma}{Lemma}
\newtheorem{corollary}{Corollary}[lemma]
\theoremstyle{remark}

\theoremstyle{remark}

\theoremstyle{plain}
\newtheorem{definition}{Definition}
\theoremstyle{plain}
\newtheorem{property}{Property}
\theoremstyle{plain}

  \makeatletter
\xpatchcmd{\proof}{\topsep6\p@\@plus6\p@\relax}{}{}{}
\makeatother

\newenvironment{myreasoning}[1][]{%
  	\proof[\textit{Reasoning}]%
  	}{\endproof}

  	\newtheorem{innercustomgeneric}{\customgenericname}
\providecommand{\customgenericname}{}
\newcommand{\newcustomtheorem}[2]{%
  \newenvironment{#1}[1]
  {%
   \renewcommand\customgenericname{#2}%
   \renewcommand\theinnercustomgeneric{##1}%
   \innercustomgeneric
  }
  {\endinnercustomgeneric}
}

\newcustomtheorem{customthm}{Theorem}
\newcustomtheorem{customlemma}{Lemma}
 \newcustomtheorem{customcorollary}{Corollary}
\newcommand{\norm}[1]{\left\lVert#1\right\rVert}

\newcommand{\R}{\mathbb{R}} 
\newcommand{\N}{\mathbb{N}} 

\usepackage{titlesec}
\titlespacing*{\section}{0pt}{0.01\baselineskip}{0.01\baselineskip}
\titlespacing*{\subsection}{0pt}{0.01\baselineskip}{0.01\baselineskip}
\titlespacing*{\chapter}{0pt}{-8ex}{10ex}
\usepackage{tikz}
\usepackage{tikzscale}
\usetikzlibrary{shapes.geometric,arrows,positioning,patterns,matrix,spy}
\usepackage{pgfplots}
\usepackage{pgfplotstable}
\usepgfplotslibrary{patchplots}
\pgfplotsset{compat=1.16}

\newlength\figureheight
\newlength\figurewidth
\usepackage[accepted]{mlsys2021}

\mlsystitlerunning{An Efficient Statistical-based Gradient Compression Technique for  Distributed Training Systems}

\title{An Efficient Statistical-based Gradient Compression Technique for Distributed Training Systems}

\begin{document}


\begin{acronym}

\acro{CER}{character error rate}
\acro{CNN}{convolution neural network}
\acro{CDF}{cumulative distribution function}
\acro{CCDF}{complementary cumulative distribution function}
\acro{cs}[CS]{compressed sensing}                   
\acrodef{cscapital}[CS]{Compressed sensing}
\acrodef{CS}[CS]{compressed sensing}

\acro{DSSGD}{distributed synchronous SGD}
\acro{DNN}{deep neural network}
\acro{DGC}{deep gradient compression}
\acro{DR}{distortion-rate}

\acro{EC}{Error compensation}

\acro{i.i.d.}{independent, identically distributed}

\acro{NLogL}{negative log likelihood}

\acro{MM}{moment matching}
\acro{ML}{maximum likelihood}
\acro{MLE}{maximum likelihood estimate}
\acro{MSE}{mean square error}

\acro{RD}[RD]{raw data}
\acro{RDL}{"random data limit"}
\acro{r.v.}{random variable}                               
\acro{R.V.}{random vector}
\acro{RMS}{root mean square}
\acro{RNN}{recurrent neural network}

\acro{PDF}{probability density function}
\acro{PoT}{peak over threshold}
\acro{PPF}{percent-point function}
\acro{SGD}{Stochastic Gradient Descent}
\acro{SPD}[SID]{sparsity-inducing distribution}
\acro{SIDCo}{Sparsity-Inducing Distribution-based Compression}
\acro{GPD}[GP]{generalized Pareto}

\acro{WER}{word error rate}
\end{acronym}
\def\figpath{Figures/ElzanatyFittedDistribution}
\newcommand{\scheme}{\textit{\ac{SIDCo} }} 
\def\d{d} 
\def\k{k} 
\def\kh{\hat{\k}} 
\newcommand{\topk}{{\text{Top}_{\k}}} 
\newcommand{\randk}{{\text{Rand}_{\k}}} 
\def\N{N}  
\def\x{\mathbf x}
\def\xi [#1]{\x_{\{#1\}}}
\def\xin[#1]{{\x}_{\{#1\}}^{n}}
\def\gr{g} 
\def\g{\mathbf \gr}
\def\gi [#1]{{\g}_{\{#1\}}}
\def\Gi [#1]{{\G}_{\{#1\}}}
\def\gin [#1]{{\g}_{\{#1\}}^{n}}
\def\Gin [#1]{{\G}_{\{#1\}}^{n}}
\def\G{G} 
\def\C{\mathbb{C}} 
\def\Chat{\mathbb{C}_{\hat{\k}}} 
\def\Tk[#1]{{\mathbb{T}}_{\k}\left\{ #1 \right\} } 
\def\Ceta{\mathbb{C}_{\eta}} 
\def\sigmak{\sigma_{\k}}  
\def\p{p} 
\def\b{\beta} 
\def\a{\alpha} 
\def\loc{a}
\def\bhat{\hat{\beta}} 
\def\ahat{\hat{\alpha}} 
\def\lochat{\hat{a}}
\def\sign{\textrm{sign}}                                              
\def\erf{\textrm{erf}}
\def\erfc{\textrm{erfc}}
	\newcommand{\normo}[1]{{\left\lVert#1\right\rVert}_{0}}


\twocolumn[
\mlsystitle{An Efficient Statistical-based Gradient Compression Technique for Distributed Training Systems}



\mlsyssetsymbol{equal}{*}

\begin{mlsysauthorlist}
\mlsysauthor{Ahmed M. Abdelmoniem}{equal,kaust}
\mlsysauthor{Ahmed Elzanaty}{equal,kaust}
\mlsysauthor{Mohamed-Slim Alouini}{kaust}
\mlsysauthor{Marco Canini}{kaust}
\end{mlsysauthorlist}

\mlsysaffiliation{kaust}{Computer, Electrical and Mathematical Sciences and Engineering (CEMSE) Division, King
	Abdullah University of Science and Technology (KAUST),
	Thuwal 23955-6900, Saudi Arabia (email:{ahmed.sayed,
		ahmed.elzanaty,slim.alouini,marco}@kaust.edu.sa)}

\mlsyscorrespondingauthor{Marco Canini}{marco@kaust.edu.sa}

\mlsyskeywords{Distributed Deep Learning, Gradient Compression, Statistical Methods}

\vskip 0.3in

\begin{abstract}
The recent many-fold increase in the size of deep neural networks makes efficient distributed training challenging. Many proposals exploit the compressibility of the gradients and propose lossy compression techniques to speed up the communication stage of distributed training. Nevertheless, compression comes at the cost of reduced model quality and extra computation overhead. In this work, we design an efficient compressor with minimal overhead. Noting the sparsity of the gradients, we propose to model the gradients as random variables distributed according to some \acp{SPD}. We empirically validate our assumption by studying the statistical characteristics of the evolution of gradient vectors over the training process. We then propose \scheme\!, a threshold-based sparsification scheme that enjoys similar threshold estimation quality to \ac{DGC} while being faster by imposing lower compression overhead. Our extensive evaluation of popular machine learning benchmarks involving both \ac{RNN} and \ac{CNN} models shows that \scheme\! speeds up training by up to $\approx\!41.7\times$, $7.6\times$, and $1.9\times$ compared to the no-compression baseline, $\topk$, and \ac{DGC} compressors, respectively.
\end{abstract}
]

\printAffiliationsAndNotice{\mlsysEqualContribution} 

\section{Introduction}

As \acp{DNN} continue to become larger and more sophisticated, and ever increasing amounts of training data are used~\cite{megatronml,gpt3}, scaling the training process to run efficiently on a distributed cluster is currently a crucial objective that is attracting a multitude of efforts~\cite{nvidiasummit,PipeDream,bytescheduler,DML_survey}.
Modern deep learning toolkits~\cite{pytorch,tensorflow,horovod} are capable of distributed data-parallel training whereby the model is replicated and training data are partitioned among workers. Training \acp{DNN} in such settings in practice relies on synchronous distributed \ac{SGD} or similar optimizers (refer to \cref{apdx:dsgd} for more details).
Let $\N$ be the number of workers, and $\xi[i] \in \R^{\d}$ denote the model parameters with $\d$ dimensions at iteration $i$. At the end of the $i^{\text{th}}$ iteration, each worker runs the back-propagation algorithm to produce a local stochastic gradient, $\gin[i] \in \R^d$, at worker $n$. Then, each worker updates its model parameters using the final gradient aggregated across all workers as $\xi[i+1] = \xi[i] - \lambda \frac{1}{N} \sum_{n=1}^{N} \gin[i]$, where $\lambda$ is the learning rate. Gradient aggregation involves communication, which is either between the workers in a peer-to-peer fashion (typically through collective communication primitives like all-reduce) or via a parameter server architecture. Due to the synchronous nature of the optimizer, workers cannot proceed with the $(i+1)^{\text{th}}$ iteration until the aggregated gradient is available. Therefore, in distributed training workloads, communication is commonly one of the predominant bottlenecks~\cite{lin2018deep,Fang2019}. 

Addressing this communication bottleneck is the focus of this paper, where we pursue the path of improving training by reducing the communicated data volume via lossy gradient compression. Compression entails two main challenges: $(i)$ it can negatively affect the training accuracy (because the greater the compression is, the larger the error in the aggregated gradient), and $(ii)$ it introduces extra computation latency (due to the compression operation itself).
While the former can be mitigated by applying compression to a smaller extent or using compression with error-feedback~\cite{lin2018deep,ef-sgd}, the latter, if gone unchecked, can actually slow down training compared to not compressing. Surprisingly, much of the prior works in this area ignored the computation overheads of compression. Given that modern clusters for deep learning workloads nowadays use high speed, low latency network fabrics (e.g., $100$~Gbps Ethernet or InfiniBand), we argue that the efficiency of compression needs to be explicitly accounted for.

Motivated by the above observations, we propose \scheme\! compression.\footnote{Our code release is available at \url{https://github.com/sands-lab/SIDCo}.} \scheme\! builds on a sound theory of signal compressibility and enjoys linear complexity in the size of model parameters. Importantly, this affords for an implementation that parallelizes very efficiently using modern GPUs and other hardware targets. Thus, our work addresses a previously-overlooked yet crucial technical obstacle to using compression in practice, especially for communication-bounded training of large models.

\subsection{Related Work}

Efficient communication in distributed training has received extensive attention~\cite{PipeDream,Wangni18,grace}. One approach tries to maximize the overlap between the computation and communication to hide the communication overhead~\cite{PipeDream,bytescheduler}. However, the gains from these methods are bounded by the length of computation and are modest when the training is dominantly communication-bound. Alternatively, many approaches adopt methods that reduce the amount of communication, volume~\cite{lin2018deep} or frequency~\cite{patel2019communication}. In this work, we focus on gradient compression as it shows considerable benefits. 

\smartparagraph{Gradient Compression} is a well-known volume reduction technique~\cite{lin2018deep,Fang2019,Wangni18,Ahmed-AAAI-2020,grace}. Each worker applies a compression operator  $\C$ to  $\gin[i]$ to produce a compressed gradient vector that is transmitted for aggregation. Generally, the compressor $\C$ involves  quantization and/or sparsification operations.

\smartparagraph{Gradient Quantization} 
represents gradients with fewer bits for each gradient element. Under some conditions, quantization is known to achieve the same convergence as no compression~\cite{wu_memqsgd,pmlr-v119-fu20c}.
\ac{EC} is used to attain convergence when gradients are quantized using fewer bits~\cite{wu_memqsgd,ef-sgd,grace}. 
Given the standard 32-bit float number representation, the volume reduction of quantization is limited by $32\times$, i.e., $1$~bit out of $32$~bits, 
which may not be sufficient for large models or slow networks and it requires expensive encoding to pack the quantized bits~\cite{Ahmed-CONEXT-2020}.

\smartparagraph{Gradient Sparsification} 
selects a subset of gradient elements. It is generally more flexible than quantization, as it can reduce volume by up to $d\times$ and adapts easily to network conditions~\cite{Ahmed-DC2-INFOCOM21}.
It was shown that in some cases, up to 99.9\% of the non-significant gradient elements can be dropped with limited impact on convergence~\cite{lin2018deep,aji_sparse,shi2019understanding}. Gradient sparsification using $\topk$ -- selecting the top $k$ elements by their magnitude -- is known to yield better convergence compared to other compression schemes, e.g., Random-$k$~\cite{lin2018deep,Alistarh18_sparse}. However, $\topk$ or its variants are notorious for being computationally inefficient~\cite{grace}. 
$\topk$ selection does not perform well on accelerators such as GPUs~\cite{ShanbhagMIT2018}. For instance, in many cases, it is reported that $\topk$ imposes high overheads and worsens the run-time of distributed training~\cite{shi2019understanding, grace}.

\subsection{Background and Motivation}
The main challenge with using gradient compression (e.g., sparsification or quantization) is the computational overhead it introduces in the training.  If the overhead is greater than the reduction gains in communication time, the overall iteration time increases. Hence, to be useful, a robust compressor should have a low overhead~\cite{Fang2019,shi2019understanding}.
As presented earlier, one of the dominantly robust compressors is $\topk$, however it is also computationally heavy~\cite{Fang2019,shi2019understanding,shi2021towards}.
Because of this, $\topk$, for large models, results in either an increased training time or unsatisfactory performance benefits.

Numerous efforts based on algorithmic or heuristic approaches have been dedicated to enhancing the performance of $\topk$~\cite{lin2018deep,shi2019understanding,ShanbhagMIT2018,SketchML}. Existing fast implementations of $\topk$ are compute-intensive (e.g., on CPU, the computational complexity is $\mathcal{O}(\d~\log_{2}\k)$)~\cite{ShanbhagMIT2018}. Recently, more optimized implementations for multi-core hardware are proposed, which greatly depend on the data distribution and work best for a small number of $k$~\cite{ShanbhagMIT2018}. For instance, the Radix select algorithm used in PyTorch is $\mathcal{O}\left(\lceil{b}/{r}\rceil~\d\right)$ where $b$ is the number of bits in the data values and $r$ is the radix size~\cite{pytorch}. Yet, using gradient vectors of various sizes, $\topk$ is the slowest on GPUs and not the fastest on CPUs, as shown later in our micro-benchmark and in \cref{apdx:moremicrobench}.

In the context of gradient compression, {\em Threshold-based methods}, aiming to overcome the overhead of $\topk$, select in linear time gradient elements larger in magnitude than a threshold $\eta$. \ac{DGC}~\cite{lin2018deep} proposes to sample a random subset of the gradients (e.g., 1\%), apply $\topk$ on the sampled sub-population to find a threshold which is then used to obtain the actual $\topk$ elements hierarchically.\footnote{Aside from the expensive random sampling, in worst case, DGC invokes $\topk$ twice, once on the subset to obtain a threshold and another to obtain $\k$ elements if the number of elements obtained via the threshold are more than the target $\k$.} Even though DGC leads to improved performance over $\topk$, its computational complexity is still in the same order of $\topk$'s complexity. {\em Threshold estimation methods} on the other hand, are shown to attain linear time complexity~\cite{aji_sparse,Alistarh18_sparse,Dryden2016CommunicationQF}.

Recently, several works have leveraged certain features of the gradients to enhance the training process~\cite{Narang2018,pmlr-v119-fu20c}. Some approaches leveraged these features and devised heuristics to estimate and find the $\topk$ threshold which exhibits lower compression overhead compared to $\topk$ and \ac{DGC}~\cite{Fang2019,shi2019understanding}. 
%
In particular, RedSync~\cite{Fang2019} finds the threshold by moving the ratio between the maximum and mean values of the gradient; GaussianKSGD~\cite{shi2019understanding} adjusts an initial threshold obtained from fitting a Gaussian distribution through an iterative heuristic to obtain the $\topk$ elements. Nevertheless, the threshold estimation of these methods is of bad quality and the number of selected elements, $\hat{k}$,  varies significantly from the target $\k$ (\cref{sec:experiments}).
 
In this work, we propose a statistical approach to estimate an accurate threshold for selecting the $\topk$ elements with minimal overhead. In particular, we exploit the compressibility of the gradients and opt for \acp{SPD} that fit the gradients well. For instance, double exponential (i.e., Laplace), double gamma and double generalized Pareto distributions have been used as sparsity-promoting priors in Bayesian estimation framework \cite{MonMouUma:18,ArmDunLee:13,BabMolKat:10}. Our study of the gradients supports the assumption for their compressibility and suitability for modeling the gradients as \acp{r.v.} distributed according to one of the \acp{SPD}. 


 \begin{figure*}[t]
  \centering
   \begin{subfigure}[ht]{0.4\linewidth}
  \includegraphics[width=1\linewidth]{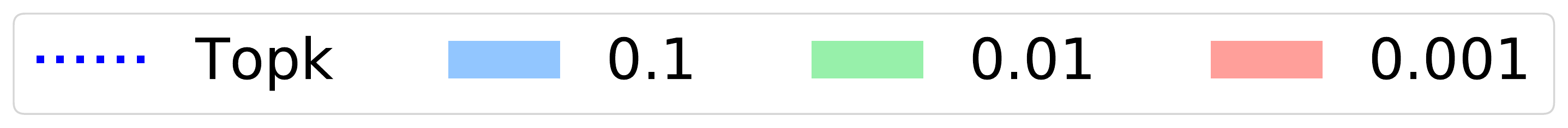}
 \end{subfigure}
  \\
    \begin{subfigure}[ht]{0.30\linewidth}
   \includegraphics[width=1\textwidth]{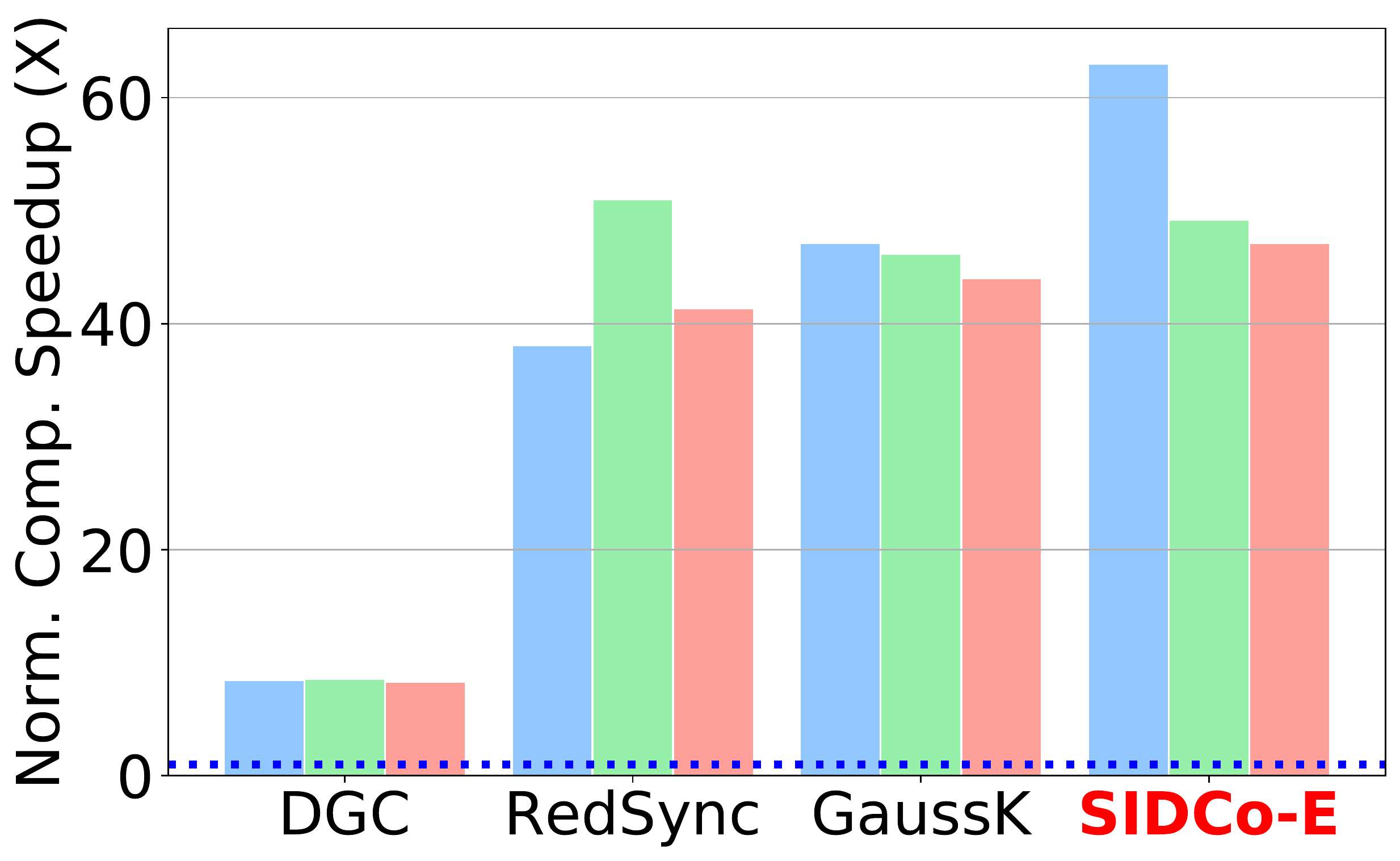}
      \caption{Compression with GPU}
    \label{fig:vgg16-cuda-speedup}
     \end{subfigure}
        \hfill
  \begin{subfigure}[ht]{0.30\linewidth}
     \includegraphics[width=1\textwidth]{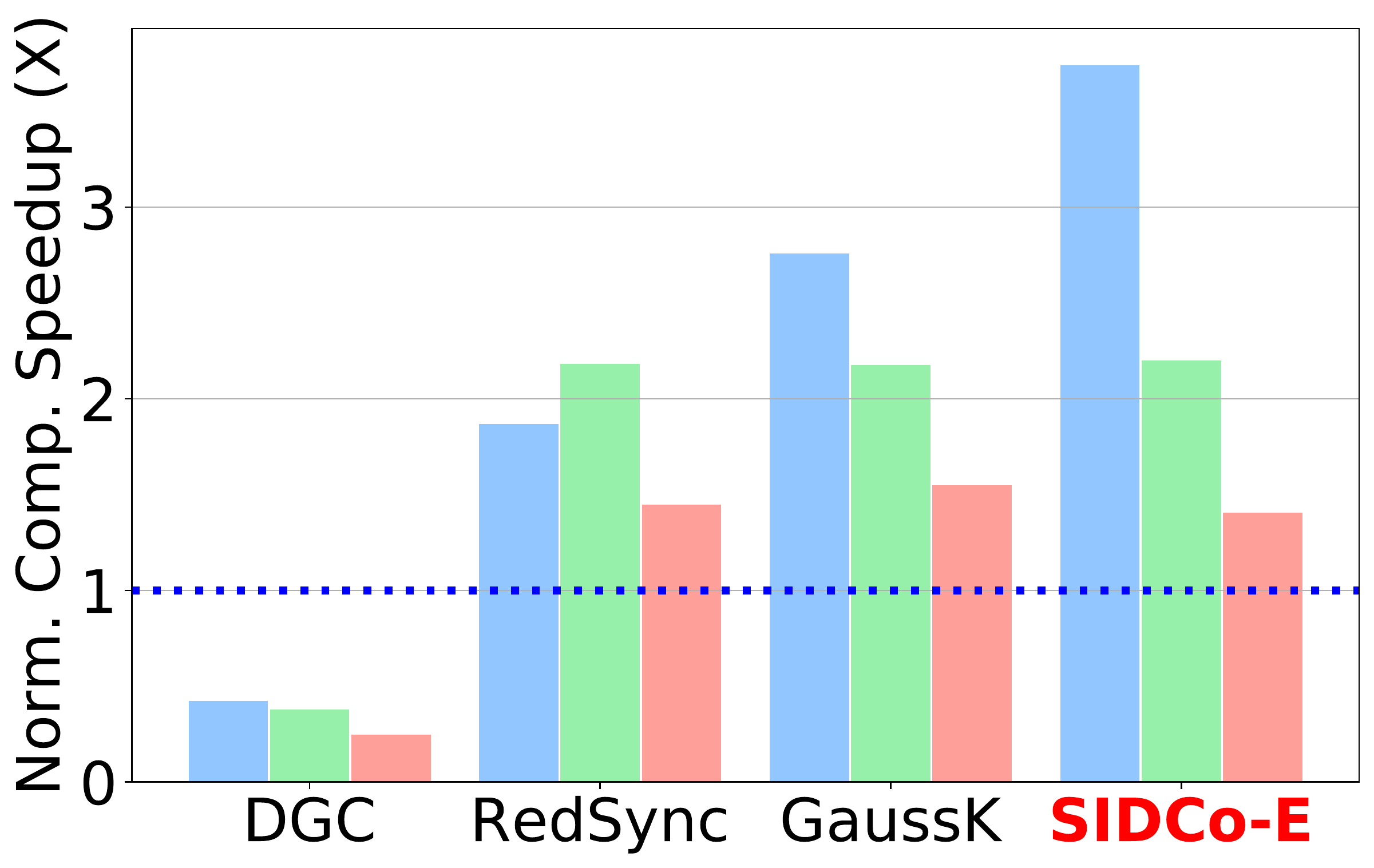}
    	\caption{Compression with CPU}
     \label{fig:vgg16-cpu-speedup}
    \end{subfigure}
    \hfill
    \begin{subfigure}[ht]{0.31\linewidth}
   \includegraphics[width=1\textwidth]{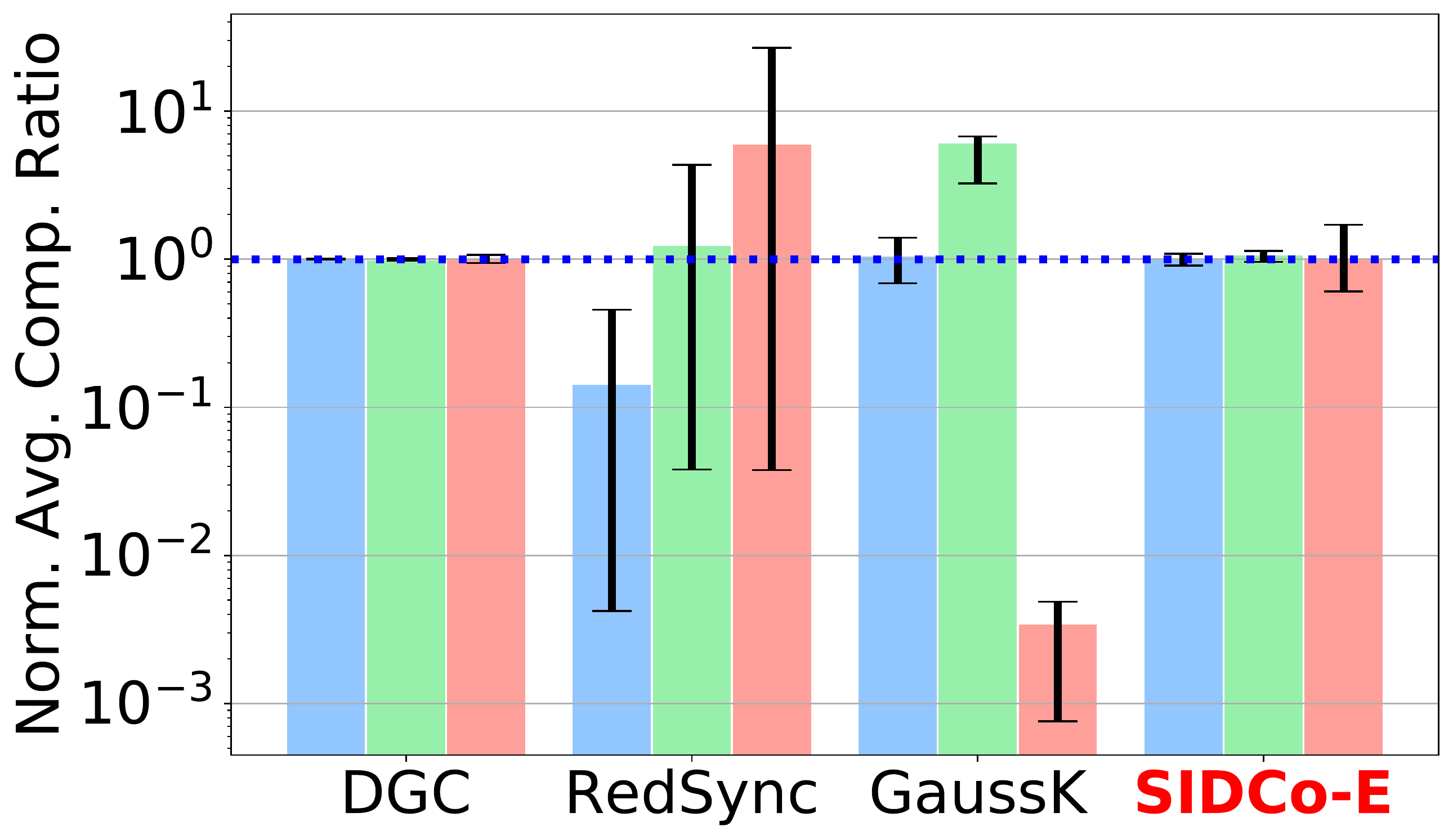}
    \caption{Quality of Threshold Estimation}
     \label{fig:vgg16-good-8}
    \end{subfigure}
    \caption{The compression speedups over $\topk$ using different compression ratios $(0.1, 0.01, 0.001)$, on (a) GPU and (b) CPU. (c) shows the average estimation quality of the target $k$. The experiments are performed for VGG16 (\cref{tab:models}) with the setup detailed in \S\ref{sec:algorithm}.}
\end{figure*}

\begin{table*}[!t]
\caption{Summary of the benchmarks used in this work.}
\centering
\scalebox{0.61}{
    \begin{tabular}{ccccrrrrccc }
    \toprule
        Task & \makecell{Neural \\ Network} & Model 	&	Dataset	& \makecell{Training \\ Parameters} & \makecell{Per-Worker \\Batch Size} & \makecell{Learning \\ Rate} & Epochs & \makecell{Comm \\ Overhead} & \makecell{Local \\ Optimizer} & \makecell{Quality \\ metric}   
        \\\midrule
       
         \multirow{3}{*}{\makecell{Language \\ Modeling}} & \multirow{3}{*}{RNN} & \multirow{3}{*}{\makecell{ LSTM\\ \citep{lstm} \\ 2 layers-1500 hidden units}} & \multirow{3}{*}{\makecell{PTB\\ \citep{ptb}}} & 66,034,000 & 20 & 22 & 30 & \textbf{\textcolor{red}{94\%}} & NesterovMom-SGD & Test Perplexity
         \\\\\\
          \midrule
          
         \multirow{2}{*}{\makecell{Speech \\ Recognition}} & \multirow{2}{*}{RNN} &  \multirow{2}{*}{\makecell{ LSTM \\ 5 layers-1024 hidden units}} & \multirow{2}{*}{\makecell{AN4\\ \citep{an4}} }& 43,476,256 & 20 & 0.004 & 150 & \textbf{\textcolor{magenta}{80\%}} & NesterovMom-SGD & WER \& CER  
         \\\\ 
         \midrule
         
        \multirow{6}{*}{\makecell{Image \\ Classification}} & \multirow{6}{*}{CNN} & {\makecell{ResNet-20\\ \citep{resnet152}}} \ & {\makecell{CIFAR-10\\ \citep{cifar10}}}	&	269,467 & 512 & 0.1 & 140 & \textcolor{black}{10\%} & SGD &\multirow{6}{*}{Top-1 Accuracy} 
        \\ 
        & & {\makecell{VGG16\\ \citep{vgg}}} & CIFAR-10 &	14,982,987 & 512 & 0.1 & 140 &  \textbf{\textcolor{orange}{60\%}} & SGD &  
        \\
         & & ResNet-50 & {\makecell{ImageNet\\ \citep{imagenet}}} &	25,559,081 & 160 & 0.2 & 90 & \textbf{\textcolor{orange}{72\%}}  & NesterovMom-SGD &   
         \\
          & & VGG19 & ImageNet & 143,671,337 & 160 & 0.05  & 90 & \textbf{\textcolor{magenta}{83\%}} &  NesterovMom-SGD &  \\ 
        \bottomrule
    \end{tabular}
    \label{tab:models}
    }
\end{table*}

To motivate our approach, we conduct initial micro-benchmark experiments to evaluate the compression overhead of sparsification techniques: $\topk$, \ac{DGC} (which uses random sub-sample for threshold calculation), RedSync and GaussianKSGD (which heuristically estimate the threshold), and one of our proposed \scheme\! schemes that estimates the threshold via a multi-stage fitting (\cref{sec:analysis}). We use both CPU and GPU to benchmark the performance (see~\cref{apdx:clusters}). We show the speed-up of different compressors normalized by the compression speed of $\topk$. We observe from the results that methods based on random sub-sampling (e.g., \ac{DGC}) excel on GPU (\cref{fig:vgg16-cuda-speedup}), but they imposes huge overhead on CPU and leads to DGC performing significantly worse than $\topk$ on CPU (\cref{fig:vgg16-cpu-speedup}). In contrast, methods that are based on estimating a threshold over which only $\k$ elements are selected, impose consistently lower compression overhead compared to $\topk$ and \ac{DGC} on both GPU and CPU. This shows that, except for linear time threshold-based methods, a variable compression overhead is to be expected on different architectures (e.g., CPU, GPU, TPU, FPGA or AI chips).\footnote{We note that many efforts are dedicated to the optimization and enabling of fast training on low-cost devices such as CPUs instead of opting for expensive hardware accelerations~\cite{Vincent2011,Das2018,Beidi2020}.} \cref{fig:vgg16-good-8} shows the normalized actual compression ratio (i.e., $\hat{k}/k$) for various schemes; note that the heuristic approaches fail to obtain the right threshold, leading to unpredictable behavior. 

\subsection{Contributions}
In this work, we make the following contributions:
\begin{itemize}[noitemsep,topsep=0pt,leftmargin=10pt]
\item We exploit the sparsity of the gradients via modeling the gradients as \acp{r.v.} with \acp{SPD} and propose a multi-stage fitting technique based on \ac{PoT} which works well with aggressive sparsification ratios and adapts to the distribution changes of the gradient.
\item We design \scheme\!, a threshold sparsification  method with  closed-form expressions for three \acp{SPD} to keep the compression overhead as low as possible. 
\item We show that \scheme\! consistently outperforms existing approaches via an extensive set of numerical and experimental evaluation on different benchmarks.
\end{itemize} 

\section{Proposed Gradient Model and Threshold Estimation}
\label{sec:analysis}
We discuss the compressibility of the gradients and their statistical distribution. Then, two threshold-based schemes are proposed that leverage the compressibility of the gradients.

\subsection{Gradient Compressibility} Signals, including gradient vectors of \acp{DNN}, can be efficiently compressed by exploiting some of their inherent features. Among these features, sparsity and compressibility are the key drivers for performing signal compression~\cite{Mal:08,Elzanaty19,ElzGioChi:19}.
%

We start by a precise definition of compressible signals.
\begin{definition}[Compressible Signals \cite{BarDavDua:11}]\label{def:compressable}
     The signal ${\g}\in \mathbb{R}^\d$ is compressible if the magnitudes of its sorted coefficients obey the following power law decay: 
\begin{equation}\label{eq:powerlaw}
\tilde{g}_{j} \leq c_{1}\, j^{-\p} \quad \forall j \in \{1,2,\cdots, \d \},  
\end{equation}
where $\tilde{\g}$ is the  sorted vector of ${|\g|}$ in descending order, ${\tilde{g}}_{j}$ is the $j^{\text{th}}$ element of  $\tilde{\g}$, and $\p>1/2$ is the decay exponent, for some constant $c_{1}$. For compressible signals with power law decay, the sparsification error, $\sigmak(\g)$, is bounded as
\begin{equation}\label{eq:bestkapprox}
    \sigmak (\g) \triangleq \norm{\g -\Tk[\g]}_{2} \leq c_{2}\, \k^{1/2-\p},
    \end{equation}
where $\norm{\x}_{q}={\left(\sum_{j=1}^{\d} \x_{j}^{q}\right)}^{1/q} $ is the $\ell_{q}$ norm of $\x$, $\Tk[\cdot]$ is the $\topk$ sparsification operator that keeps only the largest $\k$ elements in magnitude and set the others to zero, $\Tk[\g]$ is a \k-sparse vector with only $\k$ non-zero elements, and $c_2$ is a constant. The signal is more compressible if it decays faster, i.e., $\p$ is higher~\cite{Devore:98}.
\end{definition}
 \begin{property}[Gradients Compressibility]\label{property:gradientcompressible}
	The gradients generated during the training of most \acp{DNN} are compressible in the sense of \cref{def:compressable}. 
 \end{property}
 %
 \begin{myreasoning}
 From \cref{def:compressable}, it can be verified whether the gradient vectors are compressible. 
%
In \cref{apdx:statmethods}, we empirically validate that the gradients generated during the training of widely-adopted \acp{DNN} respect the condition for compressibility stated in   \eqref{eq:powerlaw} and \eqref{eq:bestkapprox}. 

 \end{myreasoning}
\subsection{Gradient Modeling}
The target now is to find the distribution of the gradient vector, while accounting for the compressibility of the gradients. The selection of sparsity-promoting priors that are able to efficiently capture the statistical characteristics of the  gradients with low computational complexity is a challenging task. However, we notice an essential property for the distribution of the gradients that permits high compression gains with low computational overhead.  
\begin{property}
\label{property:sparspromotdist}
 Gradients generated from many \acp{DNN} during the training can be modeled as \acp{r.v.} distributed according to some \aclp{SPD}, i.e., double exponential, double gamma and double \ac{GPD} distributions. More precisely, we have
 \begin{equation}
    \G \mathrel{\dot\sim}  \operatorname{Distribution}(\boldsymbol{\Theta}),
 \end{equation}
 where $\operatorname{Distribution}({\cdot})$ is one of the three \acp{SPD} with parameters indicated by the vector $\boldsymbol{\Theta}$ that generally depends on the iteration and worker's data. Also, the \ac{PDF} of $\G$, $f_{\G}(\gr;\boldsymbol{\Theta})$, is symmetric around zero.  
 \vspace{-0.1cm}
\end{property}
%
\vspace{-0.1cm}
\begin{myreasoning}
      Since the gradients are compressible as indicated by Property~\ref{def:compressable}, they  can be well approximated by sparse vectors with minimal error, as implied from \eqref{eq:bestkapprox}. Hence,  the distributions that promote sparsity are good candidates for fitting (or modeling) the gradient vectors.\footnote{For threshold estimation, we are interested in the distribution of the amplitude of a random element in the gradient vector.} For instance, the double exponential, double gamma,  double \ac{GPD}, and Bernoulli-Gaussian distributions have been used as priors that promote sparsity in \cite{MonMouUma:18,ArmDunLee:13,BabMolKat:10,Elzanaty19}.
   Property~\ref{property:sparspromotdist} is empirically verified for several \ac{DNN} architectures and datasets in \cref{sec:algorithm} and \cref{apdx:graddist}. 
\label{apdx:graddist}
\label{sec:empvalid}
\begin{figure*}[!ht]
\centering
\begin{subfigure}[h]{0.46\textwidth}
\includegraphics[width=1\textwidth]{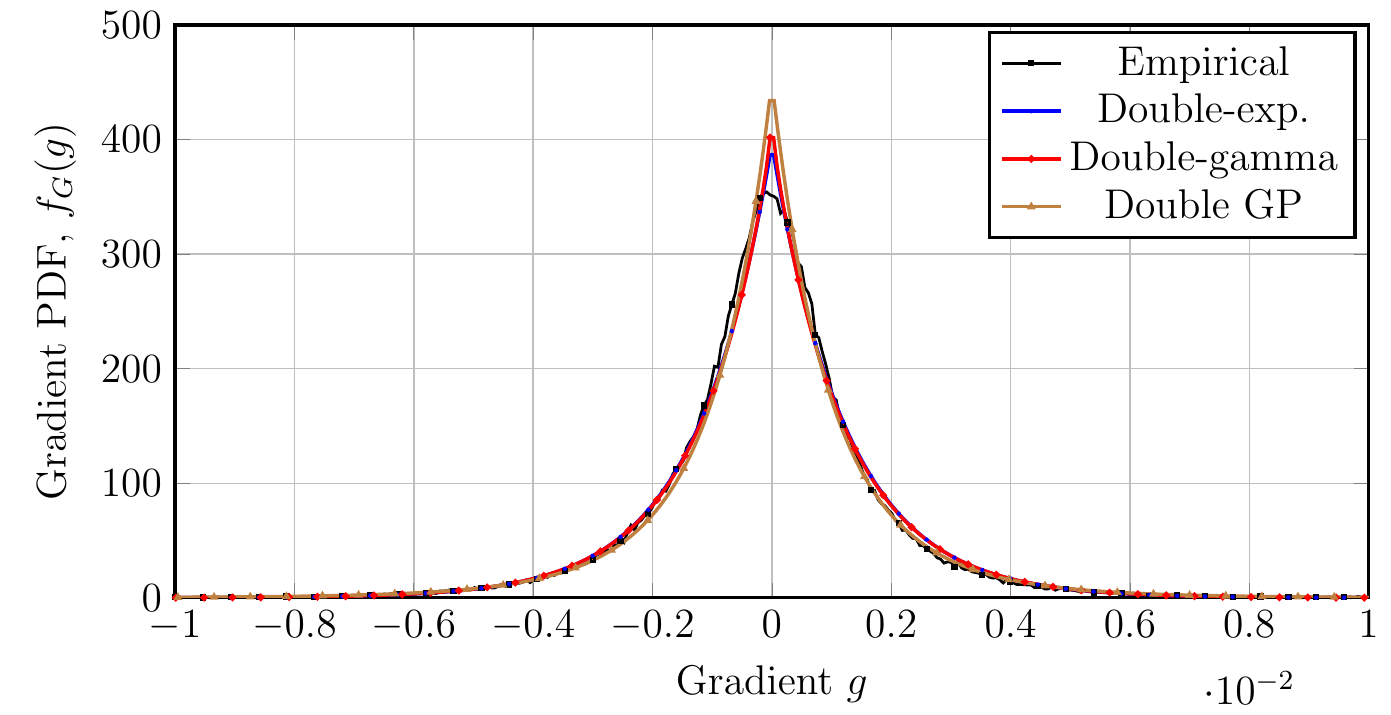}
    \caption{}
    \label{fig:PDF1}
    \end{subfigure}
\hfill
\begin{subfigure}[h]{0.46\textwidth}
 	\includegraphics[width=1\textwidth]{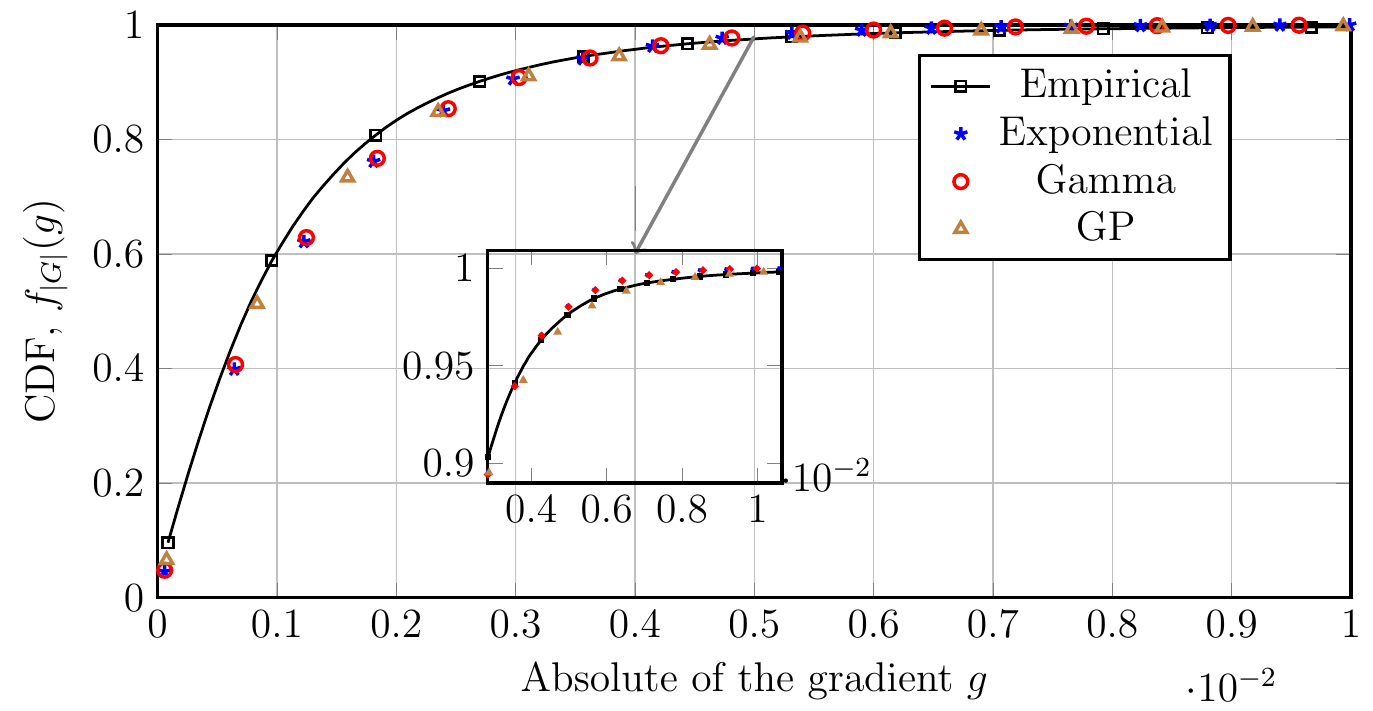}
    \caption{}
    \label{fig:CDF1}
\end{subfigure}
\\
\begin{subfigure}[h]{0.46\textwidth}
	\includegraphics[width=1\textwidth]{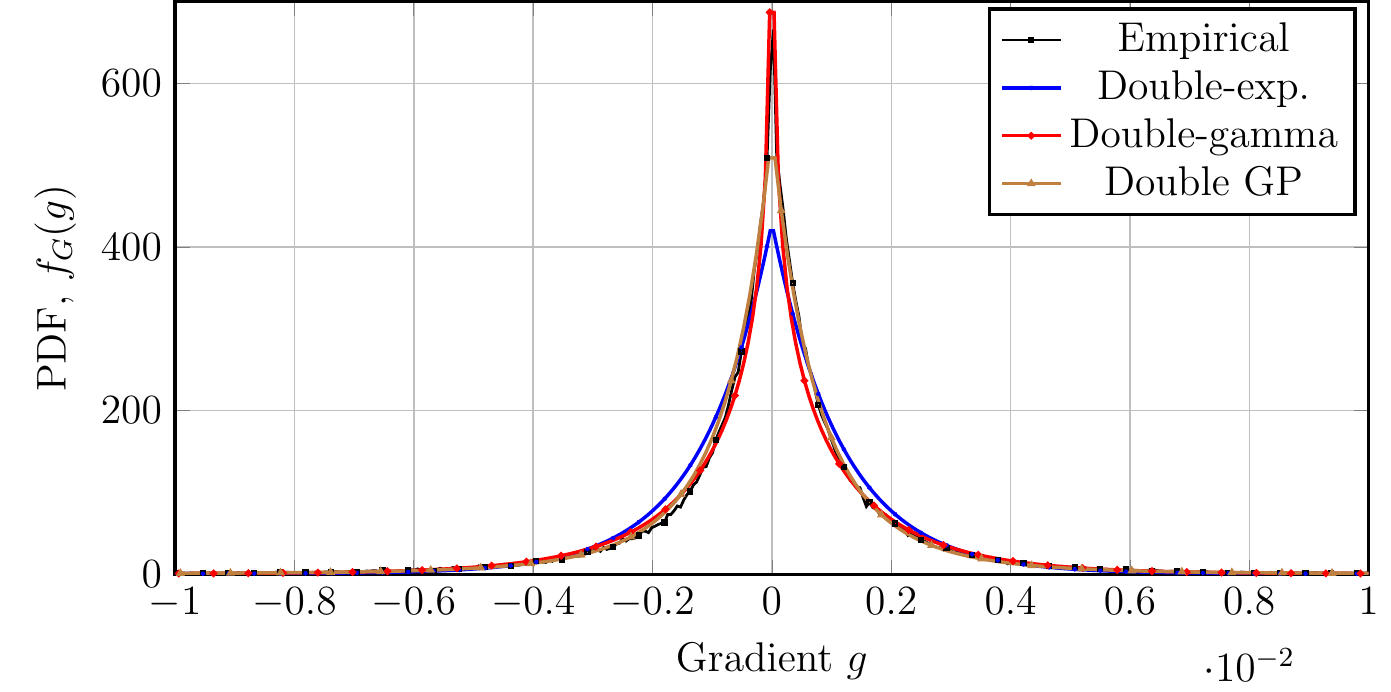}
    \caption{}
    \label{fig:PDF2}
    \end{subfigure}
\hfill
\begin{subfigure}[h]{0.46\textwidth}
\centering
\includegraphics[width=1\textwidth]{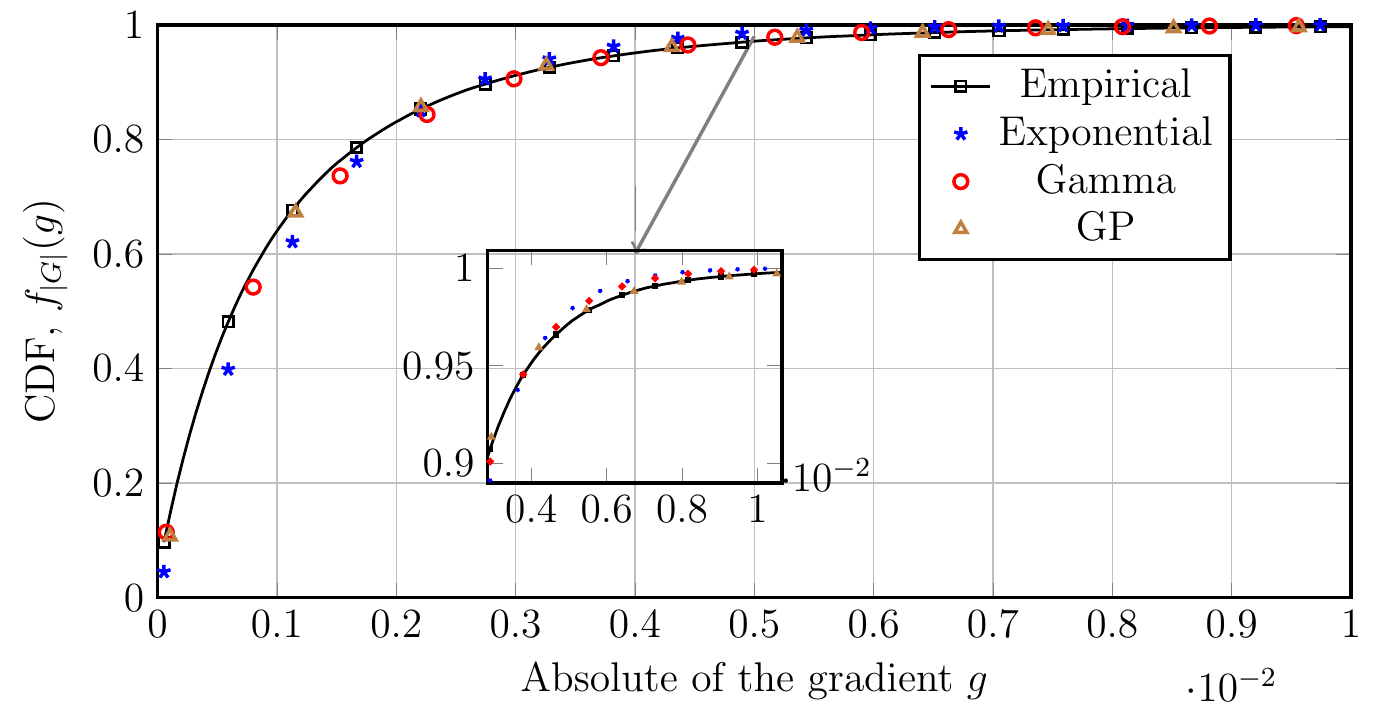}
\caption{}
    \label{fig:CDF2}
\end{subfigure}
\caption{Fitting using the three \acp{SPD} for the gradient vector along with the empirical distribution generated from training ResNet-20 on CIFAR10 using $\topk$ compressor without \ac{EC} mechanism, for the $100^{\text{th}}$ [(a) PDF, (b) CDF] and $10000^{\text{th}}$ [(c) PDF, (d) CDF] iterations.}
\label{fig:fitteddistributions}
\end{figure*}%
For instance, we consider the gradients resulting from the training of ResNet-20 with \ac{SGD}. The collected gradients are fitted by the three proposed \acp{SPD}, i.e., double exponential, double gamma, and double \ac{GPD} distributions. In \cref{fig:fitteddistributions}, the empirical distribution of the gradients and their absolutes, without \ac{EC} mechanism, are shown along with the distributions of the three fitted \ac{SPD} for two iterations.  We can notice in \cref{fig:PDF1} that the three proposed distributions can capture the main statistical characteristic of the gradients, as their \acp{PDF} approximate the empirical distribution for most of the gradient domain. This can be understood because of the compressibility of the gradients illustrated before. 
The compressibility of  \acp{r.v.} distributed according to one of the \acp{SPD} can be attributed to the shape of their \acp{PDF}, where the most probable values are those with small amplitudes. 
From \cref{fig:PDF1} and \cref{fig:PDF2}, it can be seen that the gradients at iteration $10000$ (\cref{fig:PDF2}) are more sparse than those at iteration $100$ (\cref{fig:PDF1}), where the \ac{PDF} at iteration $10000$ has higher values at smaller gradient values and it has faster tail. 
Regarding the \ac{CDF} of the absolute value of the gradients in \cref{fig:CDF1} and \cref{fig:CDF2}, we can see that the \acp{SPD} well approximate the empirical \ac{CDF}.
However, at the tail of the distribution, they tend to overestimate/underestimate the \ac{CDF} slightly. The reason is that the fitting is biased toward the majority of the data with lower values, as the gradient vector is sparse.   
\end{myreasoning}
\subsection{Single-Stage Threshold Estimator}\label{sec:sst-thresh}
We now describe the proposed compression scheme. First, the threshold that yields the target compression ratio, $\delta\triangleq \k/\d$, is derived for each of the three proposed \acp{SPD}. Then, we present a single-stage thresholding scheme for moderate compression ratios. For aggressive compression ratios with $\delta \ll 1$, e.g., $\delta \leq 0.001$, we propose a multi-stage thresholding scheme to accurately estimate the  threshold.
%
The sparsification threshold can be computed from the fitted distribution of the gradients as follows:
\begin{lemma}
		For $\G {\sim}\operatorname{Distribution}(\boldsymbol{\Theta})$ with \ac{CDF} $F_{\G}(\gr;\boldsymbol{\Theta})$, the threshold $\eta$ that yields the $\topk$ vector  with average target compression ratio $\delta \triangleq \k/\d$ can be derived as
	\begin{align} \label{eq:thresholdabs}
		\eta(\delta) &=  F_{|\G|}^{-1}(1-\delta;{\widehat{\boldsymbol{\Theta}}})  \\ &=F_{\G}^{-1}\left(1-\frac{\delta}{2};{\widehat{\boldsymbol{\Theta}}}\right),  \label{eq:threshold}
	\end{align}
where $\widehat{\boldsymbol{\Theta}}$ is the estimated parameters for the gradient distribution, $F_{|\G|}(g ;\widehat{\boldsymbol{\Theta}})$ is the \ac{CDF} of the absolute gradient ,  $F_{|\G|}^{-1}(p;\widehat{\boldsymbol{\Theta}}) \triangleq \left\{g\in \mathbb{R}^{+} : F_{|\G|}(g;\widehat{\boldsymbol{\Theta}})=p  \right\} $ is the inverse  \ac{CDF} of the absolute gradient at probability $p$, and $F_{\G}^{-1}(p;{\widehat{\boldsymbol{\Theta}}})$ is the inverse \ac{CDF} of the gradient, also known as quantile function or \ac{PPF}. 
\vspace{-0.2cm}
\end{lemma}

\begin{proof}
From Property~\ref{property:sparspromotdist}, the gradients can be modeled as \acp{r.v.} distributed according to a \ac{SPD} with \ac{CDF} $F_{G}(g)$. Next, we would like to drive a threshold $\eta$ such that on average the absolute values of $\k$ elements out of $\d$ are larger than $\eta$. The problem can be seen as a binomial random process, where the number of trials is $\d$, the success probability is the probability that the absolute of the gradient is larger than $\eta$, i.e., $p \triangleq \mathbb{P}\left\{\left|\G\right| \geq \eta  \right\}$, and the average number of successes (exceedances) is $\k$. In this process, the number of exceedances is a binomial \ac{r.v.} with $\d$ number of trials and  $p$ probability of success \cite{Papoulis:02}. The mean of the number of exceedances is $\d\,p$. In order to have, on average, $\k$ elements out of $\d$, we should have $\mathbb{P}\left\{\left|\G\right| \geq \eta  \right\}=\delta.$
Hence, the threshold $\eta$ is the $100(1-\delta)$ percentile of the distribution of absolute gradients as in \eqref{eq:thresholdabs}. From the symmetry of the gradient distribution around zero, we have  $\mathbb{P}\left\{\left|\G\right| \geq \eta  \right\}=2\,\mathbb{P}\left\{\G \leq -\eta  \right\}$. Therefore, from \eqref{eq:thresholdabs}, we get 
$
    \eta\!=\! - {F}_{\G}^{-1}\left({\delta}/{2};{\widehat{\boldsymbol{\Theta}}}\right)\!=\!F_{\G}^{-1}\left(1\!-\!{\delta}/{2};
    {\widehat{\boldsymbol{\Theta}}}\right).
$
\vspace{-10pt}
\end{proof}
In the following, we report the threshold calculation for gradients modeled by double exponential distribution. The corresponding analysis for double gamma and  \ac{GPD} is presented in \cref{apdx:threshodcalculation}. 
\begin{corollary}\label{corollary:Laplacethreshold}
For double exponentially distributed gradients with scale parameter $\b$ and location zero (symmetric around zero), i.e., $ \G{\sim}  \operatorname{Laplace}(\b)$, the threshold that achieves  $\delta$ compression ratio can be computed as
\begin{equation}\label{eq:LaplaceThreshold}
    \eta= \hat{\b} \log\left(\frac{1}{\delta}\right), \qquad \quad \hat{\b}\triangleq \frac{1}{\d}\,\sum_{j=1}^{\d} \left|\gr_{j}\right|,
\end{equation}
where $\hat{\b}$ is the \ac{MLE} of the scale parameter. 
\vspace{-0.1cm}
\end{corollary}
\begin{proof}
For $ \G {\sim} \operatorname{Laplace}(\b)$, the gradient absolute is modeled as exponential distribution with scale $\b$,  $\left|\G \right| \sim \operatorname{Exp}(\b)$ \cite{EvaHasPea:93}. From the inverse \ac{CDF} of exponential distribution at probability $p$, i.e., {$F_{|\G|}^{-1}=-{\b} \log(1-p)$}, the \ac{MLE} of $\b$ \cite{EvaHasPea:93}, and \eqref{eq:thresholdabs}, the threshold in \eqref{eq:LaplaceThreshold} follows. 
\end{proof}
\vspace{-7pt}
\textbf{Gradient compression through thresholding:}
After computing the threshold, the compressed gradient vector is found as $\widehat{{\gr}}_{j} =\Ceta\left\{\gr_{j} \right\} \triangleq \gr_{j}\,  \mathbb{I}_{\left\{\left|\gr_{i}\right| \geq \eta \right\}},$ for each $j \in \{1,2,\cdots, \d\},$
where the vector $\widehat{\g} \in \mathbb{R}^{\d}$ is the compressed gradient vector, $\mathbb{I}_{\{\text{condition}\}}$ is an indicator function that equals one when the condition is satisfied and zero otherwise. In the following, we denote by  $\bar{\g}$ and $\kh$ the vector that contains only the exceedance non-zero gradients and their number, respectively.\footnote{Note that the compressed vector $\widehat{{\gr}}_{j}$  coincides with the $\topk$ sparsified gradient with $\k=\kh$, i.e., $\Ceta\left\{\gr_{j} \right\}=\mathbb{T}_{\kh} \left\{\gr_{j} \right\}$.}

\textbf{Possible issues in far tail fitting:}
The target compression ratio $\delta$ can be as low as $10^{-4}$. Therefore, in order to accurately estimate the threshold, the fitted distribution should tightly resemble the gradient distribution at the tail. 
This is quite challenging because the estimation of the parameters tends to account more for the majority of data at the expense of the tail. Hence, the threshold obtained from single-stage fitting is accurate up to some moderate compression ratios.  For lower compression ratios, the threshold tends to underestimate/overestimate the target $\delta$. Hence, a more accurate tail fitting method is required to reduce the bias induced by the majority of non-significant gradients, as we show next.

\subsection{Multi-Stage Threshold Estimator}
\label{sec:mst-thresh}
 We propose a multi-stage fitting approach to overcome the far tail estimation problem. For convenience, we start with the two-stage approach. First, the gradients are fitted with one of the three \acp{SPD} and compressed using the proposed procedure in \cref{sec:sst-thresh} with a threshold $\eta_{1}$ computed to yield  an initial compression ratio $\delta_{1} \triangleq \k_{1}/\d > \delta$. Then, the  vector of the exceedance gradients, $\bar{\g}$, is used to fit another distribution, defined precisely below. Then, another threshold $\eta_{2}$ is computed to achieve a compression ratio $\delta_{2} \triangleq \k_{2}/\k_{1}$ with respect to the exceedance gradients. The second compression ratio is chosen such that the overall compression ratio of the original gradient vector is the target ratio $\delta$, i.e., $\delta_{2}={\delta}/{\delta_{1}}$. Then, the estimated threshold from the last stage is applied to compress the original gradient vector. This procedure can be extended to multi-stages such that $\delta=\prod_{m=1}^{M}\delta_{m}$, where $M$ is the number of stages. 
 
 The remaining question is whether the exceedance (known also as \ac{PoT}) gradients have the same distribution as the original gradients before the compression. The extreme value theory in statistics can provide an answer for this question \cite{KotNad:00,Smith:84,Leadbetter:91,Coles:01}. Let $\kh_{m}$ be the number of exceedance gradients after the $m^{th}$ thresholding stage. Then, if we apply a threshold operator on a sequence of  \acp{r.v.}, $|\G_{1}|, |\G_{2}|, \cdots, |\G_{\kh_{m-1}}|$, the distribution of the exceedance \acp{r.v.}, ${|\bar{\G}_{1}|}, {|\bar{\G}_{2}|}, \cdots, {|\bar{\G}_{\kh_{m}}|}$,  can be approximated by a \ac{GPD} distribution for large enough threshold and vector dimension, irrespective of the original distribution of the gradients. 
 
Next, we exploit the extreme value theory to compute the threshold for the multi-stage approach.  
\begin{lemma}\label{lemma:PoT}
Considering that for the $m^{th}$ thresholding stage with $m\geq 2$, the absolute of the exceedance gradients, $\bar{|\G|_{m}}$, can be modeled as  $\operatorname{GP}(\a_{m},\b_{m},\loc_{m})$, where $-1/2 < \a_{m} < 1/2$, $\b_{m}$, and $\loc_{m}={\eta}_{m-1}$ are the shape, scale, and location parameters. The threshold that achieves a compression ratio $\delta_{m}$ is obtained as
\begin{align}\label{eq:msthresholddgpd}
    \eta_{m} &= \frac{\bhat_{m}}{\ahat_{m}} \left(e^{-\ahat_{m} \log\left(\delta_{m}\right)}-1 \right) +{\eta}_{m-1}, \\
    \ahat_{m} &\triangleq \frac{1}{2}\,  \left(1-\frac{\bar{\mu}^2}{\bar{\sigma}^2} \right), \\
    \bhat_{m} &\triangleq \frac{1}{2}\, \bar{\mu} \left(\frac{\bar{\mu}^2}{\bar{\sigma}^2} +1  \right),
\end{align}
where ${\eta}_{m-1}$ is the threshold computed at the previous stage and
  $\bar{\mu}$ and $\bar{\sigma}^2$ are the sample mean and variance of  $|{\bar{\g}}_{m}|-\eta_{m-1}$, respectively. For the proof of \cref{lemma:PoT}, please refer to \cref{apdx:prooflemmaPoT}.
\vspace{-0.1cm}
\end{lemma}
\begin{corollary}\label{corollary:expPoT}
If the absolute of the gradients is modeled as  exponentially distributed \acp{r.v.}, $\left|\G_{m} \right| \sim \operatorname{Exp}(\b_{m})$, the distribution of the exceedance gradients over the threshold $\eta_{m-1}$, after proper shifting, is still  exponentially distributed, i.e., ${|\bar{\G}_{m}|}-\eta_{m-1}  \sim \operatorname{Exp}(\b_{m})$. The new stage threshold is 
\begin{align}
    \eta_{m}&= {\bhat}_{m} \log\left(\frac{1}{\delta_{m}}\right)+\eta_{m-1},\\ {\bhat}_{m}&\triangleq \frac{1}{\kh_{m-1}}\,\sum_{j=1}^{\kh_{m-1}} \left|\bar{\gr}_{j}\right|-\eta_{m-1},
\end{align}
where $\bar{\gr}_{j}$ is the $j^\text{th}$ element of the vector ${\bar{\g}}_{m}$. \textnormal{The proof is provided in \cref{apdx:proofcorollaryexpPoT}.} 
\end{corollary}
In the proposed scheme, we exploit \cref{corollary:expPoT} such that when the absolute of the gradients is fitted by an exponential distribution in the first stage, the latter stages for the exceedance gradients are also fitted by  exponential distributions, i.e., multi-stage exponential. On the other hand, for gamma-fitted absolute gradients in the first stages, the latter stages are fitted by a \ac{GPD} distribution, from \cref{lemma:PoT}, i.e., gamma-\ac{GPD}.  Finally, for \ac{GPD} distributed absolute gradients in the first stage, the \ac{GPD} is still used for the \ac{PoT} data, from \cref{lemma:PoT}, i.e., multi-stage \ac{GPD}. 



\begin{figure*}[t!]
	\captionsetup[subfigure]{justification=centering}
	\centering
	\begin{subfigure}[ht]{0.5\linewidth}
		\includegraphics[width=1\linewidth]{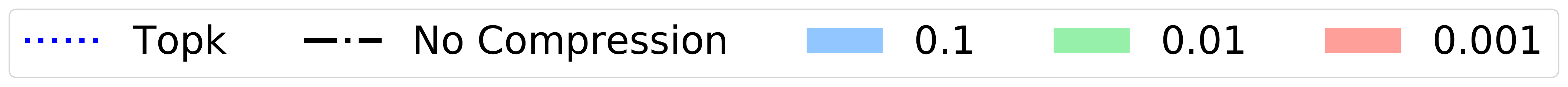}
	\end{subfigure}
	\\
	\begin{subfigure}[ht]{0.30\linewidth}
		\includegraphics[width=\linewidth]{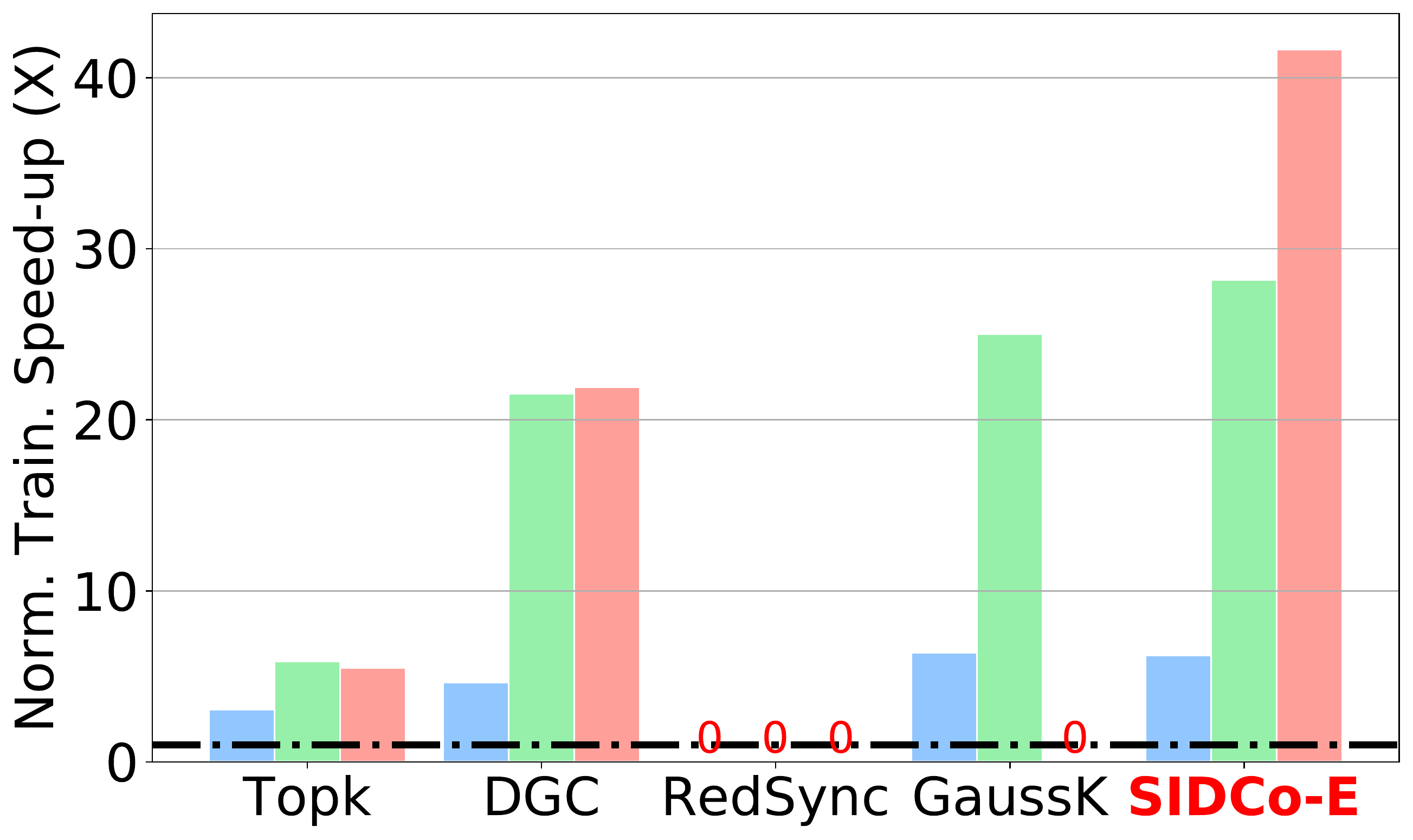}
		\caption{LSTM-PTB (Speed-up).}
		\label{fig:ptb-speedup-8}
	\end{subfigure}
	\hfill
	\begin{subfigure}[ht]{0.30\linewidth}
		\includegraphics[width=\linewidth]{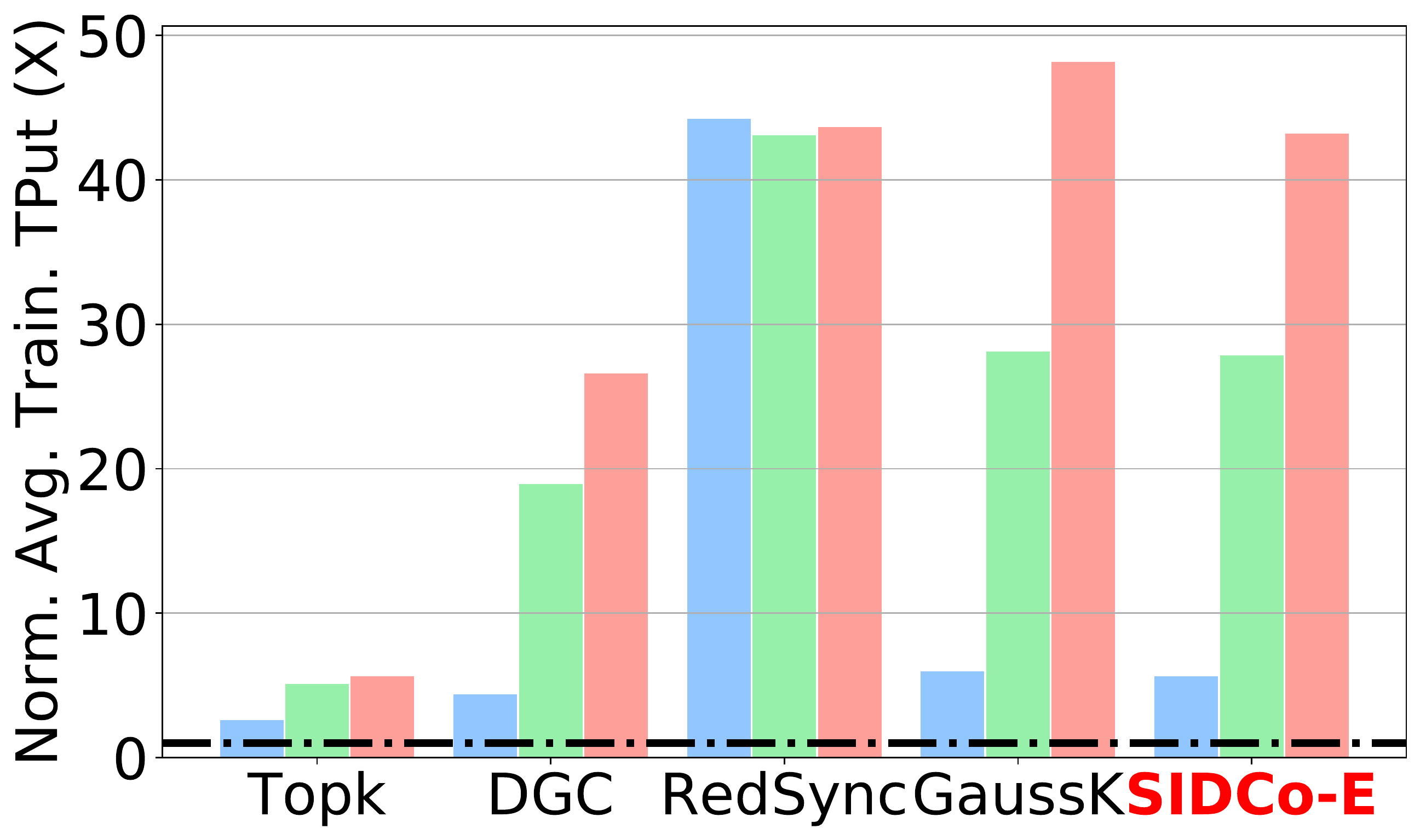}
		\caption{LSTM-PTB (Throughput).}
		\label{fig:ptb-throughput-8}
	\end{subfigure}
	\hfill
	\begin{subfigure}[ht]{0.30\linewidth}
		\includegraphics[width=\linewidth]{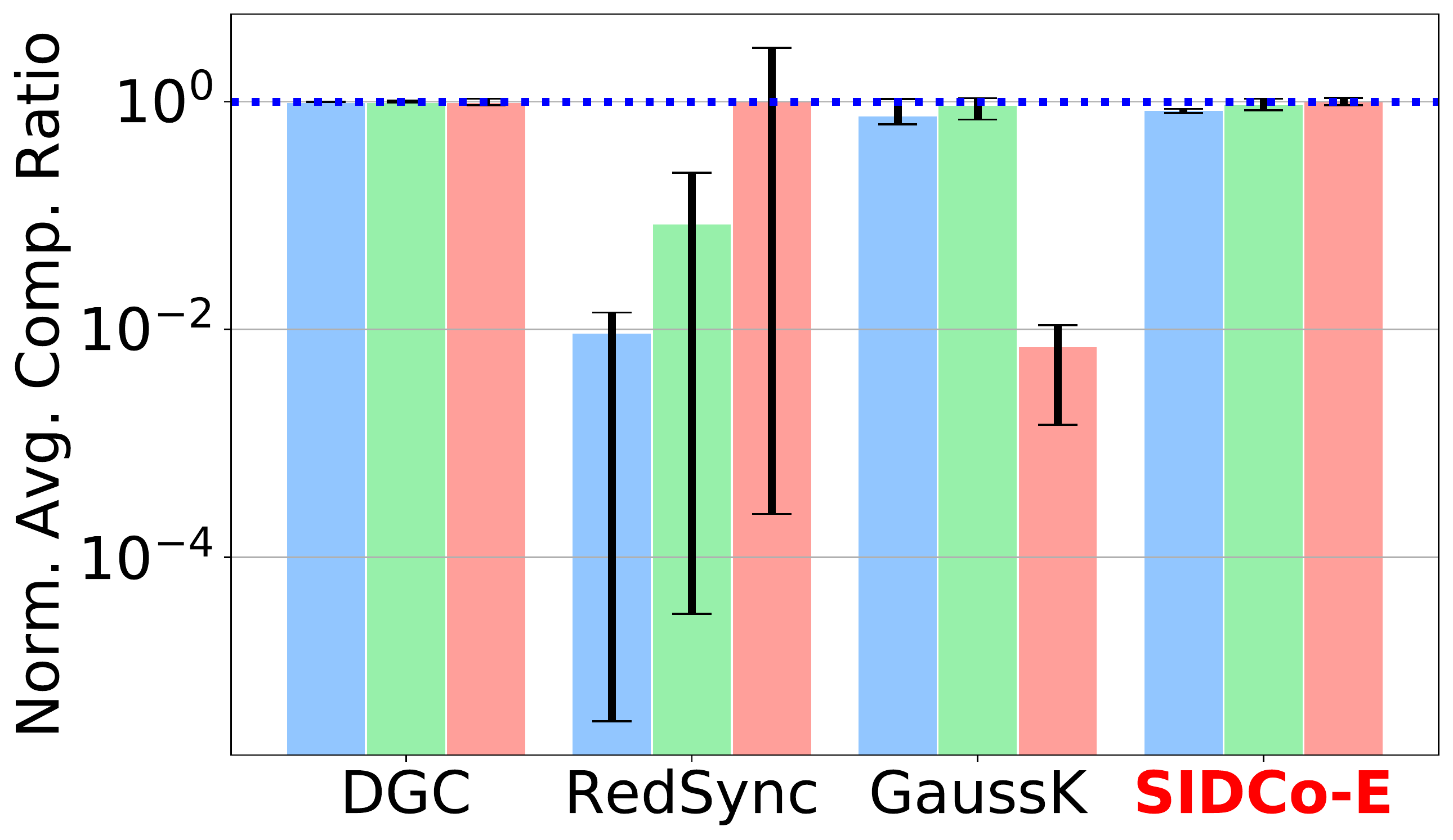}
		\caption{LSTM-PTB (Estimation Quality).}
		\label{fig:ptb-comp-8}
	\end{subfigure}
	\\
	\begin{subfigure}[ht]{0.30\linewidth}
		\includegraphics[width=\linewidth]{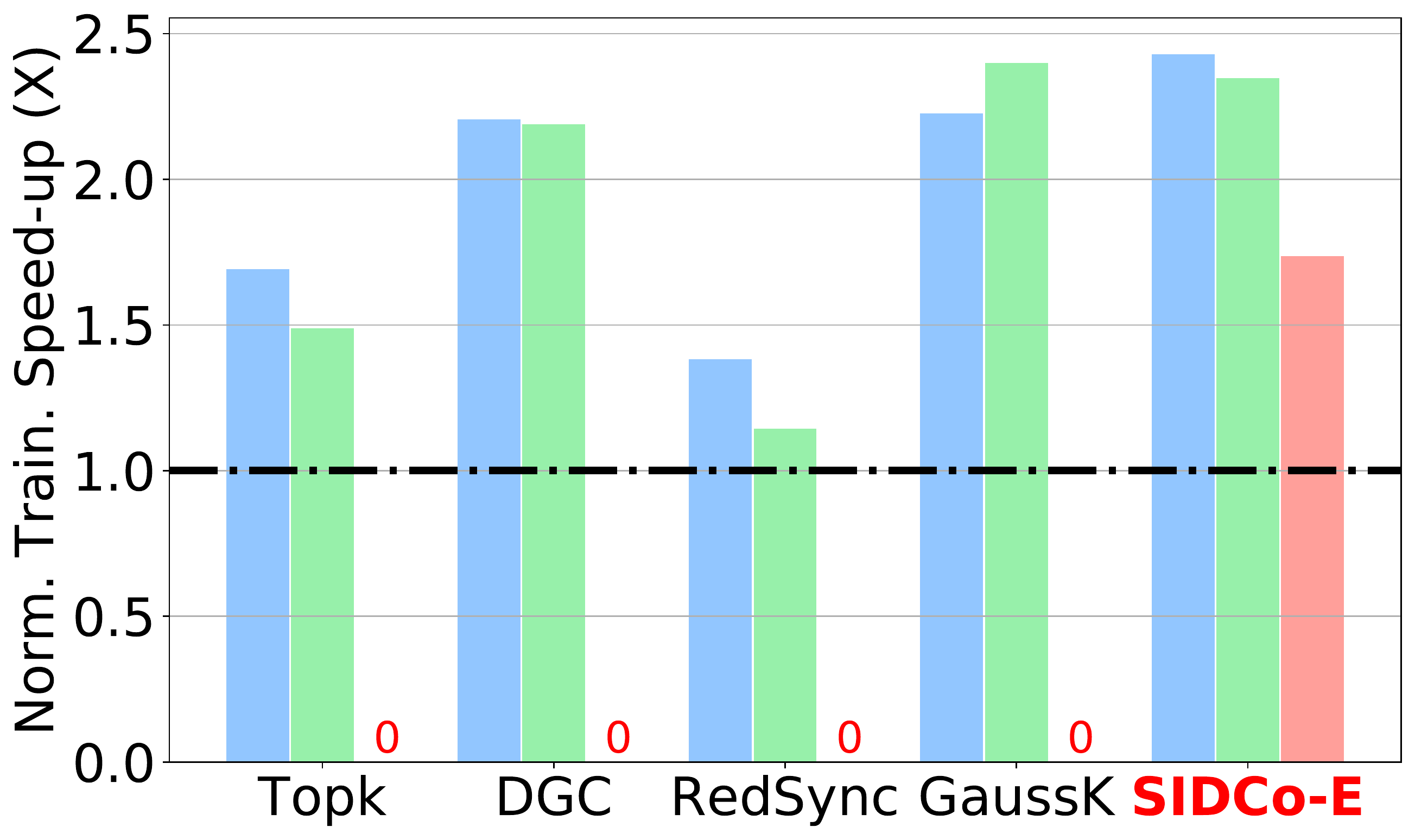}
		\caption{LSTM-AN4 (Speed-up).}
		\label{fig:an4-speedup-8}
	\end{subfigure}
	\hfill
	\begin{subfigure}[ht]{0.30\linewidth}
		\includegraphics[width=\linewidth]{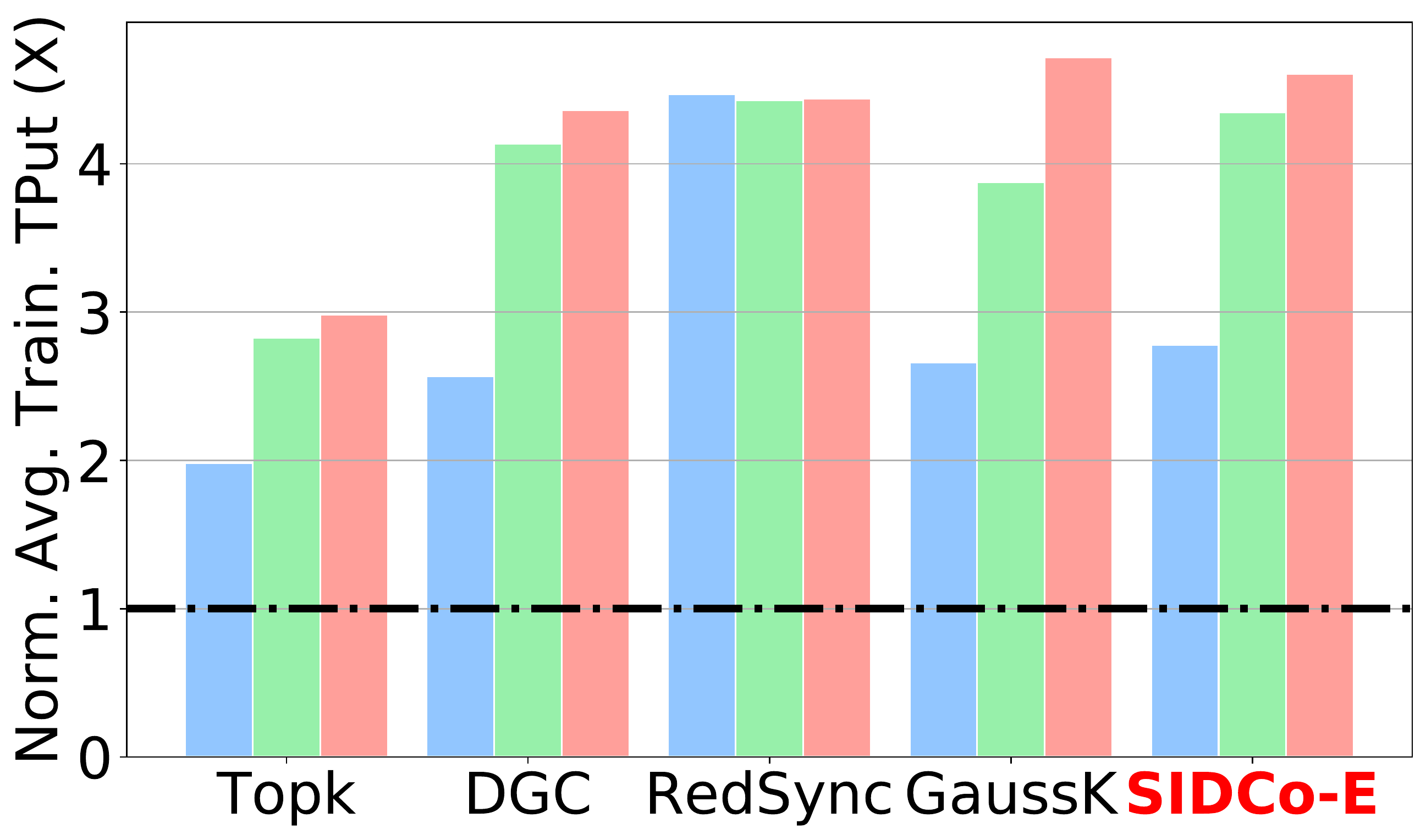}
		\caption{LSTM-AN4 (Throughput).}
		\label{fig:an4-throughput-8}
	\end{subfigure}
	\hfill
	\begin{subfigure}[ht]{0.30\linewidth}
		\includegraphics[width=\linewidth]{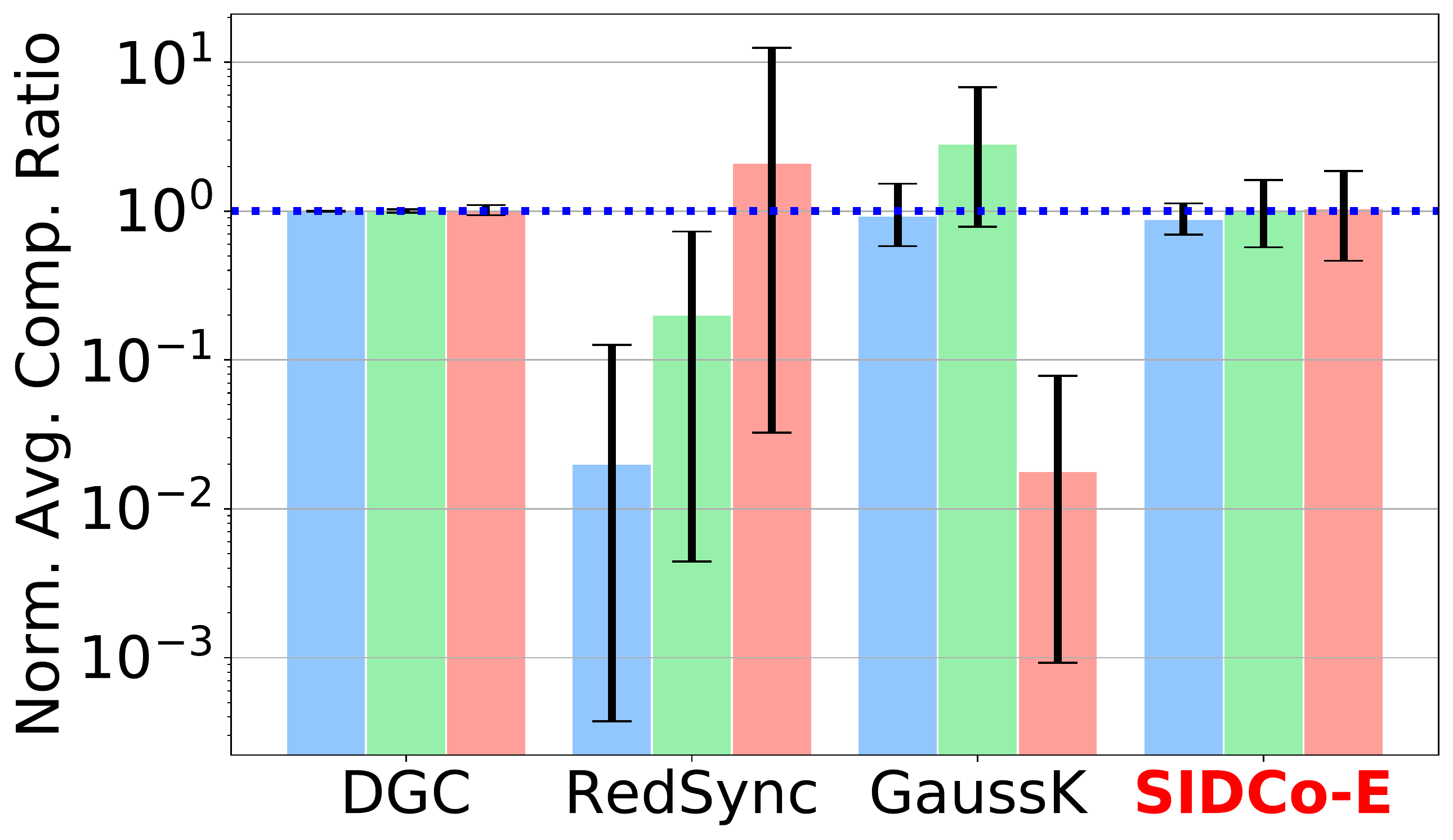}
		\caption{LSTM-AN4 (Estimation Quality).}
		\label{fig:an4-comp-8}
	\end{subfigure}
	\caption{Performance of training RNN-LSTM on PTB [(a),(b),(c)] and AN4  [(d),(e),(f)] datasets.}
	\label{fig:lstm}
\end{figure*}
%
\section{\ac{SIDCo} Algorithm}
\label{sec:algorithm}
\scheme\! leverages \acp{SPD} to obtain a threshold via the multi-stage threshold estimator described in \cref{sec:mst-thresh}. We select the number of stages, $M$, via an adaptive algorithm such that the estimation error, averaged over $Q$ iterations, is bounded below a predefined error tolerance, i.e, 
\begin{align}\label{eq:boundederror}
&\left\lvert  \hat{\delta}-\delta  \right\rvert \leq \epsilon \, \delta, && 0 \leq \epsilon < 1\,.
\end{align}
First, we describe the algorithm that \scheme\! follows to perform the gradient compression. The full pseudo-code is shown in \cref{algo:algo1} of the Appendix. The algorithm in each iteration, takes as input the gradient vector and produces a compressed vector. The vector is sparsified through the multi-stage fitting strategy described in \cref{sec:mst-thresh}. In each stage, the function {\em Thresh\_Estimation} uses the chosen \ac{SPD} to obtain a threshold. The algorithm dynamically adapts the number of stages $M$ by monitoring the quality of its estimated selection of elements and adjusting $M$ using function {\em Adapt\_Stages}. 

The algorithm starts by calling the {\em sparsify} function which takes the gradient and target ratio as the parameters. Then, the algorithm applies a multi-stage estimation loop of $M$ iterations. In each iteration, the vector is partially sparsified with the previously estimated threshold obtained from the previous stage $m-1$. Then, given the ratio $\delta_m$ at loop step $m$, the chosen \ac{SPD} distribution fitting is invoked via the function {\em Thresh\_Estimation} to obtain a new threshold. At the last stage (i.e., step $M$ of the loop), the resulting estimation threshold should approximate the threshold that would obtain the target ratio $\delta$ of the input vector. Then, the estimated threshold is used to sparsify the full gradient vector and obtain the values and their corresponding indices. For each invocation of the algorithm in each training iteration, the algorithm maintains statistics like the average ratio of the quality of its estimations over the past training steps $Q$. Then, at the end of every $Q$ training steps, the algorithm invokes {\em Adapt\_Stages} which adjusts the current number of stages $M$ based on user-defined allowable error bounds of the estimation (i.e., $\epsilon_H$ and $\epsilon_L$). After the adjustment, the next algorithm invocation will use the new number of stages $M$. The number of stages is adjusted only if the obtained ratio is not within the error bounds. 
\subsection{Convergence Analysis}
In the following, we present the convergence analysis of \scheme\!.
\begin{lemma}
\label{lemma:convanalysis}
Let  $\hat{\delta}$ be the average achieved compression ratio of the proposed scheme with bounded discrepancy with respect to the target $\delta$ with error tolerance $\epsilon$ as in \eqref{eq:boundederror} which is assured by Algorithm~\ref{algo:algo1} in the Appendix. Also, 
let $i$ be the current training iteration, then the convergence rate of the proposed scheme coincides with that of the \ac{SGD} if 
\begin{equation}
i>\mathbb{O}\left(\frac{1}{\delta^2 \, (1-\epsilon)^2}\right).
\end{equation}
\end{lemma}
\begin{proof}
The convergence of the proposed scheme would mainly follow the existing convergence analysis of $\topk$~\cite{Alistarh18_sparse, aji_sparse, stich2018sparsified}, because \scheme\!
is designed to estimate a threshold for obtaining the top $\k$ elements. In contrast to $\topk$, the number of non-zero elements in the proposed scheme is a binomial \acs{r.v.}, $\hat{K}$, and not a constant. Second, the expected value of the estimated number of non-zero elements, $\hat{k} \triangleq \mathbb{E}\{\hat{K}\}$, may not coincide with the target $\k$, due to a possible  mismatch between the assumed \acs{SPD} of the stochastic gradients and their original distribution. The complete proof is detailed in Appendix~\ref{appnd:conanalyproof}.
\end{proof}
The proper selection of the distribution as a \ac{SPD} permits the actual compression rate to approach the designed compression ratio with small $\epsilon$.  This can be seen from the extensive numerical results in plots showing the estimation quality of \cref{fig:vgg16-good-8}, \cref{fig:resnet20-good-8}, \cref{fig:resnet50-comp-8}, and \cref{fig:compratio}.
 One can notice that on average $\hat{\k}/\k \approx 1 $, hence it resembles $\topk$. For some rare extreme cases, we have $\hat{\delta} \geq 0.8\, \delta$ (i.e., $\epsilon=20\%$), meaning that we need at most about $50\%$ more iterations than $\topk$ to reach the rate of \ac{SGD}. 
\section{Experimental Evaluation}
\label{sec:algorithm}



This evaluation answers the following questions:\\
$\bullet$ What benefits, in terms of training speed-up and model quality, does \scheme\! provide compared to state-of-the-art approaches (\emph{gains in training time to accuracy})?\\
$\bullet$ Are the improvements of \scheme\! only due to its faster training over other schemes (\emph{training throughput gains})?\\
$\bullet$ How accurate is the the threshold estimation of \scheme\! compared to the state-of-the-art (\emph{estimation quality})?

In the following, we describe the main results, and present more experimental results and scenarios in \cref{apdx:moreexp}.



\subsection{Experimental Settings}
\label{sec:experiments}
%
%
Unless otherwise mentioned, the default settings of the experiments are as follows.\\
\smartparagraph{Environment:}
We perform our experiments on 8 server machines equipped with dual 2.6 GHz 16-core Intel Xeon Silver 4112 CPU, 512GB of RAM, and 10 Gbps NICs. Each machine has an NVIDIA V100 GPU with 16GB of memory. The servers run Ubuntu 18.04, Linux kernel 4.15.0. We use PyTorch 1.1.0 with CUDA 10.2 as the ML toolkit. We use Horovod 0.16.4 configured with OpenMPI 4.0.0 for collective communication.

\smartparagraph{Benchmarks and hyper-parameters:}  The benchmarks and hyper-parameters are listed in \cref{tab:models}. We use both \ac{CNN} and \ac{RNN} models for image classification and language modeling tasks, respectively. We use compression ratios ($\delta$) of 0.1 (10\%), 0.01 (1\%) and 0.001 (0.1\%) to span a wide range of the trade-off between compression and accuracy similar to prior work \cite{aji_sparse,Alistarh18_sparse,lin2018deep}. Further details of the environment, tasks and settings of the experiments are given in \cref{apdx:clusters}.
 

\smartparagraph{Compressors: } We compare \scheme\! with $\topk$, \ac{DGC}, RedSync and GaussianKSGD.
The \ac{EC} mechanism is employed to further enhance the convergence of \ac{SGD} with compressed gradients \cite{lin2018deep,ef-sgd}. For \scheme\!, we set $\delta_{1}=0.25$, $\epsilon=20\%$, and $Q=5$ iterations to adapt the stages as in \cref{algo:algo1}. For conciseness, we present the performance of \scheme\! with double exponential fitting (shown in the figures as \ac{SIDCo}-E).\footnote{The results for double \ac{GPD} (\ac{SIDCo}-GP) and double gamma (\ac{SIDCo}-P), presented in \cref{apdx:expalldist}, are quite similar.}  

\smartparagraph{Metrics:} We quantify the performance of a given scheme (i.e., \scheme\!, Top-$k$, DGC, RedSync or GaussianKSGD) via the following metrics:\\
$\bullet$ \emph{Normalized Training Speed-up:} We evaluate the model quality at iteration $T$ (the end of training) and divide it by the time taken to complete $T$ iterations. We normalize this quantity by the same measurement calculated for the baseline case. This is the normalized training speed-up relative to the baseline;\\
$\bullet$ \emph{Normalized Average Training Throughput:} is the average throughput normalized by the baseline's throughput which illustrates the speed-up from compression irrespective of its impact on model quality;\\
$\bullet$ \emph{Estimation Quality:} is the compression ratio ($\hat{k}/d$) averaged over the training divided by the target ratio ($\delta=k/d$) along with the $90\%$ confidence interval as error-bars. 

\begin{figure}[!t]
\captionsetup[subfigure]{justification=centering}
\centering
 \begin{subfigure}[ht]{0.65\linewidth}
  \includegraphics[width=1\linewidth]{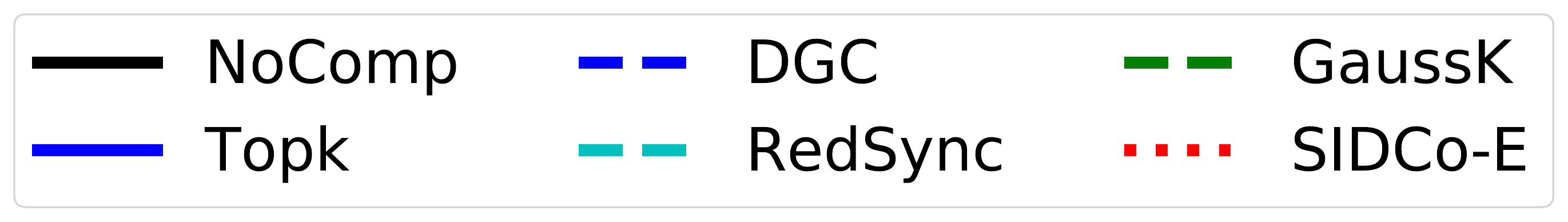}
 \end{subfigure}
    \begin{subfigure}[ht]{0.48\linewidth}
  \includegraphics[width=\linewidth]{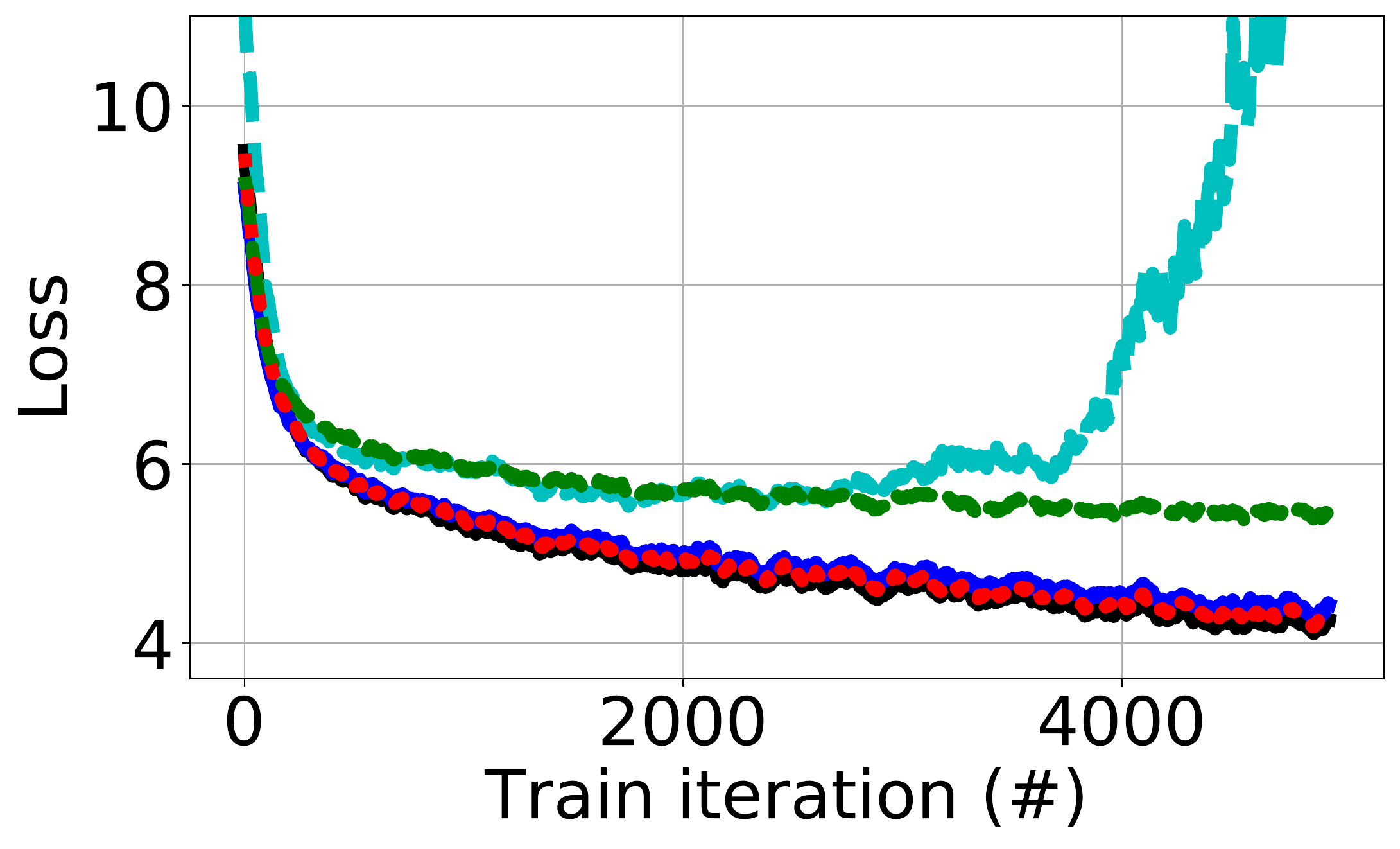}
	\caption{Train loss vs iterations}
	\label{fig:ptb-loss-0.001}
    \end{subfigure}
    \hfill
  \begin{subfigure}[ht]{0.48\linewidth}
    \includegraphics[width=\linewidth]{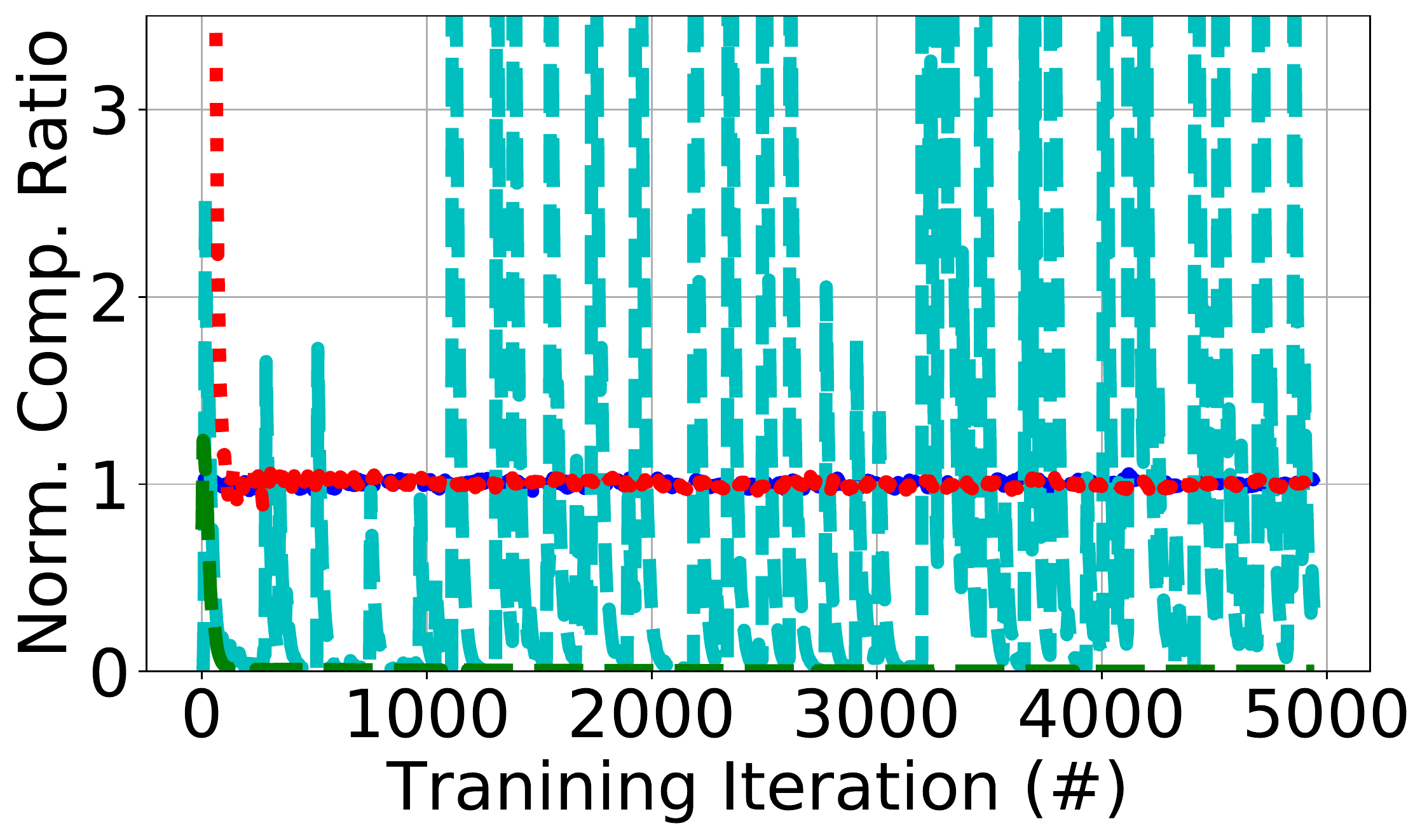}
	\caption{Thresh. Estimation Quality.}
	\label{fig:ptb-avgcomp-0.001}
     \end{subfigure}
     \\
    \begin{subfigure}[ht]{0.48\linewidth}
  \includegraphics[width=\linewidth]{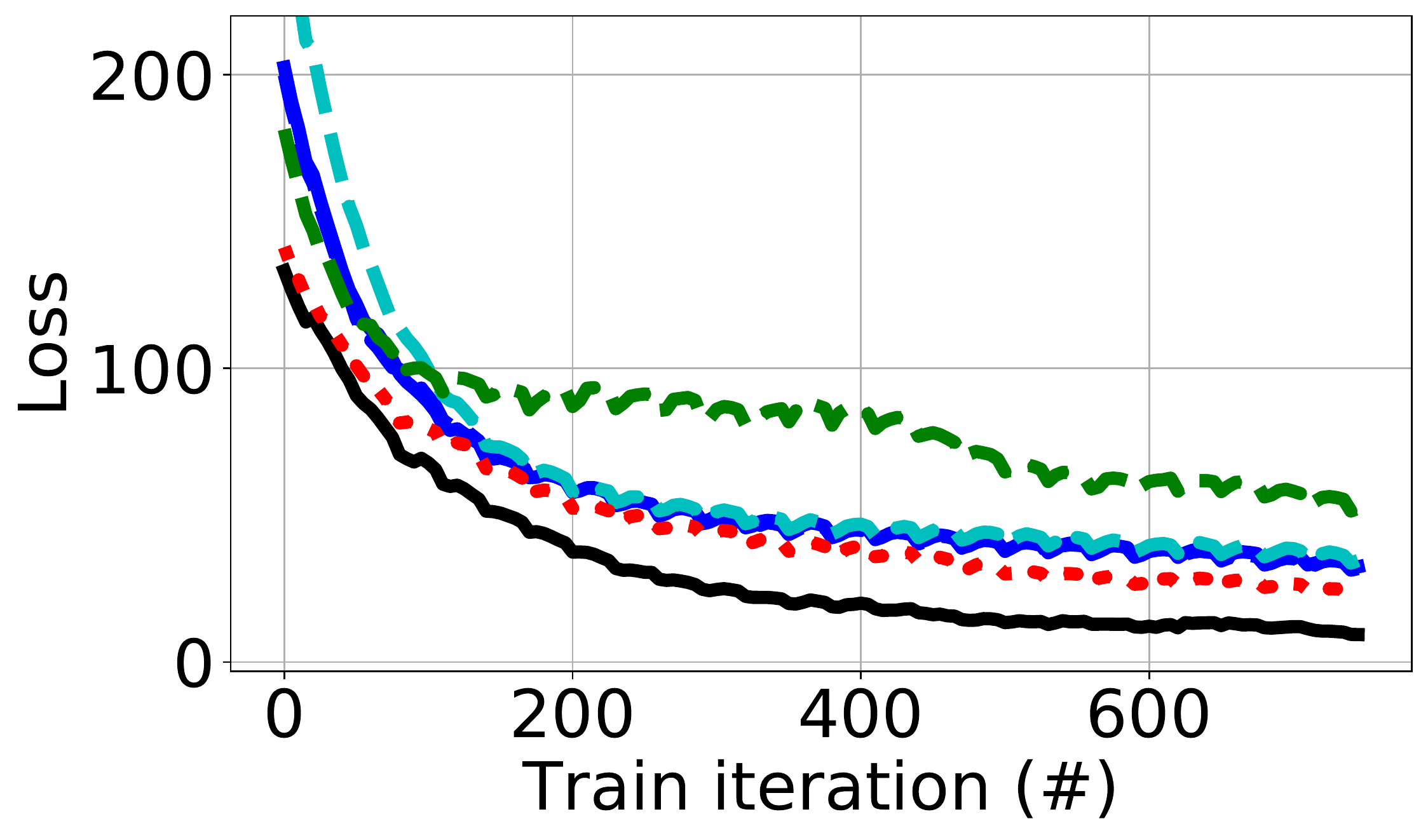}
	\caption{Train loss vs iterations}
	\label{fig:an4-loss-0.001}
    \end{subfigure}
    \hfill
  \begin{subfigure}[ht]{0.48\linewidth}
    \includegraphics[width=\linewidth]{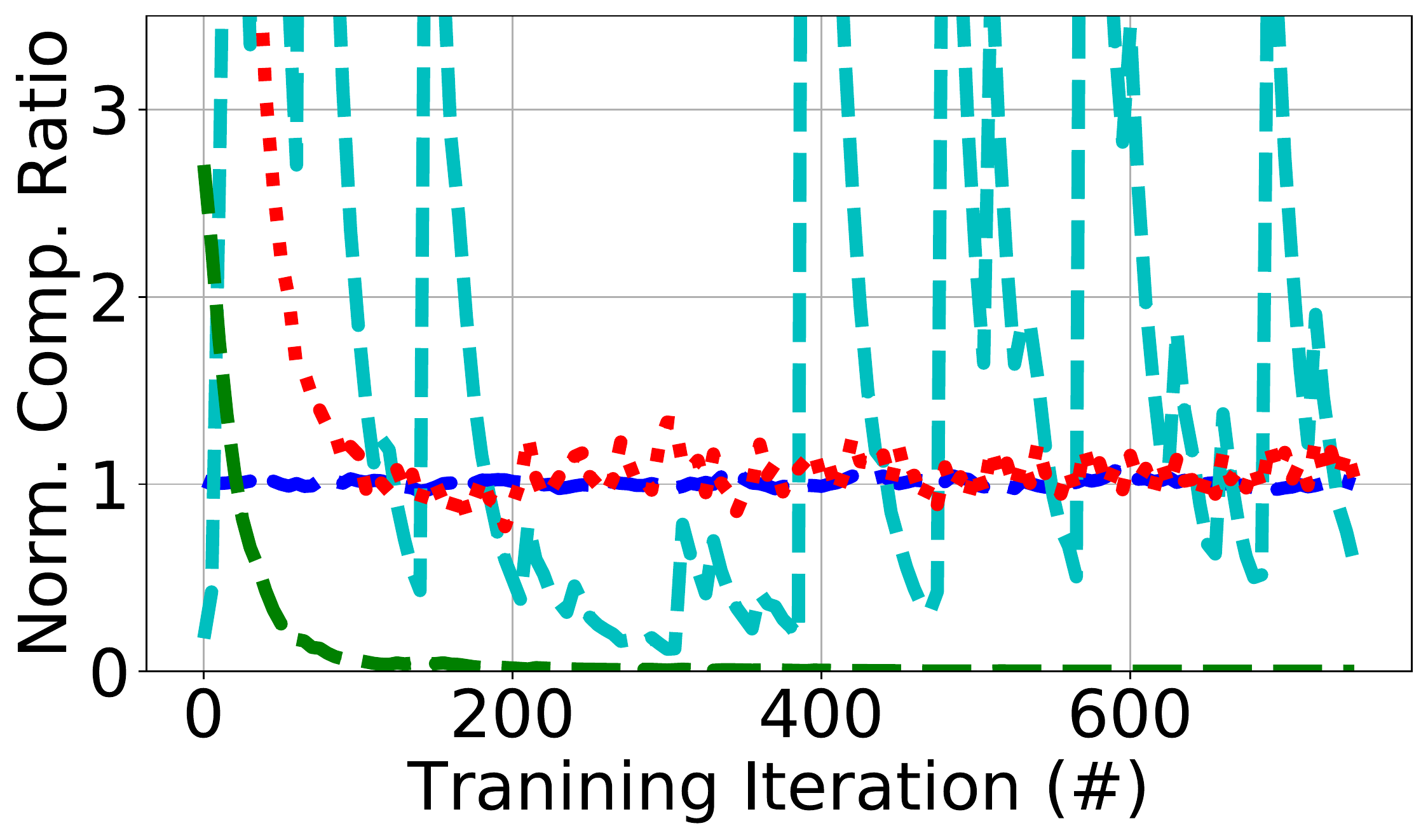}
	\caption{Thresh. Estimation Quality}
	\label{fig:an4-avgcomp-0.001}
     \end{subfigure}
\caption{The training performance for the LSTM model on PTB and AN4 datasets with compression ratio of $0.001$.}
\label{fig:rnn-extra}
\end{figure}


\begin{figure*}[t!]
	\captionsetup[subfigure]{justification=centering}
	\centering
	\begin{subfigure}[ht]{0.5\linewidth}
		\includegraphics[width=1\linewidth]{Figures/experiments/legend3.pdf}
	\end{subfigure}
	\\
	\begin{subfigure}[ht]{0.31\linewidth}
		\includegraphics[width=\linewidth]{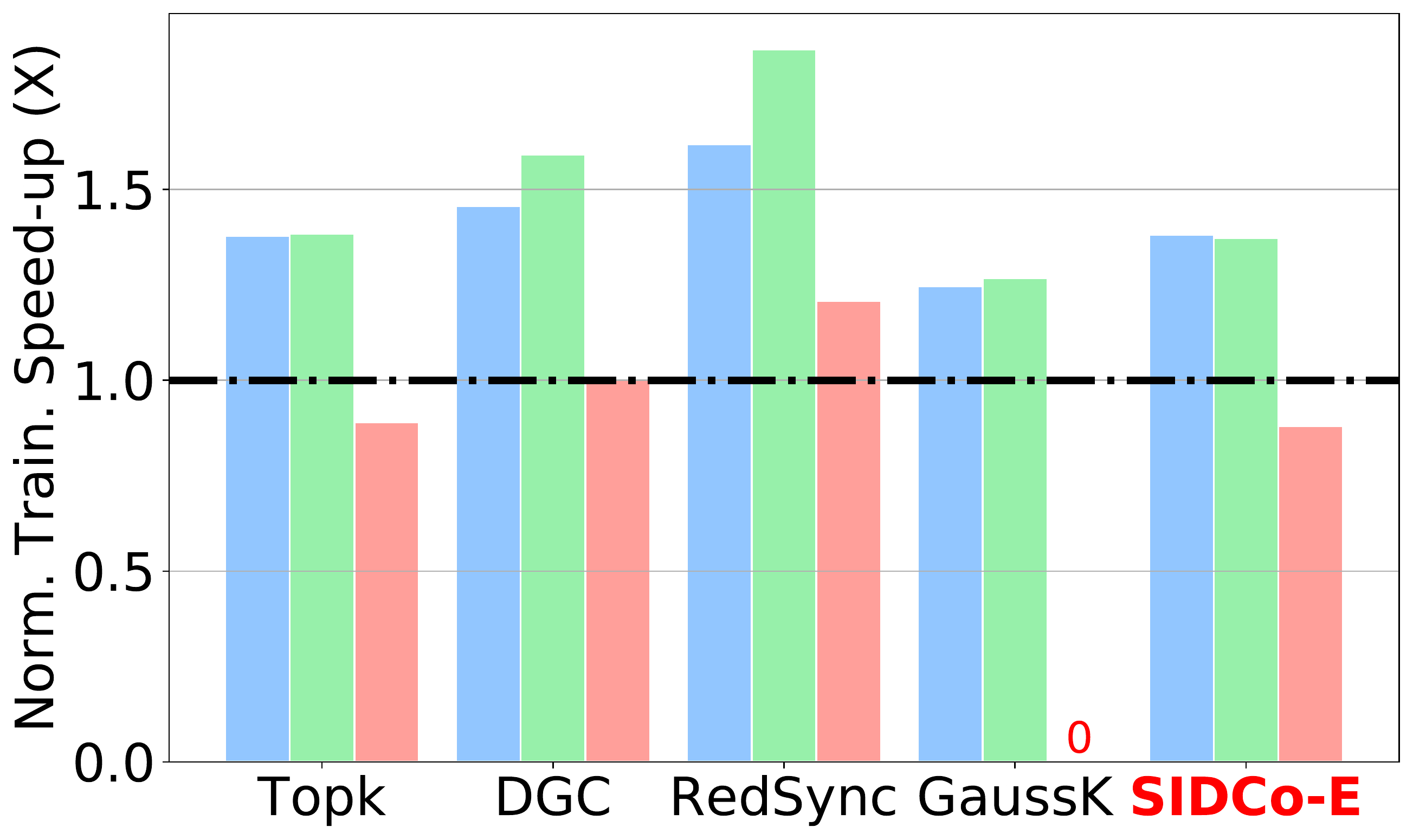}
		\caption{ResNet20-CIFAR-10 (Speedup).}
		\label{fig:resnet20-speedup-8}
	\end{subfigure}
	\hfill
	\begin{subfigure}[ht]{0.30\linewidth}
		\includegraphics[width=\linewidth]{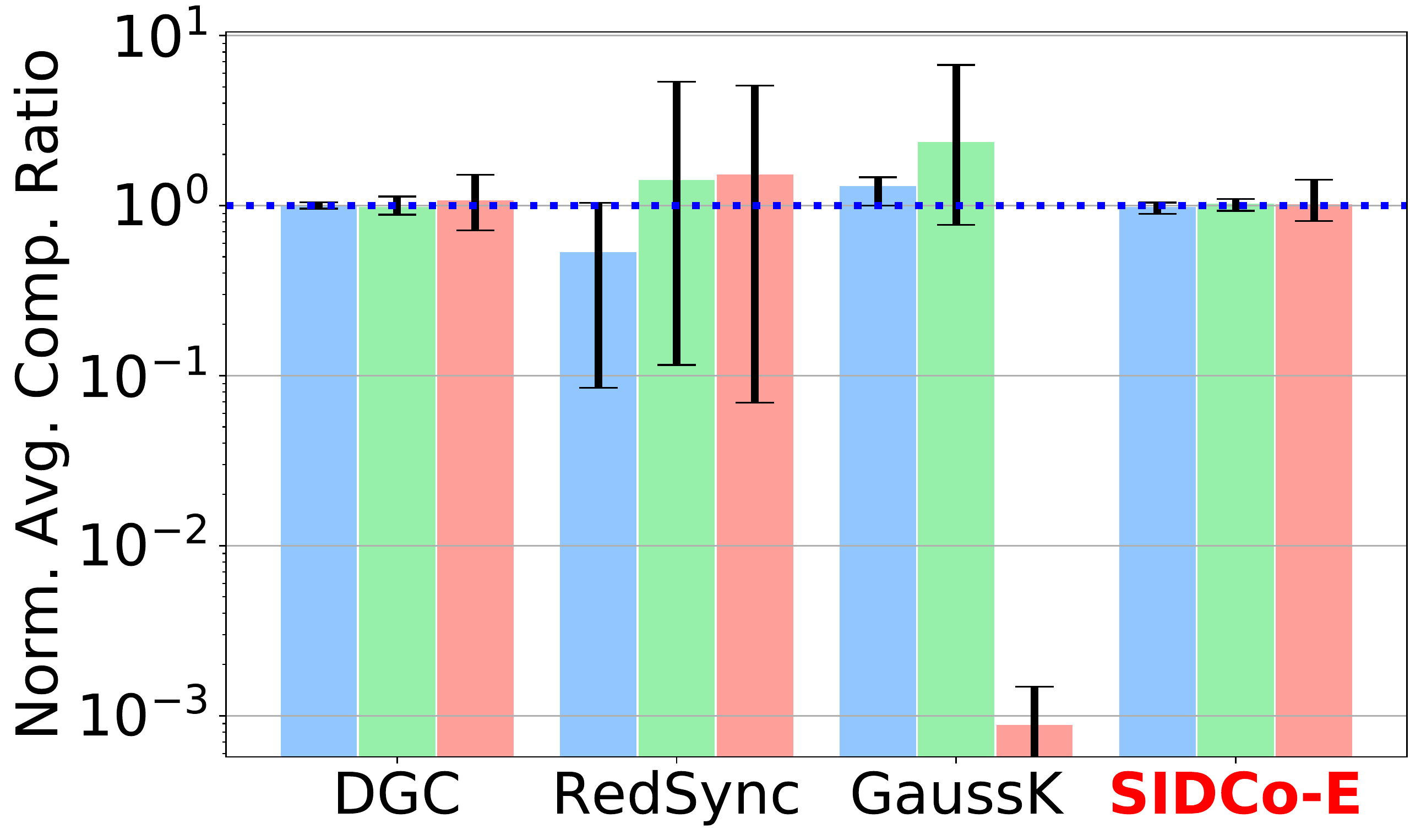}
		\caption{ResNet20-CIFAR-10 (Est. Quality).}
		\label{fig:resnet20-good-8}
	\end{subfigure}
	\hfill
	\begin{subfigure}[ht]{0.30\linewidth}
		\includegraphics[width=\linewidth]{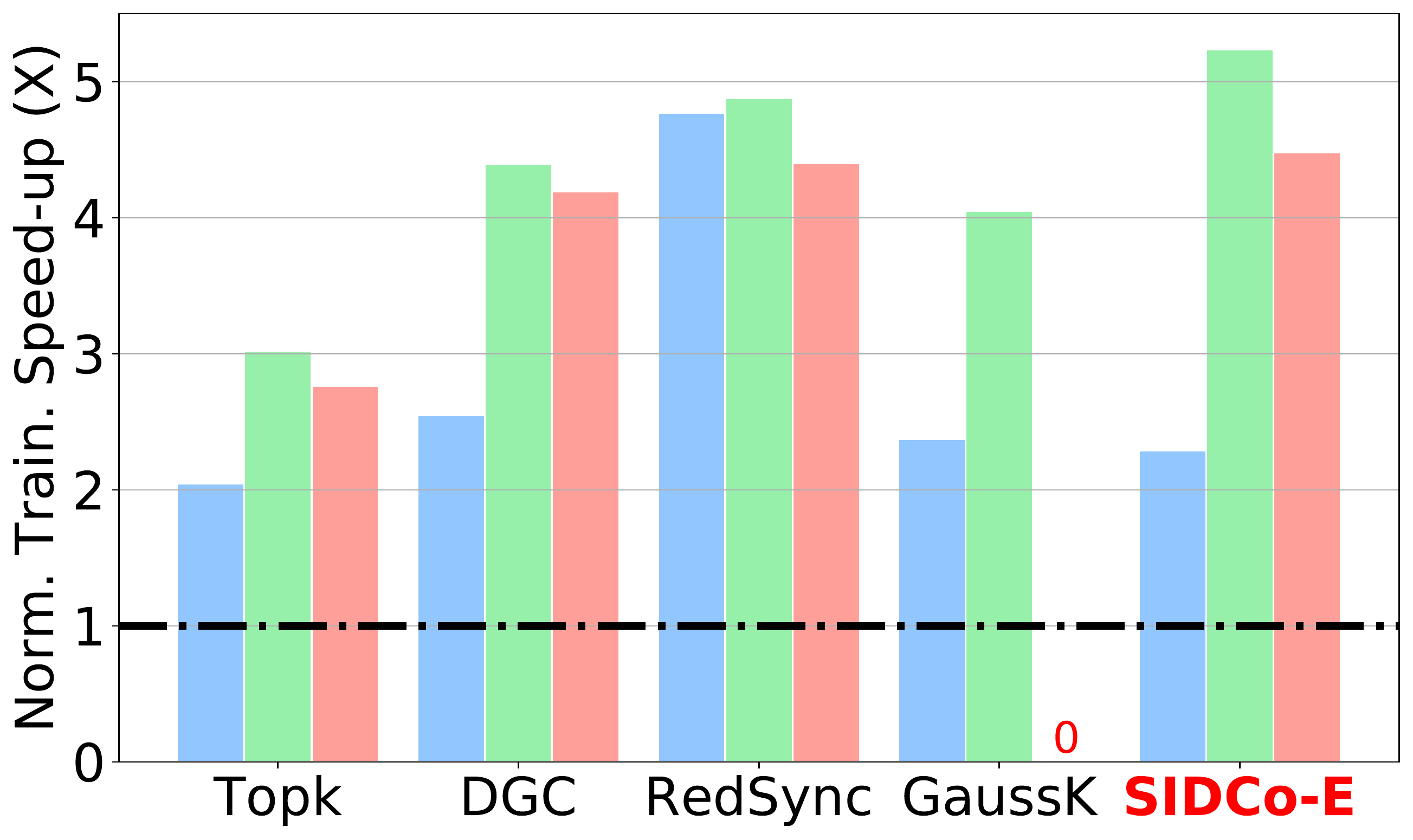}
		\caption{VGG16-CIFAR-10 (Speedup)}
		\label{fig:vgg16-speedup-8}
	\end{subfigure}
	\caption{The training performance for ResNet20 [(a),(b)] and VGG16 [(c)] on CIFAR-10 dataset.}
	\label{fig:cifar10}
\end{figure*}
Next, we present the results for the benchmarks in~\cref{tab:models}.
\subsection{Recurrent Neural Networks (RNNs)}
\textbf{RNN-LSTM on PTB: } This benchmark has the highest communication overhead (\cref{tab:models}). In~\cref{fig:ptb-speedup-8}, \scheme\! shows significant speed-up over no-compression by {\em $\approx\!41.7\times$} and improves over $\topk$ and \ac{DGC} by up to {\em $\approx\!7.6\times$} and {\em $\approx\!1.9\times$}, respectively. 
 At high compression ratio of $0.001$, both RedSync and GaussianKSGD compression methods do not converge to the target loss and test perplexity (\cref{fig:ptb-loss-0.001}) and therefore they attain zero speed-ups.
\cref{fig:ptb-throughput-8} shows that threshold estimation schemes including \scheme\! have the highest training throughput. However, in~\cref{fig:ptb-comp-8}, \ac{DGC} and \scheme\! are the only methods that accurately estimate the target ratio with high confidence. However, for GaussianKSGD at ratio of $0.001$ and RedSync at ratios of $0.01$ and $0.001$, the number of selected  elements is two orders-of-magnitude lower than the target. Moreover, over the training process, the estimation quality of RedSync has high variance, harming convergence. \cref{fig:ptb-avgcomp-0.001} shows, at target ratio of $0.001$, RedSync causes significant fluctuation in compression ratio and training does not converge. GaussianKSGD results in very low compression ratio which is close to $0$ and far from the target leading to significantly higher loss (and test perplexity) values compared to the target values. 

\textbf{RNN-LSTM on AN4:} \cref{fig:an4-speedup-8} shows that \scheme\! achieves higher gains compared to other compressors by up to {\em $\approx\!2.1\times$} for ratios of $0.1$ and $0.01$. Notability, at ratio of $0.001$, only \scheme\! achieved the target \ac{CER}. Thus, we ran other compressors for $250$ epochs to achieve the target \ac{CER} (instead of the default 150), except for GaussianKSGD, which does not converge. The gains of \scheme\! over the other compressors are increased by up to {\em $\approx\!4\times$}. The reason could be that the model is more sensitive to compression (esp., in the initial training phase). \scheme\! starts as single-stage before performing stage adaptations, leading to a slight over-estimation of $k$ and so more gradient elements are sent during training start-up. Throughput-wise, \cref{fig:an4-throughput-8} shows that threshold-estimation methods including \scheme\! enjoy higher training throughput, explaining the gains over the baseline. Similar to LSTM-PTB results,~\cref{fig:an4-comp-8} shows that on average, with low variance, \scheme\! closely matches the estimated ratios of DGC while other estimation methods have poor estimation quality. Similar to PTB, \cref{fig:an4-avgcomp-0.001} shows, at target ratio of $0.001$, RedSync causes significant fluctuation in compression ratio and GaussianKSGD results in very low compression ratio (close to 0) which is far from the target. This leads both methods to achieve significantly higher loss (or test perplexity) values compared to the target loss (or test perplexity) values. 

\subsection{Convolutional Neural Networks (CNNs)}

\begin{figure*}[!h]
\captionsetup[subfigure]{justification=centering}
\centering
\centering
 \begin{subfigure}[ht]{0.5\linewidth}
  \includegraphics[width=1\linewidth]{Figures/experiments/legend3.pdf}
 \end{subfigure}
          \\
  \begin{subfigure}[ht]{0.30\linewidth}
    \includegraphics[width=\linewidth]{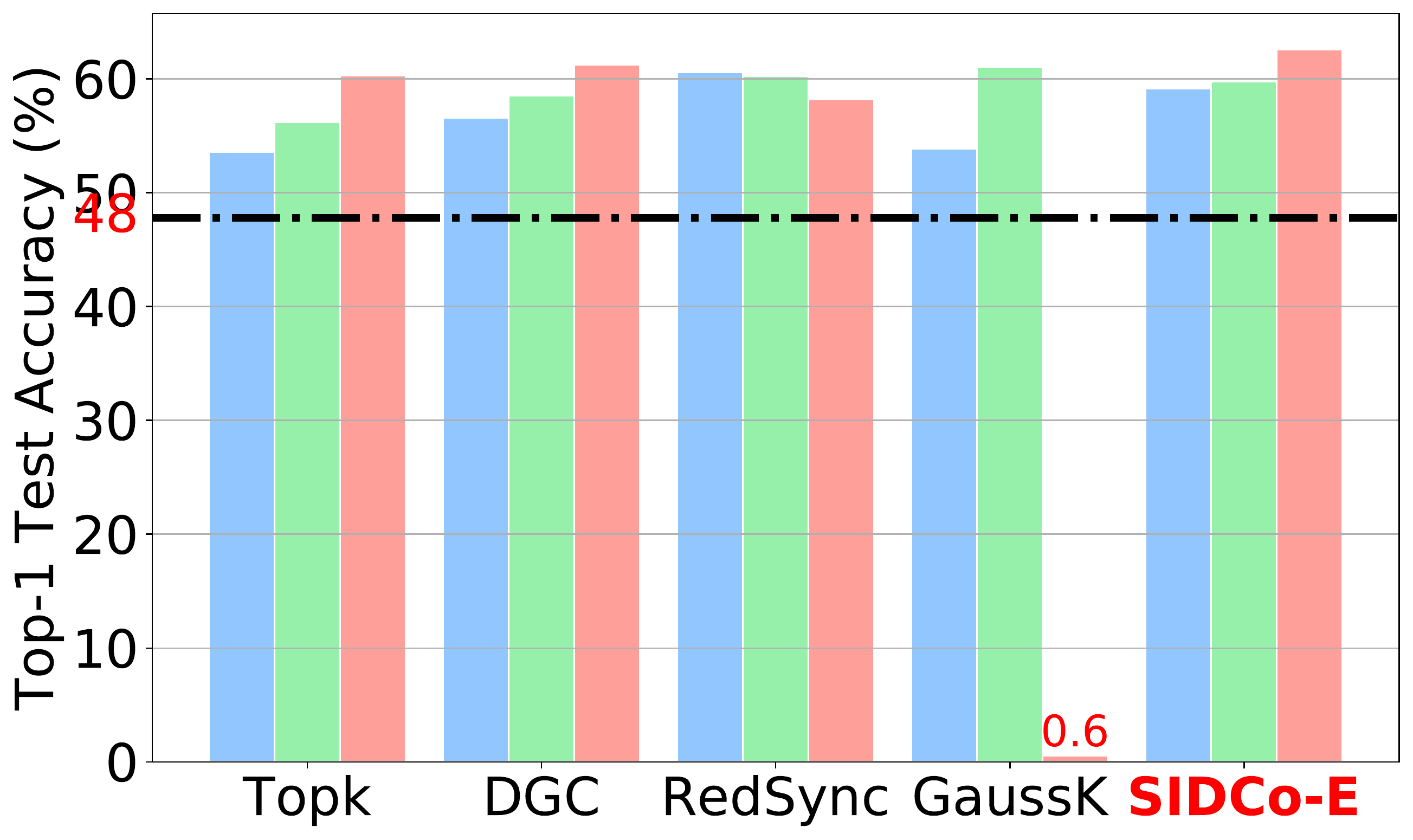}
	\caption{ResNet50-ImageNet (Accuracy)}
	\label{fig:resnet50-speedup-8}
     \end{subfigure}
     \hfill
     \begin{subfigure}[ht]{0.30\linewidth}
  \includegraphics[width=\linewidth]{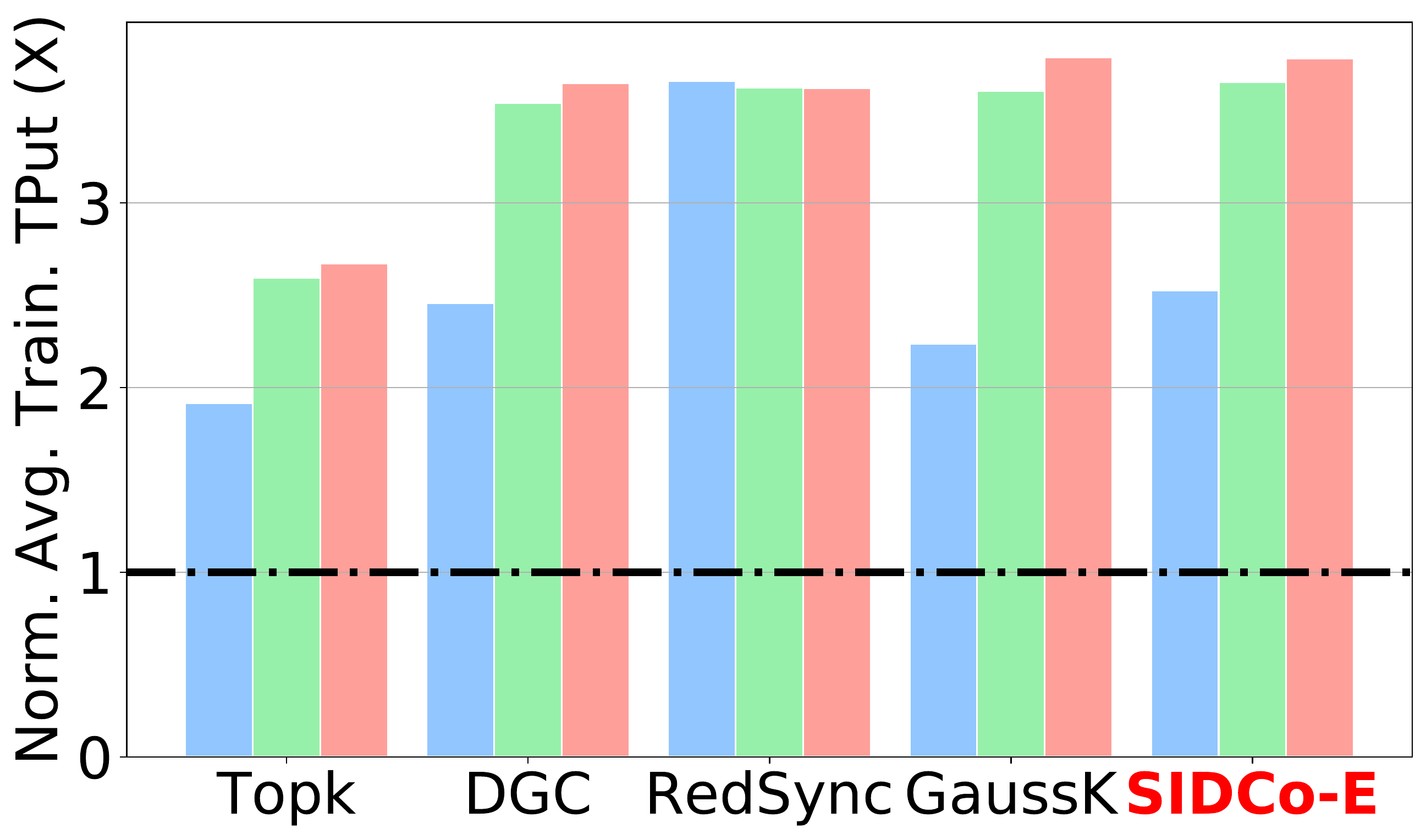}
	\caption{ResNet50-ImageNet (Throughput)}
	\label{fig:resnet50-tput-8}
    \end{subfigure}
    \hfill
        \begin{subfigure}[ht]{0.30\linewidth}
    \includegraphics[width=\linewidth]{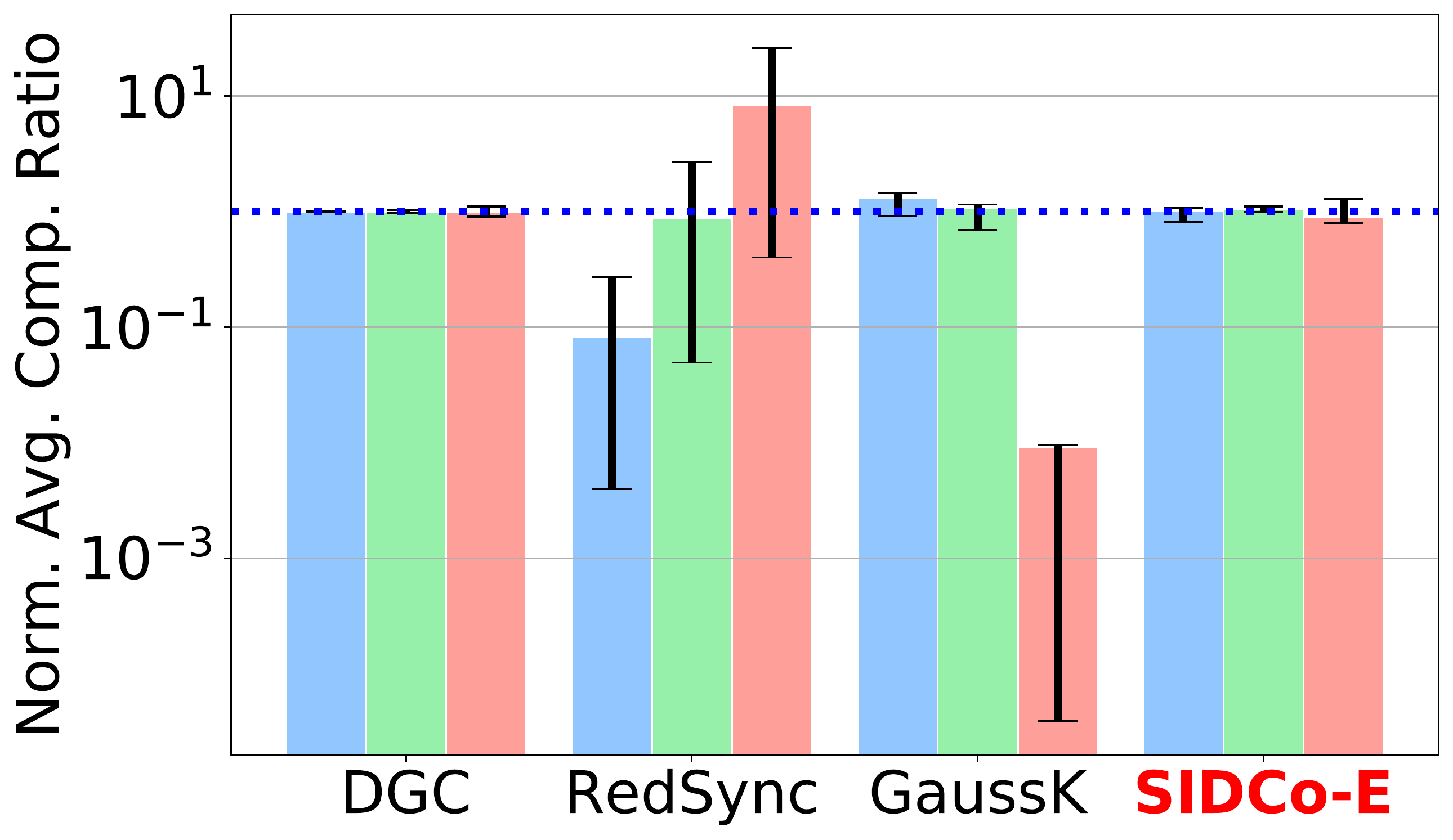}
	\caption{ResNet50-ImageNet (Est. Quality)}
	\label{fig:resnet50-comp-8}
     \end{subfigure}
     \\
     \begin{subfigure}[ht]{0.30\linewidth}
    \includegraphics[width=\linewidth]{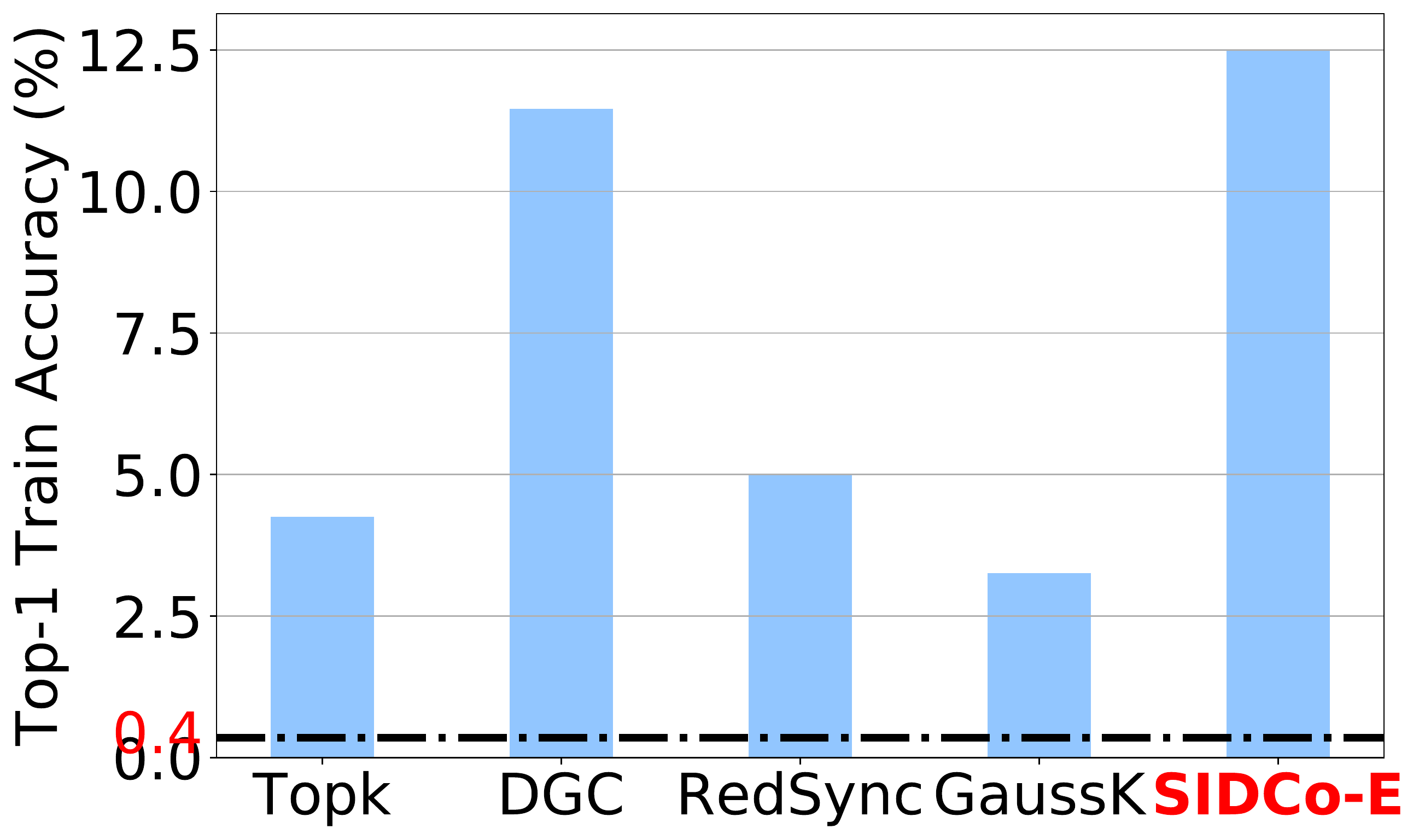}
	\caption{VGG19-ImageNet (Accuracy)}
	\label{fig:vgg19-speedup-8}
     \end{subfigure}
     \hfill
     \begin{subfigure}[ht]{0.30\linewidth}
  \includegraphics[width=\linewidth]{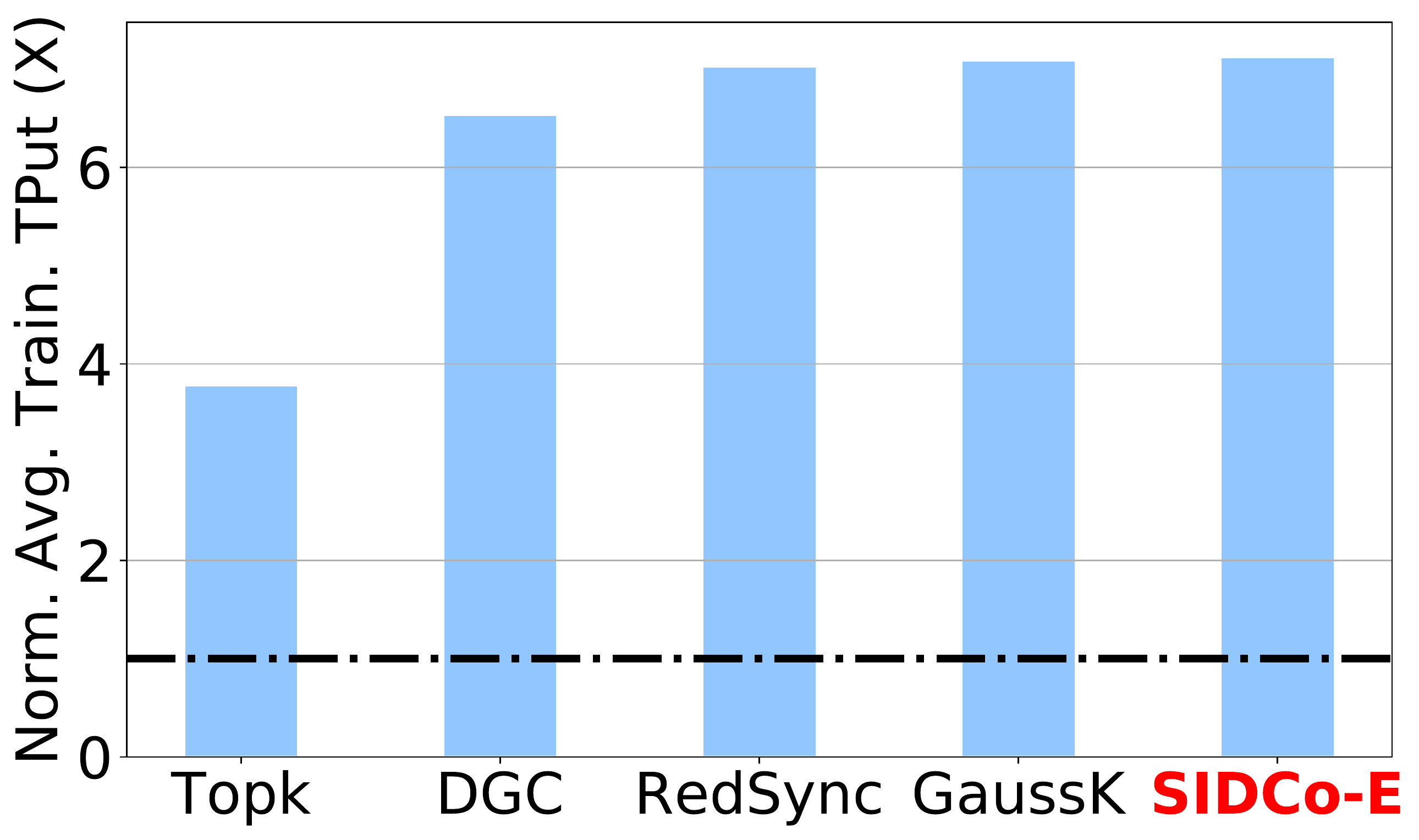}
	\caption{VGG19-ImageNet (Throughput)}
	\label{fig:vgg19-tput-8}
    \end{subfigure}
    \hfill
        \begin{subfigure}[ht]{0.30\linewidth}
    \includegraphics[width=\linewidth]{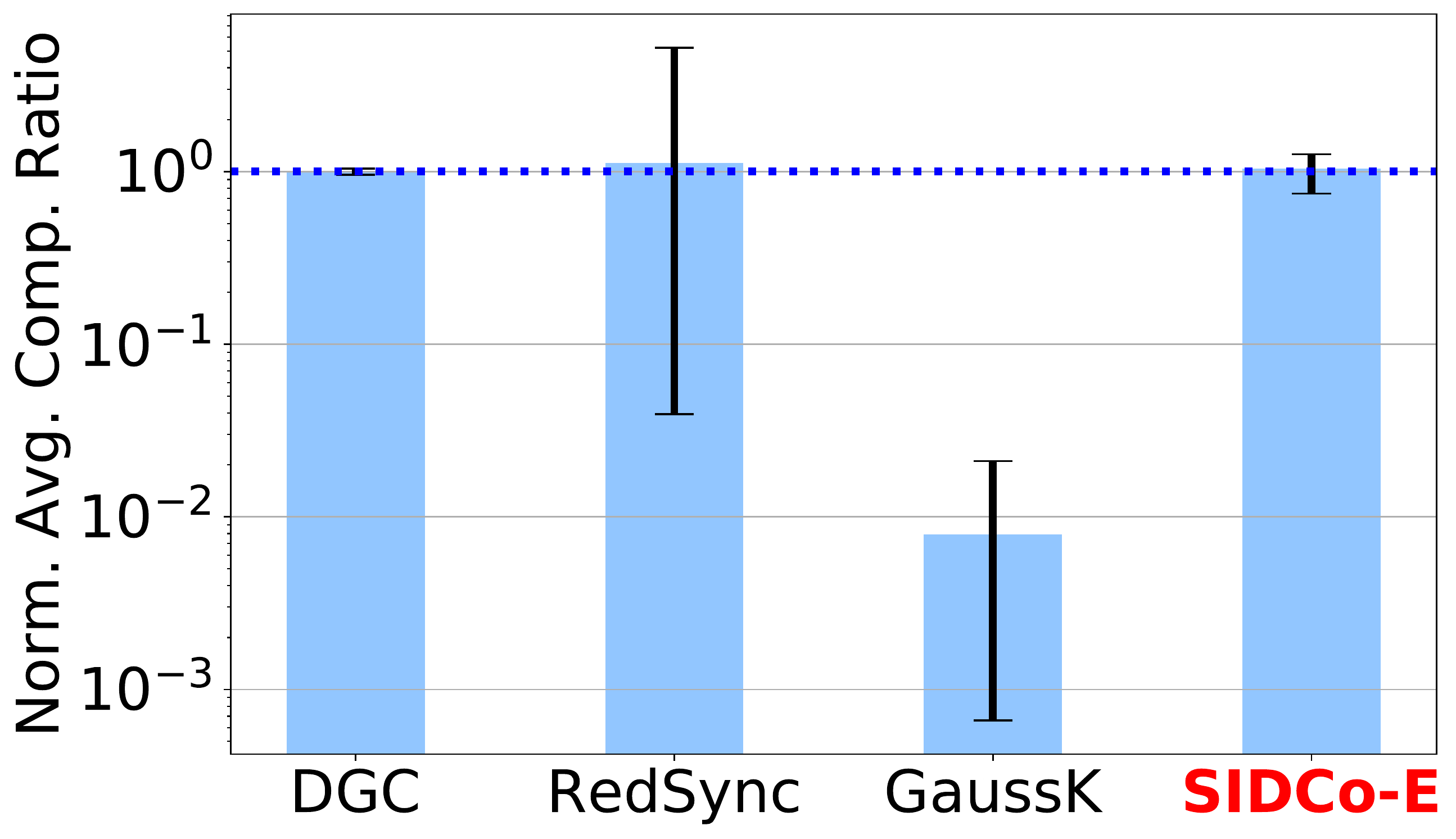}
	\caption{VGG19-ImageNet (Est. Quality)}
	\label{fig:vgg19-comp-8}
     \end{subfigure}
\caption{The training performance for ResNet50 [(a), (b), (c)] and VGG19 [(d), (e), (f)] on ImageNet dataset.}
\label{fig:imagenet}
\end{figure*}

\textbf{ResNet20 and VGG16 on CIFAR-10:} \cref{fig:resnet20-speedup-8} shows that, for ResNet20, all compressors achieve somewhat comparable and modest speed-ups over the no-compression baseline (except at ratio of 0.001, where accuracy is degraded and hence the lower speed-up than the baseline).
This is not surprising because ResNet20 is not network-bound.
However, for the larger VGG16 model, \cref{fig:vgg16-speedup-8} shows that \scheme\! achieves significant speed-ups over no-compression, $\topk$ and DGC by up to $\approx\!5\times$, $1.5\times$, and $1.2\times$, respectively. \cref{fig:resnet20-good-8} shows that, unlike other estimation schemes, \scheme\! can accurately achieve the target ratio.

\textbf{ResNet50 and VGG19 on ImageNet:} In these experiments, we set a time-limit of 5 hours per run to reduce our costs. For calculating the speed-up, we compare the top-1 accuracy achieved by different methods  at the end of training. First, for ResNet50 benchmark, we use compression ratios of $0.1$, $0.01$, and $0.001$. \cref{fig:resnet50-speedup-8} shows that \scheme\! achieves the highest accuracy that is higher than the baseline, $\topk$ and \ac{DGC} by $\approx15$, $3$, and $2$ accuracy points, i.e., normalized accuracy gains of $\approx\!40\%$, $5\%$, and $4\%$, respectively. \cref{fig:resnet50-tput-8} shows that \scheme\! attains the highest throughput among all methods (except for RedSync at $0.1$ compression).
\cref{fig:resnet50-comp-8} shows that, unlike GaussianKSGD and RedSync, which both result in estimation quality far from the target with high variance, \scheme\! estimates the threshold with very high quality for all ratios. Similar trends are observed for the VGG19 benchmark where we use compression ratio of $0.001$. As shown in \cref{fig:vgg19-speedup-8,fig:vgg19-tput-8,fig:vgg19-comp-8}, \scheme\! estimates the threshold with high quality, and achieves the highest top-1 accuracy and training throughput among all methods. The accuracy gains compared to the baseline, $\topk$ and DGC are $\approx\!34\times$, $2.9\times$, and $1.13\times$, respectively. 


\smartparagraph{Takeaways:} Our approach is simple in nature, which is intentional, to make it applicable in practice. Nonetheless, our work goes beyond existing works that estimate a threshold for $\topk$ sparsification. These works either did not leverage the statistical property of the gradients (DGC) or assumed Gaussian distribution without a thorough study of the gradient (e.g., RedSync, GaussianKSGD). On a GPU, \scheme\! improves over DGC by at least $2\times$, and the speed-ups are significantly larger on the CPU as shown in ~\cref{fig:vgg16-cpu-speedup} and~\cref{apdx:moremicrobench}. As a threshold estimation method, SIDCo does not only benefit from the throughput gains of threshold methods but also from the high quality of its threshold estimation. The results in~\cref{fig:rnn-extra,fig:compratio} indicate that existing estimation methods (e.g., RedSync and GaussianKSGD) fail to achieve consistent threshold estimation behavior even though they may provide throughput gains. Their throughput gains, in many cases, are due to severe under-estimation of the target ratio, which results in lower volumes of data sent compared to other compressors.

\section{Conclusion}
We solved a practical problem in distributed deep learning. We showed that the performance of compressors other than threshold-based ones has high computational costs whereas existing threshold-estimation methods fail to achieve their target. To address these issues, we proposed \scheme\!, a multi-stage threshold-based compressor through imposing a sparsity prior on the gradients. We evaluated \scheme\! and compared it with popular compressors using common benchmarks involving RNN and CNN architectures. \ac{SIDCo}, unlike existing threshold estimation methods, can efficiently approximate the target threshold and results in significant gains of up to $\approx\!41.7\times$, $7.5\times$, and $1.9\times$ over no-compression baseline, $\topk$ and \ac{DGC} compression methods, respectively. Also, we expect further gains for large and communication-bounded models. In the future, we will explore ways to estimate a threshold for which compression satisfies other quality targets.

\small
\bibliographystyle{mlsys2021}
\bibliography{main}

\begin{thebibliography}{62}
\providecommand{\natexlab}[1]{#1}
\providecommand{\url}[1]{\texttt{#1}}
\expandafter\ifx\csname urlstyle\endcsname\relax
  \providecommand{\doi}[1]{doi: #1}\else
  \providecommand{\doi}{doi: \begingroup \urlstyle{rm}\Url}\fi

\bibitem[Abadi et~al.(2016)Abadi, Barham, Chen, Chen, Davis, Dean, Devin,
  Ghemawat, Irving, Isard, Kudlur, Levenberg, Monga, Moore, Murray, Steiner,
  Tucker, Vasudevan, Warden, Wicke, Yu, and Zheng]{tensorflow}
Abadi, M., Barham, P., Chen, J., Chen, Z., Davis, A., Dean, J., Devin, M.,
  Ghemawat, S., Irving, G., Isard, M., Kudlur, M., Levenberg, J., Monga, R.,
  Moore, S., Murray, D.~G., Steiner, B., Tucker, P., Vasudevan, V., Warden, P.,
  Wicke, M., Yu, Y., and Zheng, X.
\newblock {TensorFlow: A System for Large-Scale Machine Learning}.
\newblock In \emph{OSDI}, 2016.

\bibitem[Abdelmoniem \& Canini(2021)Abdelmoniem and
  Canini]{Ahmed-DC2-INFOCOM21}
Abdelmoniem, A.~M. and Canini, M.
\newblock {DC2: Delay-aware Compression Control for Distributed Machine
  Learning}.
\newblock In \emph{{INFOCOM}}, 2021.

\bibitem[Abramowitz \& Stegun(1965)Abramowitz and Stegun]{AbraSte:65}
Abramowitz, M. and Stegun, I.~A.
\newblock \emph{{Handbook of Mathematical Functions, With Formulas, Graphs, and
  Mathematical Tables}}, volume~55.
\newblock US Government printing office, fourth edition, 1965.

\bibitem[Aji \& Heafield(2017)Aji and Heafield]{aji_sparse}
Aji, A.~F. and Heafield, K.
\newblock {Sparse Communication for Distributed Gradient Descent}.
\newblock In \emph{EMNLP}, 2017.

\bibitem[Alistarh et~al.(2018)Alistarh, Hoefler, Johansson, Konstantinov,
  Khirirat, and Renggli]{Alistarh18_sparse}
Alistarh, D., Hoefler, T., Johansson, M., Konstantinov, N., Khirirat, S., and
  Renggli, C.
\newblock {The Convergence of Sparsified Gradient Methods}.
\newblock In \emph{NeurIPS}, 2018.

\bibitem[AN4()]{an4}
AN4.
\newblock {CMU Census Database}, 1991.
\newblock http://www.speech.cs.cmu.edu/databases/an4/index.html.

\bibitem[Armagan et~al.(2013)Armagan, Dunson, and Lee]{ArmDunLee:13}
Armagan, A., Dunson, D.~B., and Lee, J.
\newblock {Generalized double Pareto shrinkage}.
\newblock \emph{Statistica Sinica}, 23\penalty0 (1), 2013.

\bibitem[{Babacan} et~al.(2010){Babacan}, {Molina}, and
  {Katsaggelos}]{BabMolKat:10}
{Babacan}, S.~D., {Molina}, R., and {Katsaggelos}, A.~K.
\newblock {Bayesian Compressive Sensing Using Laplace Priors}.
\newblock \emph{IEEE Transactions on Image Processing}, 19\penalty0 (1), 2010.

\bibitem[Baraniuk et~al.(2011)Baraniuk, Davenport, Duarte, and
  Hegde]{BarDavDua:11}
Baraniuk, R., Davenport, M.~A., Duarte, M.~F., and Hegde, C.
\newblock {An Introduction to Compressive Sensing}, 2011.
\newblock \url{https://legacy.cnx.org/content/col11133/1.5/}.

\bibitem[Beidi et~al.(2020)Beidi, Medini, Farwell, Gobriel, Tai, and
  Shrivastava]{Beidi2020}
Beidi, C., Medini, T., Farwell, J., Gobriel, S., Tai, C., and Shrivastava, A.
\newblock Slide : In defense of smart algorithms over hardware acceleration for
  large scale deep learning systems.
\newblock In \emph{MLSys}, 2020.

\bibitem[Bond(2001)]{Bond:01}
Bond, S.~A.
\newblock A review of asymmetric conditional density functions in
  autoregressive conditional heteroscedasticity models.
\newblock In Knight, J. and Satchell, S. (eds.), \emph{Return Distributions in
  Finance}, Quantitative Finance, chapter~2. Butterworth-Heinemann, Oxford,
  2001.

\bibitem[Brown et~al.(2020)Brown, Mann, Ryder, Subbiah, Kaplan, Dhariwal,
  Neelakantan, Shyam, Sastry, Askell, Agarwal, Herbert-Voss, Krueger, Henighan,
  Child, Ramesh, Ziegler, Wu, Winter, Hesse, Chen, Sigler, Litwin, Gray, Chess,
  Clark, Berner, McCandlish, Radford, Sutskever, and Amodei]{gpt3}
Brown, T.~B., Mann, B., Ryder, N., Subbiah, M., Kaplan, J., Dhariwal, P.,
  Neelakantan, A., Shyam, P., Sastry, G., Askell, A., Agarwal, S.,
  Herbert-Voss, A., Krueger, G., Henighan, T., Child, R., Ramesh, A., Ziegler,
  D.~M., Wu, J., Winter, C., Hesse, C., Chen, M., Sigler, E., Litwin, M., Gray,
  S., Chess, B., Clark, J., Berner, C., McCandlish, S., Radford, A., Sutskever,
  I., and Amodei, D.
\newblock {Language Models are Few-Shot Learners}.
\newblock \emph{arXiv 2005.14165}, 2020.

\bibitem[Coles(2001)]{Coles:01}
Coles, S.
\newblock \emph{{An Introduction to Statistical Modeling of Extreme Values}}.
\newblock Springer London, 2001.

\bibitem[Das et~al.(2018)Das, Mellempudi, Mudigere, Kalamkar, Avancha,
  Banerjee, Sridharan, Vaidyanathan, Kaul, Georganas, Heinecke, Dubey, Corbal,
  Shustrov, Dubtsov, Fomenko, and Pirogov]{Das2018}
Das, D., Mellempudi, N., Mudigere, D., Kalamkar, D.~D., Avancha, S., Banerjee,
  K., Sridharan, S., Vaidyanathan, K., Kaul, B., Georganas, E., Heinecke, A.,
  Dubey, P., Corbal, J., Shustrov, N., Dubtsov, R., Fomenko, E., and Pirogov,
  V.~O.
\newblock Mixed precision training of convolutional neural networks using
  integer operations.
\newblock In \emph{ICLR}, 2018.

\bibitem[Dean et~al.(2012)Dean, Corrado, Monga, Chen, Devin, Le, Mao, Ranzato,
  Senior, Tucker, Yang, and Ng]{Dean2012}
Dean, J., Corrado, G.~S., Monga, R., Chen, K., Devin, M., Le, Q.~V., Mao,
  M.~Z., Ranzato, M., Senior, A., Tucker, P., Yang, K., and Ng, A.~Y.
\newblock {Large Scale Distributed Deep Networks}.
\newblock In \emph{NeurIPS}, 2012.

\bibitem[Deng et~al.(2009)Deng, Dong, Socher, Li, Li, and Li]{imagenet}
Deng, J., Dong, W., Socher, R., Li, L.-J., Li, K., and Li, F.~F.
\newblock {ImageNet: a Large-Scale Hierarchical Image Database}.
\newblock In \emph{CVPR}, 2009.

\bibitem[DeVore(1998)]{Devore:98}
DeVore, R.~A.
\newblock Nonlinear approximation.
\newblock \emph{Acta Numerica}, 7, 1998.

\bibitem[Dieuleveut \& Patel(2019)Dieuleveut and Patel]{patel2019communication}
Dieuleveut, A. and Patel, K.~K.
\newblock {Communication Trade-offs for Local-SGD with Large Step Size}.
\newblock In \emph{NeurIPS}, 2019.

\bibitem[Dryden et~al.(2016)Dryden, Jacobs, Moon, and
  Van~Essen]{Dryden2016CommunicationQF}
Dryden, N., Jacobs, S.~A., Moon, T., and Van~Essen, B.
\newblock {Communication Quantization for Data-Parallel Training of Deep Neural
  Networks}.
\newblock In \emph{Workshop on ML in HPC (MLHPC)}, 2016.

\bibitem[Dutta et~al.(2020)Dutta, Bergou, Abdelmoniem, Ho, Sahu, Canini, and
  Kalnis]{Ahmed-AAAI-2020}
Dutta, A., Bergou, E.~H., Abdelmoniem, A.~M., Ho, C.-Y., Sahu, A.~N., Canini,
  M., and Kalnis, P.
\newblock {On the Discrepancy between the Theoretical Analysis and Practical
  Implementations of Compressed Communication for Distributed Deep Learning}.
\newblock In \emph{AAAI}, 2020.

\bibitem[Elzanaty et~al.(2019{\natexlab{a}})Elzanaty, Giorgetti, and
  Chiani]{ElzGioChi:19}
Elzanaty, A., Giorgetti, A., and Chiani, M.
\newblock {Limits on Sparse Data Acquisition: {RIC} Analysis of Finite
  {Gaussian} Matrices}.
\newblock \emph{IEEE Transactions on Information Theory}, 65\penalty0 (3),
  2019{\natexlab{a}}.

\bibitem[Elzanaty et~al.(2019{\natexlab{b}})Elzanaty, Giorgetti, and
  Chiani]{Elzanaty19}
Elzanaty, A., Giorgetti, A., and Chiani, M.
\newblock {Lossy Compression of Noisy Sparse Sources Based on Syndrome
  Encoding}.
\newblock \emph{IEEE Transactions on Communications}, 67\penalty0 (10),
  2019{\natexlab{b}}.

\bibitem[Evans et~al.(1994)Evans, Hastings, and Peacock]{EvaHasPea:93}
Evans, M., Hastings, N., and Peacock, B.
\newblock \emph{Statistical distributions}.
\newblock Wiley, New York, second edition, 1994.

\bibitem[Fang et~al.(2019)Fang, Fu, Yang, and Hsieh]{Fang2019}
Fang, J., Fu, H., Yang, G., and Hsieh, C.-J.
\newblock {RedSync: Reducing synchronization bandwidth for distributed deep
  learning training system}.
\newblock \emph{Journal of Parallel and Distributed Computing}, 133, 2019.

\bibitem[Fu et~al.(2020)Fu, Hu, He, Jiang, Shao, Zhang, and
  Cui]{pmlr-v119-fu20c}
Fu, F., Hu, Y., He, Y., Jiang, J., Shao, Y., Zhang, C., and Cui, B.
\newblock {Don't Waste Your Bits! Squeeze Activations and Gradients for Deep
  Neural Networks via TinyScript}.
\newblock In \emph{ICML}, 2020.

\bibitem[Gajjala et~al.(2020)Gajjala, Banchhor, Abdelmoniem, Dutta, Canini, and
  Kalnis]{Ahmed-CONEXT-2020}
Gajjala, R., Banchhor, S., Abdelmoniem, A.~M., Dutta, A., Canini, M., and
  Kalnis, P.
\newblock {Huffman Coding Based Encoding Techniques for Fast Distributed Deep
  Learning}.
\newblock In \emph{DistributedML}, 2020.

\bibitem[Goyal et~al.(2017)Goyal, Doll{\'{a}}r, Girshick, Noordhuis,
  Wesolowski, Kyrola, Tulloch, Jia, and He]{Goyal2017}
Goyal, P., Doll{\'{a}}r, P., Girshick, R., Noordhuis, P., Wesolowski, L.,
  Kyrola, A., Tulloch, A., Jia, Y., and He, K.
\newblock {Accurate, Large Minibatch SGD: Training ImageNet in 1 Hour}.
\newblock \emph{arXiv 1706.02677}, 2017.

\bibitem[Gross \& Wilber(2016)Gross and Wilber]{Gross2016}
Gross, S. and Wilber, M.
\newblock {Training and investigating Residual Nets}, 2016.
\newblock \url{http://torch.ch/blog/2016/02/04/resnets.html}.

\bibitem[Hannun et~al.(2014)Hannun, Case, Casper, Catanzaro, Diamos, Elsen,
  Prenger, Satheesh, Sengupta, Coates, and Ng]{deepspeech}
Hannun, A., Case, C., Casper, J., Catanzaro, B., Diamos, G., Elsen, E.,
  Prenger, R., Satheesh, S., Sengupta, S., Coates, A., and Ng, A.~Y.
\newblock Deep speech: Scaling up end-to-end speech recognition.
\newblock \emph{arXiv 1412.5567}, 2014.

\bibitem[He et~al.({{2015}})He, Zhang, Ren, and Sun]{resnet152}
He, K., Zhang, X., Ren, S., and Sun, J.
\newblock {{Deep Residual Learning for Image Recognition}}.
\newblock In \emph{{{Proc. of CVPR}}}, pp.\  {{770--778}}, {{2015}}.

\bibitem[Hochreiter \& Schmidhuber(1997)Hochreiter and Schmidhuber]{lstm}
Hochreiter, S. and Schmidhuber, J.
\newblock {Long Short-Term Memory}.
\newblock \emph{Neural Computing}, 9\penalty0 (8), 1997.

\bibitem[Hosking \& Wallis(1987)Hosking and Wallis]{HosWal:87}
Hosking, J. and Wallis, J.
\newblock {Parameter and Quantile Estimation for the Generalized Pareto
  Distribution}.
\newblock \emph{Technometrics}, 29\penalty0 (3), 1987.

\bibitem[Jiang et~al.(2018)Jiang, Fu, Yang, and Cui]{SketchML}
Jiang, J., Fu, F., Yang, T., and Cui, B.
\newblock {SketchML: Accelerating Distributed Machine Learning with Data
  Sketches}.
\newblock In \emph{SIGMOD}, 2018.

\bibitem[Karimireddy et~al.(2019)Karimireddy, Rebjock, Stich, and
  Jaggi]{ef-sgd}
Karimireddy, S.~P., Rebjock, Q., Stich, S., and Jaggi, M.
\newblock {Error Feedback Fixes Sign SGD and other Gradient Compression
  Schemes}.
\newblock In \emph{ICML}, 2019.

\bibitem[Kotz \& Nadarajah(2000)Kotz and Nadarajah]{KotNad:00}
Kotz, S. and Nadarajah, S.
\newblock \emph{{Extreme Value Distributions: Theory and Applications}}.
\newblock World Scientific, 2000.

\bibitem[Krizhevsky(2009)]{cifar10}
Krizhevsky, A.
\newblock {Learning Multiple Layers of Features from Tiny Images}.
\newblock Technical report, University of Toronto, 2009.

\bibitem[Kurth et~al.(2018)Kurth, Treichler, Romero, Mudigonda, Luehr,
  Phillips, Mahesh, Matheson, Deslippe, Fatica, Prabhat, and
  Houston]{nvidiasummit}
Kurth, T., Treichler, S., Romero, J., Mudigonda, M., Luehr, N., Phillips, E.,
  Mahesh, A., Matheson, M., Deslippe, J., Fatica, M., Prabhat, P., and Houston,
  M.
\newblock {Exascale Deep Learning for Climate Analytics}.
\newblock In \emph{SC}, 2018.

\bibitem[Leadbetter(1991)]{Leadbetter:91}
Leadbetter, M.~R.
\newblock {On a basis for ‘Peaks over Threshold’ modeling}.
\newblock \emph{Statistics \& Probability Letters}, 12\penalty0 (4), 1991.

\bibitem[Lin et~al.(2018)Lin, Han, Mao, Wang, and Dally]{lin2018deep}
Lin, Y., Han, S., Mao, H., Wang, Y., and Dally, W.
\newblock {Deep Gradient Compression: Reducing the Communication Bandwidth for
  Distributed Training}.
\newblock In \emph{ICLR}, 2018.

\bibitem[Mallat(2009)]{Mal:08}
Mallat, S.
\newblock \emph{{A Wavelet Tour of Signal Processing: The Sparse Way}}.
\newblock Academic Press, 2009.

\bibitem[Marcus et~al.(1999)Marcus, Santorini, Marcinkiewicz, and Taylor]{ptb}
Marcus, M., Santorini, B., Marcinkiewicz, M., and Taylor, A.
\newblock {Treebank-3}, 1999.
\newblock \url{https://catalog.ldc.upenn.edu/LDC99T42}.

\bibitem[Minka(2002)]{Minka:02}
Minka, T.~P.
\newblock {Estimating a Gamma distribution}.
\newblock Technical report, Microsoft Research, 2002.

\bibitem[Monga et~al.(2018)Monga, Mousavi, and Srinivas]{MonMouUma:18}
Monga, V., Mousavi, H.~S., and Srinivas, U.
\newblock {Sparsity Constrained Estimation in Image Processing and Computer
  Vision}.
\newblock In \emph{Handbook of Convex Optimization Methods in Imaging Science},
  pp.\  177--206. Springer International Publishing, 2018.

\bibitem[Narang et~al.(2018)Narang, Diamos, Elsen, Micikevicius, Alben, Garcia,
  Ginsburg, Houston, Kuchaiev, Venkatesh, and Wu]{Narang2018}
Narang, S., Diamos, G., Elsen, E., Micikevicius, P., Alben, J., Garcia, D.,
  Ginsburg, B., Houston, M., Kuchaiev, O., Venkatesh, G., and Wu, H.
\newblock {Mixed precision training}.
\newblock In \emph{ICLR}, 2018.

\bibitem[Narayanan et~al.(2019)Narayanan, Harlap, Phanishayee, Seshadri,
  Devanur, Ganger, Gibbons, and Zaharia]{PipeDream}
Narayanan, D., Harlap, A., Phanishayee, A., Seshadri, V., Devanur, N.~R.,
  Ganger, G.~R., Gibbons, P.~B., and Zaharia, M.
\newblock {PipeDream: Generalized Pipeline Parallelism for DNN Training}.
\newblock In \emph{SOSP}, 2019.

\bibitem[Olver(1997)]{Olver:97}
Olver, F.
\newblock \emph{Asymptotics and special functions}.
\newblock CRC Press, 1997.

\bibitem[Papoulis \& Pillai(2002)Papoulis and Pillai]{Papoulis:02}
Papoulis, A. and Pillai, S.
\newblock \emph{{Probability, Random Variables, and Stochastic Processes}}.
\newblock McGraw-Hill, 2002.

\bibitem[Peng et~al.(2019)Peng, Zhu, Chen, Bao, Yi, Lan, Wu, and
  Guo]{bytescheduler}
Peng, Y., Zhu, Y., Chen, Y., Bao, Y., Yi, B., Lan, C., Wu, C., and Guo, C.
\newblock {A Generic Communication Scheduler for Distributed DNN Training
  Acceleration}.
\newblock In \emph{SOSP}, 2019.

\bibitem[pytorch.org()]{pytorch}
pytorch.org.
\newblock {PyTorch}.
\newblock \url{https://pytorch.org/}.

\bibitem[Sergeev \& Balso(2018)Sergeev and Balso]{horovod}
Sergeev, A. and Balso, M.~D.
\newblock {Horovod: fast and easy distributed deep learning in TensorFlow}.
\newblock \emph{{{arXiv 1802.05799}}}, 2018.

\bibitem[Shanbhag et~al.(2018)Shanbhag, Pirk, and Madden]{ShanbhagMIT2018}
Shanbhag, A., Pirk, H., and Madden, S.
\newblock {Efficient Top-K Query Processing on Massively Parallel Hardware}.
\newblock In \emph{SIGMOD}, 2018.

\bibitem[Shi et~al.(2019)Shi, Chu, Cheung, and See]{shi2019understanding}
Shi, S., Chu, X., Cheung, K.~C., and See, S.
\newblock {Understanding Top-k Sparsification in Distributed Deep Learning}.
\newblock \emph{arXiv 1911.08772}, 2019.

\bibitem[Shi et~al.(2020)Shi, Zhou, Song, Wang, Zhu, Huang, Jiang, Zhou, Guo,
  Xie, Lan, Ouyang, Zhang, Wei, Gong, Lin, Gao, Meng, Xu, Guo, Yang, Chen, Wu,
  and Chu]{shi2021towards}
Shi, S., Zhou, X., Song, S., Wang, X., Zhu, Z., Huang, X., Jiang, X., Zhou, F.,
  Guo, Z., Xie, L., Lan, R., Ouyang, X., Zhang, Y., Wei, J., Gong, J., Lin, W.,
  Gao, P., Meng, P., Xu, X., Guo, C., Yang, B., Chen, Z., Wu, Y., and Chu, X.
\newblock Towards scalable distributed training of deep learning on public
  cloud clusters.
\newblock \emph{arXiv preprint arXiv:2010.10458}, 2020.

\bibitem[Shoeybi et~al.(2019)Shoeybi, Patwary, Puri, LeGresley, Casper, and
  Catanzaro]{megatronml}
Shoeybi, M., Patwary, M., Puri, R., LeGresley, P., Casper, J., and Catanzaro,
  B.
\newblock {Megatron-LM: Training Multi-Billion Parameter Language Models Using
  Model Parallelism}.
\newblock \emph{ArXiv 1909.08053}, 2019.

\bibitem[Simonyan \& Zisserman(2015)Simonyan and Zisserman]{vgg}
Simonyan, K. and Zisserman, A.
\newblock {Very Deep Convolutional Networks for Large-Scale Image Recognition}.
\newblock In \emph{ICLR}, 2015.

\bibitem[Smith(1984)]{Smith:84}
Smith, R.~L.
\newblock {Threshold Methods for Sample Extremes}.
\newblock In \emph{Statistical extremes and applications}. Springer, 1984.

\bibitem[Stich et~al.(2018)Stich, Cordonnier, and Jaggi]{stich2018sparsified}
Stich, S.~U., Cordonnier, J.-B., and Jaggi, M.
\newblock {Sparsified SGD with Memory}.
\newblock In \emph{NeurIPS}, 2018.

\bibitem[Vanhoucke et~al.(2011)Vanhoucke, Senior, and Mao]{Vincent2011}
Vanhoucke, V., Senior, A., and Mao, M.~Z.
\newblock {Improving the speed of neural networks on CPUs}.
\newblock In \emph{Deep Learning and Unsupervised Feature Learning Workshop -
  NeurIPS}, 2011.

\bibitem[Verbraeken et~al.(2020)Verbraeken, Wolting, Katzy, Kloppenburg,
  Verbelen, and Rellermeyer]{DML_survey}
Verbraeken, J., Wolting, M., Katzy, J., Kloppenburg, J., Verbelen, T., and
  Rellermeyer, J.~S.
\newblock {A Survey on Distributed Machine Learning}.
\newblock \emph{ACM Computing Surveys}, 53\penalty0 (2), 2020.

\bibitem[Wangni et~al.(2018)Wangni, Wang, Liu, and Zhang]{Wangni18}
Wangni, J., Wang, J., Liu, J., and Zhang, T.
\newblock {Gradient Sparsification for Communication-Efficient Distributed
  Optimization}.
\newblock In \emph{NeurIPS}, 2018.

\bibitem[Wu et~al.(2018)Wu, Huang, Huang, and Zhang]{wu_memqsgd}
Wu, J., Huang, W., Huang, J., and Zhang, T.
\newblock {Error Compensated Quantized SGD and its Applications to Large-scale
  Distributed Optimization}.
\newblock In \emph{ICML}, 2018.

\bibitem[Xu et~al.(2020)Xu, Ho, Abdelmoniem, Dutta, Bergou, Karatsenidis,
  Canini, and Kalnis]{grace}
Xu, H., Ho, C.-Y., Abdelmoniem, A.~M., Dutta, A., Bergou, E.~H., Karatsenidis,
  K., Canini, M., and Kalnis, P.
\newblock Compressed communication for distributed deep learning: Survey and
  quantitative evaluation.
\newblock Technical report, KAUST, 2020.
\newblock \url{http://hdl.handle.net/10754/662495}.

\end{thebibliography}

\normalsize
\clearpage

\appendix



\begin{algorithm*}[!t]
	\caption{\scheme\! Algorithm}
	\label{algo:algo1}
	\KwIn{$ {\bf g}$\: the gradient vector to be compressed} \KwIn{$\delta$\: the target compression ratio}
	\KwIn{$i$\: the training iteration number }
	\KwIn{$Q$: the frequency of invoking stages adaption}
	\KwIn{($\epsilon_H$,$\epsilon_L$): upper and lower bounds of estimation error for adapting the stages.}
	\tcc{The discrepancy tolerance $\epsilon$ in \eqref{eq:boundederror} can be computed as $\epsilon=\max(\epsilon_H,\epsilon_L)$}
	Define $\hat{k}, \hat{k}_{avg}$:  number of and average number of elements obtained from threshold estimation.\\
	\Fn{Sparsify(${\bf g}, \delta$)}{
		\tcc{Multi-stage threshold estimation}
		${\bf temp}=copy({\bf g})$\\
		$\eta_0 = 0$ \\
		\For{each stage $m$ in (1,$M$)}
		{
			${\bf temp}={\bf temp}-\eta_{m-1}$\\
			$\eta_m=\text{Thresh\_Estimation}({\bf temp}, \delta_m, DIST)$\\
			$\widehat{I} =abs({\bf temp}) > \eta_m $ - find the index of top absolute values larger than $\eta_m$. \\
			Use $\widehat{I}$ to filter {\bf temp} which keeps the top values of {\bf temp} and sets the others to zero.\\
		}
		\If{$i \mod Q == 0$}
		{
		    \tcc{Call stages adaption function to adjust number of stages}
			$M$=Adapt\_Stages($M$, $\frac{\hat{k}_{avg}}{Q}$ )\\
			$\hat{k}_{avg} = 0$
		}
		\Else
		{
			$\hat{k}_{avg} = \hat{k}_{avg} + \hat{k}$\\
		}
		\tcc{Threshold sparsification: find indices of top absolute values larger than the threshold $\eta_M$, then use them to filter the elements of the gradient {\bf g}}
		return $\Chat({\bf g}) = {\bf g}[abs({\bf g}) \geq \eta_M]$

	}
	\Fn{Thresh\_Estimation($\g, \delta$, DIST)}{
		\If{DIST is Exponential}{
			$\widehat{\mu}$ = mean(abs($\g$))\\
			return $ -\widehat{\mu}\, \log(\delta)$
		}
		\If{DIST is GPareto}{
			$\widehat{\mu}, \widehat{\sigma}$ = mean\_var(abs($\g$))\\
			${\displaystyle \hat{\alpha} = 0.5\, \left(1-\frac{\widehat{\mu}^2}{\widehat{\sigma}^2} \right)}$\\
			${\displaystyle \hat{\beta} = 0.5\, \widehat{\mu}\, \left(\frac{\widehat{\mu}^2} {\widehat{\sigma}} + 1\right)}$\\
			return ${\displaystyle \frac{\hat{\beta}}{\hat{\alpha}}\, \left(\exp({-\hat{\alpha} \text{log}(\delta)})-1\right)}$
		}
		\If{DIST is Gamma}{
			$\widehat{\mu}$ = mean(abs($\g$))\\
			$s$ = log($\widehat{\mu}$) - mean(log(abs($\g$)))\\
			${\displaystyle \hat{\alpha} = \frac{3 - s+ \sqrt{(s - 3)^2 + 24\,s}}{12\,s}}$\\
			${\displaystyle \hat{\beta}=\frac{\widehat{\mu}}{\hat{\alpha}}}$\\
			return ${\displaystyle -\hat{\beta} \, (\text{log}(\delta) + \text{log}(\text{gamma}(\hat{\alpha})))}$
		}
		
	}
	\Fn{Adapt\_Stages$(M, \hat{k}_{avg})$}
	{
		\tcc{Choose the number of stages}
		\If {$\hat{k}_{avg} > k * (1 + \epsilon_H)$  }
		{
			$M = M - 1$
		}
		\If{$\hat{k}_{avg} <  k * (1 - \epsilon_L)$}
		{
			$M = M + 1$
		}
		return min(max($M$, 1),$M_{max}$)
	}
\end{algorithm*}


\clearpage
\section{Distributed Synchronous SGD (DSSGD) with Sparsification}
\label{apdx:dsgd}

Here, we describe the Distributed Synchronous version of \ac{SGD} optimization algorithm which is the main work-horse behind most of the distributed training~\cite{aji_sparse}. \cref{algo:dsgd} presents the specifics of the algorithm with sparsification. The algorithm executes in parallel on each worker and starts with sampling a batch-sized sub-set from the full training dataset. Then, each worker performs a local pass including forward pass to obtain a loss value followed by a backward pass to calculate the gradient vector for updating the model parameters. At this point and before moving along with the model update, the workers have to synchronize by performing aggregation (or averaging) of their gradient vectors and use the aggregated gradient for the update. The gradient is sparsified using a chosen compressor (e.g., $\topk$, DGC, \scheme\!, .., etc) and target sparsification ratio. For example, to invoke \scheme\! compressor, one would invoke function {\emph Sparsify} of \cref{algo:algo1} which takes as input the gradient $\g$ and target sparsification ratio $\delta$. Then, the aggregation can be either accomplished via means of a parameter server which has a global copy of the model parameters and receives the gradients from the workers and update its local model and then the workers can pull the up-to-date model parameters at any time~\cite{Dean2012}. The other way is to perform aggregation in a peer-to-peer fashion via means of collective operation like All-Reduce or All-Gather which requires no extra parameters server for the aggregation~\cite{horovod}. The peer-to-peer collective communication  methods are widely adopted by most frameworks \cite{pytorch, horovod} and known to scale well in practice \cite{Goyal2017} and hence is adopted in this work. 

\subsection{Discussion on \scheme\! Algorithm}
We highlight a few technical aspects of \scheme\! algorithm presented in \cref{algo:algo1}.
\smartparagraph{Scalability concerns:} The compression algorithm has no scalability issues since it executes locally and does not rely on inter-node communication. Also, the compressor only depends on the size of the gradient vector leading to the same compression time on all the workers regardless of the number of training workers that run in parallel.

\smartparagraph{Algorithm's dependence on training iteration:} the gradient sparsity changes over iterations as shown in ~\cref{fig:fitteddistributions}~and~\cref{fig:fitteddistributionsEC}. The proposed algorithm leverages extreme-value theorem to handle sparsity variations by adapting the number of stages at each iteration. This enables adaptive fitting of the gradient at each iteration via the sparsity-inducing distribution enabling the estimation of an approximate threshold that obtains the top $k$ elements of the gradient vector. 

\smartparagraph{Sparsity and compressability of the gradients:} our algorithm relies on a principled statistical approach, which makes it robust to various sparsity levels of the gradient given that the compressibility property holds. And if not, most sparsifiers would be equally ineffective. Moreover, the compressibility property is the reason why $\topk$ is commonly used in the literature. Therefore, in this work, we seek an approximate fast threshold estimation method that exploits a common prior information of the gradients, while preserving the convergence guarantee of $\topk$, albeit with different rate depending on the accuracy of the estimated threshold.


\begin{algorithm}[!t]
\caption{Sparsified Distributed Synchronous SGD}
\label{algo:dsgd}
\tcc{Worker $n$}
\tcc{Initialization}
\KwIn{$D$: Local Dataset}
\KwIn{$B$: Minibatch size per node}
\KwIn{$N$: The total number of workers}
\KwIn{$\lambda$: The learning rate}
\KwIn{$\delta$: The target sparsification ratio}
\KwIn{$x$: Model parameters $x=(x[0],x[1], ...,x[d])$}
\tcc{loop till end of training}
\For{i = 0, 1, ...} 
{
    \tcc{Calculate stochastic gradient}
    $\gin[i]$ = 0\\
    \For{i = 1, ..., B} 
    {
        Sample data point $d$ from $D$\\
        Calculate   $\nabla f(x;d)$\\
        $\gin[i] = \gin[i] + \frac{1}{B} \nabla f(x;d)$\\ 
    }
    \tcc{Aggregate workers' gradients}
    Collective-Comm: $\Gin[i] = \frac{1}{N} \sum_{n=1}^{N} \text{Sparsify}(\gin[i], \delta)$\\
    \tcc{Update model parameters}
    $\xin[i+1] = \xin[i] + \lambda\,  \Gin[i]$\\
}
\end{algorithm}

\section{Gradient Features and Distribution}

\subsection{Validation of Gradient Compressibility}
\label{apdx:statmethods}
The compressibility of the gradient vector allows efficient compression for the gradients through sparsification techniques, e.g., $\topk$ and thresholding-based compression \cite{Elzanaty19,grace}. Here, we empirically investigate the compressibility of the gradients according to \cref{def:compressable}. In order to verify the vector compressibility, we consider the gradients generated while the training of ResNet20. The absolute of the gradients are sorted in descending order to obtain the vector $\tilde{\g}$ with $\d=269722$. In \cref{Fig:compressgradients}, the elements of the gradient vector $\tilde{\g}$, i.e., $\tilde{\gr}_{j}$, are reported vs their index,  for three iterations in the beginning, middle, and end of the training.\footnote{Note that  in \cref{Fig:compressgradients}, we focus only on the elements from $1$ to $10^5$, as for larger indices the  amplitude of the vector elements are sufficiently small.} As a benchmark, we report a power low decay example with decay exponent $p>0.5$, i.e., $p=0.7$.  It can be noticed that the gradients follow a power-law decay with decay exponent $p>0.7>0.5$; hence, they are compressible from \eqref{eq:powerlaw}. 

In \cref{fig:bestkapprox}, we report the sparsification error for the best $\k$ approximation, e.g., the $\topk$, as a function of $\k$. We also report an example of the power decay model with decay exponent $p-0.5=0.2$. We can see the best $\k$ approximation error decays faster than the benchmark. Hence, the vector can be considered compressible, according to \eqref{eq:bestkapprox}.

We also validate this behavior for various models and datasets, not reported here for conciseness. Therefore, the gradient vectors can be considered compressible in the sense of \cref{def:compressable}. 
\begin{figure*}[!t]
  \centering
  \begin{subfigure}[h]{0.48\textwidth}
      	\includegraphics[width=1\textwidth]{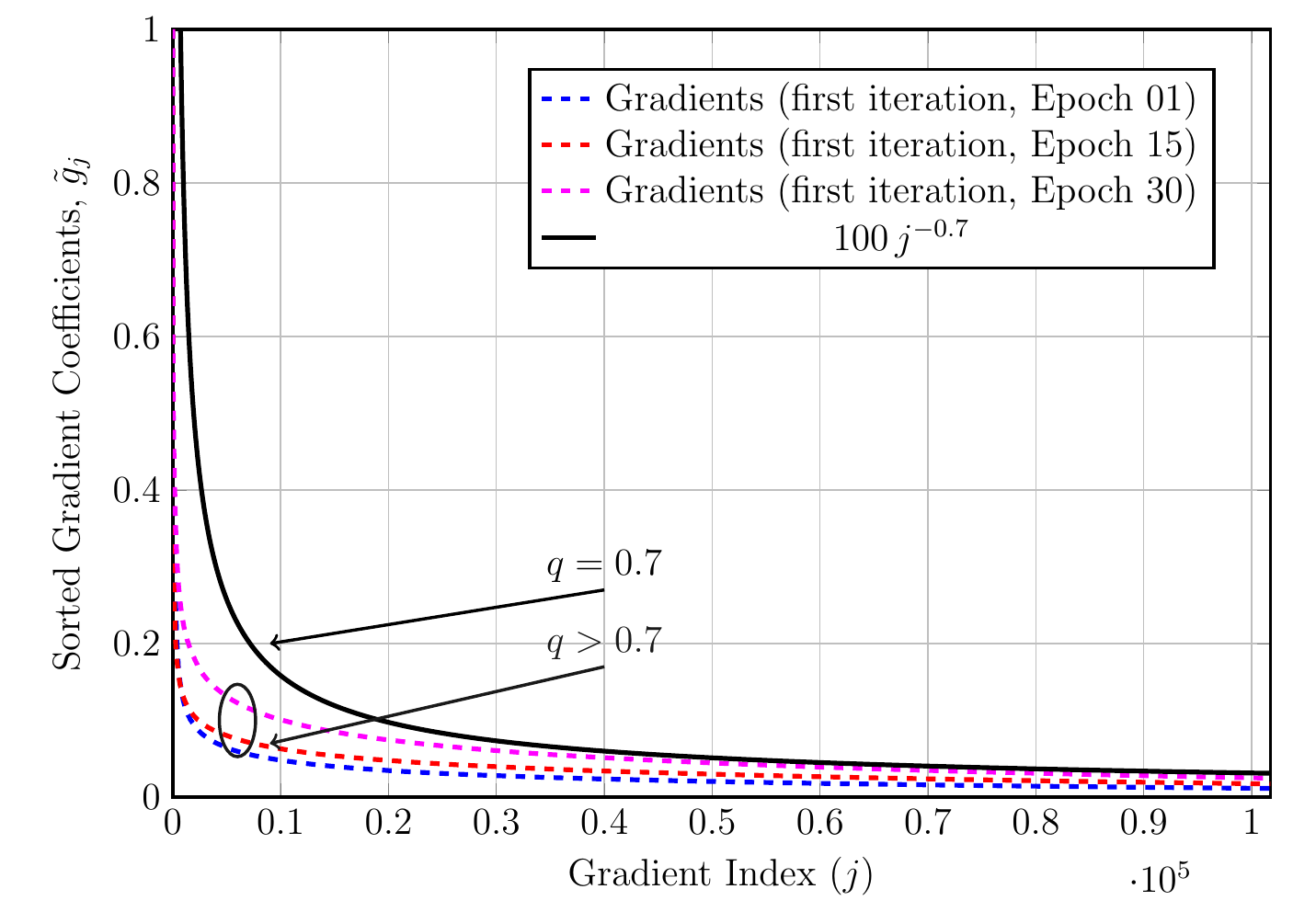}
    	\caption{\footnotesize The sorted magnitude of the gradients vs their indexes, and the fitted curve via power law in \eqref{eq:powerlaw}.} 
    	\label{Fig:compressgradients}
    \end{subfigure}
    \hfill
      \begin{subfigure}[h]{0.48\textwidth}
    	\includegraphics[width=1\textwidth]{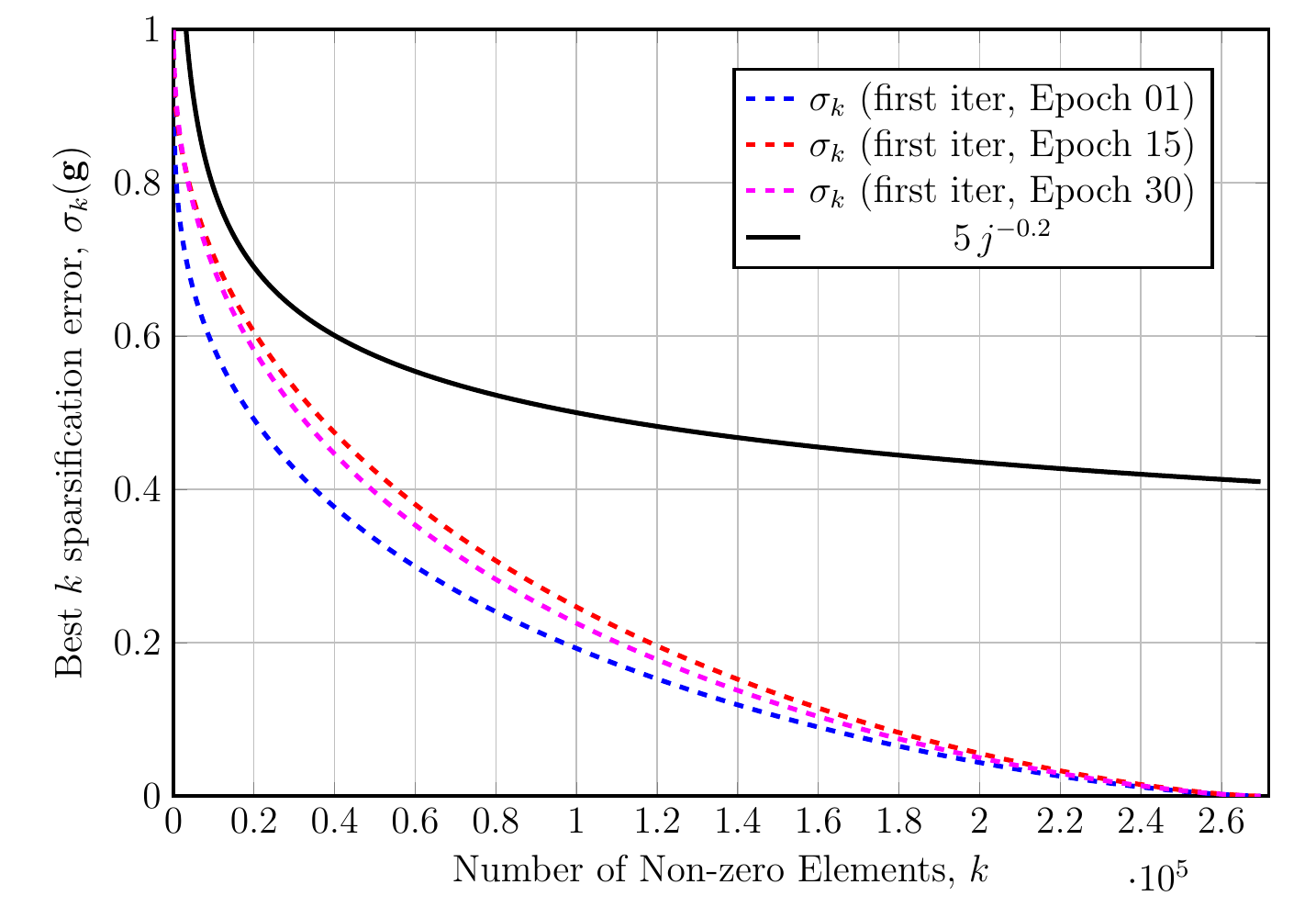}
    	\caption{\footnotesize The approximation error for the $\topk$ vs the number of non-zero elements, $\k$.}
        \label{fig:bestkapprox}
     \end{subfigure}
     \caption{The compressibility property of the gradients}
     \label{fig:compressability}
\end{figure*}

\subsection{Validation of Gradient Distributions}
\label{apdx:graddist}
\label{sec:empvalid}

In this part, we discuss the distribution of the gradients generated while training several neural networks. Since the gradient vectors are compressible, \acp{SPD} can approximate the gradient distribution. This feature, i.e. Property \ref{property:sparspromotdist}, is numerically validated as follows. 

First, we consider the gradients from training ResNet-20 with \ac{SGD}. The generated gradient vector at iteration $i$ is compressed using $\topk$ with $\delta=0.001$, and the distributed \ac{SGD} is employed as described in \cref{apdx:dsgd}. We investigate two cases: i) memoryless compression, where the \acf{EC} mechanism is not applied, as shown in \cref{fig:fitteddistributions} ii) memory-based compression, where an \ac{EC} mechanism is deployed by adding the sparsification error from the previous iteration to the gradient vector before the $\topk$ compression, i.e., $\g_{\{i\}}=\g_{\{i\}}+ \left[\g_{\{i-1\}} -\Tk[\g_{\{i-1\}}]\right]$. For both cases, we collect the uncompressed gradients from the master worker, as different workers have similar gradients in distributions. The gradient vectors are then normalized by their $\ell_{2}$ norm to easily visualize and compare the evolution of the gradient distributions over various iterations. Then, the collected gradients are fitted by the three proposed \acp{SPD}, i.e., double exponential, double gamma, and double \ac{GPD} distributions. The parameters of the distribution are estimated as indicated in  Corollary~\ref{corollary:Laplacethreshold}, Corollary~\ref{corollary:gammathreshold}, and Corollary~\ref{corollary:Gparetothreshold}.

For the training with \ac{EC} mechanism in \cref{fig:fitteddistributionsEC}, it becomes more challenging to fit the gradients, especially for larger iterations, as can be seen in \cref{fig:PDF2EC}.  This behavior arises due to the addition of the sparsification error from the previous stage to the gradients, as the resulting distribution of the gradients changes. More precisely, the gradient distribution is the convolution between the \ac{PDF} of the error from the last stage and  \ac{PDF} of the current gradient vector before the \ac{EC}. Therefore, the distribution of the gradients significantly changes over the iterations. Therefore, single-stage fitting suffers more when \ac{EC} mechanism is used, particularly for fitting the tail, as in \cref{fig:CDF1EC} and \cref{fig:CDF2EC}.

We also validate that the gradients generated from the other networks in \cref{tab:models} can be well approximated with \acp{r.v.} distributed according to one of the \acp{SPD}, which are not reported here for conciseness.  In general, there is a slight variation in the fitting accuracy among the three \acp{SPD} for various networks and datasets, due to the nature of the gradients. For example, the double exponential distribution can not capture well the gradients with an empirical distribution that decays fast. In contrast, the double gamma and double \ac{GPD} distributions have an additional shape parameter that can approximate the behavior of sparser vectors with $\alpha<1$. Nevertheless, the double-exponential behaves well when the distribution of the absolute of the gradients decays as fast as the exponential distribution.

\begin{figure*}[!ht]
\centering
\begin{subfigure}[h]{0.48\textwidth}
\centering
 	\includegraphics[width=1\textwidth]{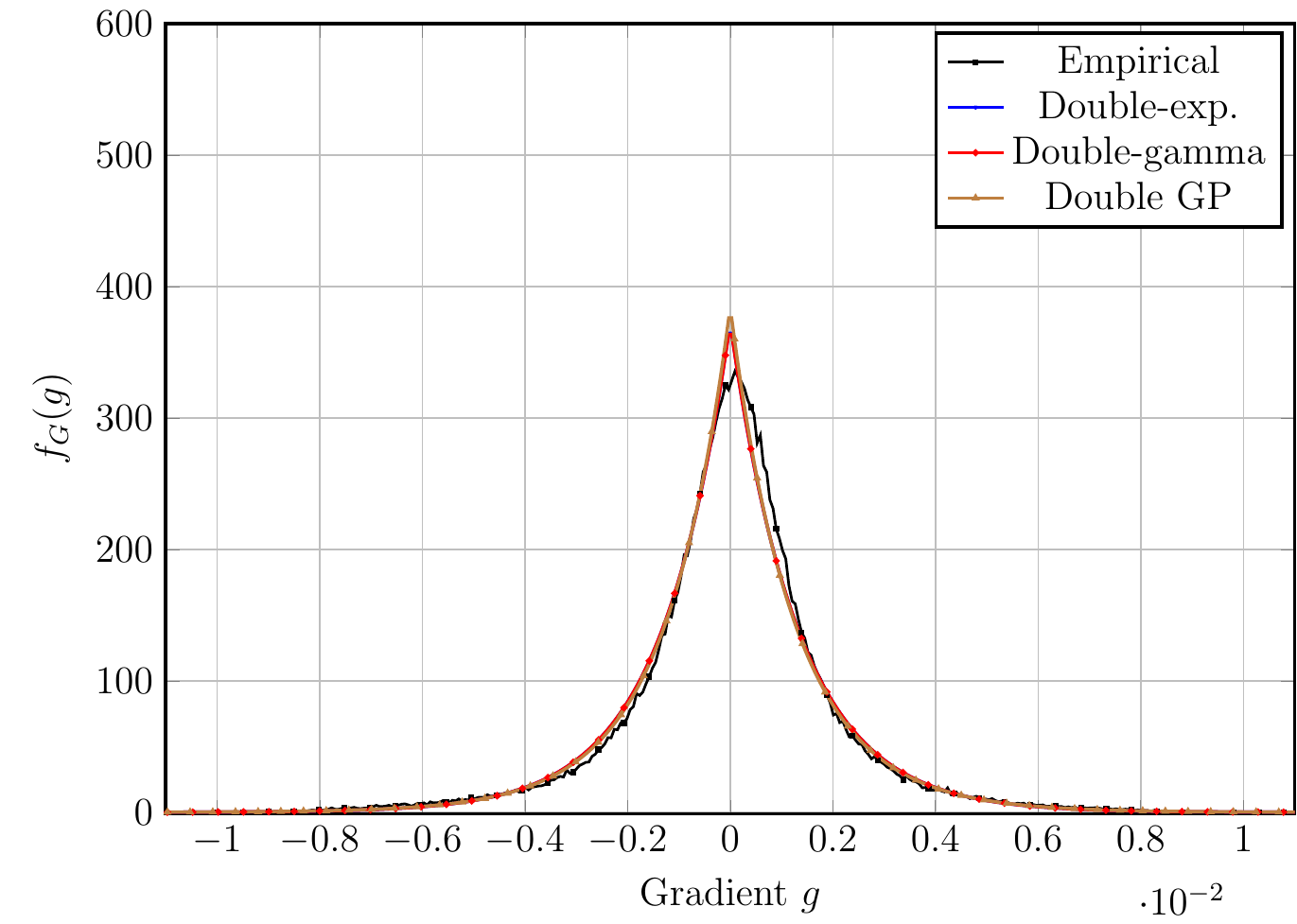}
 \caption{}
    \label{fig:PDF1EC}

    \end{subfigure}
\hfill
\begin{subfigure}[h]{0.48\textwidth}
\centering
 	\includegraphics[width=1\textwidth]{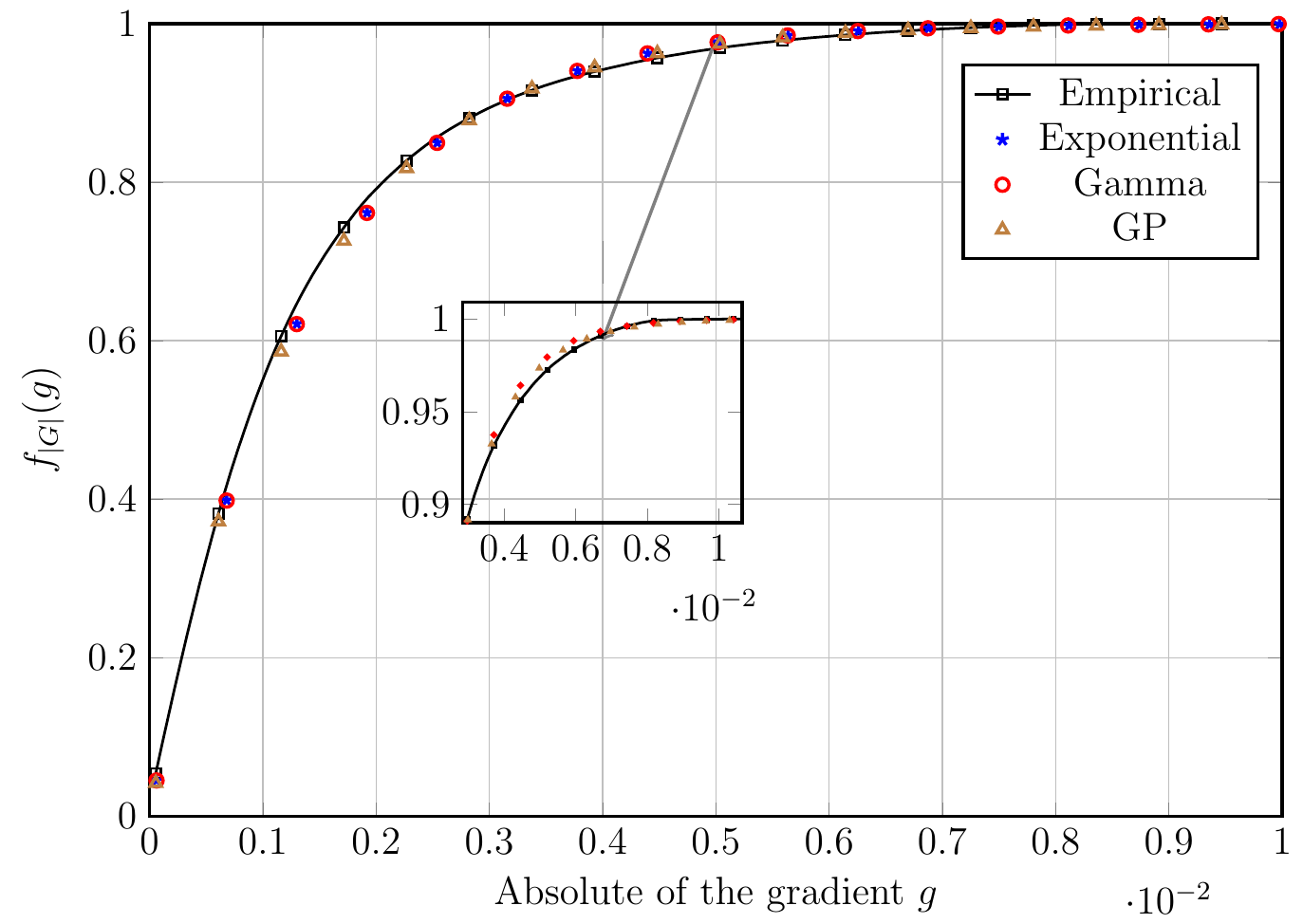}
\caption{}
    \label{fig:CDF1EC}
\end{subfigure}
\\
\begin{subfigure}[h]{0.48\textwidth}
\centering
     	\includegraphics[width=1\textwidth]{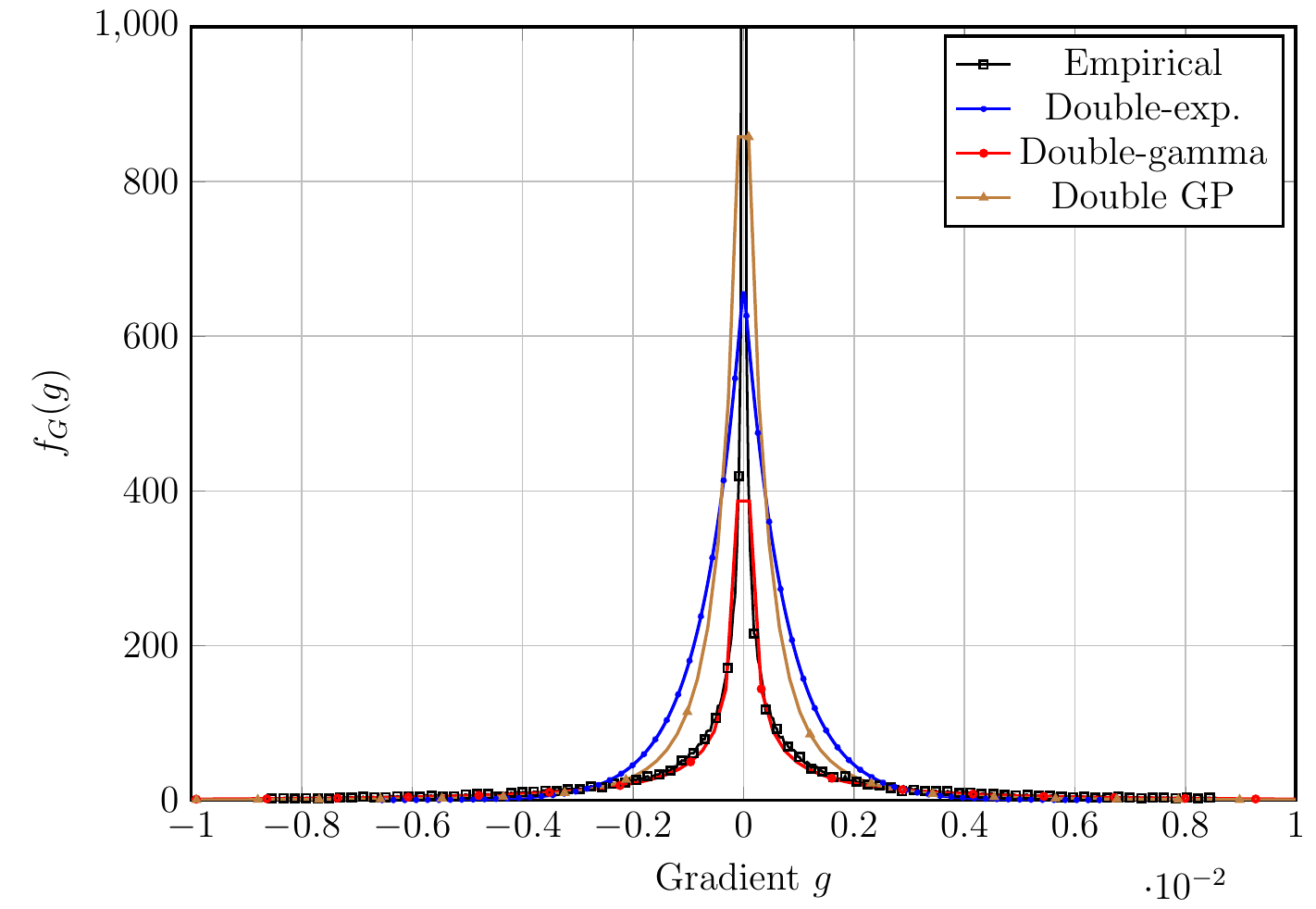}
 \caption{}
    \label{fig:PDF2EC}
    
    \end{subfigure}
\hfill
\begin{subfigure}[h]{0.48\textwidth}
\centering
 	\includegraphics[width=1\textwidth]{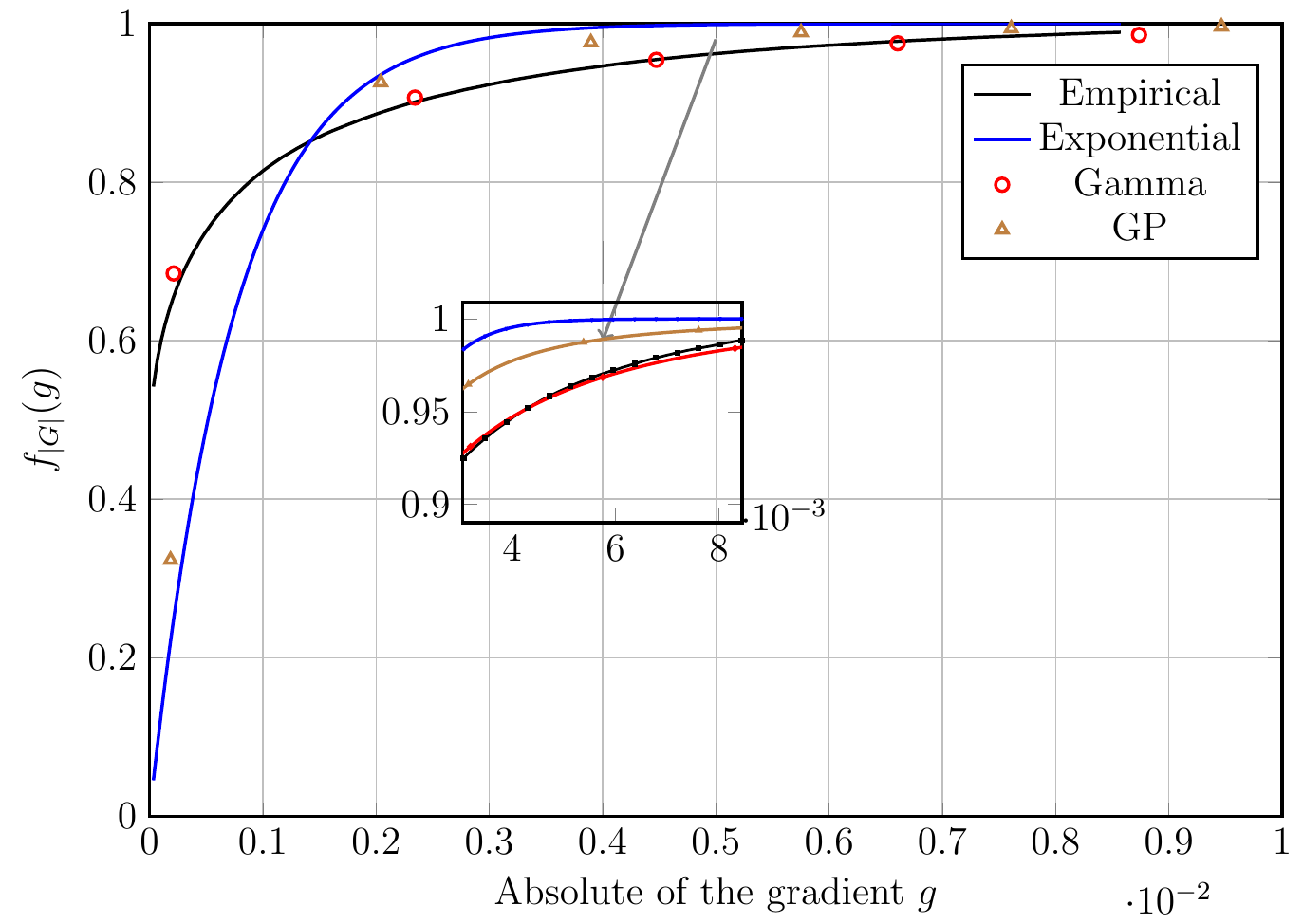}
\caption{}
    \label{fig:CDF2EC}
\end{subfigure}
\caption{Gradient fitting using the three \acp{SPD} for the gradient vector along with the empirical distribution generated from training ResNet-20 on CIFAR10 dataset using $\topk$ compressor \textbf{with \ac{EC} mechanism}, for the $100^{\text{th}}$ [(a) PDF, (b) CDF] and $10000^{\text{th}}$ [(c) PDF, (d) CDF] iterations.}
\label{fig:fitteddistributionsEC}
\end{figure*}

\subsection{Analysis of Double Gamma and Double Generalized Pareto Distributed Gradients}
\label{apdx:threshodcalculation}
\subsubsection{Threshold Calculation for Gamma Distributed Gradients} \label{apdx:gammathreshold}
\begin{customcorollary}{1.2}\label{corollary:gammathreshold}
Considering that the gradients that can be well-fitted by double gamma distribution with  shape parameter $\alpha \leq 1$, the absolute of the gradient is gamma distributed \cite{Bond:01}, i.e., $ |\G| {\sim} \operatorname{gamma}(\alpha,\b)$. The sparsifying threshold can be derived as
\begin{align} \label{eq.gammathreshold}
	\eta(\delta) &=\hat{\b} \, P^{-1}(\hat{\alpha},1-\delta) \\
	& \simeq   -\hat{\b}\, {\left[\log(\delta)+ \log({\Gamma(\hat{\alpha}))}\right]},  
	\end{align}
where	${  P(\alpha,x)\triangleq \frac{1}{\Gamma(\alpha)}\int _{0}^{x}t^{\alpha-1}\,\mathrm {e} ^{-t}\,{\rm {d}}t\,}$ is the regularized lower incomplete gamma function, and $P^{-1}(\alpha,p)\triangleq \{x: P(\alpha,x)=p\}$ is the inverse of the regularized lower incomplete gamma function \cite{AbraSte:65},
\begin{align}
\ahat &\triangleq \frac{3 - s+ \sqrt{(s - 3)^2 + 24\,s}}{12\,s}, &\bhat \triangleq  \frac{\widehat{\mu}}{\ahat},
\end{align}  
with $s \triangleq \log(\widehat{\mu})-\widehat{\mu}_{\log}$,  $\widehat{\mu}_{\log} \triangleq \frac{1}{\d}\sum_{i=1}^{\d} \log\left({|\gr|}_{i}\right)$ 
, and $\widehat{\mu}$ and $\widehat{\sigma}^2$ are  the sample mean and variance for the absolute gradient vector $|\g|$, respectively.
 \end{customcorollary}
%
 \begin{proof}
 The gradients are modeled by double gamma distribution %
 with $\alpha \leq 1$,
 with \ac{PDF} defined in \cite{Bond:01} as 
 \begin{align}
f_{\G}(\gr;\alpha,\b) & =\frac{1}{2} \frac{|\gr|^{\alpha-1} e^{-|\gr|/\beta}}{\b^\alpha\Gamma(\alpha)}, &&\text{for}\quad -\infty < \gr < \infty. 
 \end{align}
Hence, the absolute of the gradient is modeled as gamma distribution with \ac{PDF} 
\begin{align} \label{eq:gammapdf}
f_{|\G|}(\gr;\alpha,\b) & = \frac{\gr^{\alpha-1} e^{-\gr/\beta}}{\b^\alpha\Gamma(\alpha)}, &&\text{for}\quad 0 \leq \gr < \infty. 
 \end{align}
The \ac{CDF} of the gamma \ac{r.v.} which can be written from \eqref{eq:gammapdf} as
\begin{align}
F_{|\G|}(\gr;\alpha,\b) &=\int_{0}^{\gr}\frac{t^{\alpha-1} e^{-t/\beta}}{\b^\alpha\Gamma(\alpha)} \mathrm{d}t\\
&=\int_{0}^{\gr/\b}\frac{ z^{\alpha-1} e^{-z}}{\Gamma(\alpha)} \mathrm{d}z \triangleq P(\alpha, \gr/\b ) ,
\end{align}
where  $P(\alpha, x)$ is the regularized lower incomplete gamma function \cite{AbraSte:65}. The threshold in \eqref{eq.gammathreshold} follows from the inverse of the \ac{CDF} at $1-\delta$, as illustrated in \eqref{eq:thresholdabs} and by substituting the parameters of the gamma distribution  with their estimates ${\ahat}$ and $\bhat$. Nevertheless, calculating the threshold involves the inverse of incomplete gamma function which can be computationally heavy. In the following we provide a closed-form approximation for the threshold. First, we would like to find a closed-form approximation for the inverse lower incomplete function at $1-\delta$, i.e., $x \triangleq P^{-1}(\ahat,1-\delta)$.  Starting from the bound on  $P(\ahat, x)$ in \cite{Olver:97}, we have
\begin{align}
    P(\ahat, x)&= 1-\delta \geq 1- \frac{x^{\ahat-1}\,e^{-x}}{\Gamma(\ahat)} \qquad \text{for\,\,}\ahat\leq 1, x> 0. 
\end{align}
After some manipulations, we get
\begin{align}
       x   &\leq -\log\left(\delta \Gamma(\ahat)\right) -(1-\ahat)\log(x),  &&\text{for\,\,}{\ahat}\leq 1, x> 0,\\
    x&\leq - \log\left(\delta\right) -\log\left(\Gamma({\ahat})\right),  &&\text{for\,\,}{\ahat}\leq 1, x\geq 1.\label{eq:inversep}
\end{align}
Finally, by substituting $P^{-1}(\ahat,1-\delta)$ with \eqref{eq:inversep} in \eqref{eq.gammathreshold}, we get
\begin{align}
 \eta &\leq    - {\bhat}\, \left[\log(\delta)+ \log({\Gamma({\ahat}))}\right], &&\text{for\,\,}{\ahat}\leq 1, x\geq 1
\end{align}
with equality if $\ahat=1$. For $0<x<1$ or  $\ahat>1$, the bound does not hold, however, it provides a good approximation for the threshold when $\ahat$ is close to one. 

For estimating the pentameters, let us start by 
 %
the \ac{PDF} of the gamma distribution, defined as
\begin{align}
f_{|\G|}(\gr;\a,\b) & = \frac{\gr^{\a-1} e^{-\gr/\b}}{\b^\a\Gamma(\a)} \quad \text{ for } x > 0, \quad \a,\b>0
\end{align}
where $\alpha$ and $\beta$ are the shape and scale parameters, respectively. The shape parameter can be estimated from the absolute gradient vector $|\g|$ using \ac{MLE} as the solution of
\begin{align}\label{eq.alphaMLE}
\Psi(\alpha)-\log(\alpha)+\log(\widehat{\mu})- \widehat{\mu}_{\log}=0,
\end{align}
where $\Psi(x)\triangleq \frac{\d \Gamma(x)}{\d x}$ is the digamma function, $\widehat{\mu} \triangleq \frac{1}{\d}\sum_{i=1}^{\d} {|\gr|}_{i}$ is the sample mean, and $\widehat{\mu}_{\log} \triangleq \frac{1}{\d}\sum_{i=1}^{\d} \log\left({|\gr|}_{i}\right)$  \cite{Papoulis:02}. On the other hand, the scale parameter can be estimated as   $\bhat=\widehat{\mu}/\alpha$. Nevertheless, the shape parameter estimation in \eqref{eq.alphaMLE} involves solving a non-linear equation with a special function. Hence, it increases the computational complexity for the scheme, leading to higher time overhead for the compression. In order to reduce the complexity, we propose to employ a simpler closed-form approximation for the shape parameter, i.e.,
\begin{align}
&\ahat = \frac{3 - s+ \sqrt{(s - 3)^2 + 24\,s}}{12\,s},   & \bhat=\frac{\widehat{\mu}}{\alpha},
\end{align}
where $s \triangleq \log(\widehat{\mu})-\widehat{\mu}_{\log}$ \cite{Minka:02}.
 \end{proof}
\subsubsection{Threshold Calculation for Generalized Pareto Distributed Gradients} \label{apdx:GPDthreshold}
\begin{customcorollary}{1.3}\label{corollary:Gparetothreshold}
For gradients distributed as double generalized Pareto \acp{r.v.}, the absolute of the gradients is modeled as \ac{GPD} distributed \acp{r.v.}, $|\G|\sim \operatorname{GP}(\a,\b,\loc)$, where $0 < \a < 1/2$, $\b$, $\loc=0$ are the shape, scale, and location parameters. The sparsifying threshold that achieves a compression ratio $\delta$ is 
\begin{align}\label{eq:thresholddgpd}
    \eta &= \frac{\bhat}{\ahat} \left(e^{-\ahat \log(\delta)}-1 \right),
\end{align}
where
\begin{align}\label{eq:gpdestimatesappdx}
     \ahat&\triangleq \frac{1}{2}\,  \left(1-\frac{\hat{\mu}^2}{\hat{\sigma}^2}   \right), &&\bhat \triangleq \frac{1}{2}\, \hat{\mu} \left(\frac{\hat{\mu}^2}{\hat{\sigma}^2} +1  \right),
 \end{align}
 with $\hat{\mu}$ and $\hat{\sigma}^2$ being the sample mean and variance for the absolute gradient vector, $|\g|$, respectively.
\end{customcorollary}
\begin{proof}
the gradients can be well-fitted by double \ac{GPD} distribution with \ac{PDF}, indicated in \cite{ArmDunLee:13} as\footnote{The double \ac{GPD} distribution resembles the Laplacian distribution for $\a \rightarrow 0$. Similarly, the \ac{GPD} becomes exponential distribution for $\a=0$.}  
\begin{multline}
    f_{\G}(\gr)= \frac{1}{2 \b} \left(1+ \a \,\frac{|\gr|}{\b} \right)^{-(\frac{1}{\a}+1)}, \\ 0 < \a < \frac{1}{2},\,- \infty <\gr < \infty
\end{multline}
Hence, the absolute of the gradients can be modeled as \ac{GPD} distributed \acp{r.v.} with \ac{PDF}
\begin{align}
    f_{|\G|}(\gr)= \frac{1}{\b} \left(1+ \a\, \frac{\gr}{\b} \right)^{-(1/\a+1)}, && 0 < \a \leq \frac{1}{2},\, \gr \geq 0
\end{align}
and the corresponding \ac{CDF} can be written from \cite{HosWal:87} as
 \begin{align}\label{eq:cdfgpd}
     F_{|\G|}(\gr)= 1- \left(1+ \a\, \frac{\gr}{\b} \right)^{-1/\a}.
 \end{align}
 The inverse \ac{CDF} can be written from \eqref{eq:cdfgpd} as
 \begin{align} \label{eq:inversecdfgpd}
     F^{-1}_{|\G|}(p)= \frac{\b}{\a} \left(e^{-\a \log(1-p)}-1 \right). 
 \end{align}
 From  \eqref{eq:thresholdabs} and \eqref{eq:inversecdfgpd} and by substituting the distribution parameters with their estimates, provided below, the threshold in \eqref{eq:thresholddgpd} follows. 

%
 Unfortunately, there are no closed-form \ac{ML} estimators for the parameters of \ac{GPD} distributions. Hence, the  \ac{ML} estimates have to be computed through complex numerical optimization. Alternately, the parameters can be estimated in closed-from through the \ac{MM} method under some conditions on the shape parameter \cite{HosWal:87}. More precisely, for the considered range of the shape parameter, i.e., $-0.5<\alpha<0.5$, the first and second moments exit and they can be written as
 \begin{align} \label{eq:meanssgdp}
 \mu &= \frac{\b}{1+\a}, && S^2=  \frac{\b^2}{(1+\a)^2 (1+2\,\a)},
 \end{align}
 where $\mu$ and $S^2$ are the mean and mean square, respectively. Therefore, from \eqref{eq:meanssgdp} through the \ac{MM} method, the parameters can e estimated as
 \begin{align}\label{eq:gpdestimatesapndx}
     \ahat&= \frac{1}{2}\,  \left[1-\frac{\hat{\mu}^2}{\hat{\sigma}^2} \right], &&\bhat= \frac{1}{2}\, \hat{\mu} \left[\frac{\hat{\mu}^2}{\hat{\sigma}^2} +1  \right],
 \end{align}
 where $\hat{\mu}$ and $\hat{\sigma}^2$ are the sample mean and variance for the absolute gradient vector, $|\g|$, respectively.
 \end{proof} 
\subsubsection{Proof of \cref{lemma:PoT}} \label{apdx:prooflemmaPoT}
The distribution of the \ac{PoT} absolute gradients for the $m$th stage can be approximated as \ac{GPD} distribution from Theorem 4.1 in \cite{Coles:01} with \ac{CDF} 
\begin{align}
 F_{{{|\bar{\G}_{m}|}}}(\gr)=& 1- \left(1+ \a_{m}\, \frac{\gr-\eta_{m-1}}{\b_{m}} \right)^{-1/\a_{m}}, \nonumber \\
 &\gr \geq \eta_{m-1}, -1/2< \a_{m} < 1/2,
 \label{eq:mscdfgpd}
\end{align}
where the first and second moments of the \ac{r.v.} ${{{|\bar{\G}_{m}|}}}$  are finite and the \ac{PDF} is smooth for $-1/2 <\a_{m} < 1/2$ \cite{HosWal:87}. The inverse \ac{CDF} can be written from \eqref{eq:mscdfgpd} as
 \begin{align} \label{eq:msinversecdfgpd}
     F^{-1}_{{{|\bar{\G}_{m}|}}}(p)= \frac{\b_{m}}{\a_{m}} \left(e^{-\a_{m} \log(1-p)}-1 \right) +\eta_{m-1}. 
 \end{align}
 The threshold in \eqref{eq:msthresholddgpd} follows from  \eqref{eq:thresholdabs} and \eqref{eq:msinversecdfgpd} and by substituting the distribution parameters with their estimates derived as from \eqref{eq:meanssgdp}
   \begin{align}\label{eq:gpdestimatesapndxPoT}
     \ahat_{m}&= \frac{1}{2}\,  \left[1-\frac{\bar{\mu}^2}{\bar{\sigma}^2} \right], &&\bhat_{m}= \frac{1}{2}\, \bar{\mu} \left[\frac{\bar{\mu}^2}{\bar{\sigma}^2} +1  \right],
 \end{align}
 where the sample mean $\bar{\mu}$ and the variance $\bar{\sigma}^2$ are computed from absolute of the \ac{PoT} gradients shifted by the threshold, i.e., $|\tilde{\g}_{m}|-\eta_{m-1}$.
 
\subsubsection{Proof of Corollary~\ref{corollary:expPoT}} \label{apdx:proofcorollaryexpPoT}
The \ac{CCDF} of the exceedance \ac{r.v.} can be written as
\begin{align}
    &\mathbb{P}\left\{{{{|\bar{\G}_{m}|}}} \geq \gr\right\} \nonumber \\ &=\mathbb{P}\Big\{{|\G_{m}|} \geq \gr \,\Big|\, |\G_{m}|>  \eta_{m-1} \Big\},\forall \gr \geq \eta_{m-1} \\
    &=\mathbb{P}\Big\{{|\G_{m}|} \geq \eta_{m-1}\!+\!y \,\Big|\, |\G_{m}|\!>\!  \eta_{m-1} \Big\}, \forall y \!\triangleq\! \gr\!-\!\eta_{m-1} \!\geq\! 0 \\
    &= \frac{1- F_{|G_{m}|}(\eta_{m-1}+y)}{1- F_{|G_{m}|}(\eta_{m-1})}
    =e^{- \frac{\gr-\eta_{m-1}}{\b_{m}}}. \label{eq:msexp}
\end{align}
From \eqref{eq:msexp}, the \ac{PoT} gradients is distributed as exponential \ac{r.v.} with location $\eta_{m-1}$. Hence, the \ac{r.v.} ${|\bar{\G}_{m}|}-\eta_{m-1}$ is exponentially distributed. Consequently, the threshold can be calculated from \eqref{eq:LaplaceThreshold} after proper shifting.
\section{Proof of Lemma~\ref{lemma:convanalysis} for the Convergence Analysis}
\label{appnd:conanalyproof}
 Let $\bar{f}: \mathbb{R}^\d \rightarrow \mathbb{R}$ be a function that is required to be minimized. This function can be a convex or non-convex $L_{0}$-smooth function \cite{ef-sgd}. Also, the expected value of the stochastic gradient vector equals the oracle gradient, and the second moment of the stochastic gradient is upper bounded by some constant, i.e., $\sigma_{0}^2$. 
 
Let us start first with the assumption that the genie-aided distribution for the amplitude of the stochastic gradient, $F_{\G}(g)$, is known.\footnote{The genie-aided distribution assumption is relaxed later.} Hence, the threshold $\eta$ is calculated as in \cref{eq:threshold} for some compression ratio $\delta$.\footnote{Note, the genie-aided distribution of gradients' amplitude, $F_{\G}(g)$, is not similar to the oracle gradient, $\nabla \bar{f}(\x_{\{i\}})$.} After applying the threshold based compression operator $\mathbb{C}_{\eta}$,  the number of non-zero gradients in the sparsified vector  is a \ac{r.v.} distributed as binomial distribution with number of trials $\d$ and success probability $\delta$. Hence,
 the expected number of non-zero gradients matches that of $\topk$, i.e.,
    $\mathbb{E}\left\{ \normo{ \mathbb{C}_{\eta}\{\mathbf{\g}\}} \right\}=\delta \d= \k.$ 
    Therefore, the threshold based compression technique, designed to keep the $\k$ largest gradients in magnitude, has  the same $\k$-contraction property of $\topk$ on average
    \begin{align}\label{eq:contraction}
        \mathbb{E}\left\{{\left\lVert{\mathbb{C}_{{\eta}}\{\mathbf{\g}\}}- \mathbf{\g}\right\rVert}_{2}^{2}\right\} &=  \mathbb{E}\left\{{\left\lVert{\mathbb{T}_{{\k}}\{\mathbf{\g}\}}- \mathbf{\g}\right\rVert}_{2}^{2}\right\} \nonumber \\
        &\leq  (1-\delta)\,\mathbb{E}\left\{{\mathbf{\|\g\|}}_{2}^{2} \right\}.
    \end{align} 
From \eqref{eq:contraction} and Theorem II in \cite{ef-sgd} for compressed \ac{SGD} adopted with  the \ac{EC} technique,  we have   
\begin{align}
    \min_{i\in [I]} \mathbb{E}\{\norm{\nabla \bar{f}\left(\x_{\{i\}}\right)}_{2}^{2} \} &\leq \frac{4(\bar{f}(\x_{0})-\bar{f}^{*})+L_{0}\, \sigma_{0}^2}{2\sqrt{I+1}}\nonumber \\
    &+\frac{4 \, L_{0}^2\, \sigma_{0}^{2} \,(1-\delta)}{\delta^2\,(I+1)},
\end{align}
where 
$\bar{f}^{*}$ is a  minimum value for the function $\bar{f}$, and $I$ is the number of iterations over which the function is minimized. Therefore, the rate of convergence of the threshold based scheme with genie-aided distribution coincides with that of $\topk$ designed with the same compression ratio $\delta$ in remark 4 in \cite{ef-sgd}. In other words,  after  $I>\mathbb{O}\left(1/\delta^2\right)$ iteration, the thresholding scheme coincides with the \ac{SGD} convergence rate.


Now let us move to a more realistic case where we do not know the genie-aided distribution of the gradients. Indeed, there can be a discrepancy between the original and estimated distribution $\hat{F}_{\G}(g)$, which weakens the assumption of \ac{SPD}. In this case, the threshold is estimated as $\hat{\eta}=\hat{F}^{-1}_{\G}(1-\delta)$, leading to an error in the resulting average compression ratio, $\hat{\delta} \triangleq \kh/\d$, quantified as
\begin{align}
    \hat{\delta}-\delta  &\triangleq \frac{1}{\d} \left(\mathbb{E}\left\{\normo{\mathbb{C}_{\hat{\eta}}\{\mathbf{\g}\}}\right\} -\mathbb{E}\left\{\normo{\mathbb{C}_{\eta}\{\mathbf{\g}\}}\right\} \right)\\
    &=
   F_{\G}\left(\eta(\delta) \right)-  {F}_{\G}(\hat{\eta}(\delta)).  
\end{align}
In \cref{algo:algo1}, the number of thresholding stages are adapted such that 
\begin{align}
&\left\lvert  \hat{\delta}-\delta  \right\rvert \leq \epsilon \, \delta, && 0 \leq \epsilon < 1\,.
\end{align}
Hence, the actual compression ratio can be bounded as
\begin{equation}
    \delta\, \left(1- \epsilon\right) \leq \hat{\delta} \leq  \delta\, \left(1+ \epsilon \right).
\end{equation}

 For $\hat{\delta} \geq \delta$, the proposed scheme convergences with a rate faster than that of $\topk$, as the total number of iterations required to reach the \ac{SGD}'s rate is $I>\mathbb{O}\left(\frac{1}{\delta^2 \, (1+\epsilon)^2}\right)$, which is smaller than that required for $\topk$. The reason is that the proposed scheme, in this case, has a better contraction property on average. On the other hand, for $\hat{\delta} \leq \delta$, after number of iterations $I>\mathbb{O}\left(\frac{1}{\delta^2 \, (1-\epsilon)^2}\right)$, the proposed scheme coincides with the SGD convergence rate, requiring more iterations than $\topk$. In Lemma~\ref{lemma:convanalysis}, we report only the worst-case convergence rate, requiring more iterations.

\section{Experimental Specifications}
\label{apdx:clusters}

\paragraph{\textbf{Cluster 1 - Dedicated Environment}}
\begin{itemize}[noitemsep,topsep=0pt,leftmargin=10pt]
    \item 8 nodes per experiment
    \item GPUs per node: 1 $\times$ Tesla V100-SXM2 with 16GB of GPU memory
    \item GPU inter-connection: traversing PCIe and the SMP interconnect between NUMA nodes
    \item CPU: Intel(R) Xeon(R) Silver 4112 CPU @ 2.60GHz, 16 cores
    \item System memory: 512 GiB
    \item Ethernet: 25 Gbps SFI/SFP+ - Ethernet
    \item Network Topology: Star network topology
    \item OS: Ubuntu 18.04 + Linux Kernel v4.15
    \item Environment: Horovod's Docker container on DockerHub
    \item Software: PyTorch 1.1.0, Horovod 0.16, and OpenMPI v4.0
    
\end{itemize}
\paragraph{\textbf{Cluster 2 - Shared Environment}}
\begin{itemize}[noitemsep,topsep=0pt,leftmargin=10pt]
    \item 1 node per experiment
    \item GPUs per node: 8 $\times$ Tesla V100-SXM2 with 32GB of GPU memory
    \item GPU inter-connection: traversing PCIe and the SMP interconnect between NUMA nodes
    \item CPU:  Intel(R) Xeon(R) Gold 6248 CPU @ 2.50GHz, 16 cores
    \item System memory: 512 GiB
    \item Ethernet: 100 Gbps - InfiniBand
    \item Network Topology: Fat Tree topology
    \item OS: CentOS 7.7 + Linux Kernel v3.10
    \item Environment: Miniconda 4.3
    \item Software: PyTorch 1.3, Horovod 0.18, and OpenMPI 4.0
\end{itemize}

\subsection{Further Experimental and Evaluation Details}
\label{apdx:expdetails}
Here, we present more details on our experimental settings, benchmarks, hyper-parameters, etc. First, we describe the three benchmarks used in this work which covers three commonly used ML tasks in practice. The benchmarks also cover both \ac{RNN} and \ac{CNN} architectures. 

\textbf{Image Classification: } We studied ResNet20 anf VGG16 on Cifar10, and ResNet-50 and VGG19 on ImageNet. Cifar10 consists of 50,000 training images and 10,000 validation images in 10 classes~\cite{cifar10}, while ImageNet contains over 1 million training images and 50,000 validation images in 1000 classes~\cite{imagenet}. We train CIFAR10 models with vanilia SGD (without Momentum) and ImageNet models with Nesterov-momentum SGD following the training schedule in~\cite{Gross2016}. The warm-up period is set to 5 epochs for all schemes.

\textbf{Language Modeling: } The Penn Treebank corpus (PTB) dataset consists of 923,000 training, 73,000 validation and 82,000 test words~\cite{ptb}. 
We adopt the 2-layer LSTM language model architecture with 1500 hidden units per layer~\cite{lstm}. We use Nesterov-momentum SGD with gradient clipping, while learning rate decays when no improvement has been made in validation loss. The warm-up period is 5 epoch out of the 30 epochs.

\textbf{Speech Recognition: } The AN4 dataset contains 948 training and 130 test utterances~\cite{an4}. We use DeepSpeech architecture without n-gram language model~\cite{deepspeech}, which is a multi-layer RNN following a stack of convolution layers. We train a 5-layer LSTM of 800 hidden units per layer with Nesterov momentum SGD and gradient clipping, while learning rate anneals every epoch. The warm-up period is 5 epochs out of 150 epochs.

\paragraph{\textbf{Further Evaluation Details:}}
For training speed-up, we evaluate the speed-up based on the time-to-accuracy of the method that is when it can achieve (or exceed) a certain training accuracy or test perplexity. The target test accuracy is 75\% for ResNet20 and 80\% for VGG16 on CIFAR-10. The target test perplexity is 105 for PTB benchmark. The target \ac{CER} is 55 for AN4. We compare no compression, existing and proposed sparsification methods with ratios of ($k=0.1,0.01,0.001$) using 8 nodes.




\section{Extra experiments, and results}
\label{apdx:moreexp}

\begin{figure*}[!h]
\captionsetup[subfigure]{justification=centering}
\centering
\begin{subfigure}[ht]{0.8\linewidth}
  \includegraphics[width=1\linewidth]{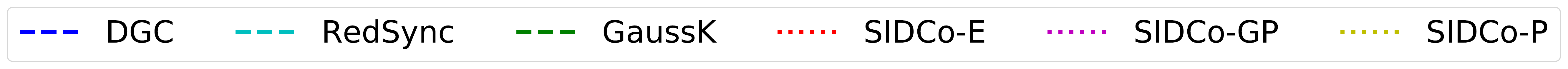}
 \end{subfigure}
  \\
  \begin{subfigure}[ht]{0.3\linewidth}
    \includegraphics[ width=\textwidth]{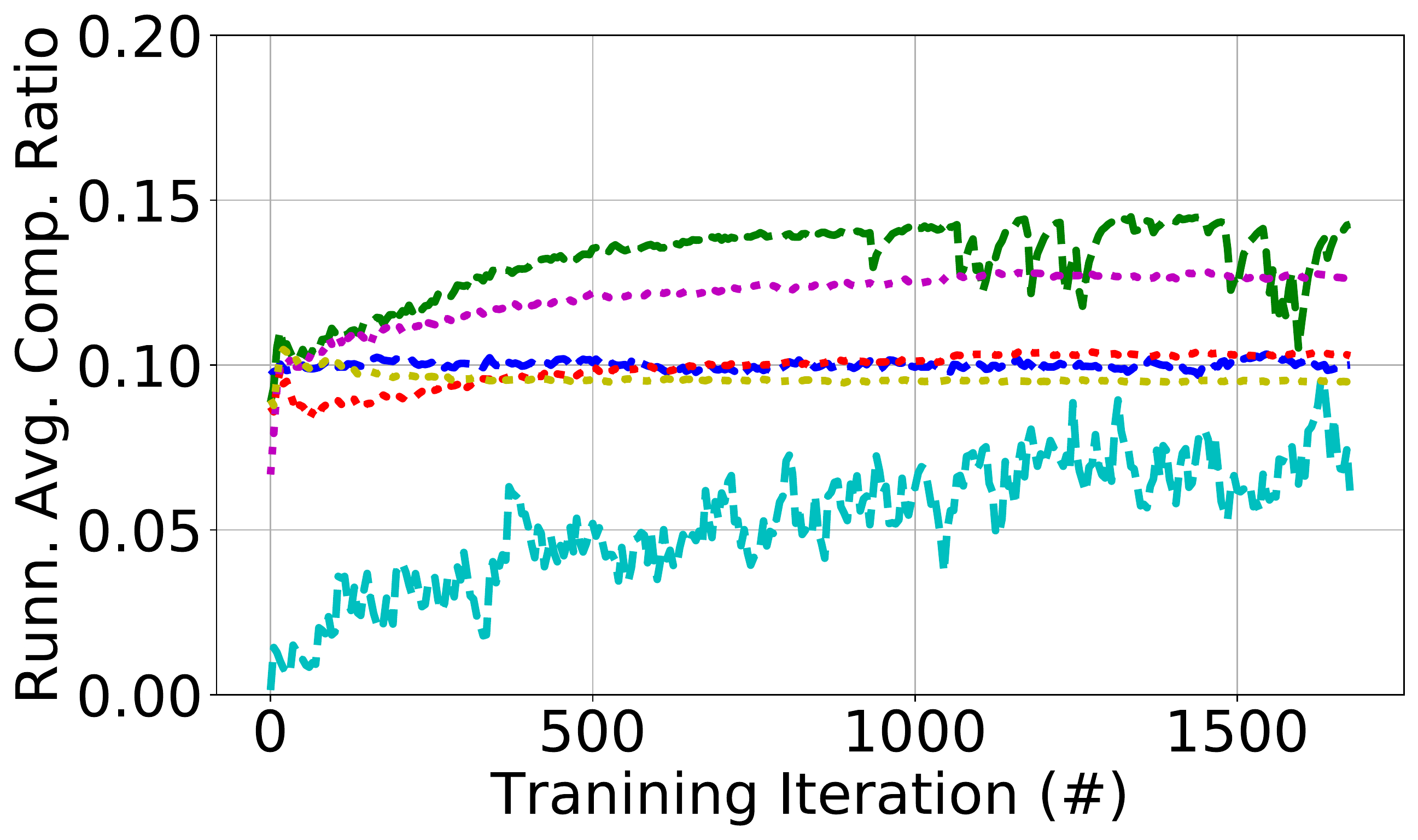}
	\caption{ResNet20 on CIFAR10 - Ratio 0.1.}
	\label{fig:resnet20-avgcomp0.1-8}
     \end{subfigure}
     \hfill
	\begin{subfigure}[ht]{0.3\linewidth}
  \includegraphics[ width=\textwidth]{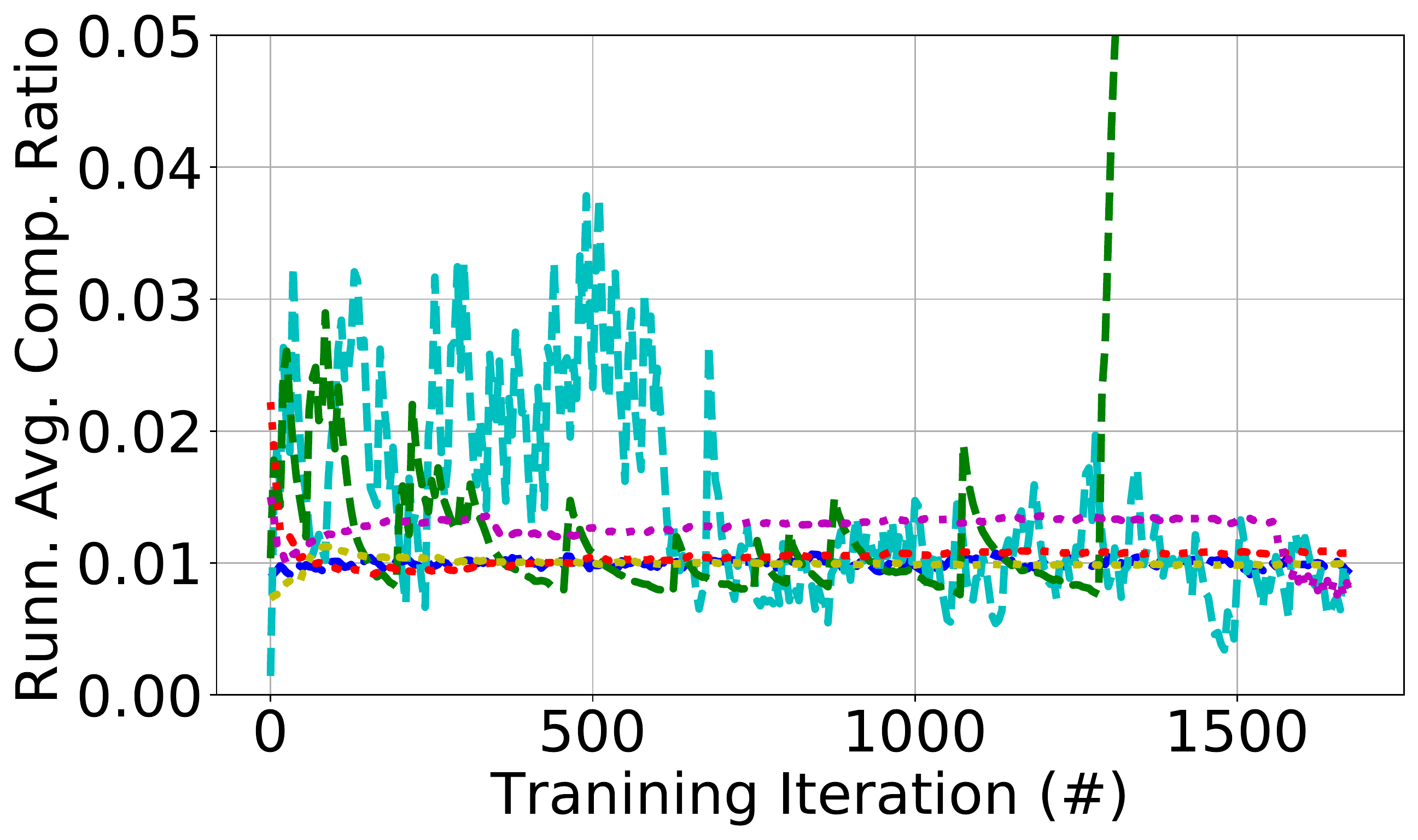}
	\caption{ResNet20 on CIFAR10 - Ratio 0.01.}
	\label{fig:resnet20-avgcomp0.01-8}
    \end{subfigure}
     \hfill
    \begin{subfigure}[ht]{0.3\linewidth}
   \includegraphics[ width=\textwidth]{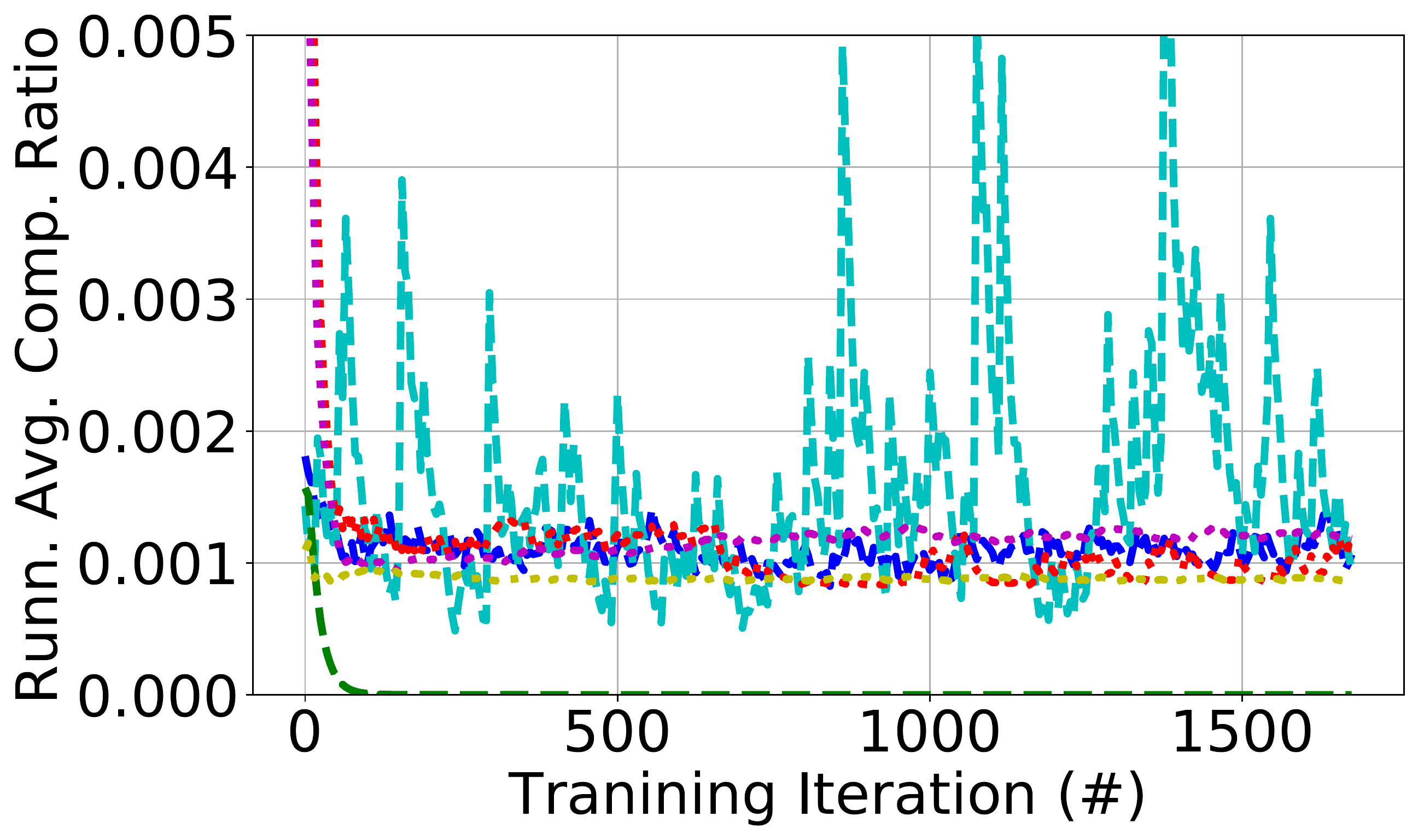}
	\caption{ResNet20 on CIFAR10 - Ratio 0.001.}
	\label{fig:resnet20-avgcomp0.001-8}
     \end{subfigure}
     \\
     \begin{subfigure}[ht]{0.3\linewidth}
    \includegraphics[ width=\textwidth]{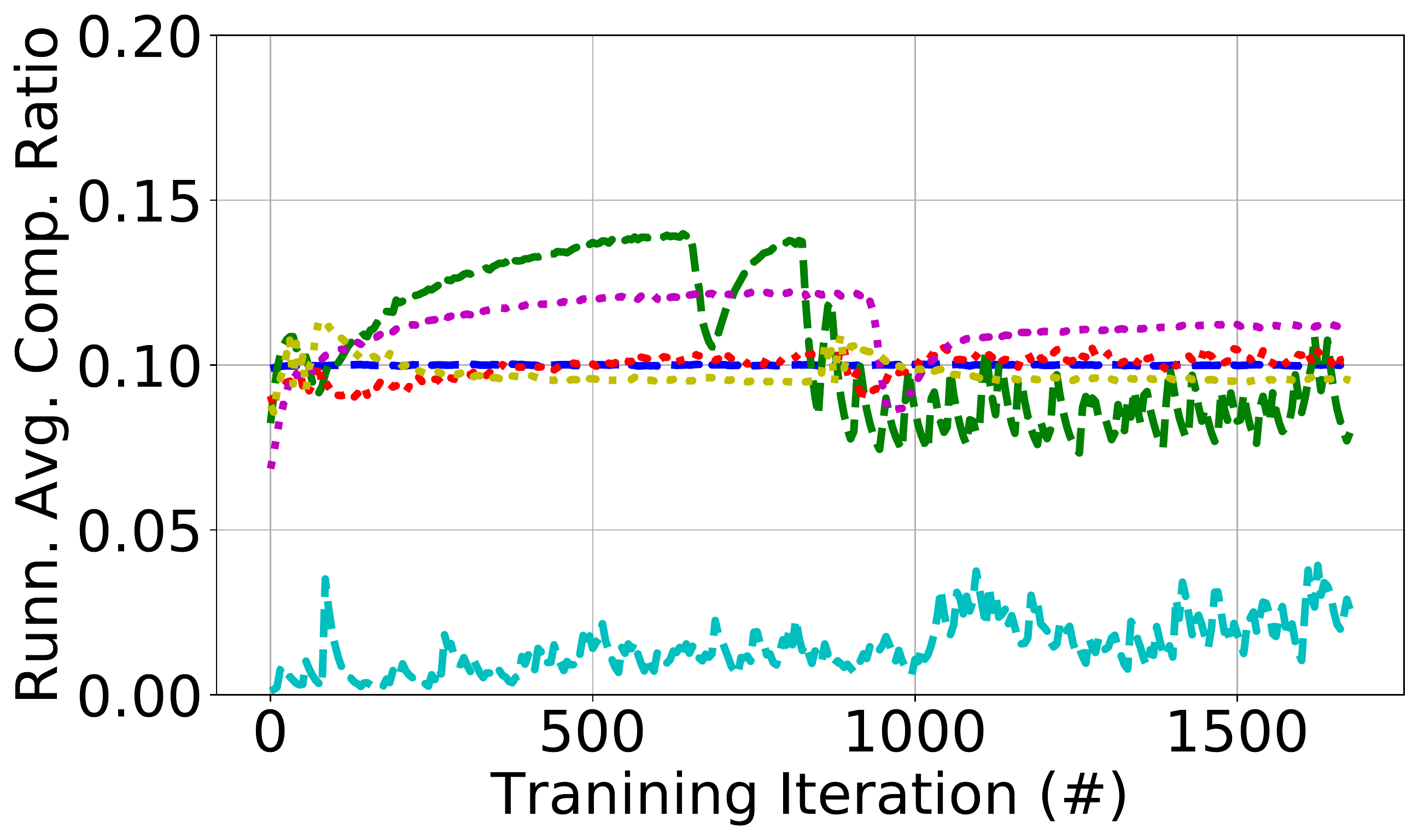}
	\caption{VGG16 on CIFAR10 - Ratio 0.1.}
	\label{fig:vgg16-avgcomp0.1-8}
     \end{subfigure}
     \hfill
	\begin{subfigure}[ht]{0.3\textwidth}
  \includegraphics[ width=\textwidth]{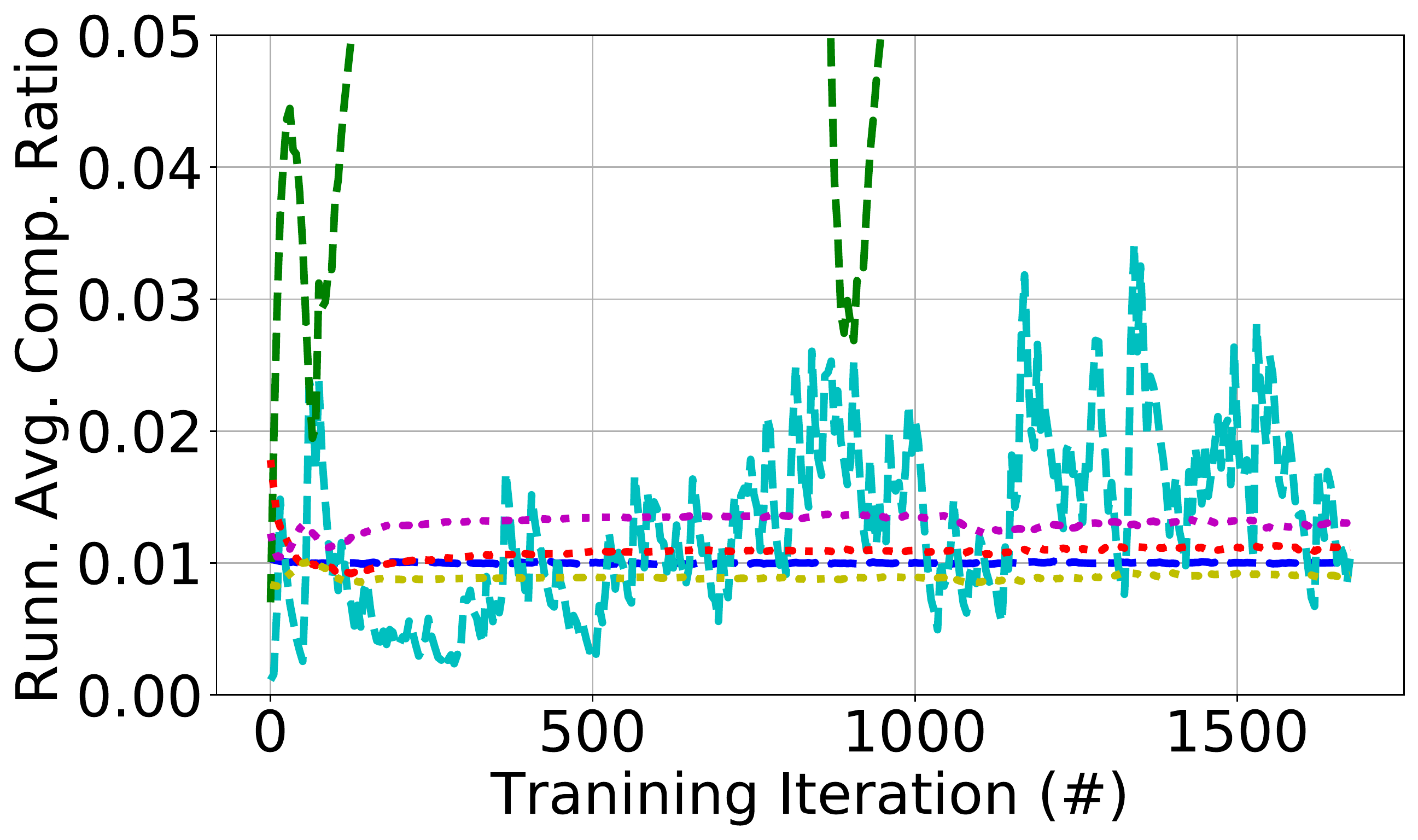}
	\caption{VGG16 on CIFAR10 - Ratio 0.01.}
	\label{fig:vgg16-avgcomp0.01-8}
    \end{subfigure}
     \hfill
    \begin{subfigure}[ht]{0.3\linewidth}
   \includegraphics[ width=\textwidth]{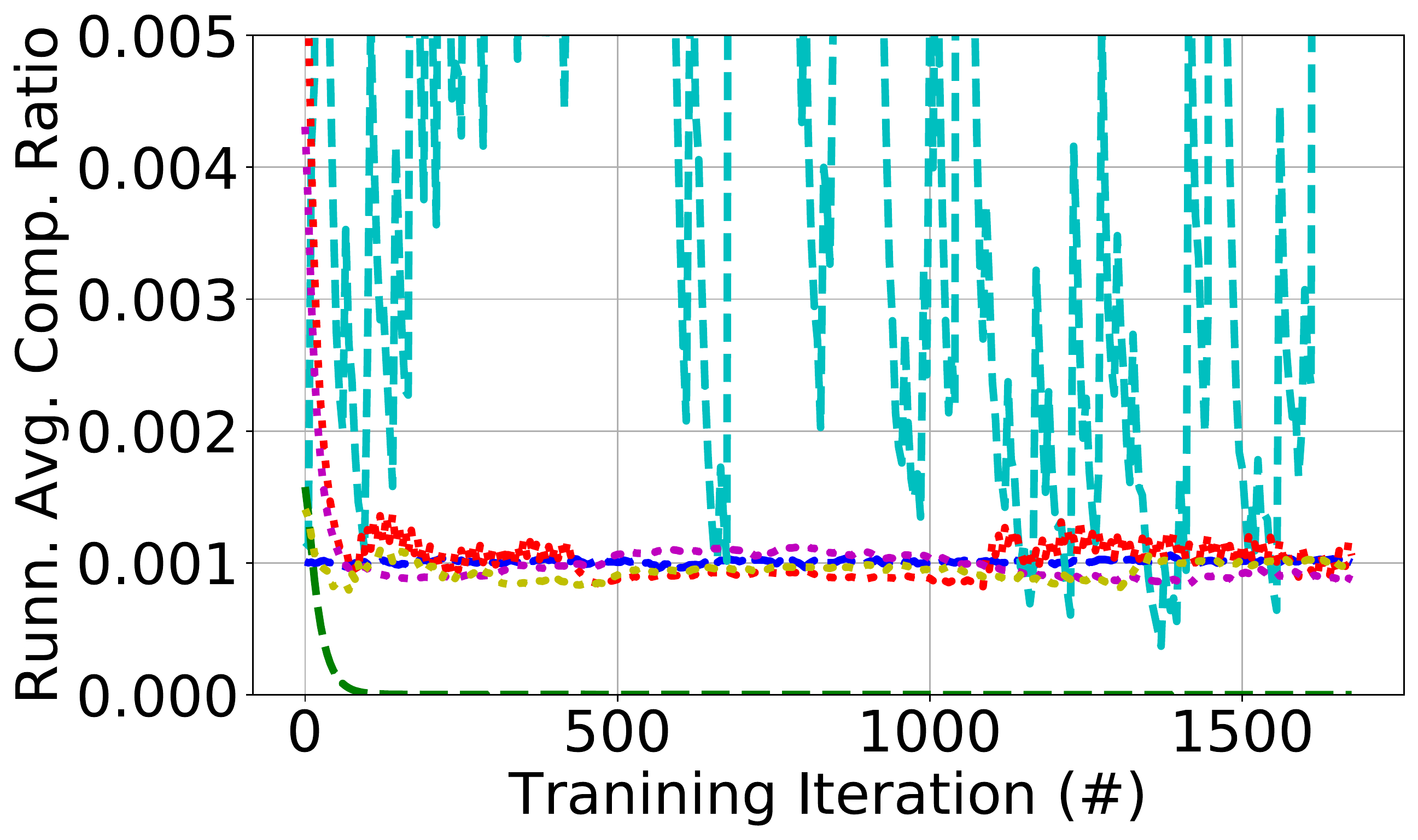}
	\caption{VGG16 on CIFAR10 - Ratio 0.001.}
	\label{fig:vgg16-avgcomp0.001-8}
     \end{subfigure}
      \\
      \begin{subfigure}[ht]{0.3\linewidth}
    \includegraphics[ width=\textwidth]{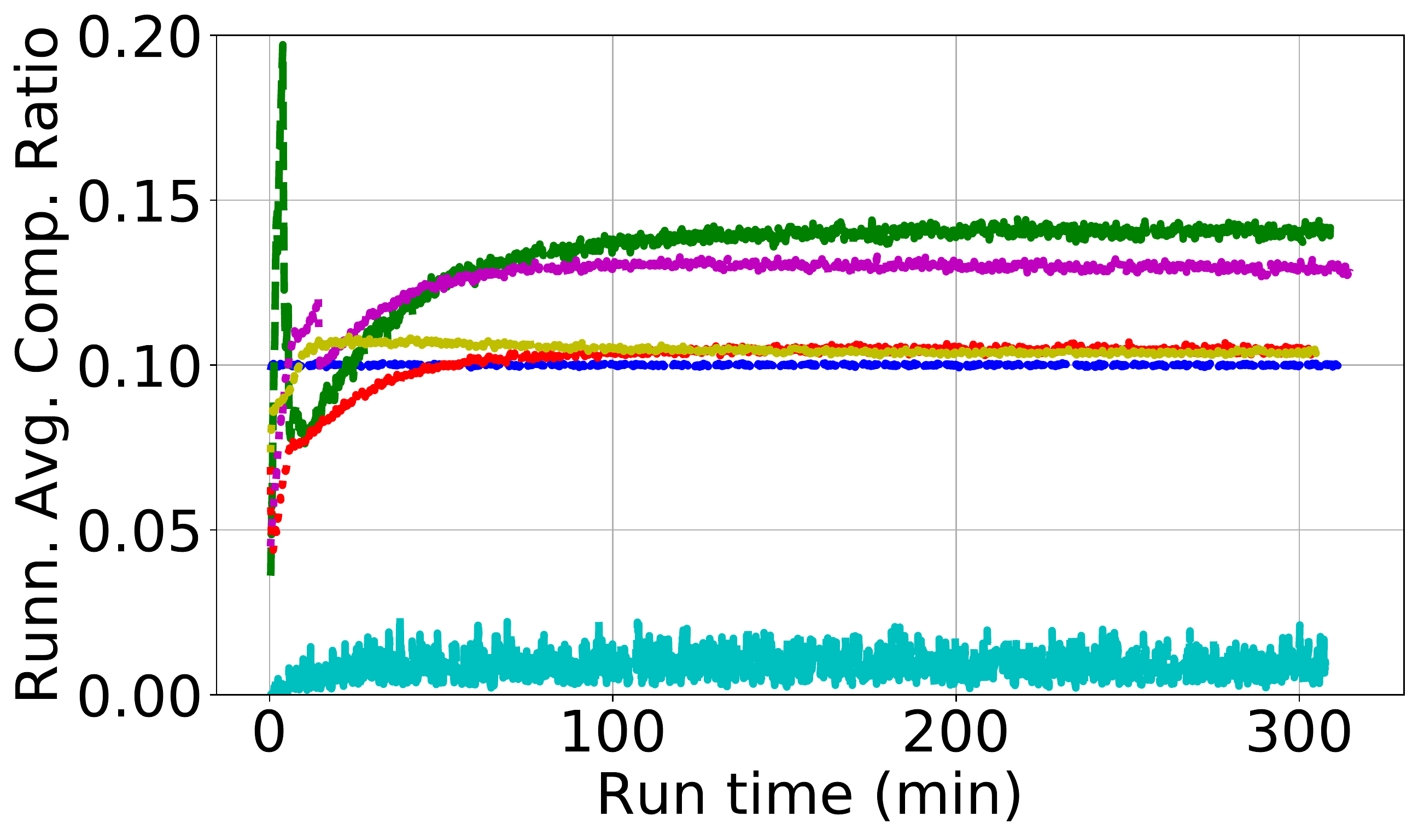}
	\caption{ResNet50 on ImageNet - Ratio 0.1.}
	\label{fig:resnet50-avgcomp0.1-8}
     \end{subfigure}
     \hfill
	\begin{subfigure}[ht]{0.3\linewidth}
  \includegraphics[ width=\textwidth]{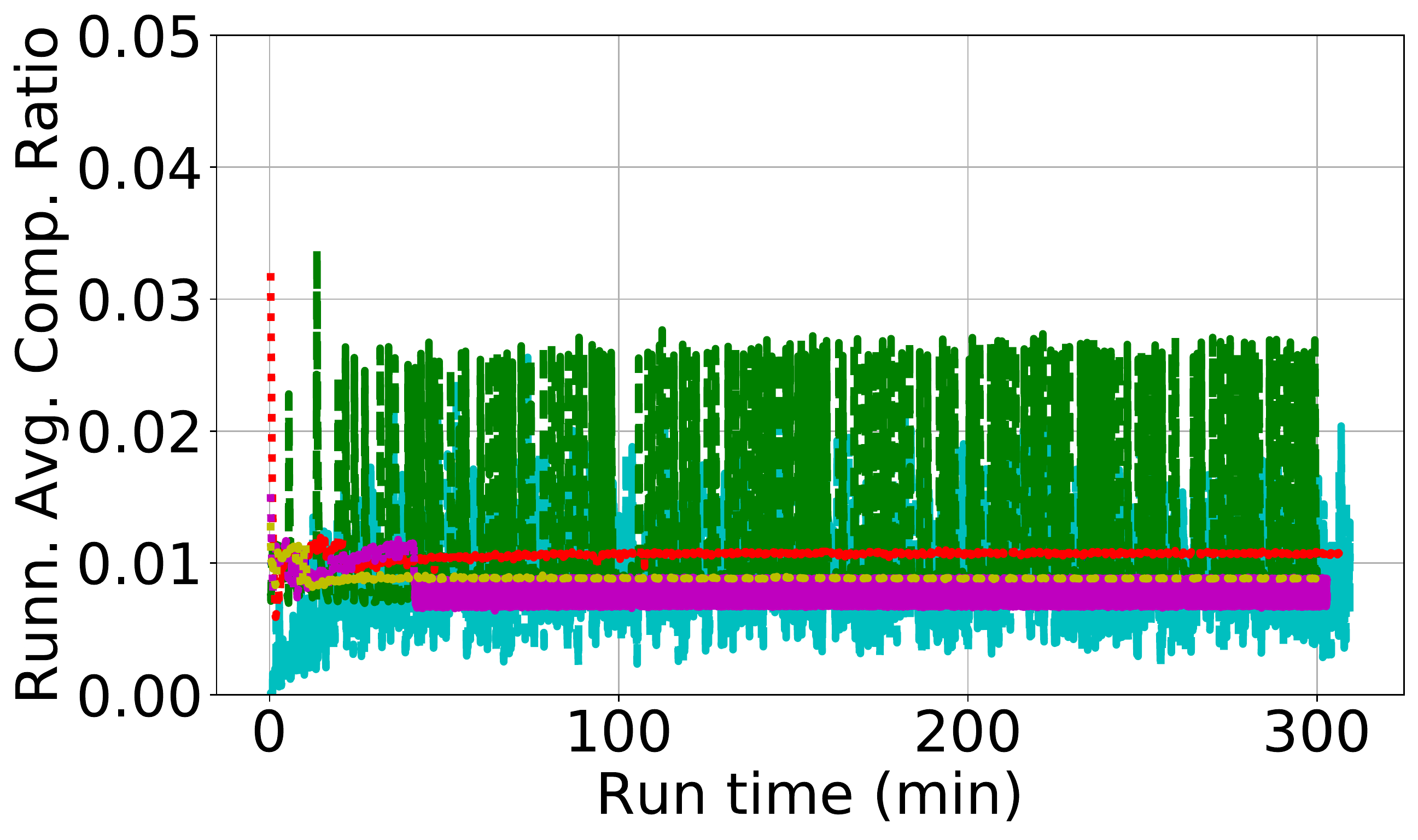}
	\caption{ResNet50 on ImageNet - Ratio 0.01.}
	\label{fig:resnet50-avgcomp0.01-8}
    \end{subfigure}
     \hfill
    \begin{subfigure}[ht]{0.3\linewidth}
   \includegraphics[ width=\textwidth]{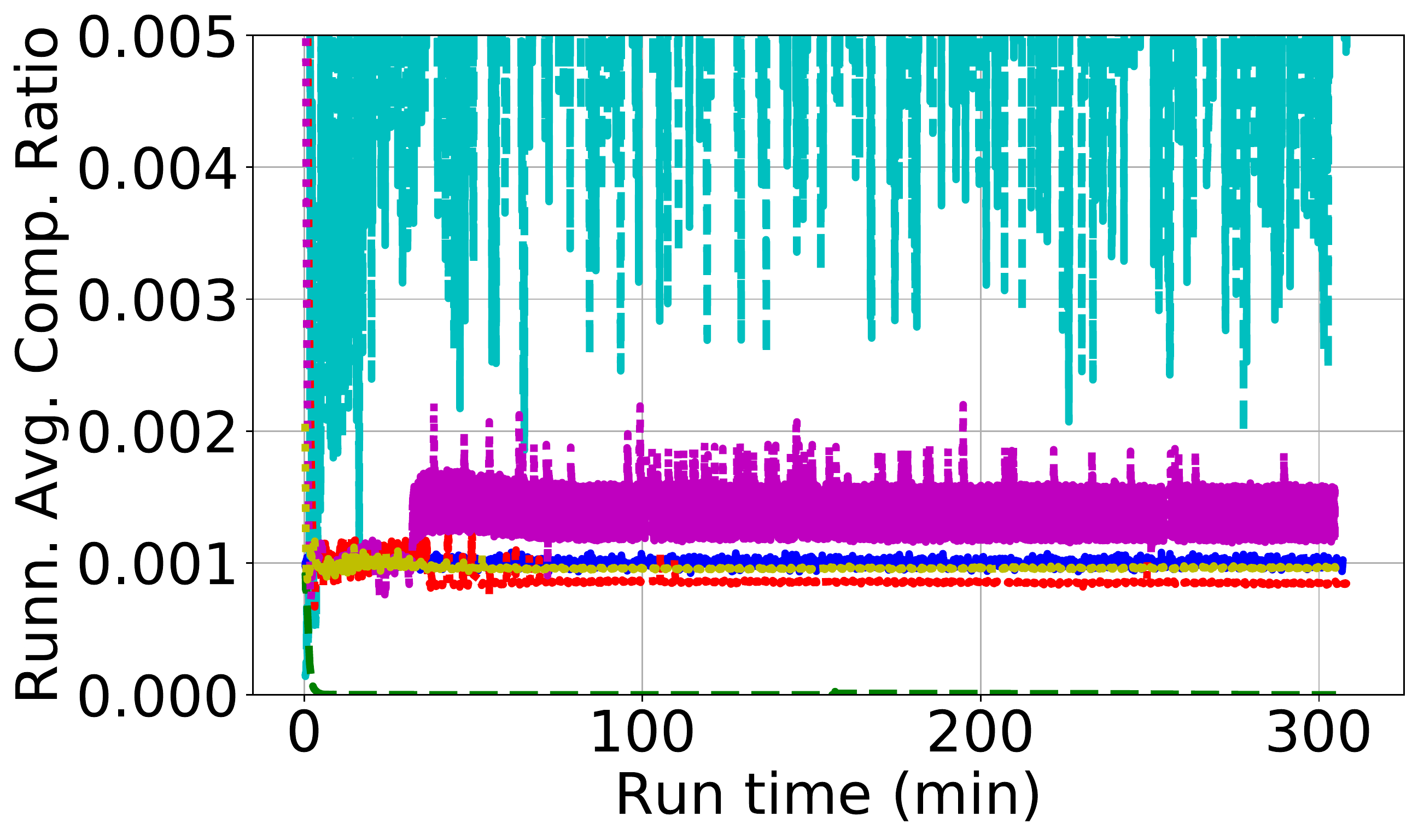}
	\caption{ResNet50 on ImageNet - Ratio 0.001.}
	\label{fig:resnet50-avgcomp0.001-8}
     \end{subfigure}
     \\
     \begin{subfigure}[ht]{0.3\linewidth}
    \includegraphics[ width=\textwidth]{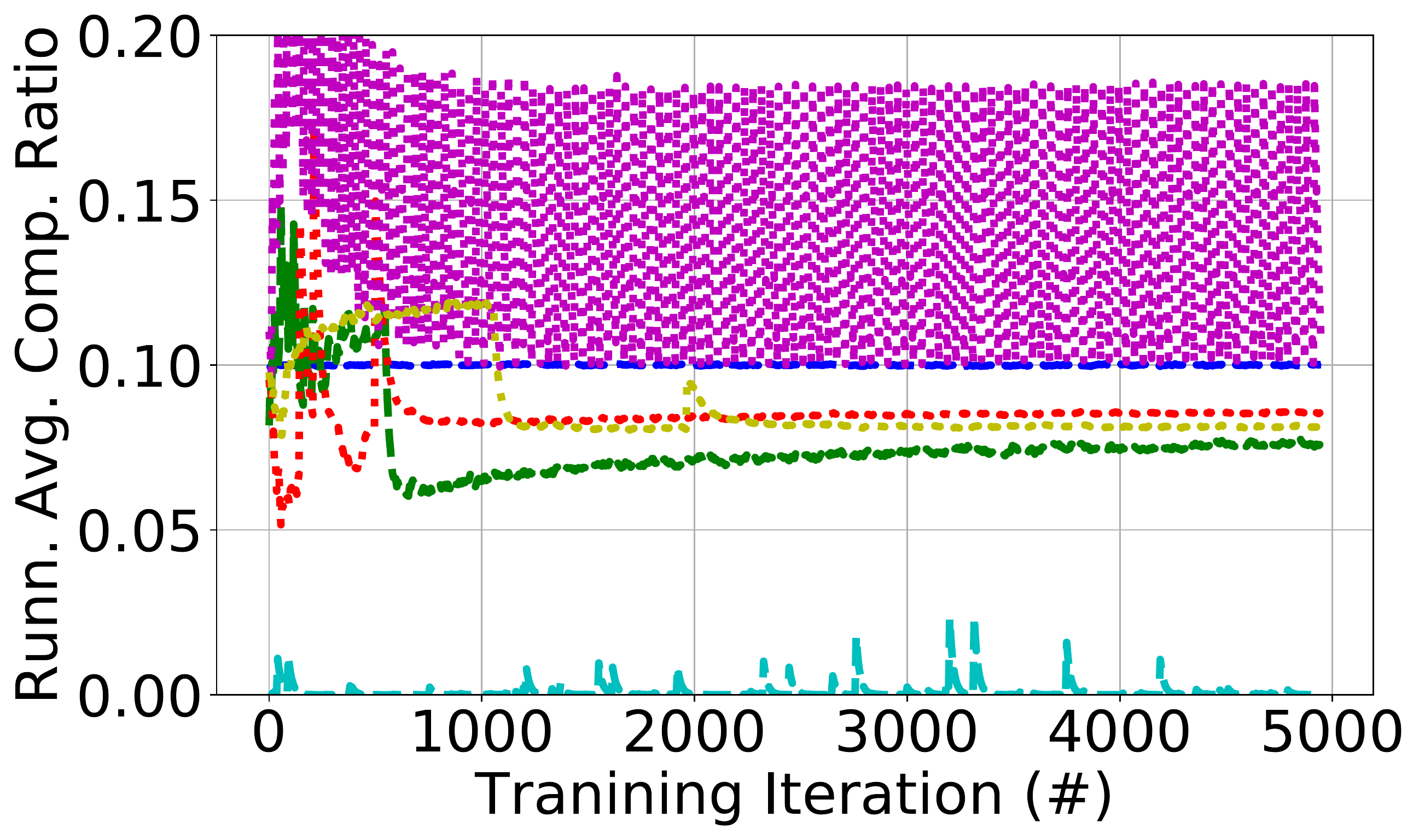}
	\caption{PTB on LSTM - Ratio 0.1.}
	\label{fig:ptb-avgcomp0.1-8}
     \end{subfigure}
     \hfill
	\begin{subfigure}[ht]{0.3\linewidth}
  \includegraphics[ width=\textwidth]{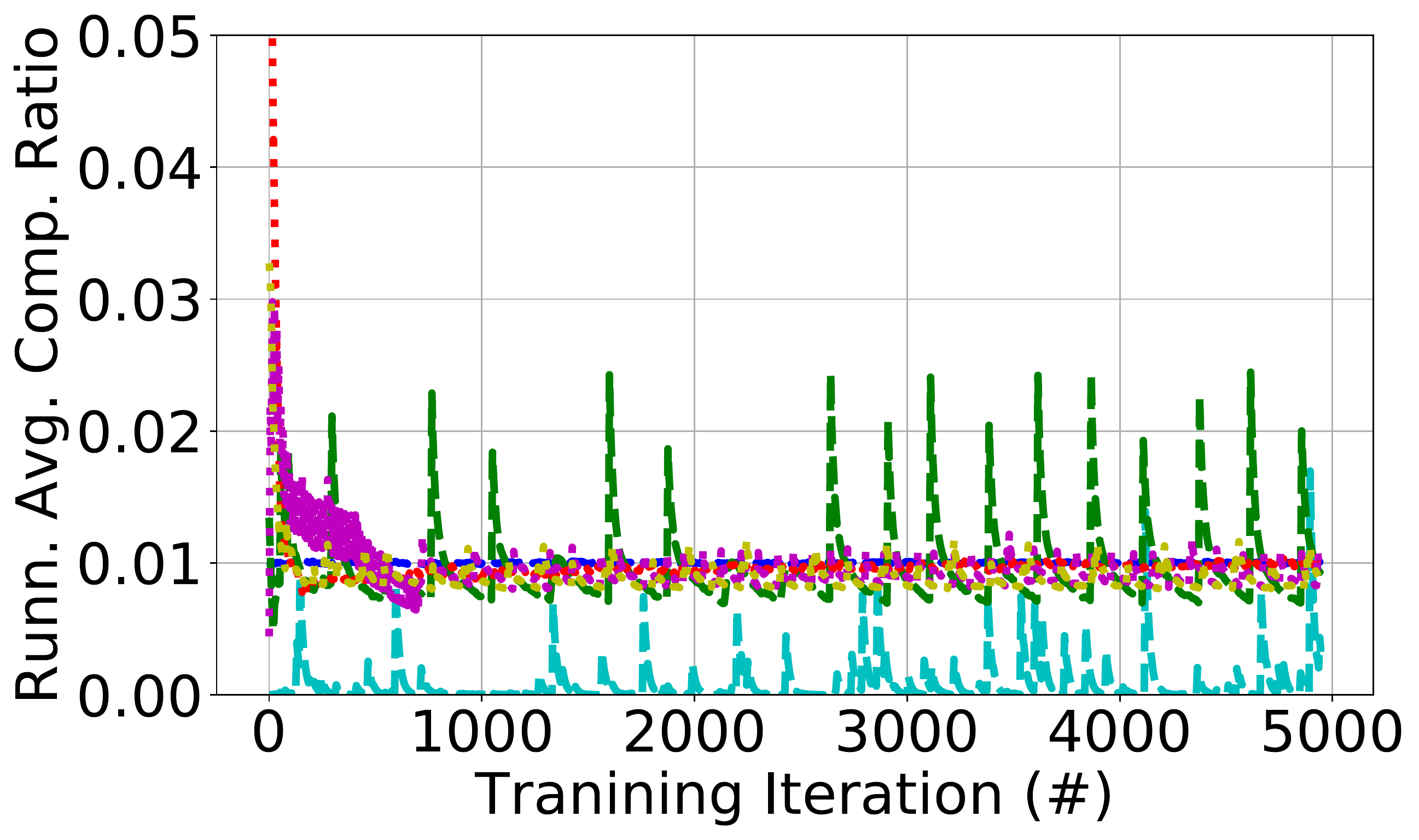}
	\caption{LSTM on PTB - Ratio 0.01.}
	\label{fig:ptb-avgcomp0.01-8}
    \end{subfigure}
     \hfill
    \begin{subfigure}[ht]{0.3\linewidth}
   \includegraphics[ width=\textwidth]{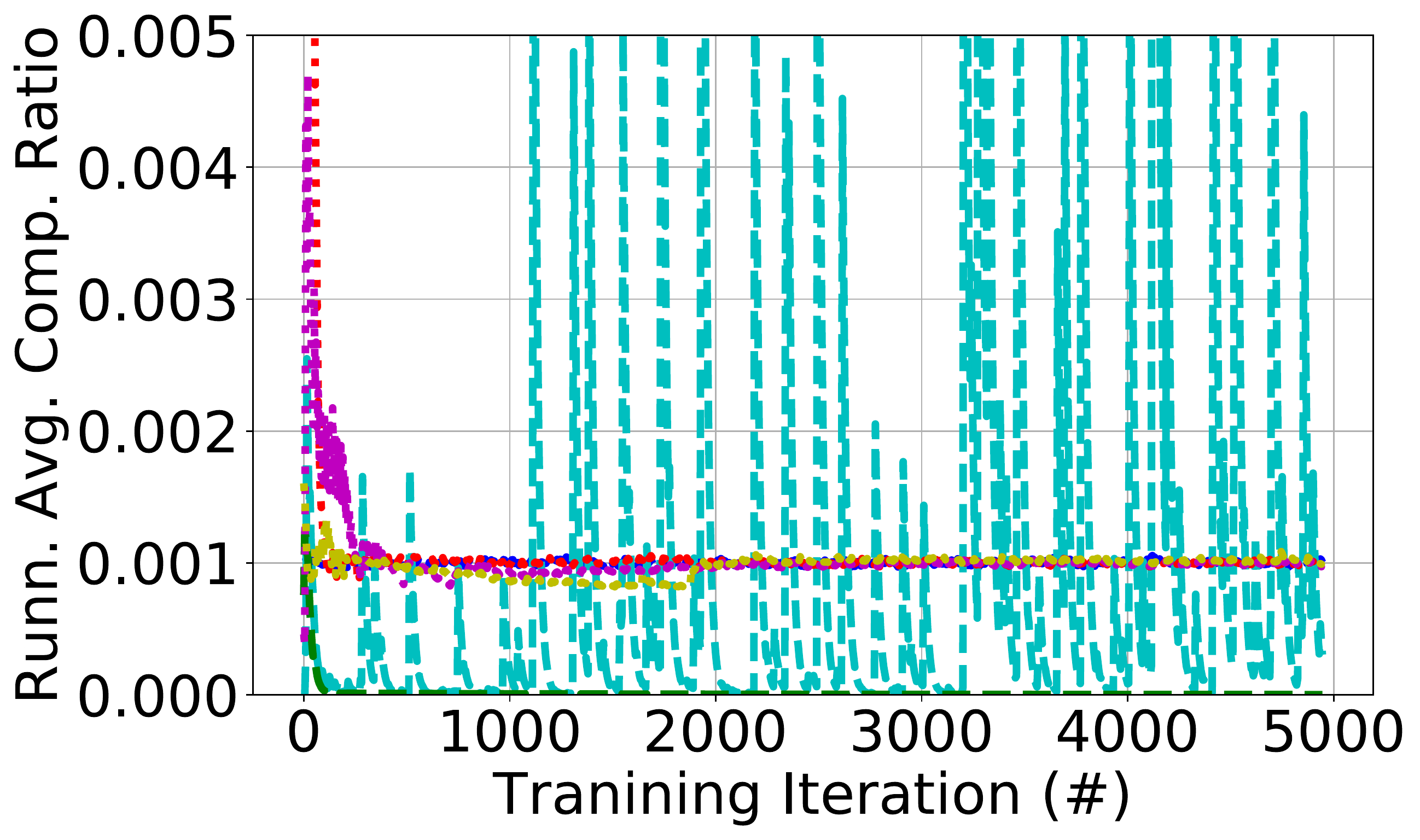}
	\caption{LSTM on PTB - Ratio 0.001.}
	\label{fig:ptb-avgcomp0.001-8}
     \end{subfigure}
     \\
     \begin{subfigure}[ht]{0.3\linewidth}
    \includegraphics[ width=\textwidth]{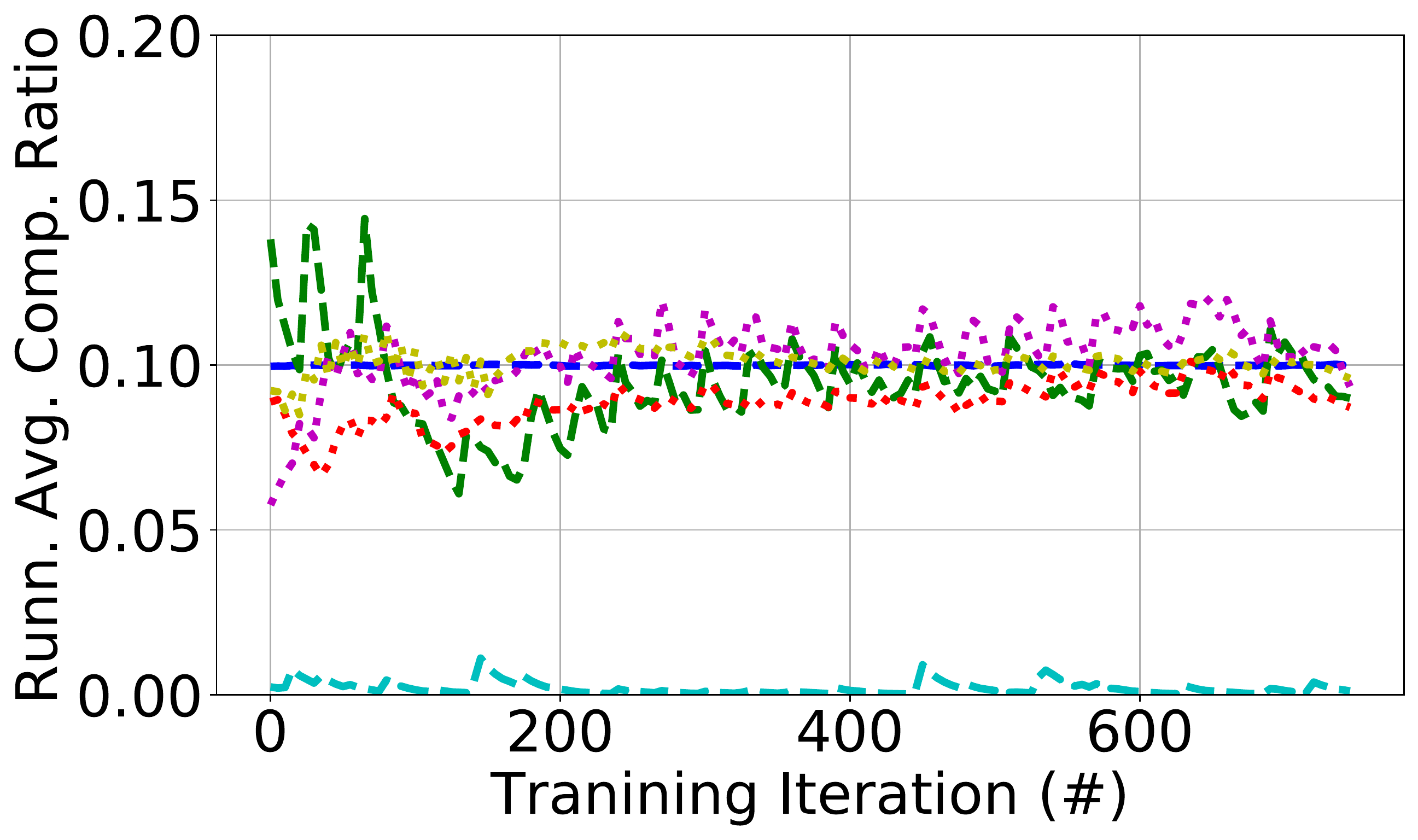}
	\caption{LSTM on AN4 - Ratio 0.1.}
	\label{fig:an4-avgcomp0.1-8}
     \end{subfigure}
     \hfill
	\begin{subfigure}[ht]{0.3\linewidth}
  \includegraphics[ width=\textwidth]{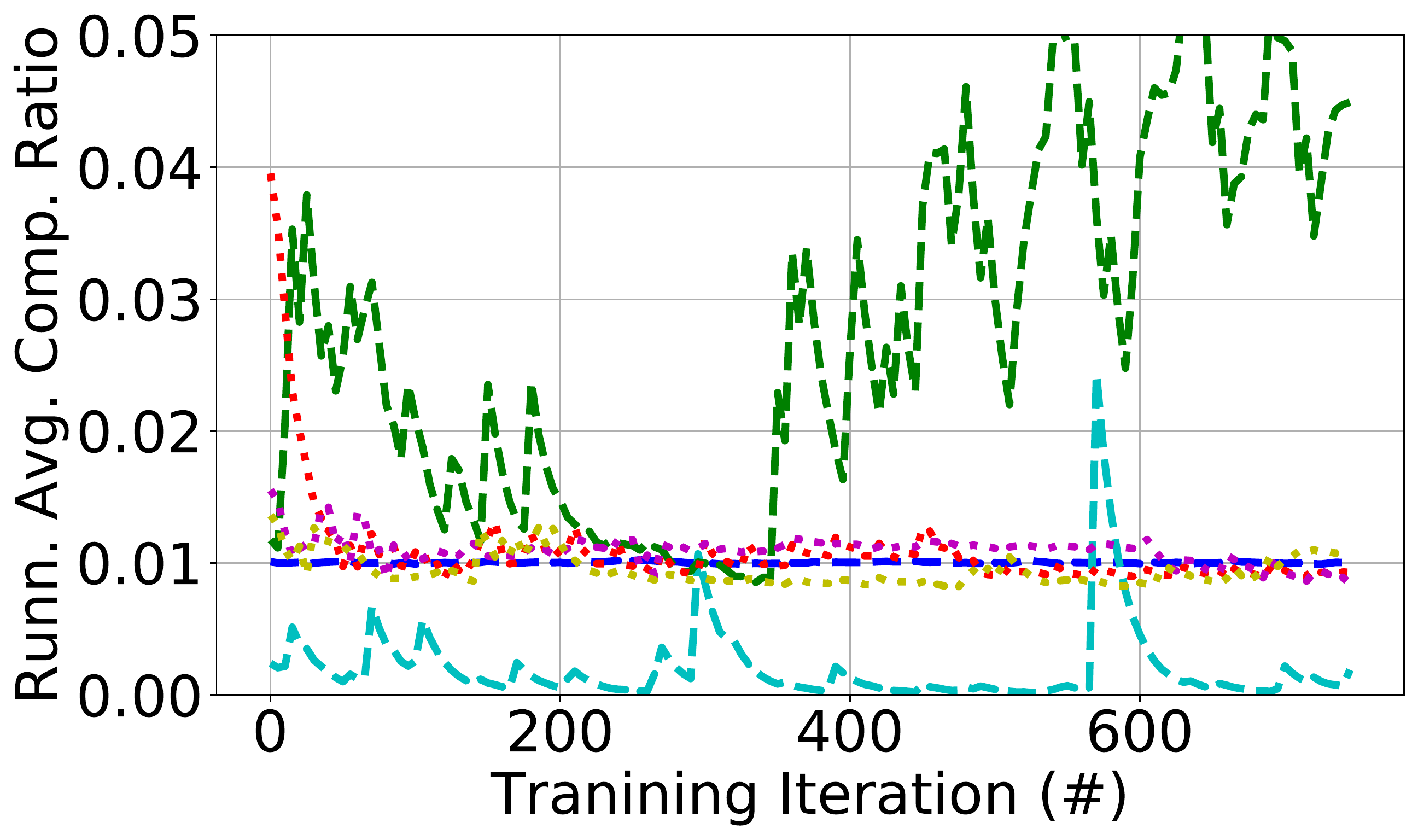}
	\caption{LSTM on AN4 - Ratio 0.01.}
	\label{fig:an4-avgcomp0.01-8}
    \end{subfigure}
     \hfill
    \begin{subfigure}[ht]{0.3\linewidth}
   \includegraphics[ width=\textwidth]{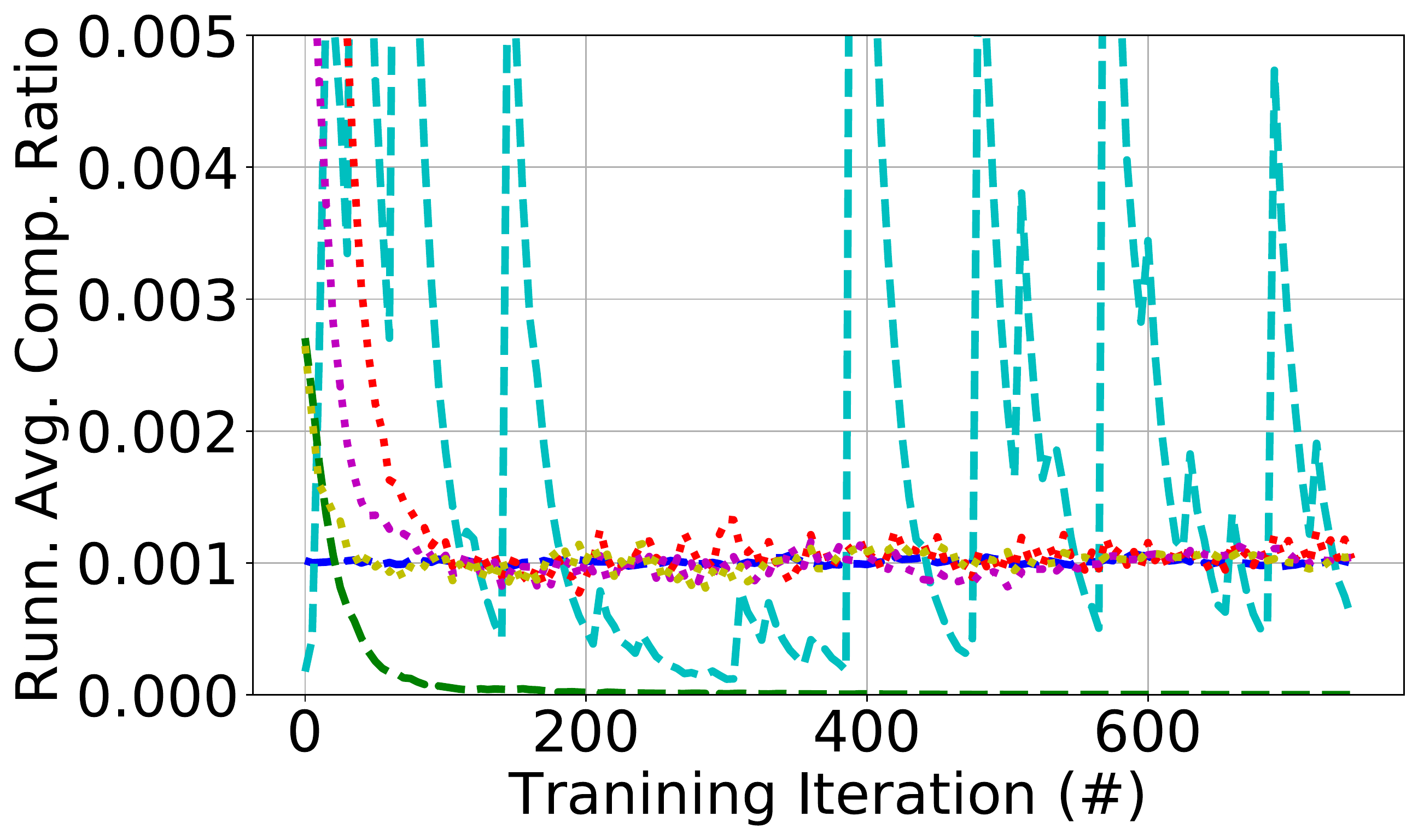}
	\caption{LSTM on AN4 - Ratio 0.001.}
	\label{fig:an4-avgcomp0.001-8}
     \end{subfigure}
\caption{Smoothed compression ratio for all benchmarks at different ratios.}
\label{fig:compratio}
\end{figure*}

In the following, we present more results including more detailed metrics and experimental scenarios. In the following, we refer to 1-stage double Gamma followed by $M-1$ stage Generalized Pareto and multi-stage Generalized Pareto, and multi-stage double exponential are refereed to \scheme\!-GP, \scheme\!-P, and \scheme\!-E respectively.
\subsection{Further Metrics and Experimental Scenarios}
\paragraph{\textbf{Quality of Estimation Methods:}} \cref{fig:compratio} shows the smoothed (or running average) of the compression ratio for all benchmarks and the three ratios ($0.1$, $0.01$, and $0.001$) used in the experiments. The results signify the quality of the obtained threshold throughout the training for DGC, RedSync, GaussianKSGD and the three \scheme\! methods. The results, in general, reinforce our previous observation that \scheme\! schemes perform quite well and achieve nearly the same threshold quality as of the sampling methods of DGC. \scheme\! schemes are also significantly better than the other estimation methods (i.e., RedSync and GaussianKSGD). Moreover, other estimation methods (e.g., RedSync and GaussianKSGD) generally results in high oscillations and their over/under-estimation can be up to $\approx\!\pm60\times$ the target. We also observe, in few cases, that the multi-stage \scheme\!-GP (i.e., Gamma-Pareto) results in slight over-estimation which is at most 2 times the target ratio. This could be attributed to the inaccuracies from the first-stage threshold estimation that uses closed-form moment-matching approximation used for fitting the double-Gamma distribution.

To support the observation presented in \cref{fig:an4-speedup-8-all} in which \scheme\!, unlike all other methods, achieved the target Character Error Rate (CER) because it over-estimated the threshold at early stage of training. In particular \cref{fig:an4-avgcomp0.001-8} shows that \scheme\!-E algorithm, at the beginning of training, uses the single-stage fitting for the target ratio which leads to threshold over-estimation for few iterations until it settles at the final number of stages. So, thanks to the multi-stage adaptation technique, it can reach to the appropriate number of stages which allows it stay at the target compression ratio. The initial extra-volume at the beginning of training, at this extreme sparsification ratio for this benchmark, leads to significant improvement in accuracy gains and explains the results presented in \cref{fig:an4-speedup-8}.

\paragraph{\textbf{Training Loss: }}
we present the training loss vs run time plots for all benchmarks using all ratios. \cref{fig:accuracy} shows the convergence of all schemes over time and the results in general confirm the speed-up results presented in \cref{sec:experiments} and \cref{apdx:expalldist}. The results highlight the gains in terms of time and accuracy from employing compression over the no-compression. They also signify that most compressors (except for GaussianKSGD and RedSync) achieve same accuracy as $\topk$ but at lower overhead than $\topk$.

\begin{figure*}[!h]
\captionsetup[subfigure]{justification=centering}
\centering
\begin{subfigure}[ht]{0.8\linewidth}
  \includegraphics[width=1\linewidth]{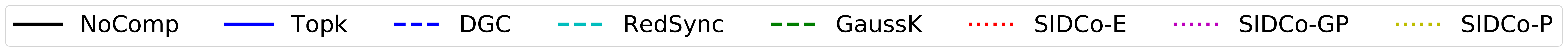}
 \end{subfigure}
  \\
  \begin{subfigure}[ht]{0.3\linewidth}
    \includegraphics[ width=\textwidth]{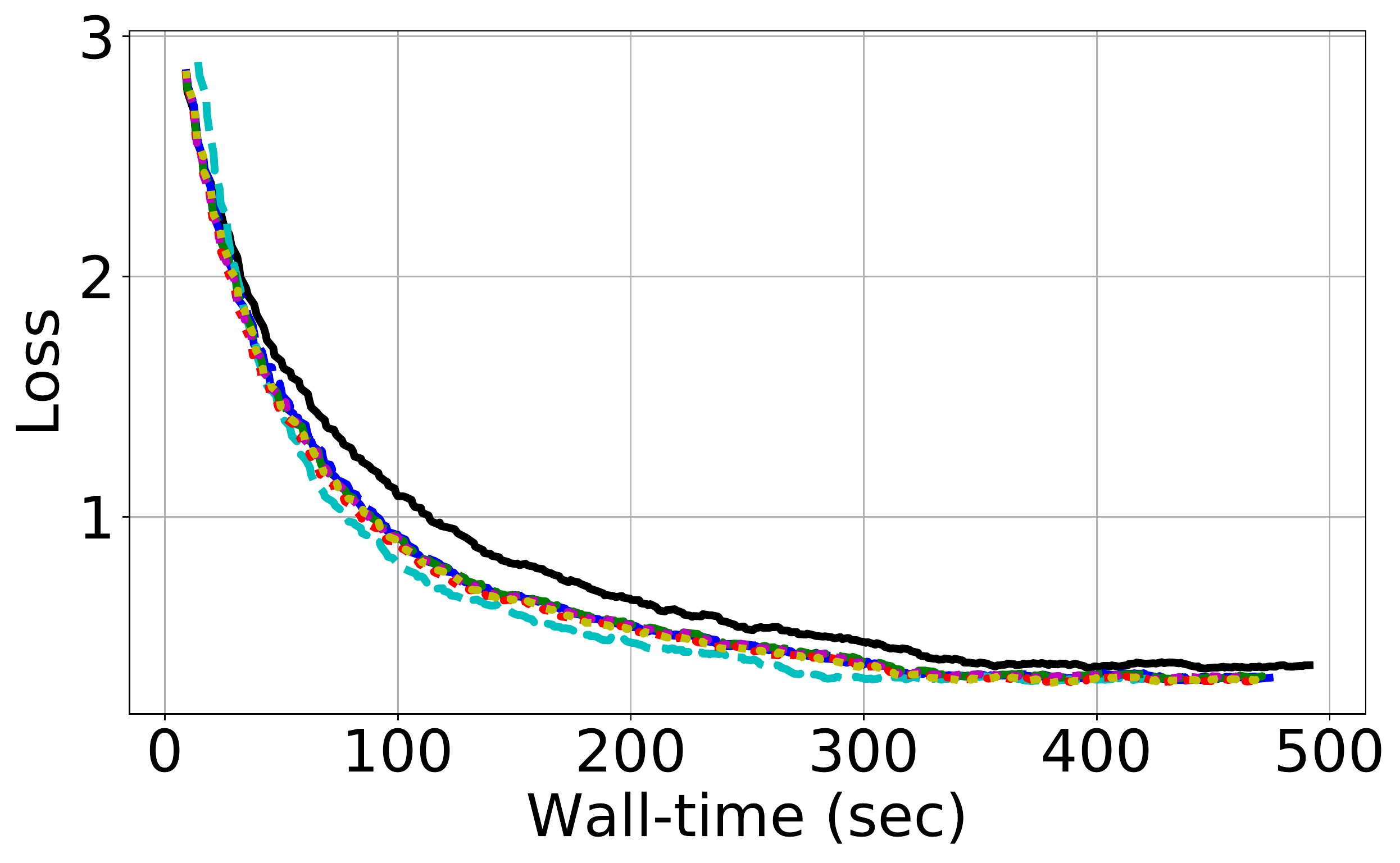}
	\caption{ResNet20 on CIFAR10 - Ratio 0.1.}
	\label{fig:resnet20-acc0.1-8}
     \end{subfigure}
     \hfill
	\begin{subfigure}[ht]{0.3\linewidth}
  \includegraphics[ width=\textwidth]{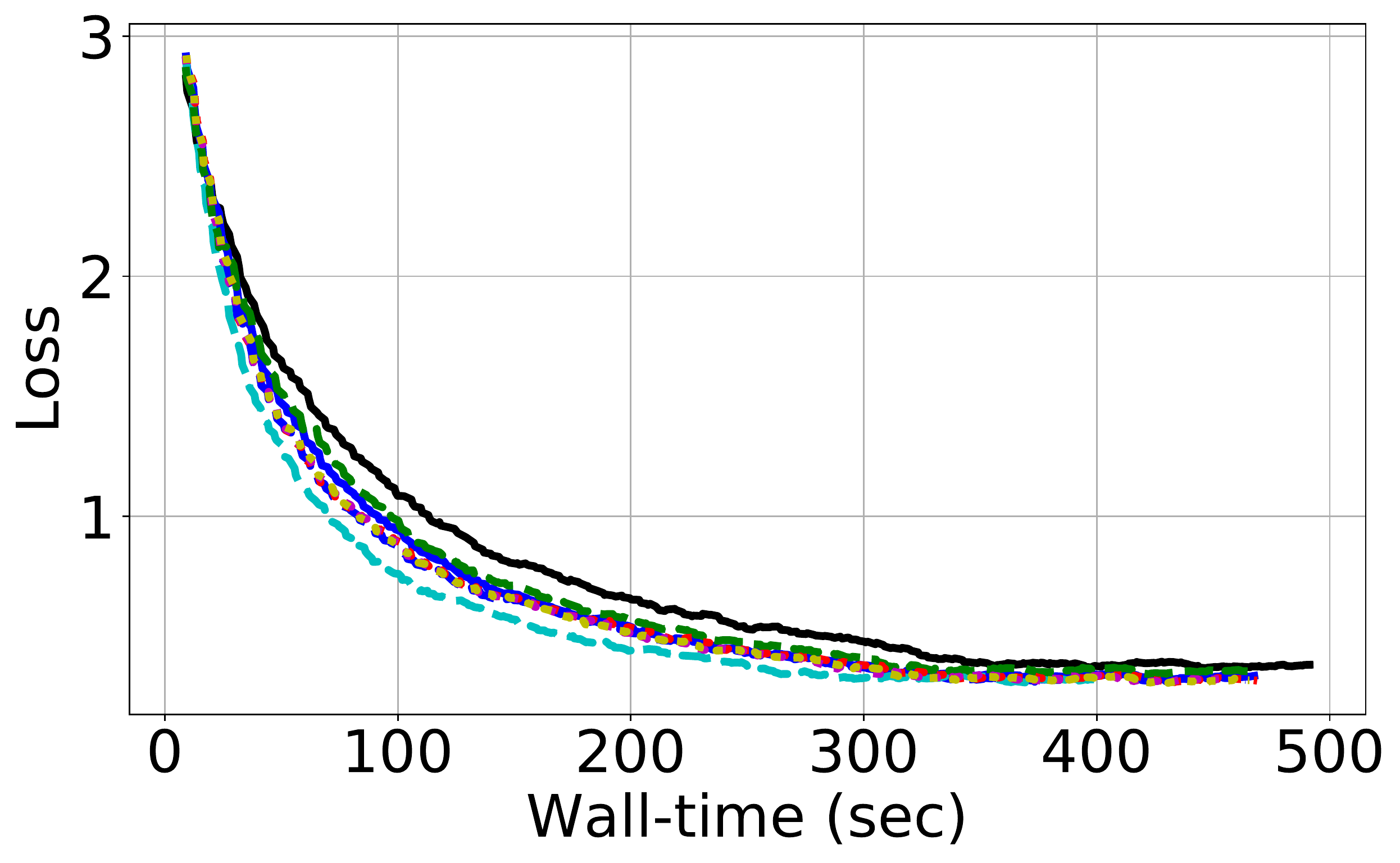}
	\caption{ResNet20 on CIFAR10 - Ratio 0.01.}
	\label{fig:resnet20-acc0.01-8}
    \end{subfigure}
     \hfill
    \begin{subfigure}[ht]{0.3\linewidth}
   \includegraphics[ width=\textwidth]{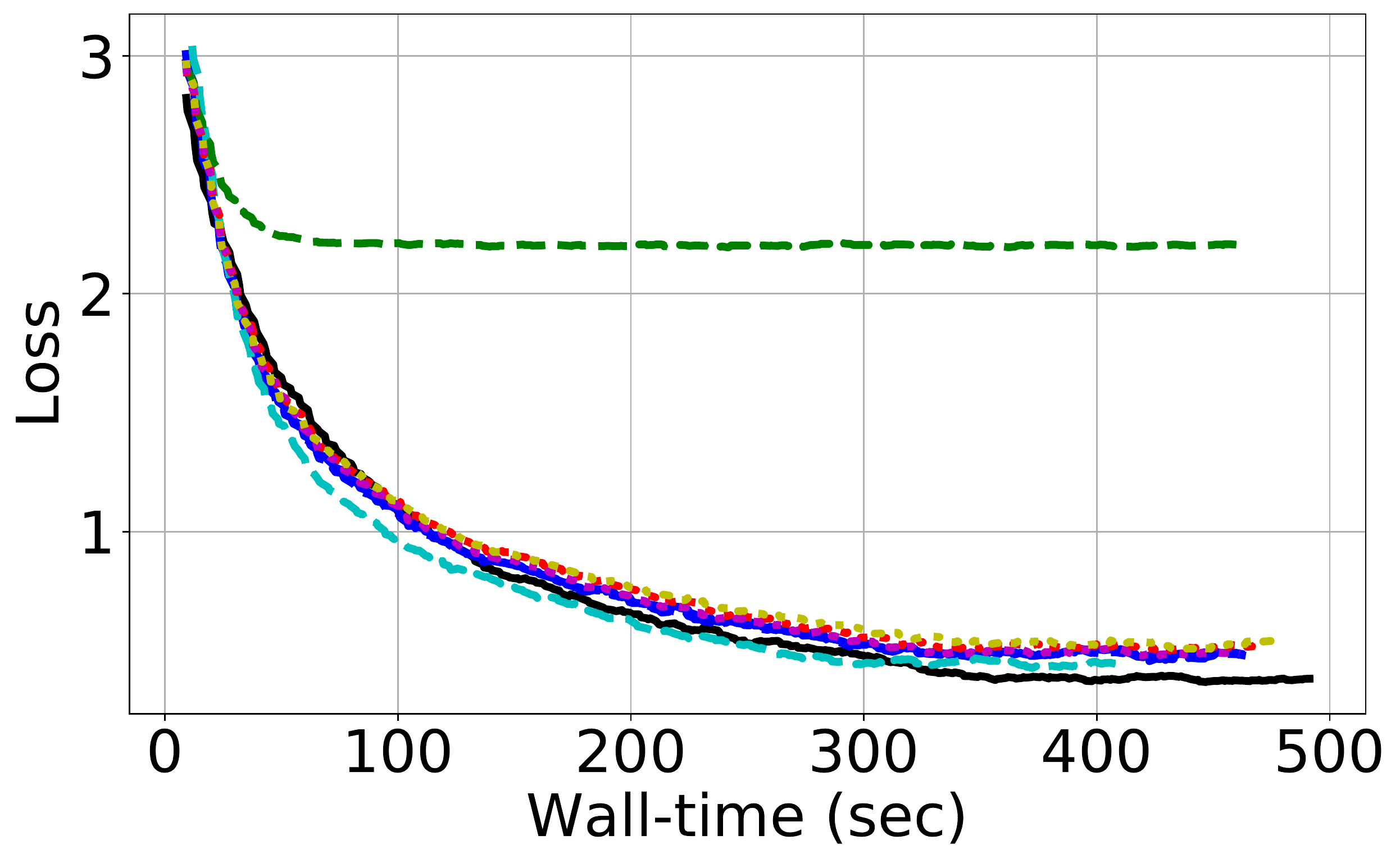}
	\caption{ResNet20 on CIFAR10 - Ratio 0.001.}
	\label{fig:resnet20-acc0.001-8}
     \end{subfigure}
     \\
     \begin{subfigure}[ht]{0.3\linewidth}
    \includegraphics[ width=\textwidth]{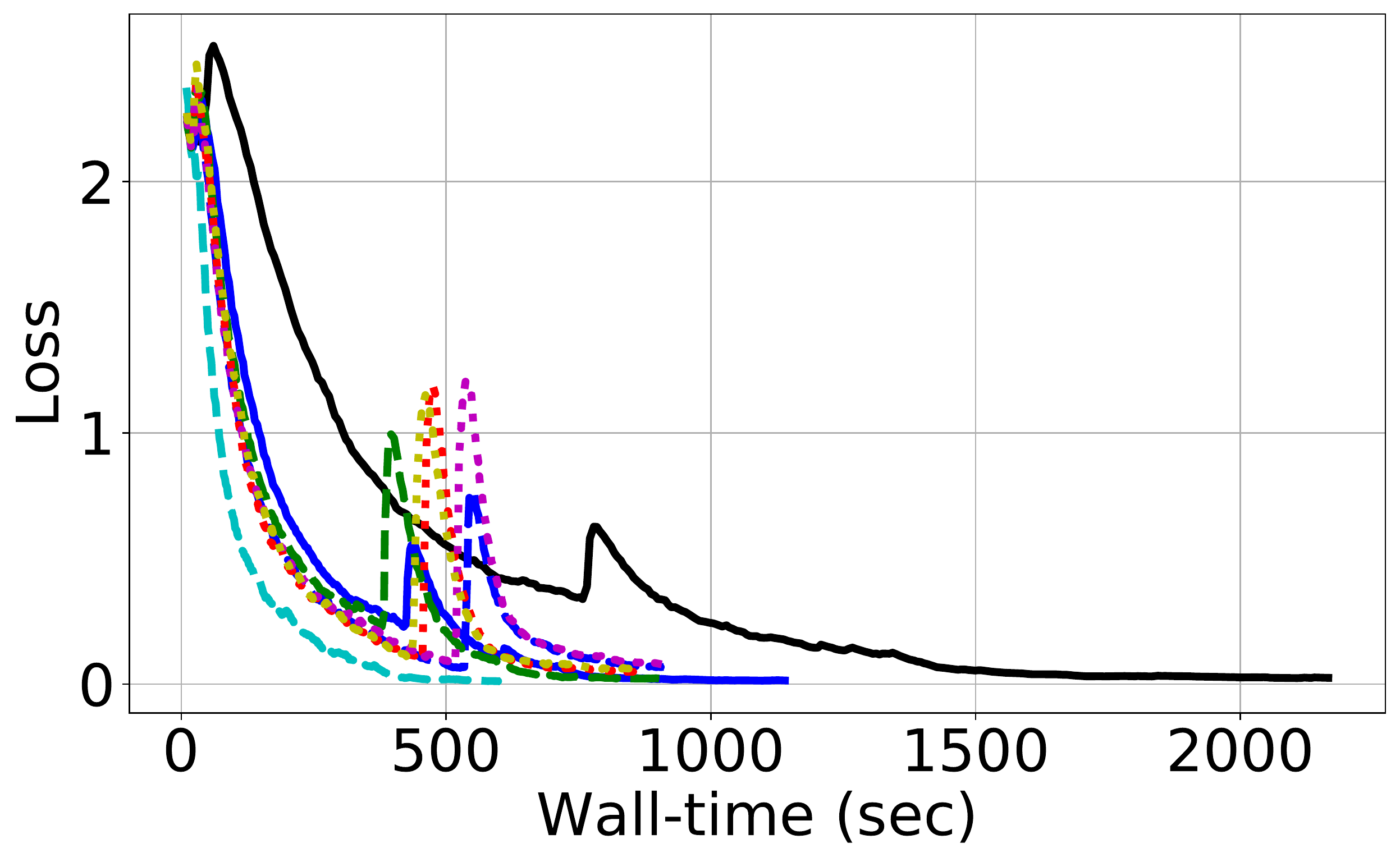}
	\caption{VGG16 on CIFAR10 - Ratio 0.1.}
	\label{fig:vgg16-acc0.1-8}
     \end{subfigure}
     \hfill
	\begin{subfigure}[ht]{0.3\textwidth}
  \includegraphics[ width=\textwidth]{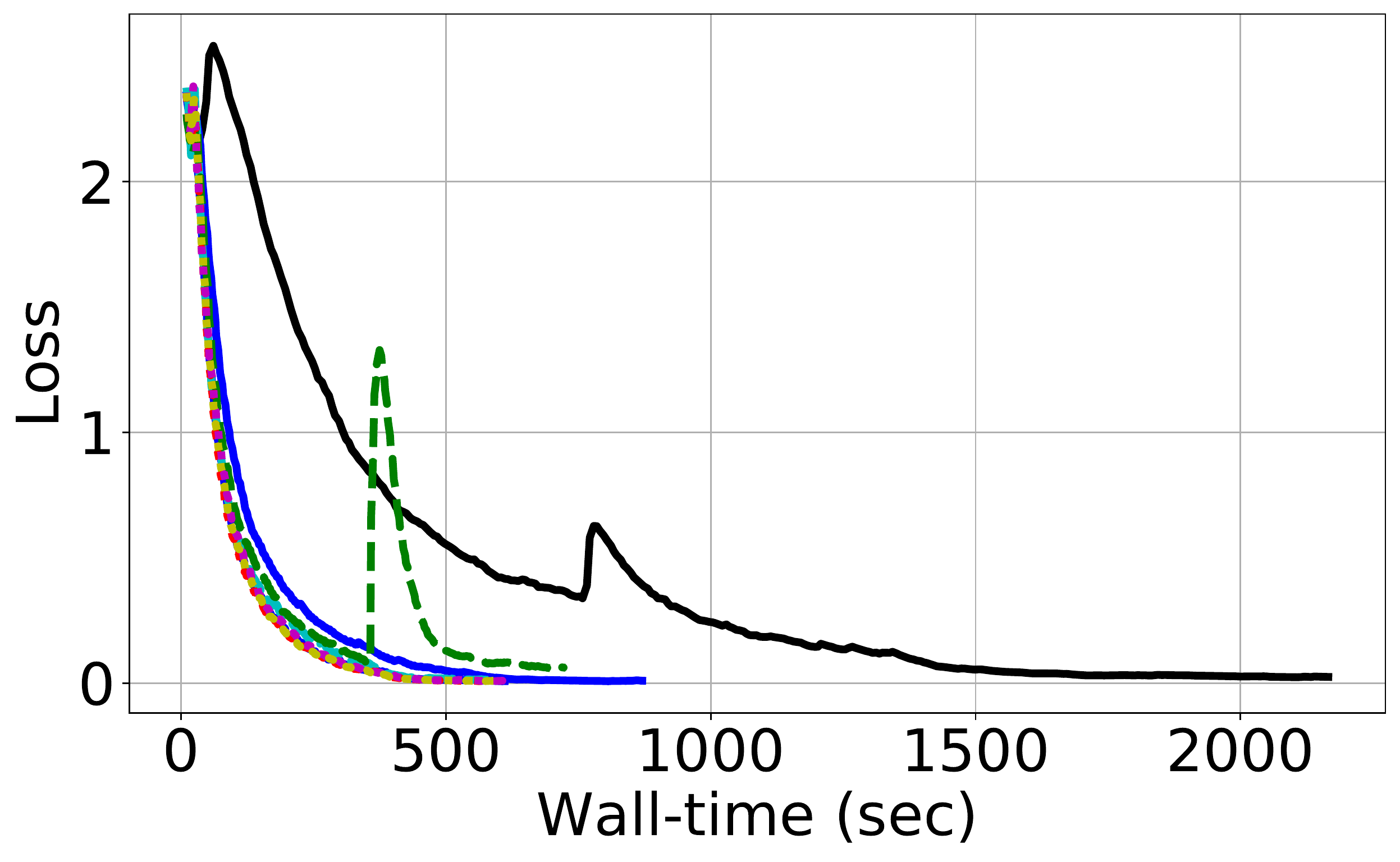}
	\caption{VGG16 on CIFAR10 - Ratio 0.01.}
	\label{fig:vgg16-acc0.01-8}
    \end{subfigure}
     \hfill
    \begin{subfigure}[ht]{0.3\linewidth}
   \includegraphics[ width=\textwidth]{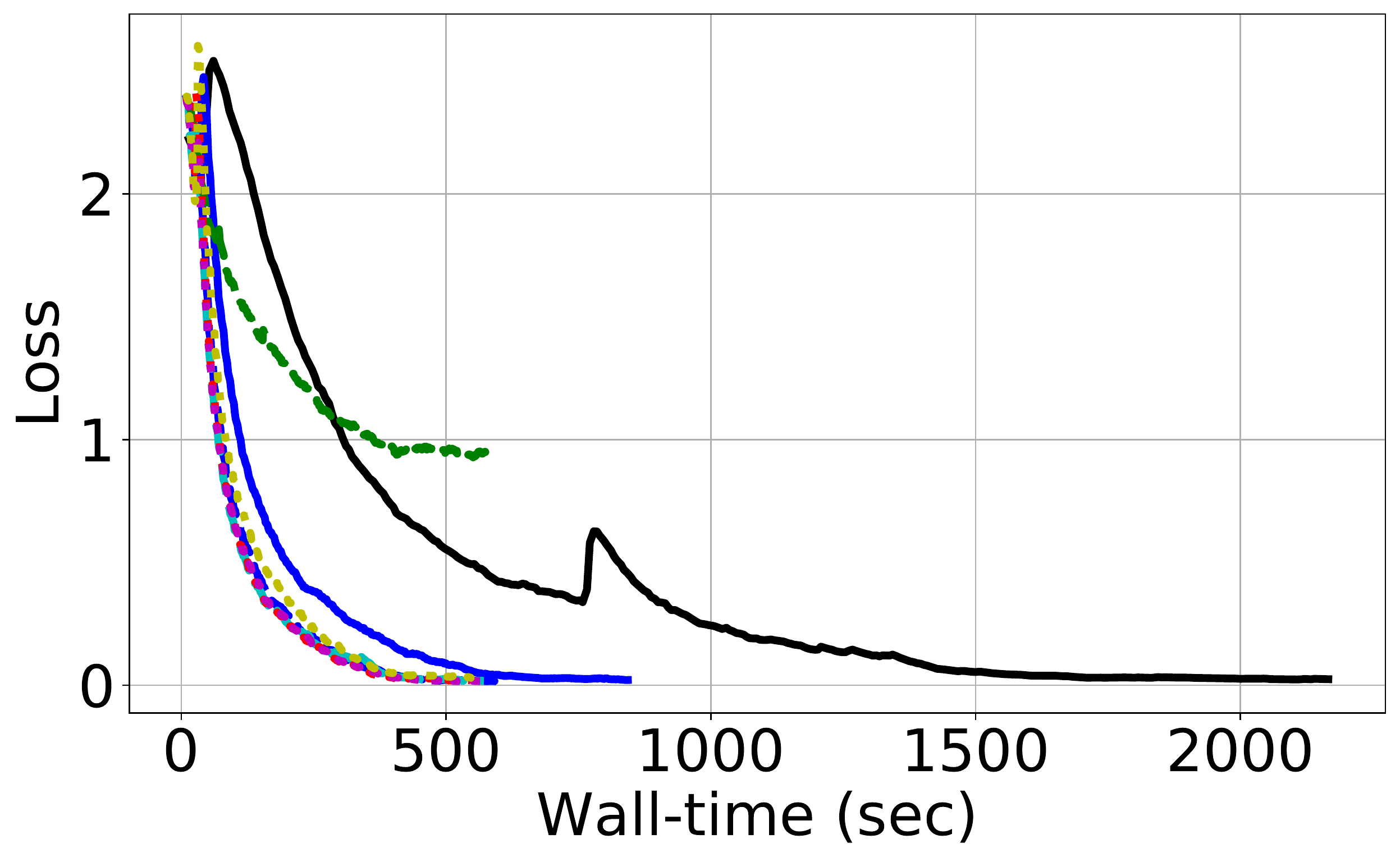}
	\caption{VGG16 on CIFAR10 - Ratio 0.001.}
	\label{fig:vgg16-acc0.001-8}
     \end{subfigure}
      \\
      \begin{subfigure}[ht]{0.3\linewidth}
    \includegraphics[ width=\textwidth]{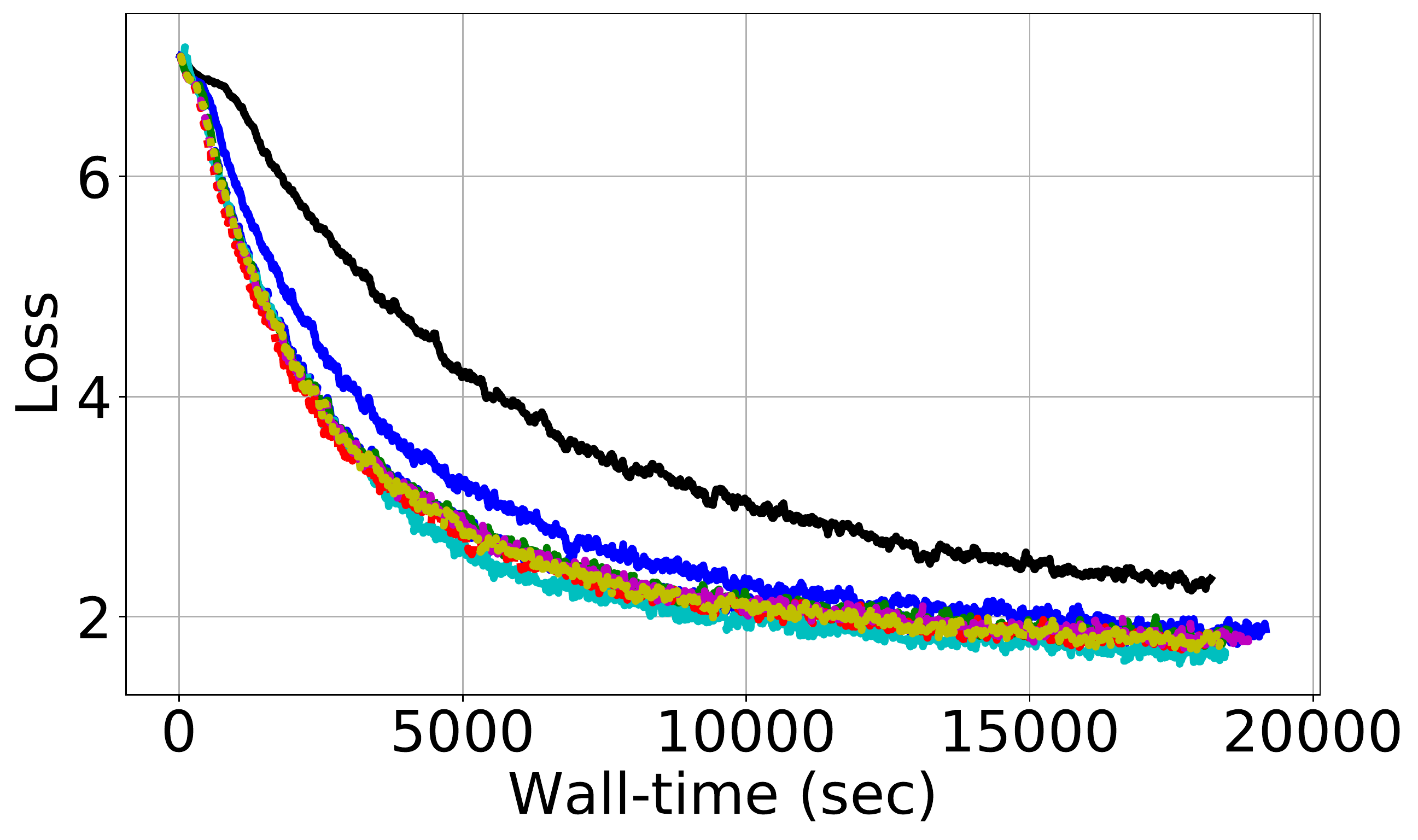}
	\caption{ResNet50 on ImageNet - Ratio 0.1.}
	\label{fig:resnet50-acc0.1-8}
     \end{subfigure}
     \hfill
	\begin{subfigure}[ht]{0.3\linewidth}
  \includegraphics[ width=\textwidth]{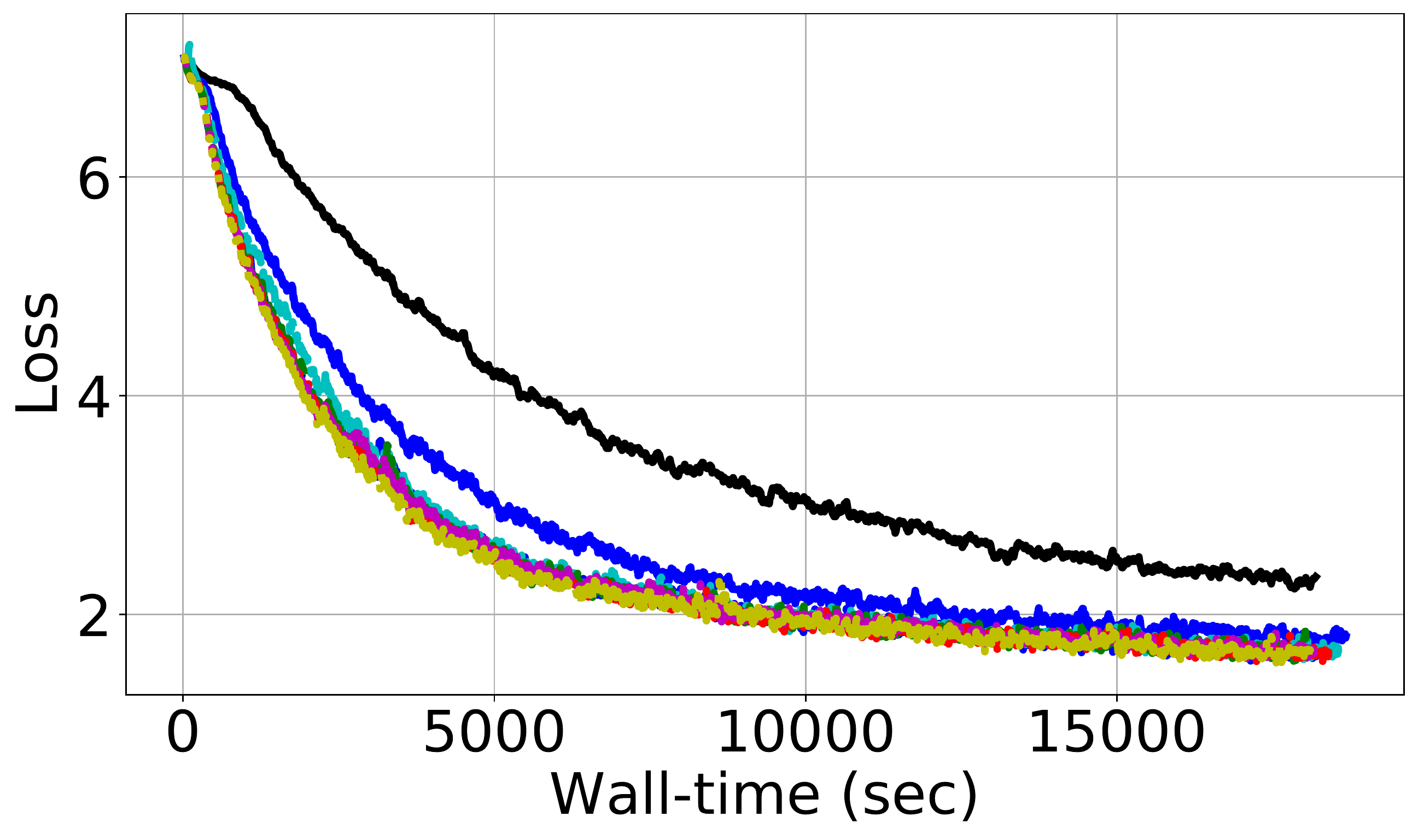}
	\caption{ResNet50 on ImageNet - Ratio 0.01.}
	\label{fig:resnet50-acc0.01-8}
    \end{subfigure}
     \hfill
    \begin{subfigure}[ht]{0.3\linewidth}
   \includegraphics[ width=\textwidth]{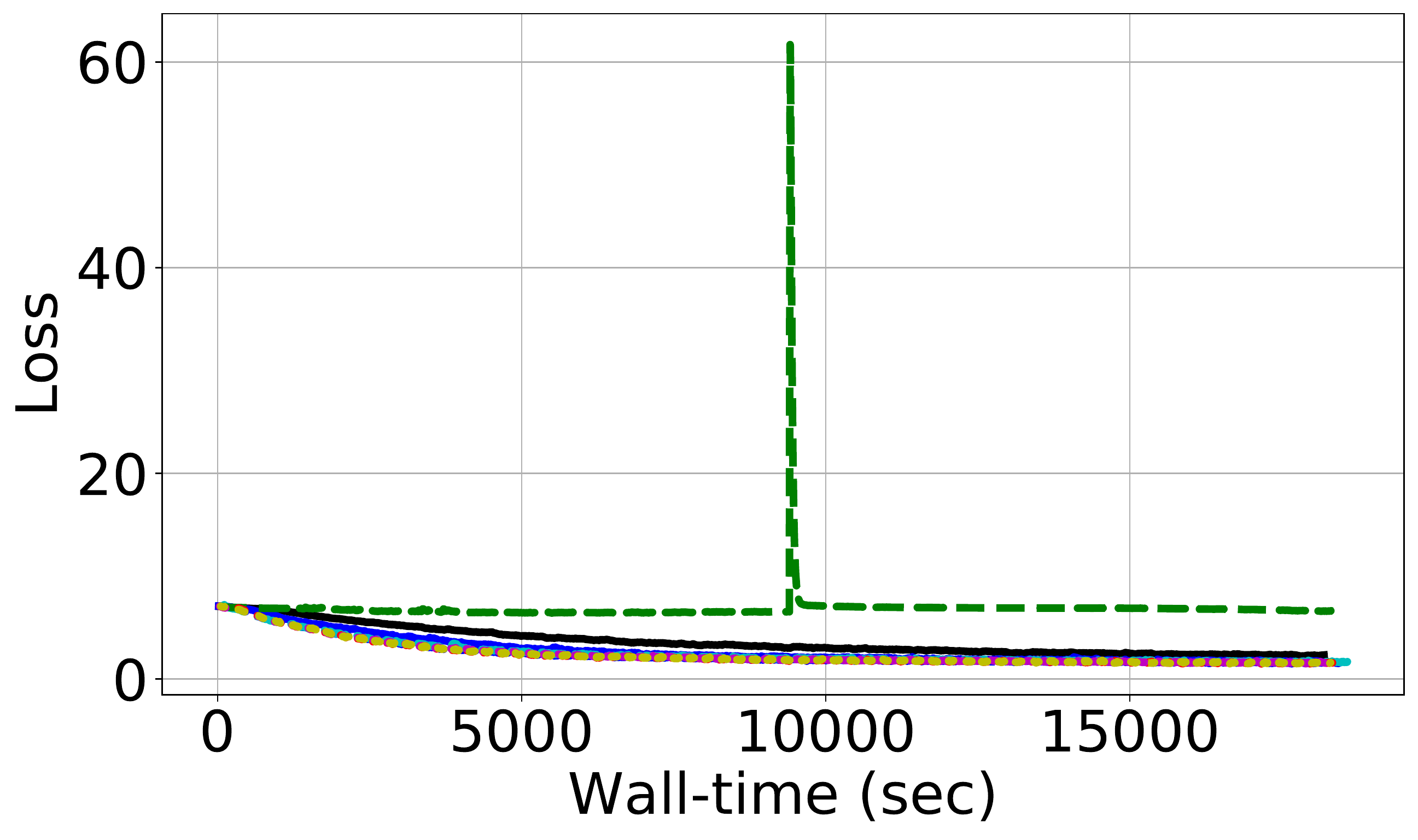}
	\caption{ResNet50 on ImageNet - Ratio 0.001.}
	\label{fig:resnet50-acc0.001-8}
     \end{subfigure}
     \\
     \begin{subfigure}[ht]{0.3\linewidth}
    \includegraphics[ width=\textwidth]{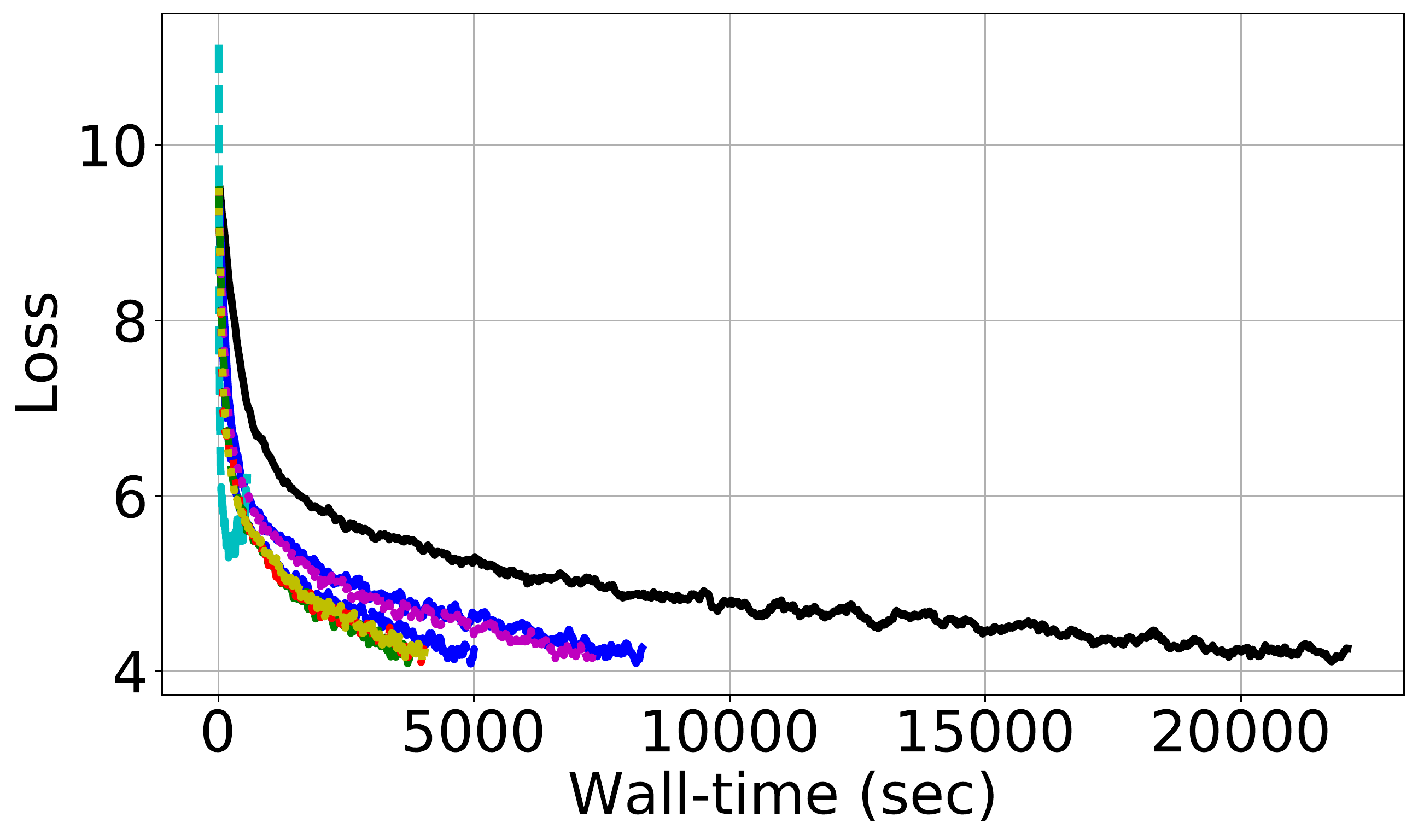}
	\caption{PTB on LSTM - Ratio 0.1.}
	\label{fig:ptb-acc0.1-8}
     \end{subfigure}
     \hfill
	\begin{subfigure}[ht]{0.3\linewidth}
  \includegraphics[ width=\textwidth]{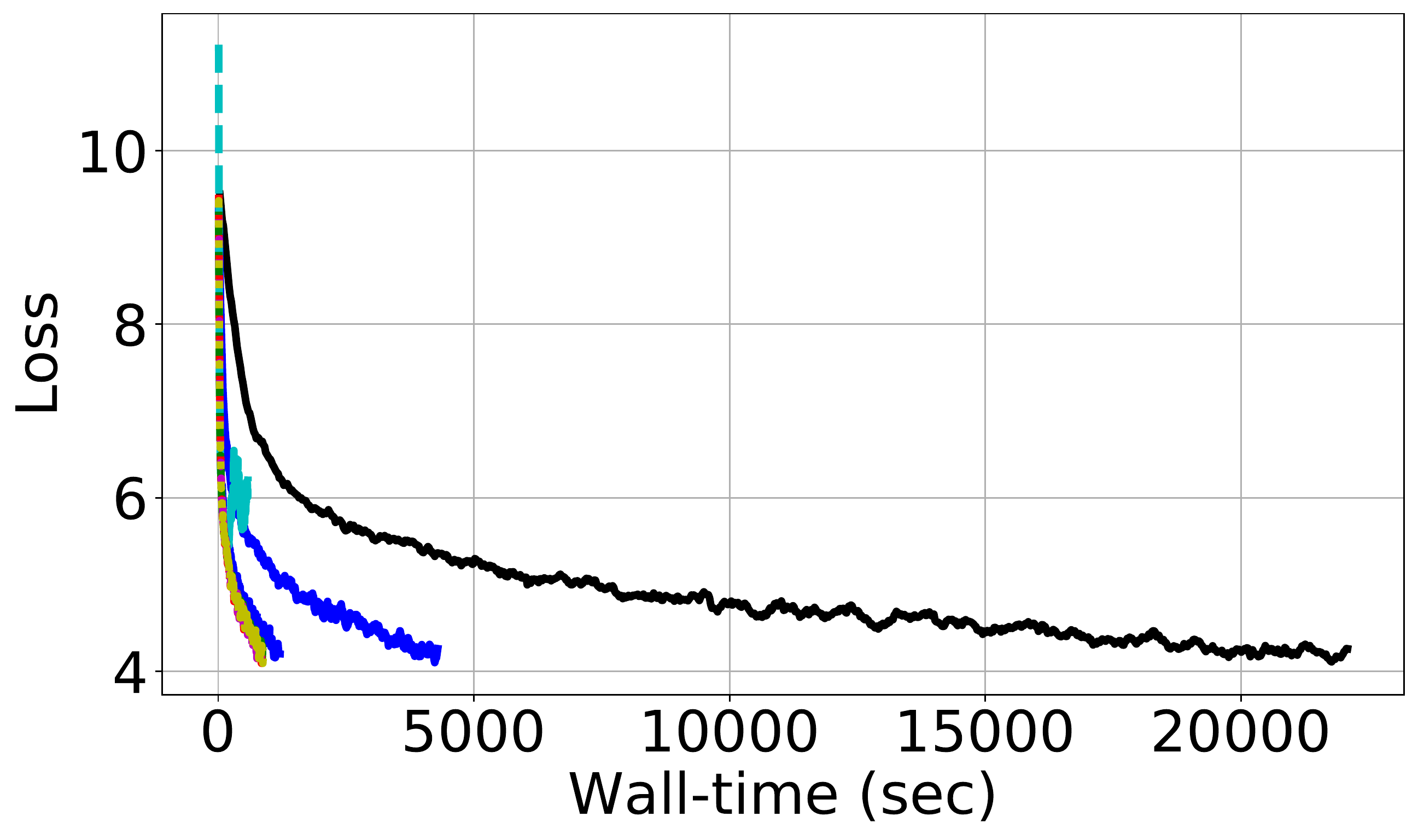}
	\caption{LSTM on PTB - Ratio 0.01.}
	\label{fig:ptb-acc0.01-8}
    \end{subfigure}
     \hfill
    \begin{subfigure}[ht]{0.3\linewidth}
   \includegraphics[ width=\textwidth]{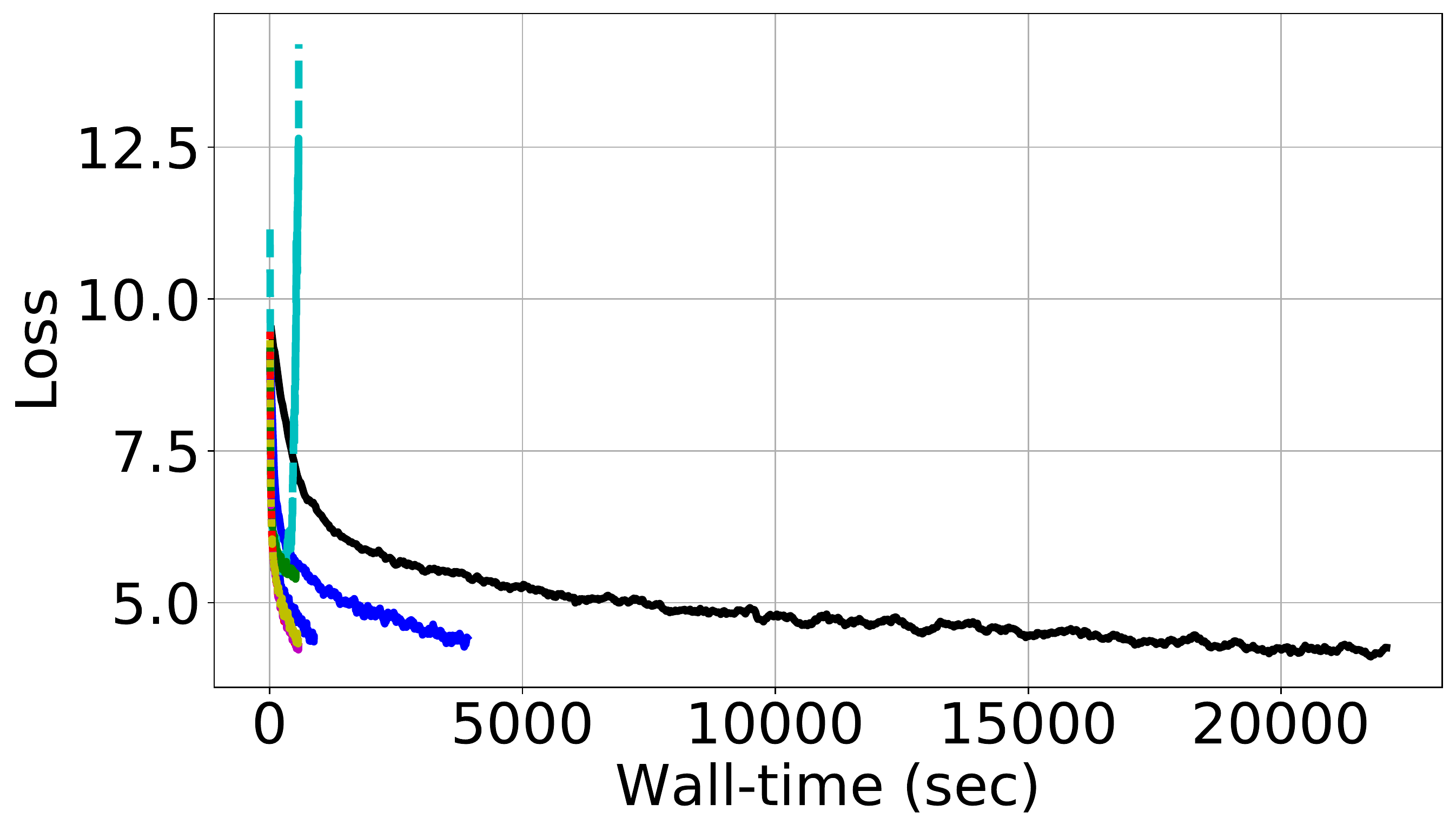}
	\caption{LSTM on PTB - Ratio 0.001.}
	\label{fig:ptb-acc0.001-8}
     \end{subfigure}
     \\
     \begin{subfigure}[ht]{0.3\linewidth}
    \includegraphics[ width=\textwidth]{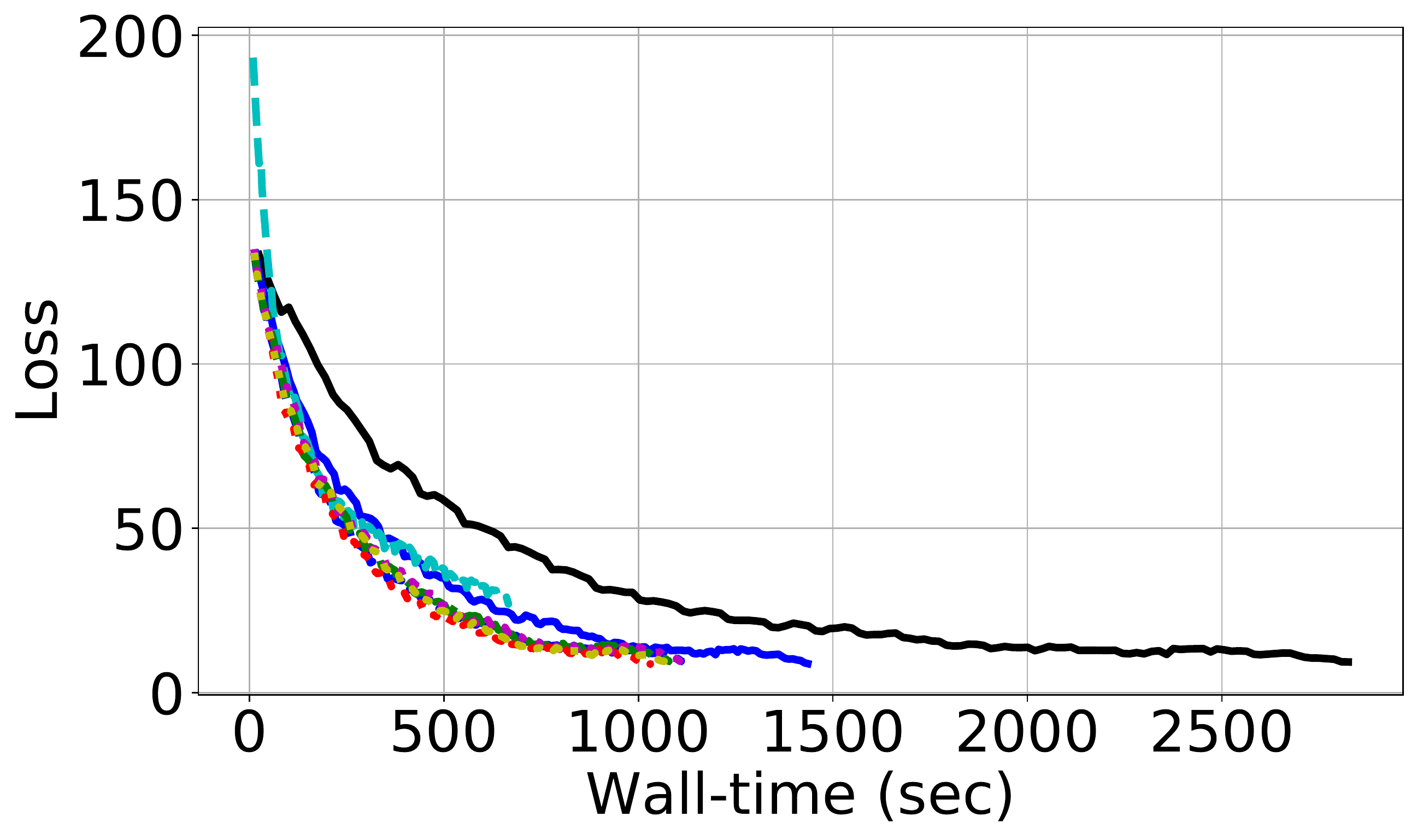}
	\caption{LSTM on AN4 - Ratio 0.1.}
	\label{fig:an4-acc0.1-8}
     \end{subfigure}
     \hfill
	\begin{subfigure}[ht]{0.3\linewidth}
  \includegraphics[ width=\textwidth]{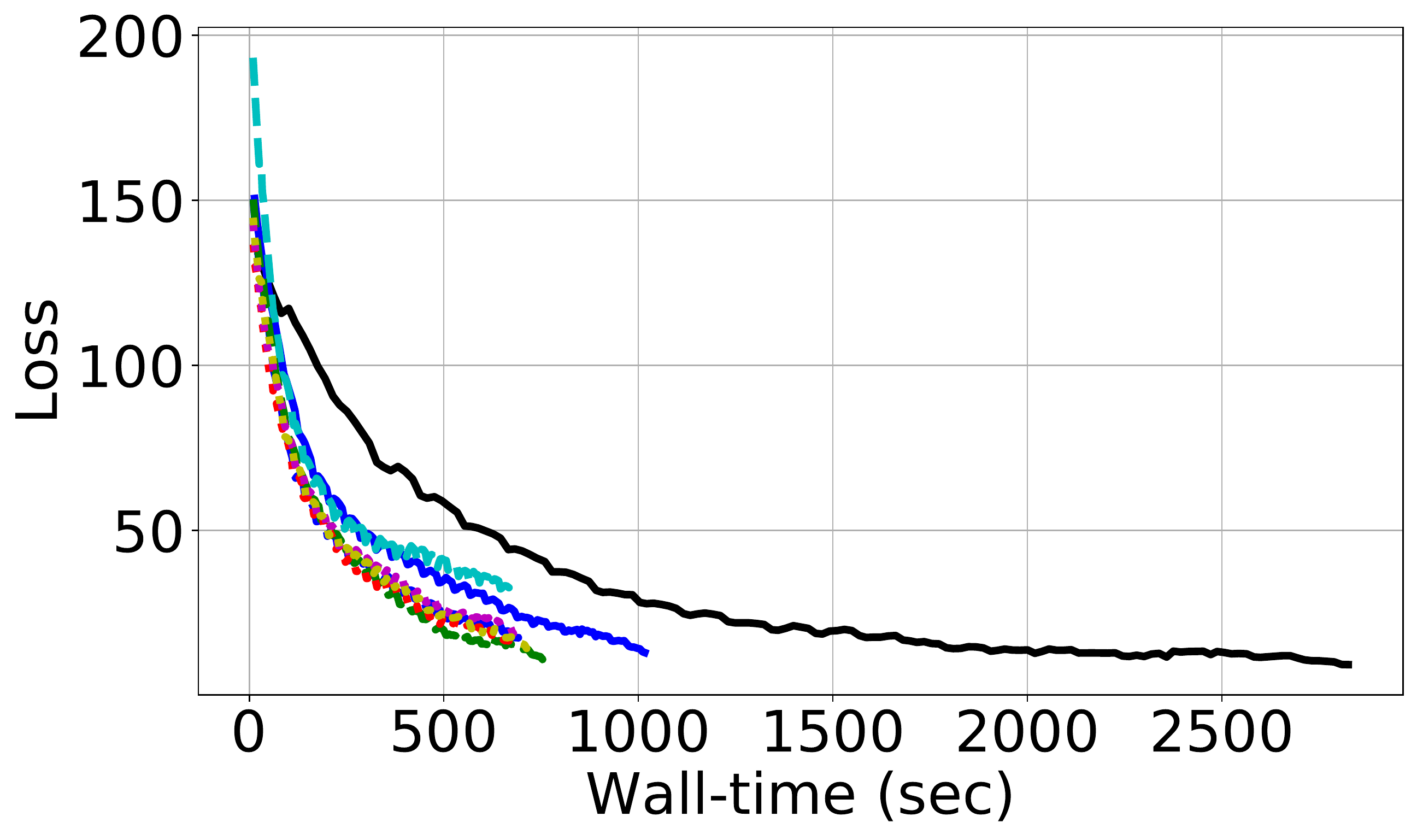}
	\caption{LSTM on AN4 - Ratio 0.01.}
	\label{fig:an4-accp0.01-8}
    \end{subfigure}
     \hfill
    \begin{subfigure}[ht]{0.3\linewidth}
   \includegraphics[ width=\textwidth]{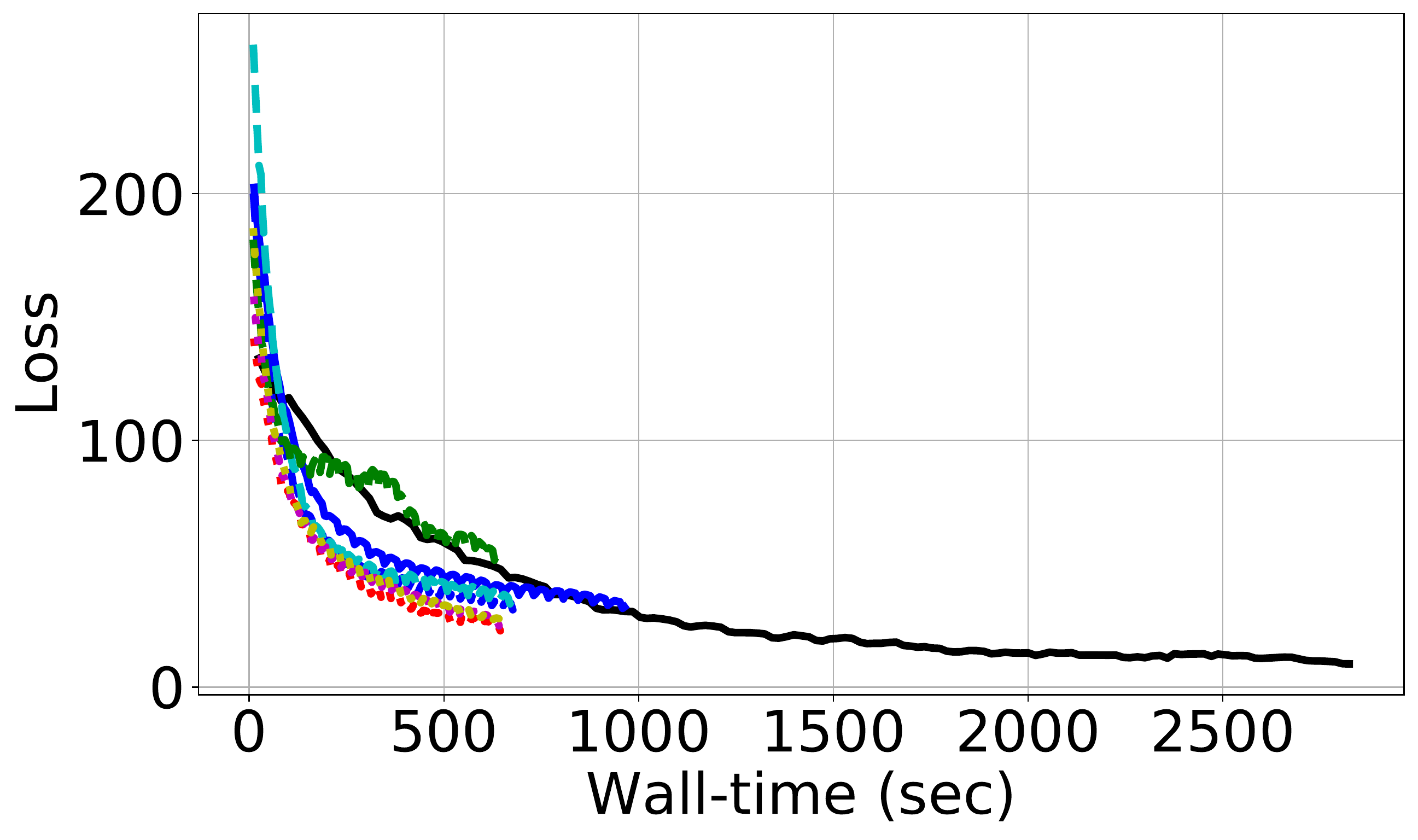}
	\caption{LSTM on AN4 - Ratio 0.001.}
	\label{fig:an4-acc0.001-8}
     \end{subfigure}
\caption{Smoothed training loss vs wall run-time for all benchmarks at different target sparsity ratios.}
\label{fig:accuracy}
\end{figure*}

\paragraph{\textbf{VGG19 on ImageNet: }}
We also present similar metrics (i.e., smoothed compression ration and training loss vs runtime) for the VGG19 benchmarks in \cref{fig:vgg19-more}. The results in \cref{fig:vgg19-avgcomp0.001-8} show that all \scheme\! methods estimate the threshold with high accuracy. They also show that GaussianKSGD miserably fails to estimate the threshold and RedSync experiences significantly high variability. \cref{fig:vgg19-avgcomp0.001-8} also shows that \scheme\! methods have noticeably higher speed-ups over all other schemes (esp., $\topk$, RedSync and GaussianKSGD).

\begin{figure*}[!h]
\captionsetup[subfigure]{justification=centering}
\centering
\begin{subfigure}[ht]{0.8\linewidth}
  \includegraphics[width=1\linewidth]{Figures/experiments/all-endtoend/mcnodes_cifar10/accuracy_ec_sgd/legend.pdf}
 \end{subfigure}
    \begin{subfigure}[ht]{0.44\linewidth}
    \includegraphics[ width=\textwidth]{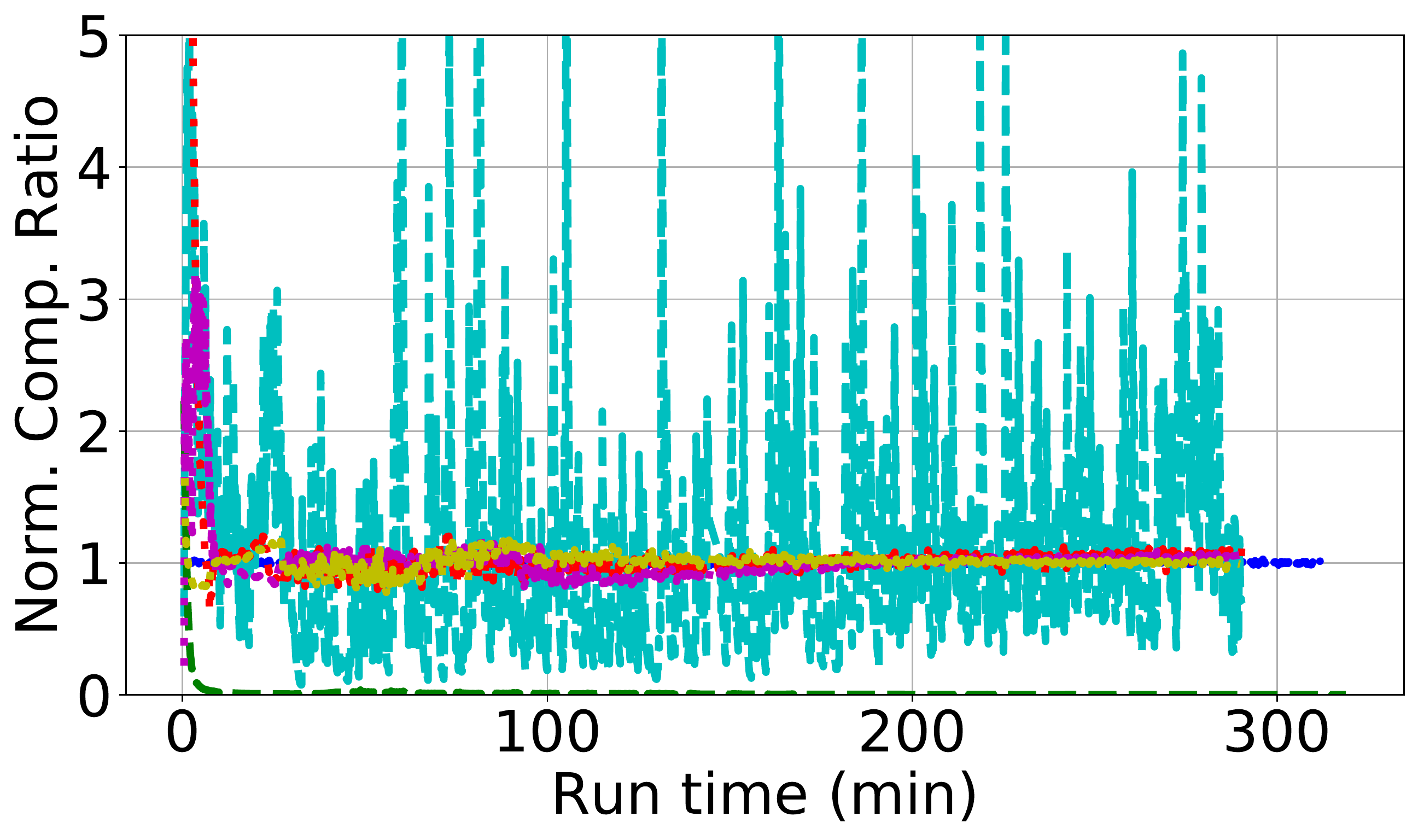}
	\caption{Smoothed compression ratio.}
	\label{fig:vgg19-avgcomp0.001-8}
     \end{subfigure}
     \hfill
     \begin{subfigure}[ht]{0.44\linewidth}
    \includegraphics[ width=\textwidth]{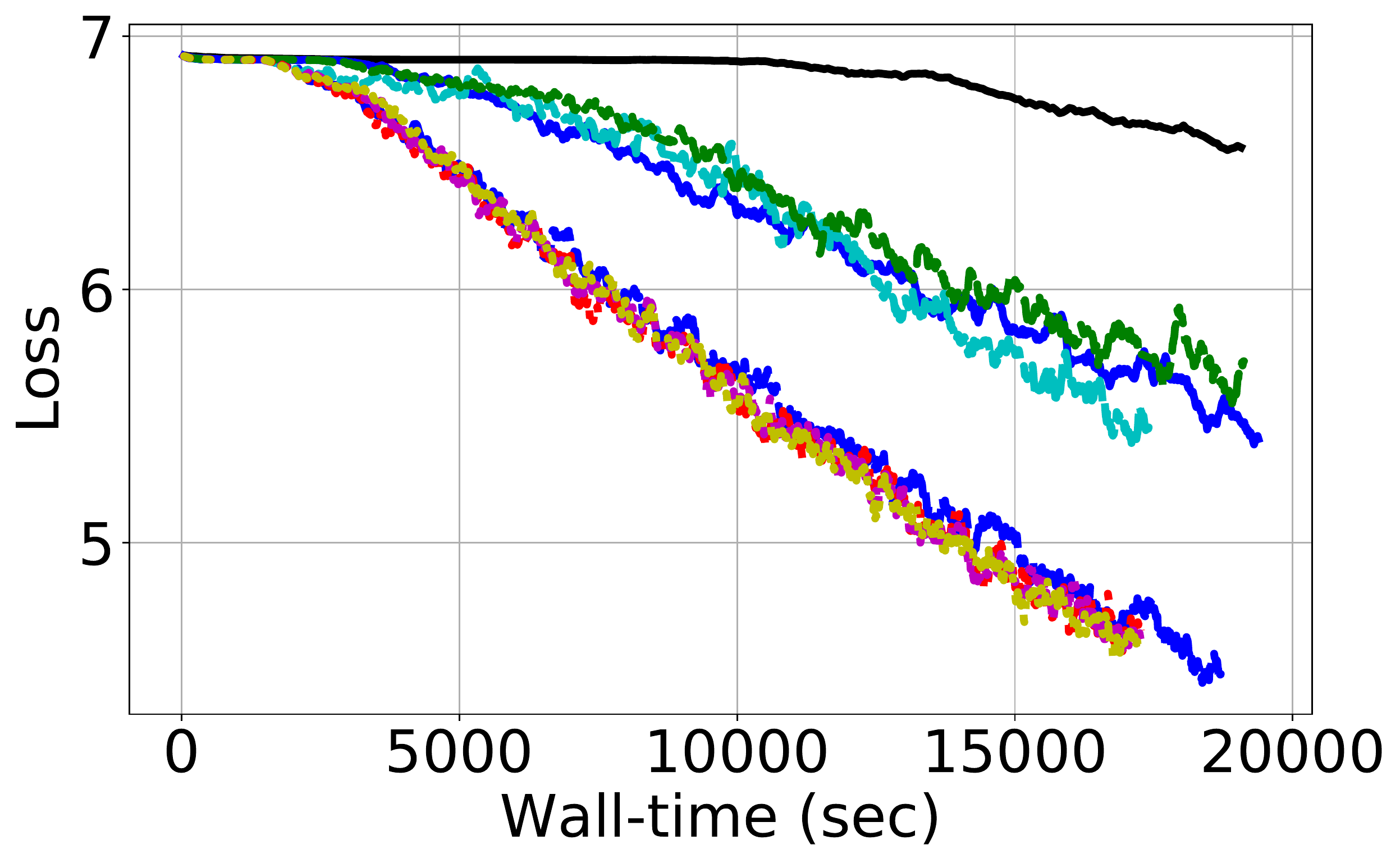}
	\caption{Training loss.}
	\label{fig:vgg19-acc0.001-8}
     \end{subfigure}
\caption{Performance metrics of training VGG19 on ImageNet using ratio of $0.001$.}
\label{fig:vgg19-more}
\end{figure*}

\paragraph{\textbf{CPU as the Compression Device:}} In this experiment, instead of using the GPU as the compression target, we use the CPU device as the compression device and report on the average training throughput. Due to the slow speed of the experiment, we only run the experiment for two epochs as we are interested in the throughput numbers. We compare the performance of $\topk$, DGC and \scheme\!-E. \cref{fig:compcpu} presents the average training throughput (the first 10 iterations are excluded in the average). First, we note that the throughput on CPU is relatively high for $\topk$ method which consistently performed the worst when GPU is the target compression device. In contrast, \ac{DGC} is now performing the worst among all methods due to the slow performance of random sampling on CPU device. On the other hand, \scheme\! consistently performs the best even on CPU as the target device. These results are not surprising as it closely matches the observations from the micro-benchmark results (\cref{apdx:moremicrobench}).

\begin{figure*}[!h]
\captionsetup[subfigure]{justification=centering}
\centering
     \begin{subfigure}[ht]{0.5\linewidth}
  \includegraphics[width=1\linewidth]{Figures/experiments/legend2.pdf}
  \end{subfigure}
  \\
     \begin{subfigure}[ht]{0.32\linewidth}
    \includegraphics[width=\linewidth]{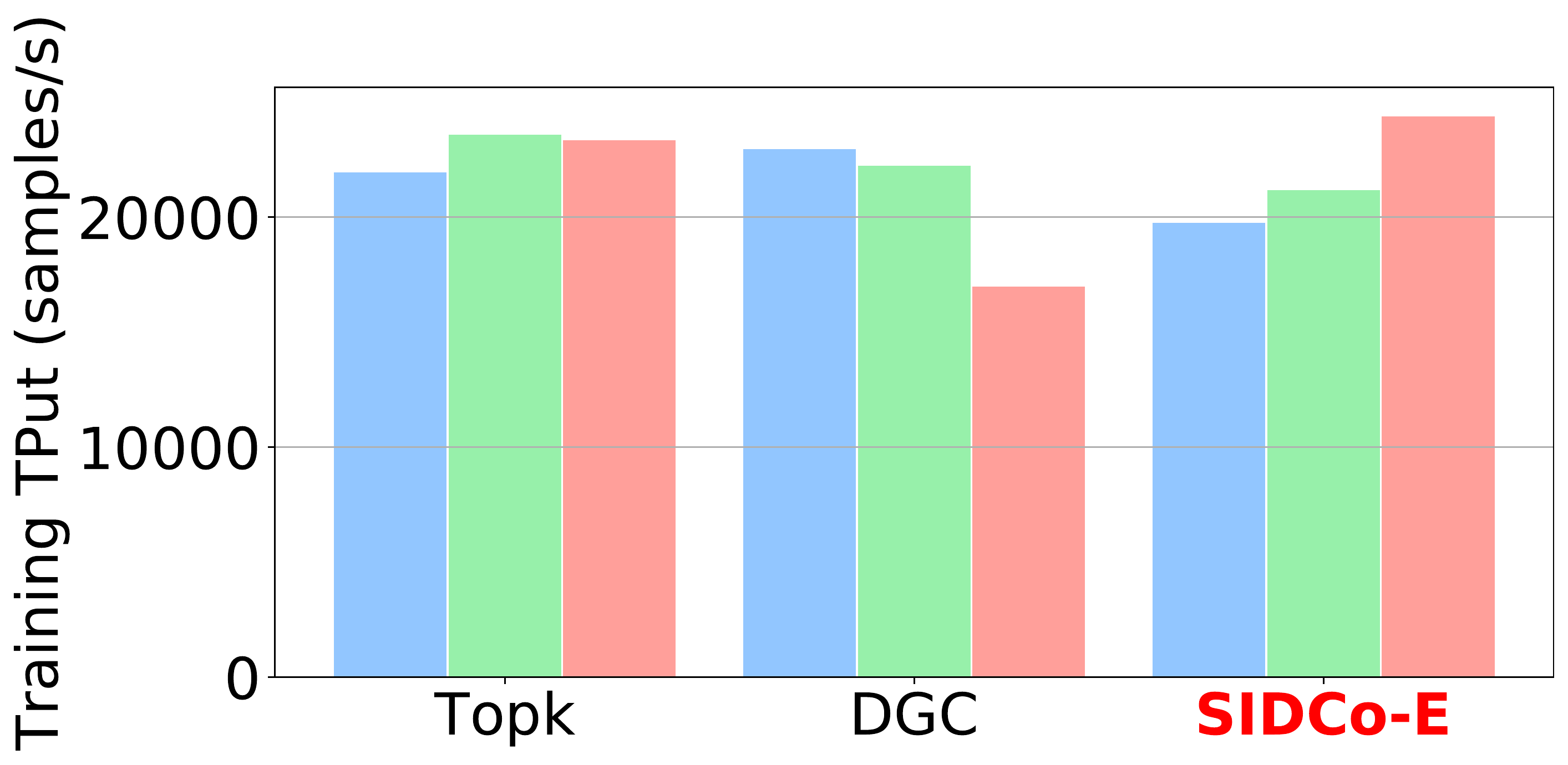}
	\caption{ResNet20 on CIFAR10 (TPut).}
	\label{fig:resnet20-cputput-8}
     \end{subfigure}
     \hfill
    \begin{subfigure}[ht]{0.32\linewidth}
  \includegraphics[width=\linewidth]{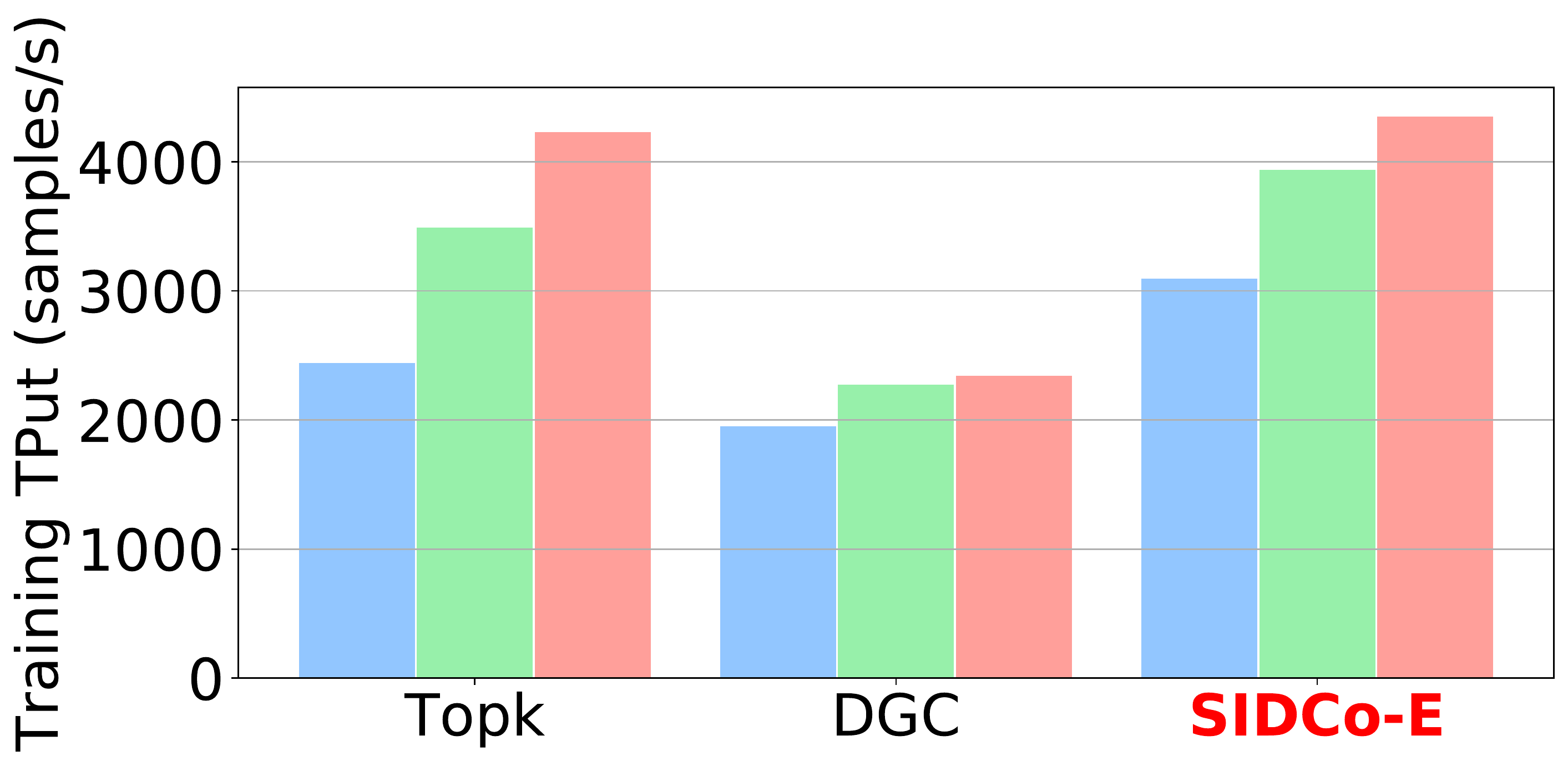}
	\caption{VGG16 on CIFAR10 (TPut).}
	\label{fig:vgg16-cputput-8}
    \end{subfigure}
    \hfill
     \begin{subfigure}[ht]{0.32\linewidth}
    \includegraphics[width=\linewidth]{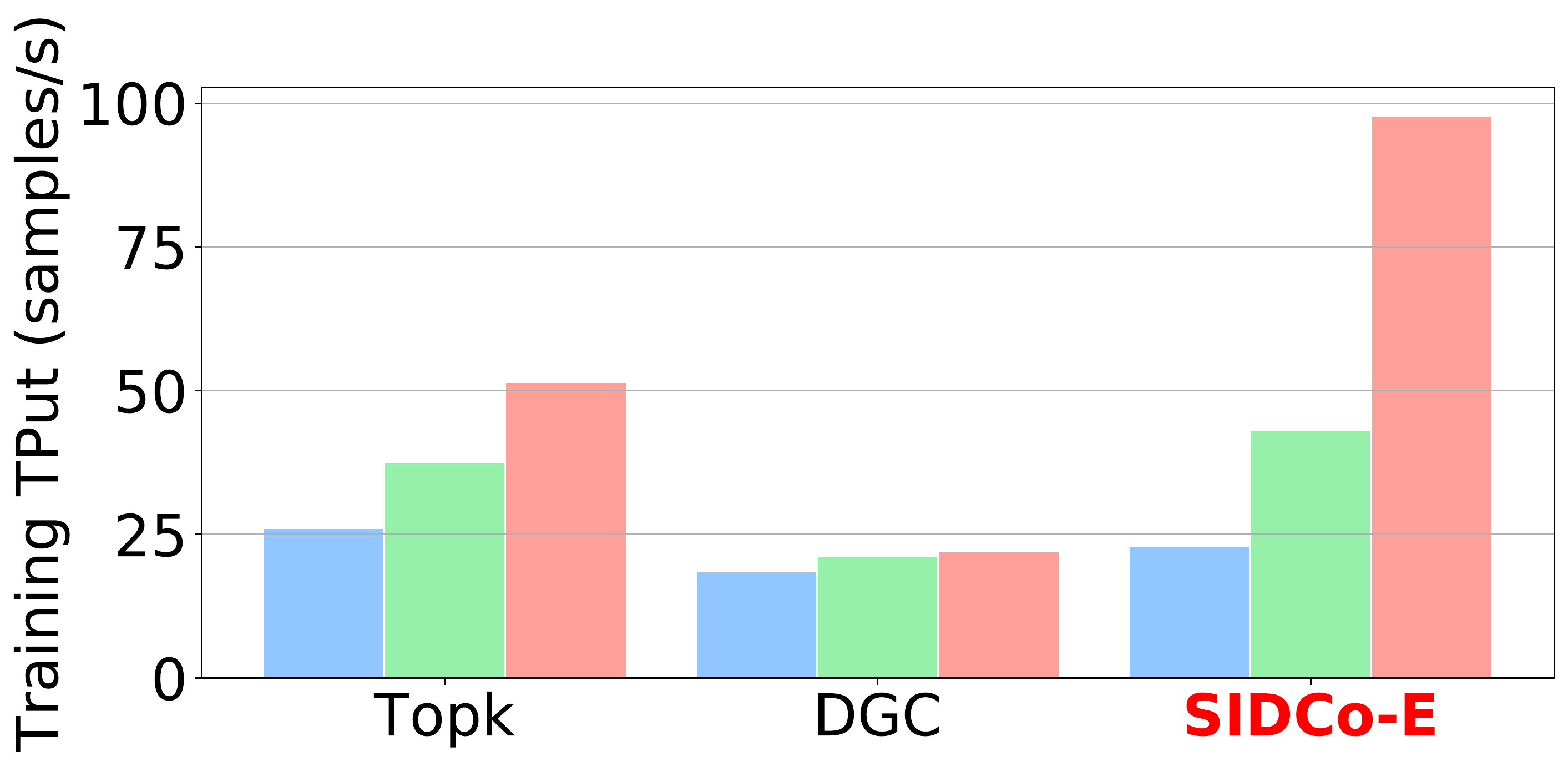}
	\caption{LSTM on PTB (TPut).}
	\label{fig:lstm-cputput-8}
     \end{subfigure}
\caption{Throughput when CPU is the compression device: (a) ResNet20 , (b) VGG16 and (c) LSTM-PTB.}
\label{fig:compcpu}
\end{figure*}

\paragraph{\textbf{Full ImageNet training on Multi-GPU node:}} In \cref{fig:ibex}, we present the results for training both ResNet50 and VGG19 on ImageNet fully for 90 epochs using a single node equipped with 8 Nvidia-V100 32GB GPUs in the shared cluster presented in \cref{apdx:clusters}. Each allocation of a node in the shared cluster is limited to 24 hours of run-time. We use compression ratio of $0.1$ for ResNet50 and $0.01$ for VGG19. Figure \ref{fig:resnet50-acc-ibex} and \ref{fig:vgg19-acc-ibex} show the top-1 test accuracy at the end of the training either due to finishing the 90 epochs or allocation is revoked. They shows that that compression can achieve the same or higher accuracy than no-compression baseline. Also, in case of VGG19, compression speed-ups allow the training to converge faster and hence the higher accuracy. \cref{fig:resnet50-tput-ibex} and \cref{fig:vgg19-tput-ibex} show the training throughput and that all methods supersedes $\topk$. Moreover, \scheme\! schemes achieve higher throughput than \ac{DGC} and $\topk$. Finally,  \cref{fig:resnet50-comp-ibex} and \cref{fig:vgg19-comp-ibex} show the estimation quality and they show that the quality is very bad for Gaussian-based fitting methods while \scheme\! schemes can achieve same estimation quality as of the sampling of DGC.
\begin{figure*}[!h]
\captionsetup[subfigure]{justification=centering}
\centering
     \begin{subfigure}[ht]{0.32\linewidth}
    \includegraphics[width=\linewidth]{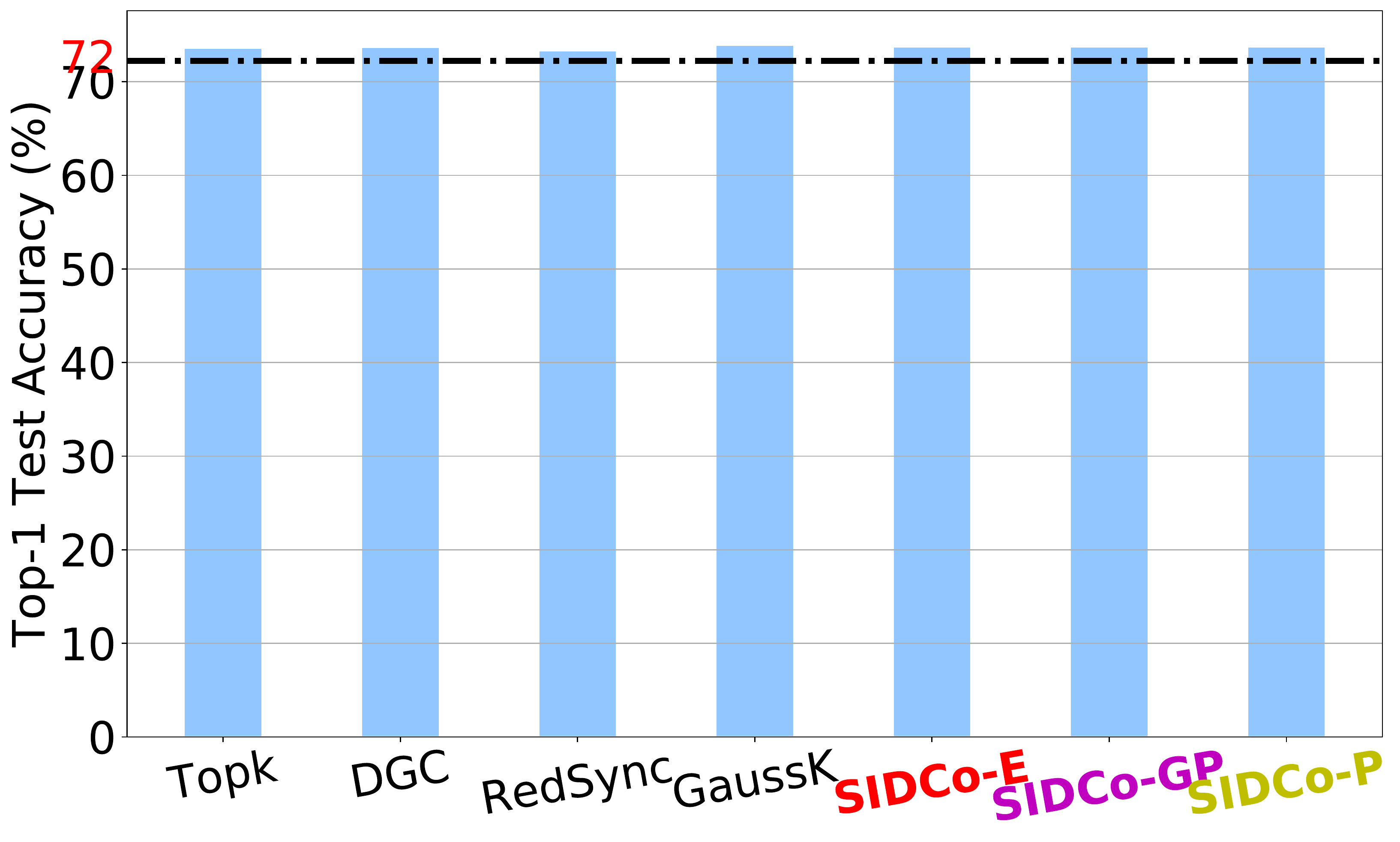}
	\caption{Training accuracy.}
	\label{fig:resnet50-acc-ibex}
     \end{subfigure}
     \hfill
    \begin{subfigure}[ht]{0.32\linewidth}
  \includegraphics[width=\linewidth]{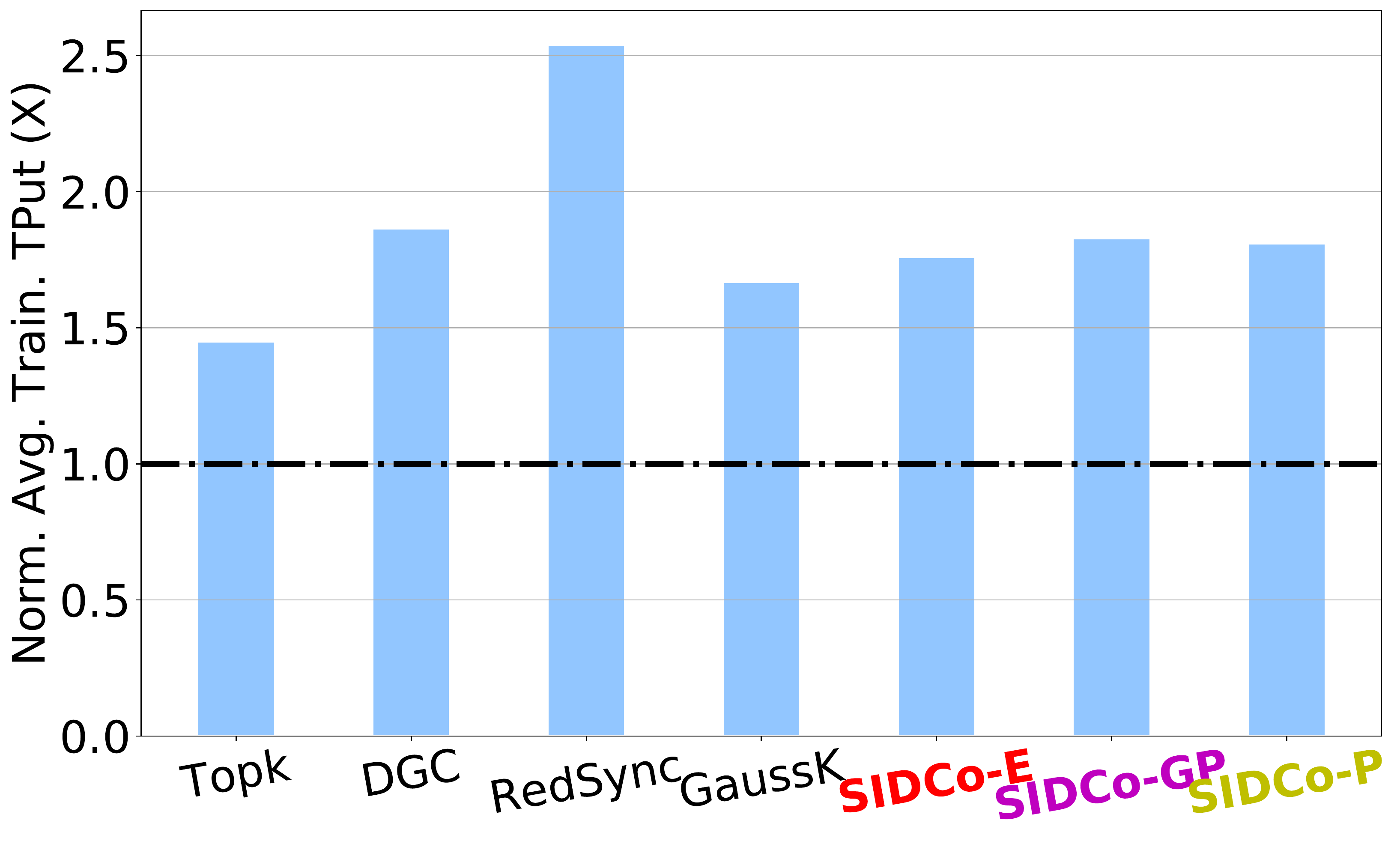}
	\caption{Training Throughput.}
	\label{fig:resnet50-tput-ibex}
    \end{subfigure}
    \hfill
     \begin{subfigure}[ht]{0.32\linewidth}
    \includegraphics[width=\linewidth]{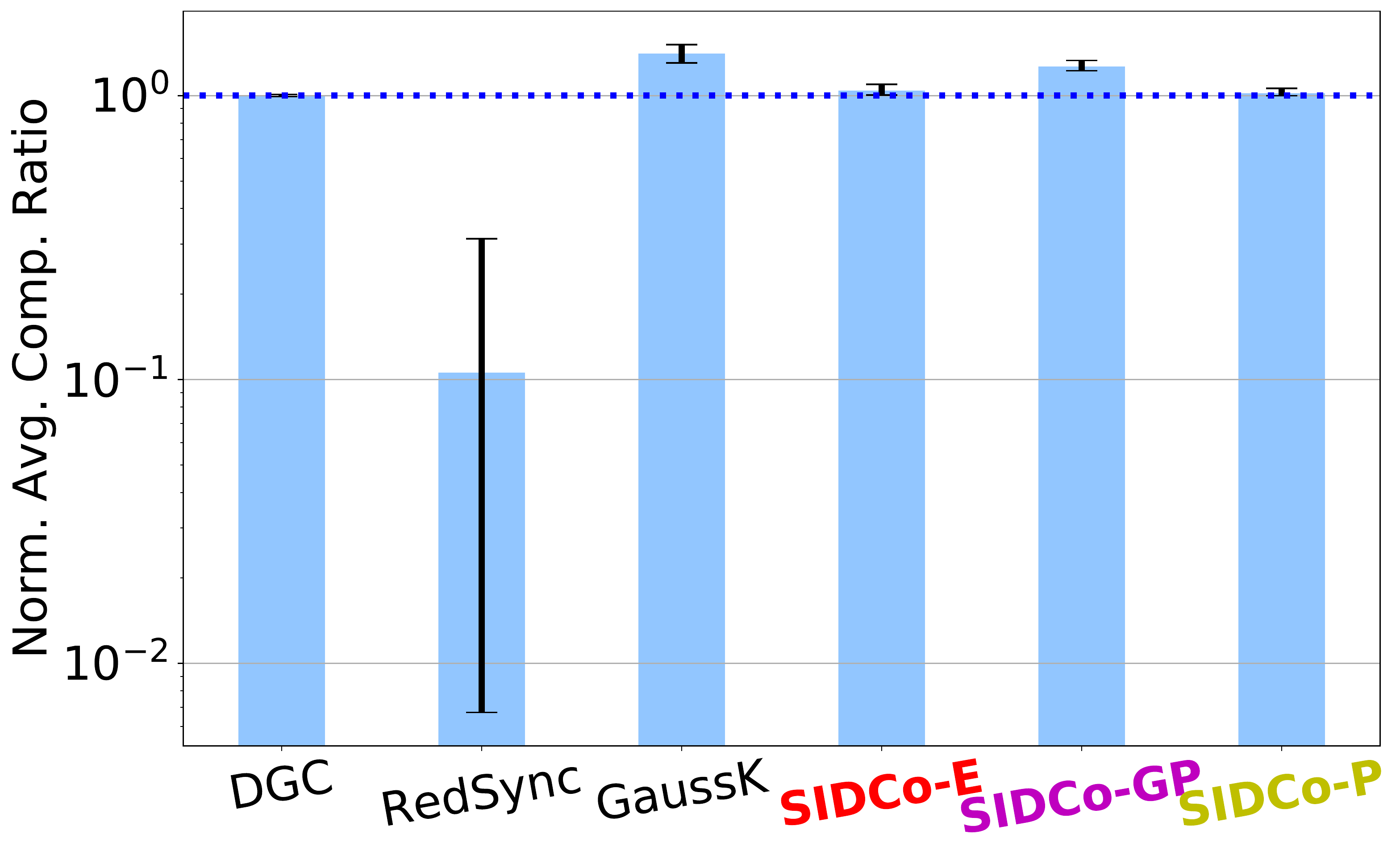}
	\caption{Estimation Quality.}
	\label{fig:resnet50-comp-ibex}
     \end{subfigure}
     \\
     \begin{subfigure}[ht]{0.32\linewidth}
    \includegraphics[width=\linewidth]{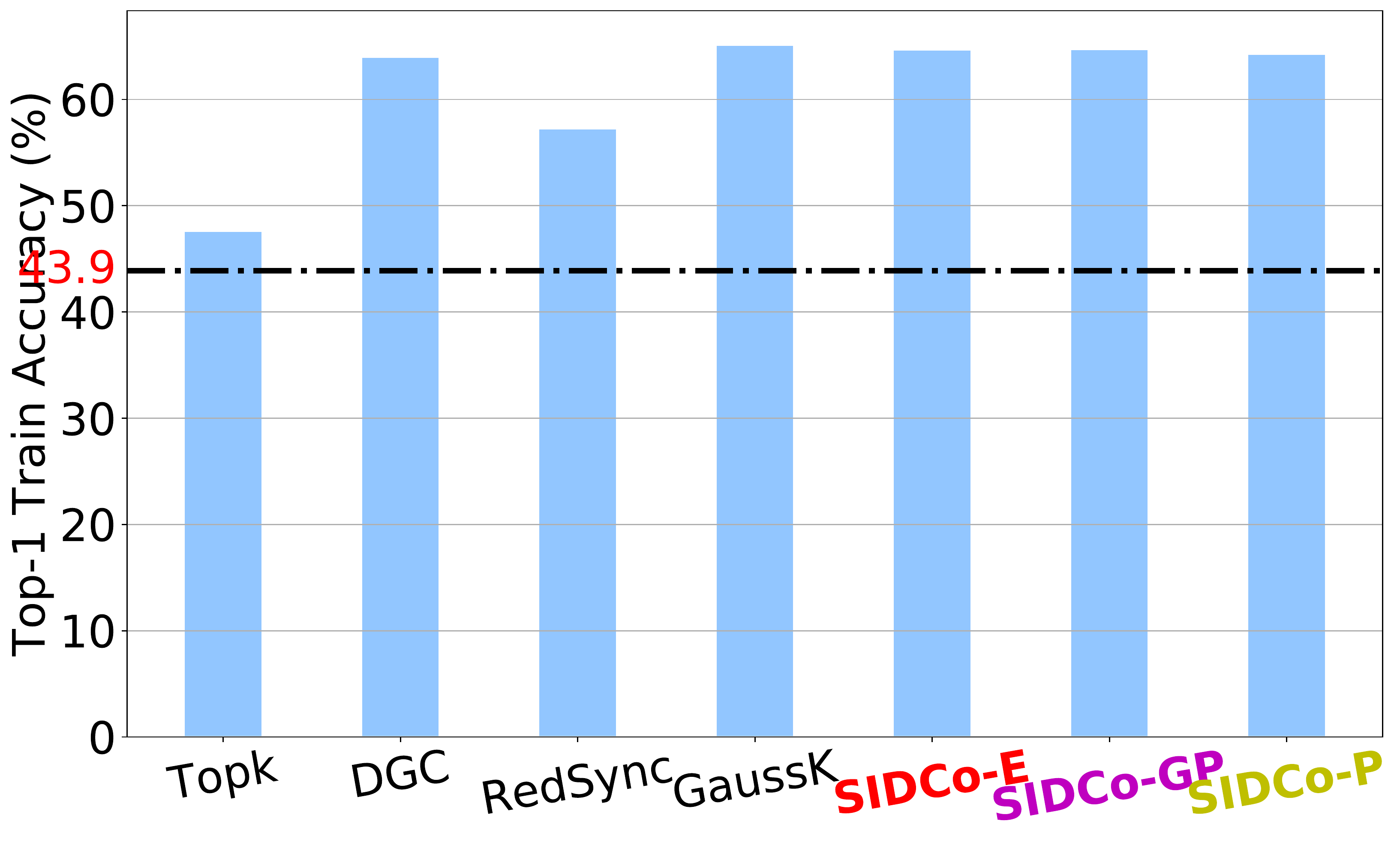}
	\caption{Training accuracy.}
	\label{fig:vgg19-acc-ibex}
     \end{subfigure}
     \hfill
    \begin{subfigure}[ht]{0.32\linewidth}
  \includegraphics[width=\linewidth]{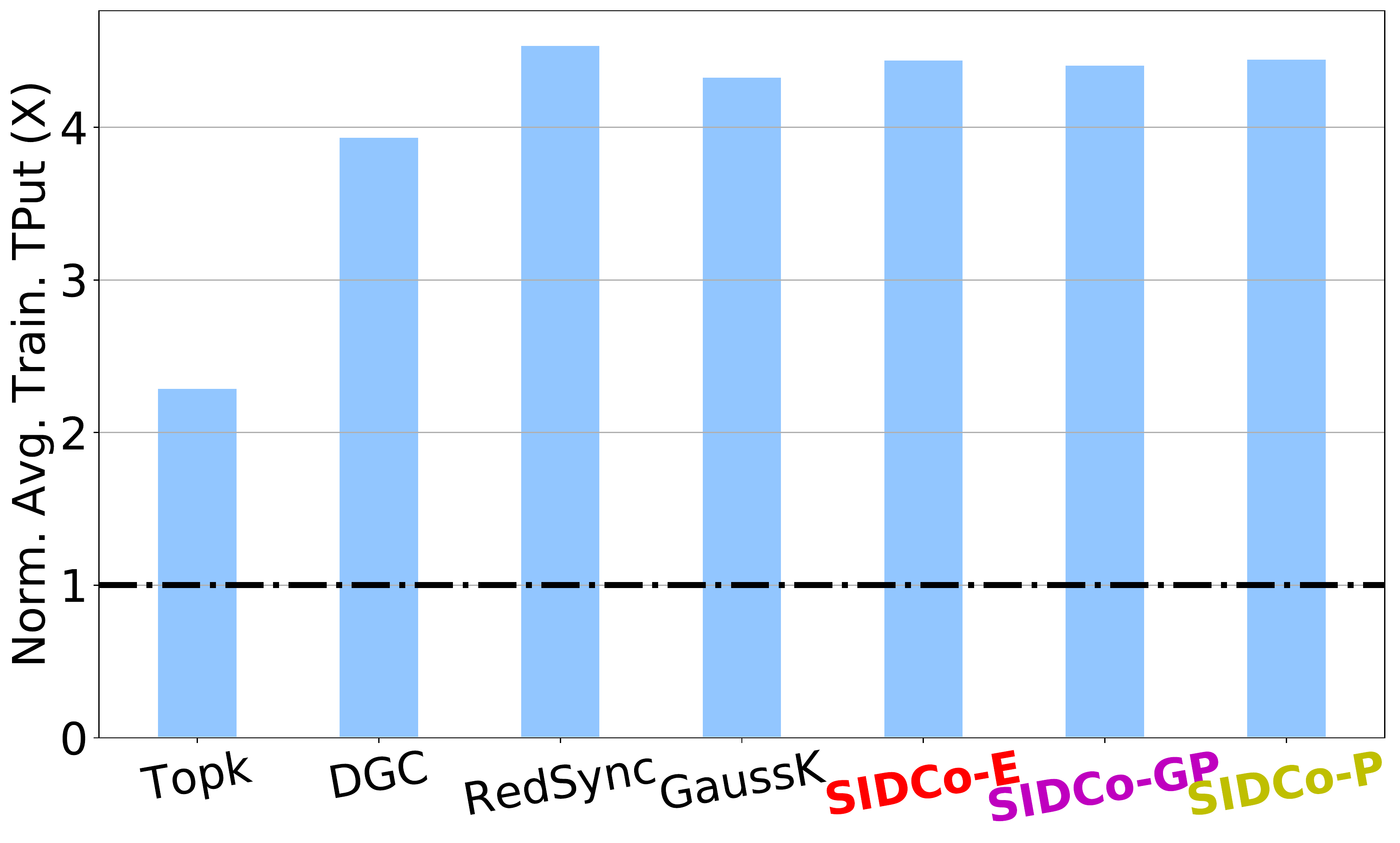}
	\caption{Training Throughput.}
	\label{fig:vgg19-tput-ibex}
    \end{subfigure}
    \hfill
     \begin{subfigure}[ht]{0.32\linewidth}
    \includegraphics[width=\linewidth]{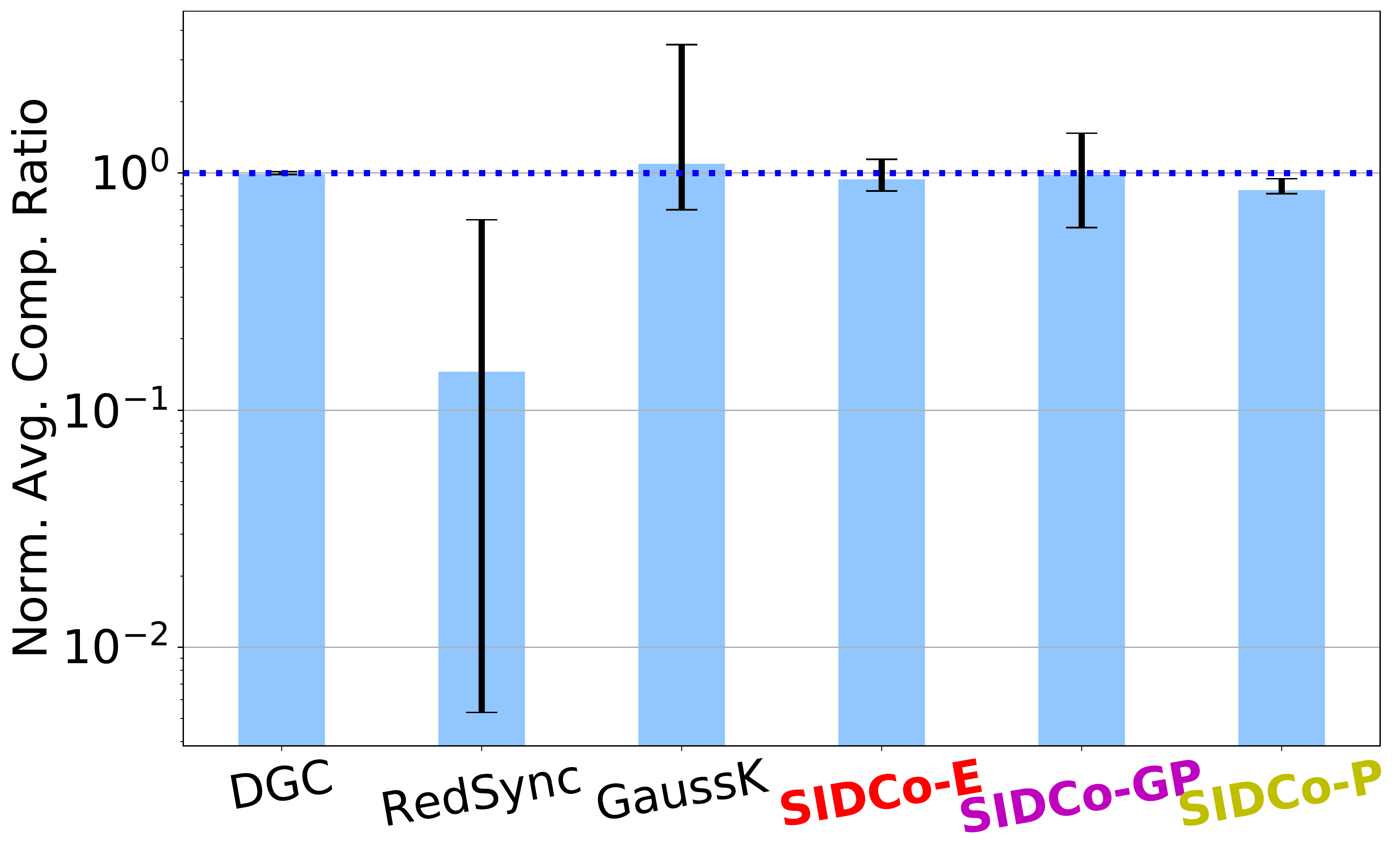}
	\caption{Estimation Quality.}
	\label{fig:vgg19-comp-ibex}
     \end{subfigure}
\caption{Training Performance of ImageNet on ResNet50 [(a), (b), (c)]  and VGG19 [(d), (e), (f)] using the multi-GPU node.}
\label{fig:ibex}
\end{figure*}

\subsection{Compression Complexity of DNN models}
\label{apdx:moremicrobench}

\textbf{Compression Overhead of Real Models:} In \cref{fig:microbench-models-speedup} and \cref{fig:microbench-models-time}, we present the compression speed-up over $\topk$ and the latency overhead for some models including ResNet20, VGG16, ResNet50 and RNN-LSTM used in training CIFAR10, ImageNet and PTB datasets, respectively. The results confirms the results, presented earlier, for VGG16, where Threshold-based methods including \scheme\! outperforms $\topk$ and \ac{DGC} both on GPU and CPU as target compression device over all models in comparison. The results also show that \ac{DGC} outperforms $\topk$ on the GPU device while $\topk$ outperforms \ac{DGC} on the CPU device. Overall, for flexibility reasons and the compatibility with various devices, both $\topk$ and DGC are not preferable.

 \begin{figure*}[!ht]
  \centering
  \begin{subfigure}[ht]{0.5\linewidth}
  \includegraphics[width=1\linewidth]{Figures/experiments/legend2.pdf}
 \end{subfigure}
  \\
      \begin{subfigure}{0.24\linewidth}
  \includegraphics[width=1\textwidth]{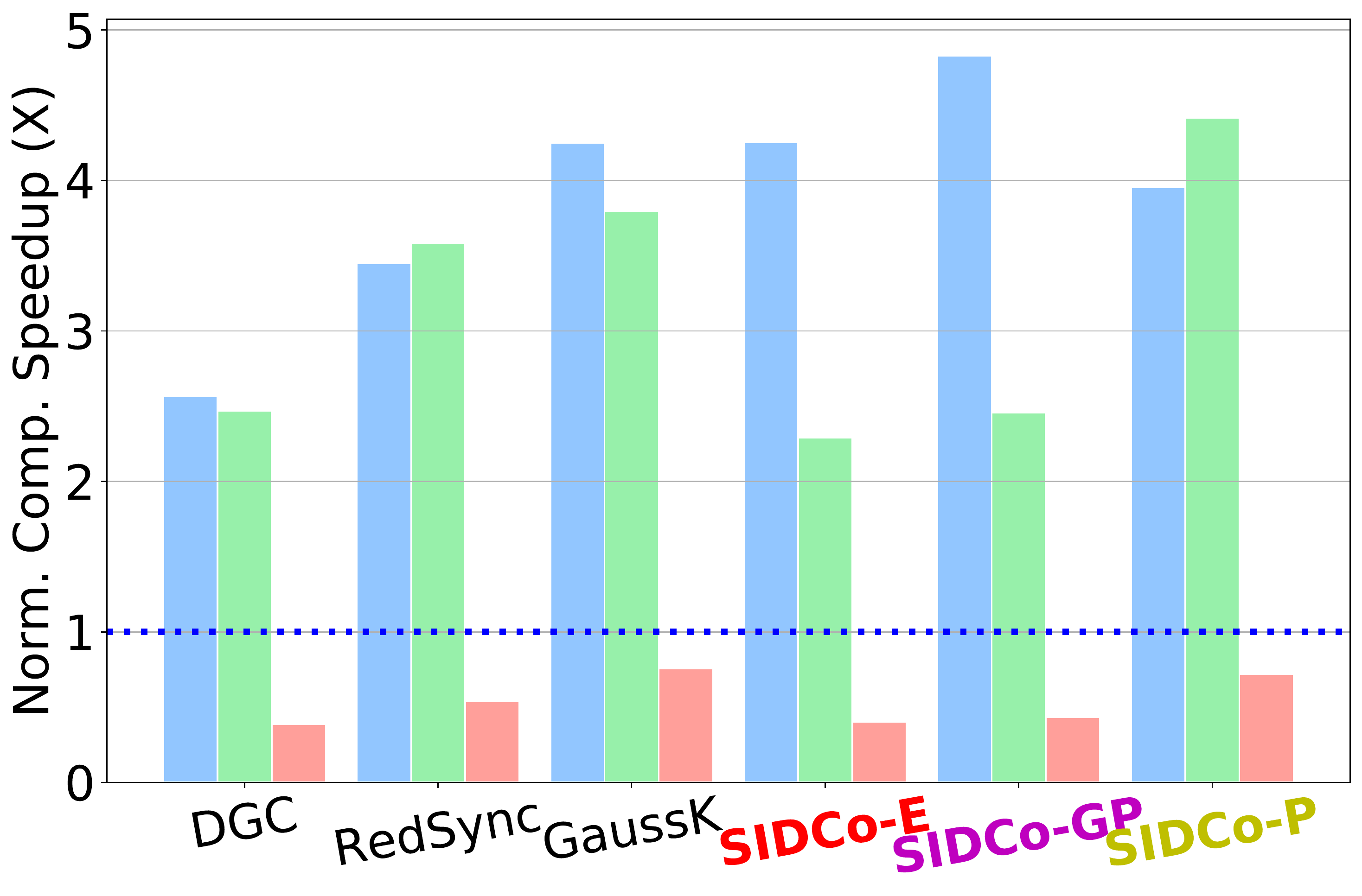}
      \caption{ResNet20 on GPU}
        \label{fig:resnet20-cuda-speedupall}
     \end{subfigure}
     \hfill
     \begin{subfigure}{0.24\linewidth}
  \includegraphics[width=1\textwidth]{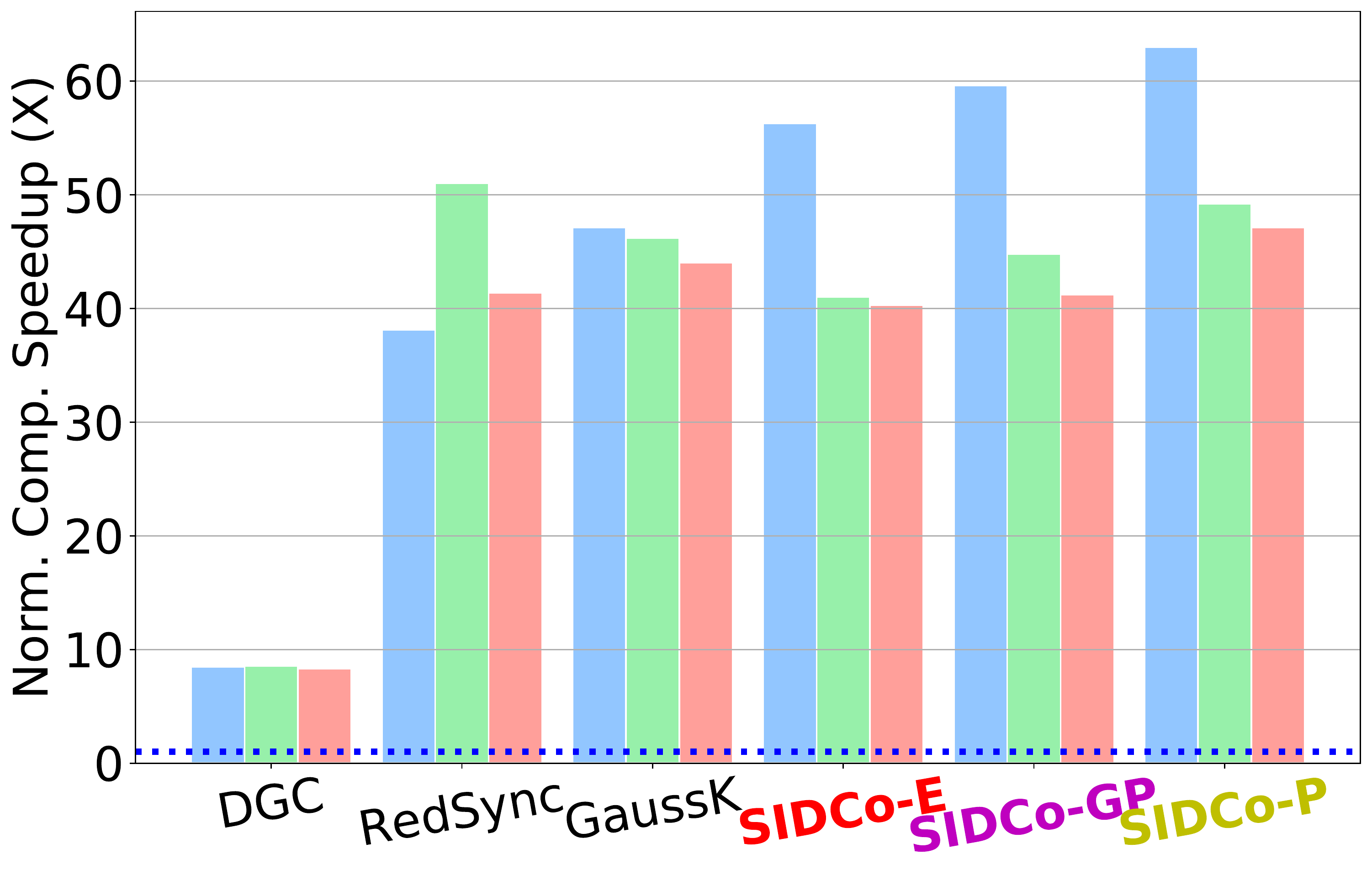}
      \caption{VGG16 on GPU}
        \label{fig:vgg16-cuda-speedupall}
     \end{subfigure}
        \hfill
  \begin{subfigure}{0.24\linewidth}
    	\includegraphics[width=1\textwidth]{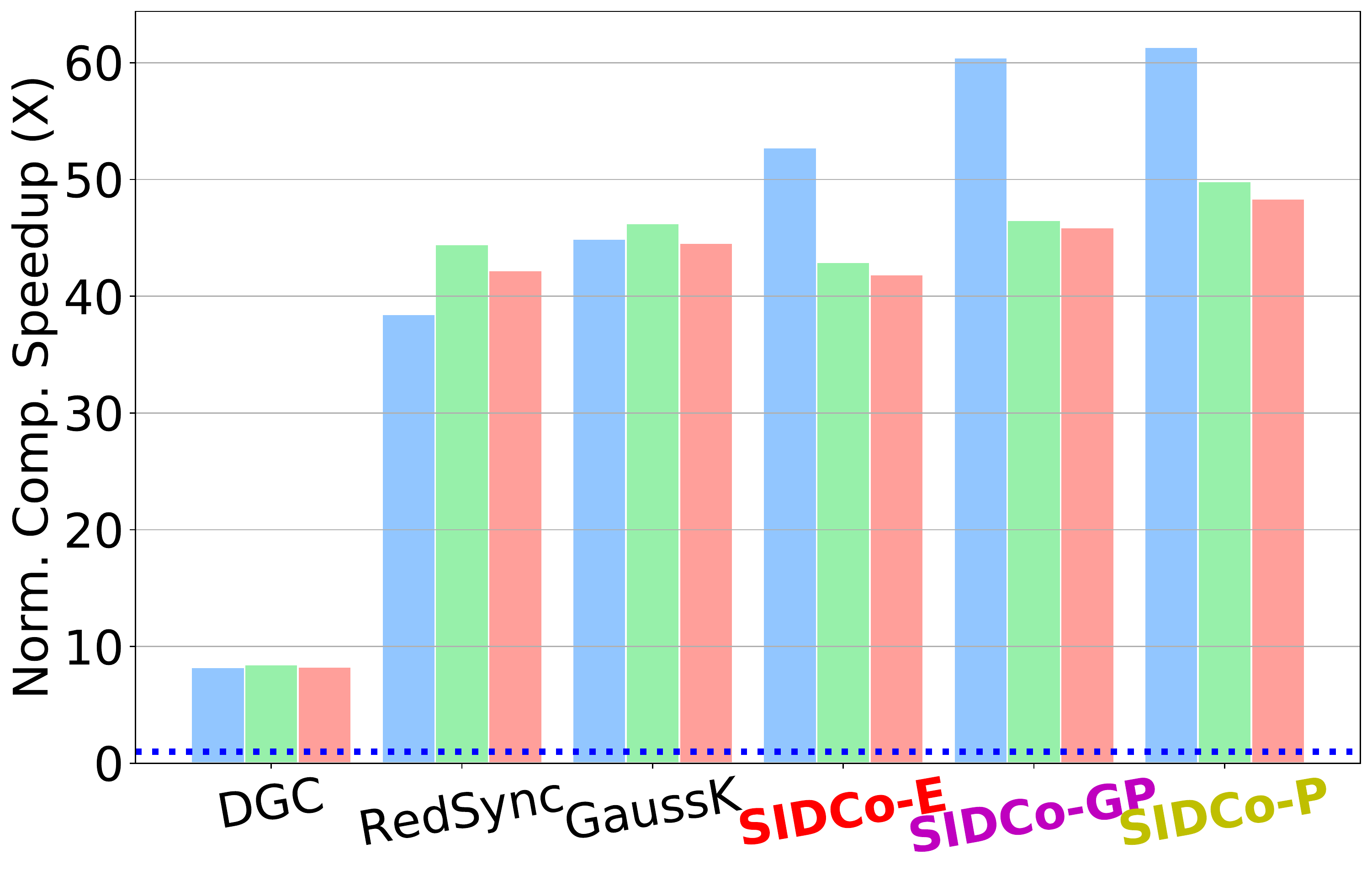}
    	\caption{ResNet50 on GPU}
     \label{fig:resnet50-cuda-speedupall}
    \end{subfigure}
    \hfill
    \begin{subfigure}{0.24\linewidth}
    \includegraphics[width=1\textwidth]{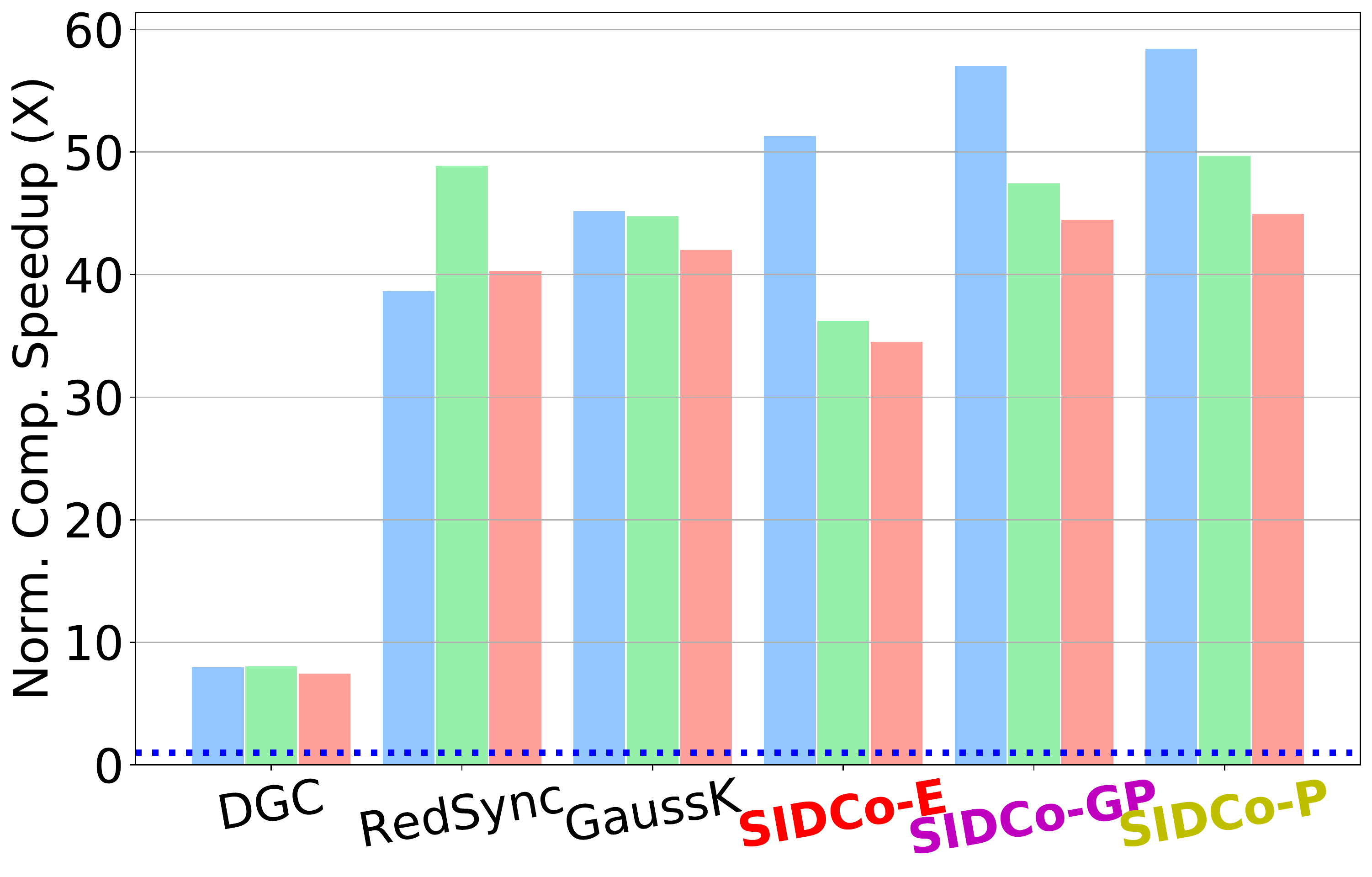}
    \caption{LSTM on GPU}
     \label{fig:lstm-cuda-speedupall}
    \end{subfigure}
    \\
      \begin{subfigure}{0.24\linewidth}
  \includegraphics[width=1\textwidth]{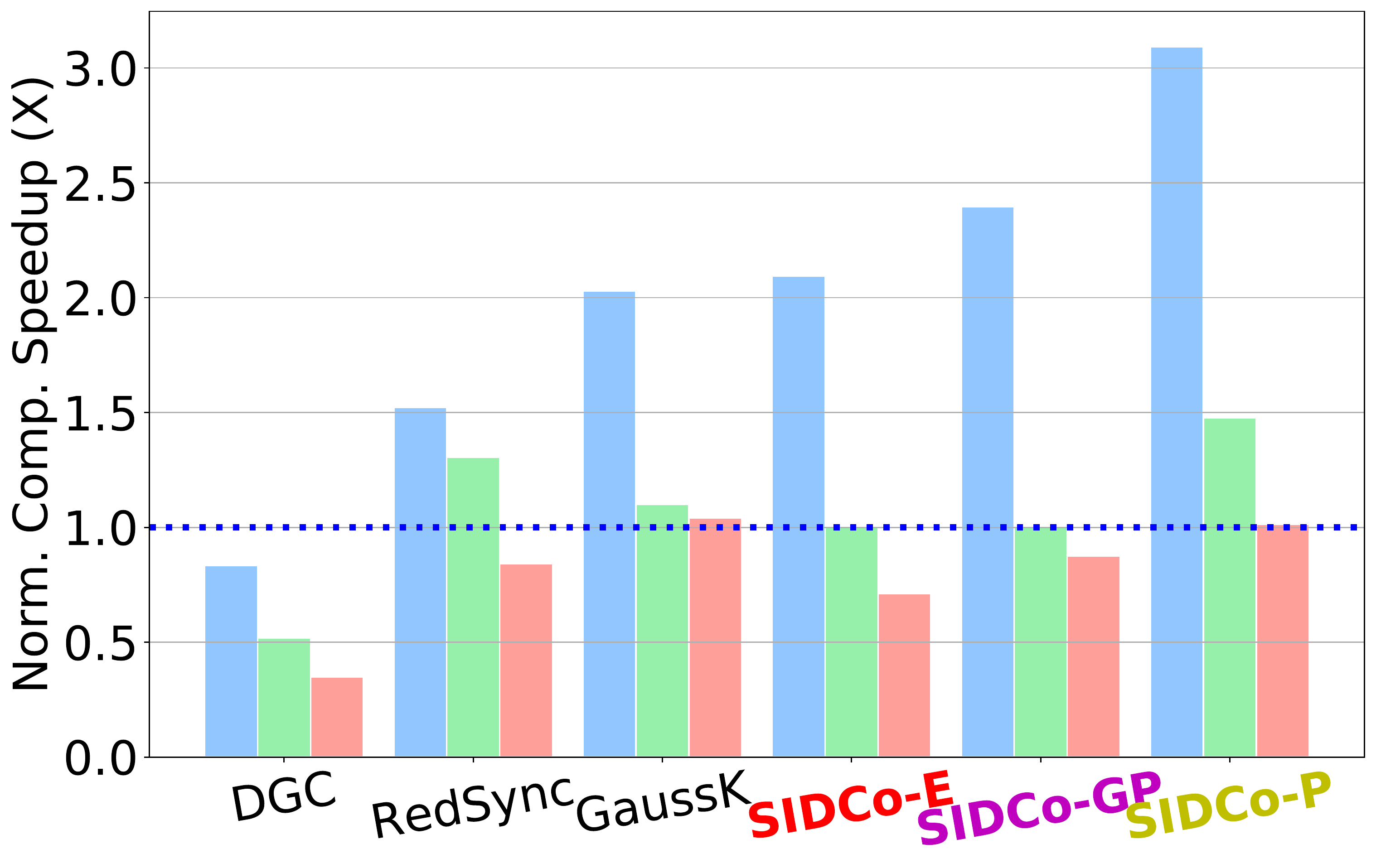}
      \caption{ResNet20 on CPU}
        \label{fig:resnet20-cpu-speedupall}
     \end{subfigure}
     \hfill
      \begin{subfigure}{0.24\linewidth}
  \includegraphics[width=1\textwidth]{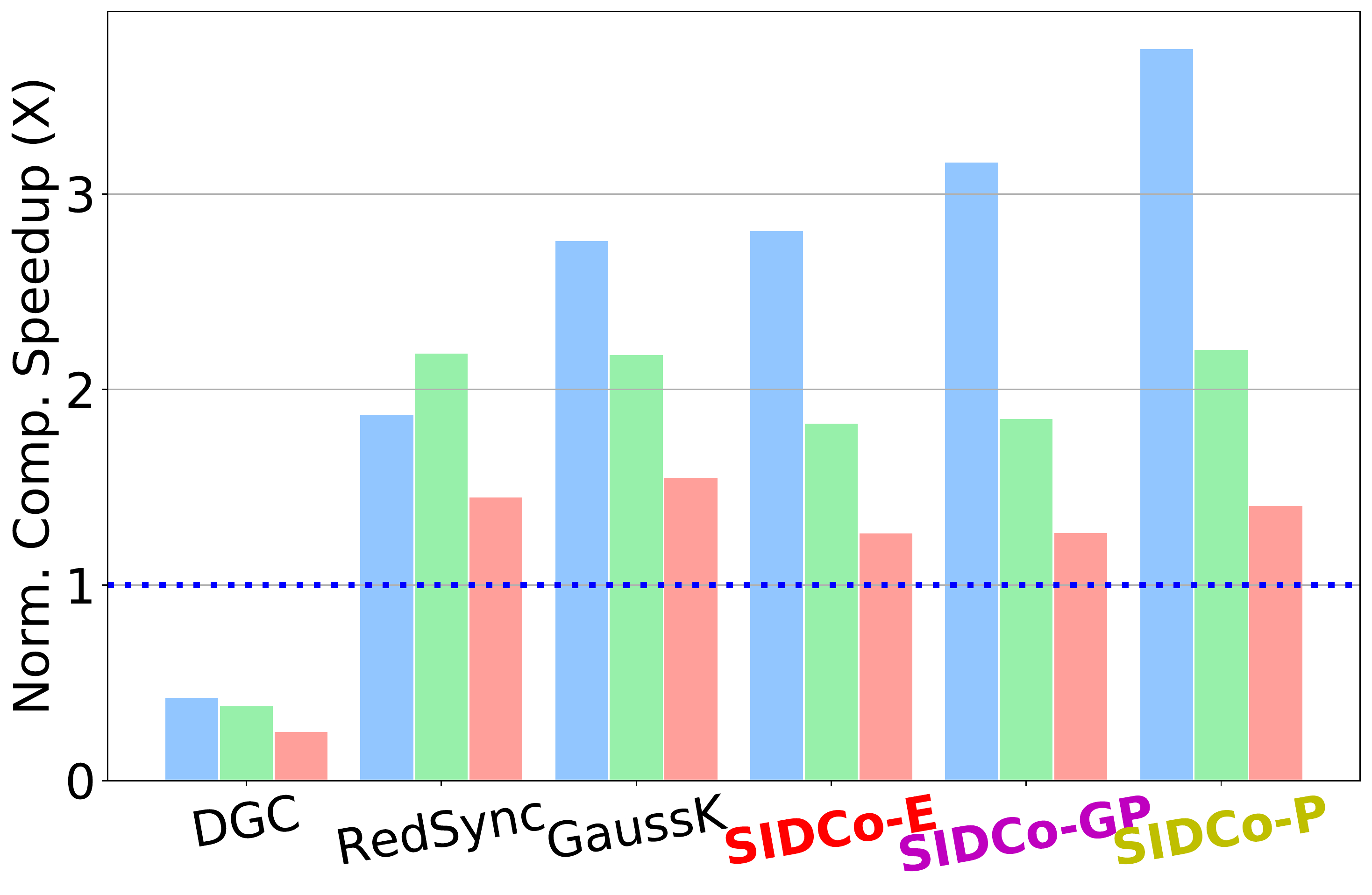}
      \caption{VGG16 on CPU}
        \label{fig:vgg16-cpu-speedupall}
     \end{subfigure}
        \hfill
  \begin{subfigure}{0.24\linewidth}
    	\includegraphics[width=1\textwidth]{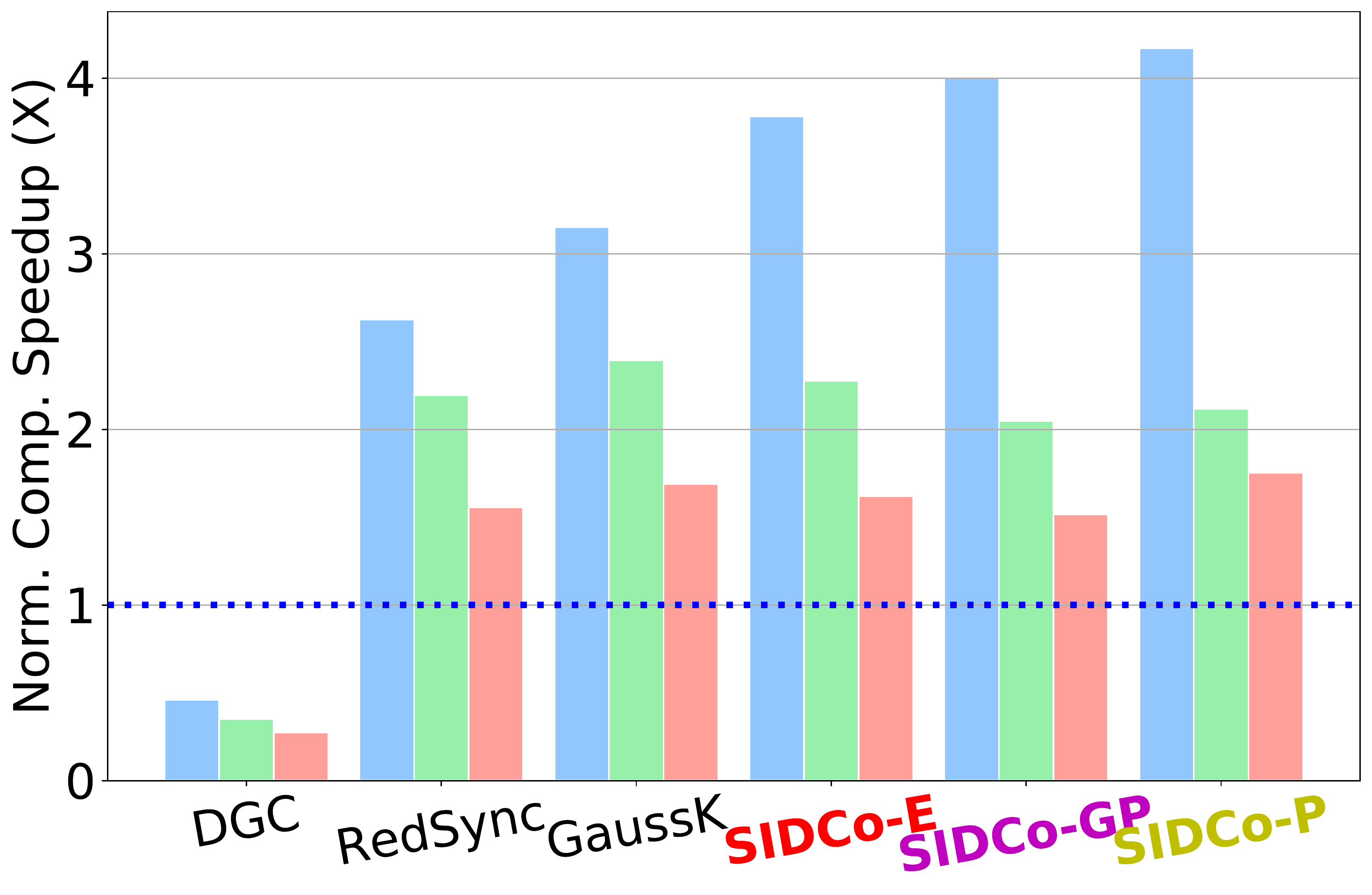}
    	\caption{ResNet50 on CPU}
     \label{fig:resnet50-cpu-speedupall}
    \end{subfigure}
    \hfill
     \begin{subfigure}{0.24\linewidth}
    \includegraphics[width=1\textwidth]{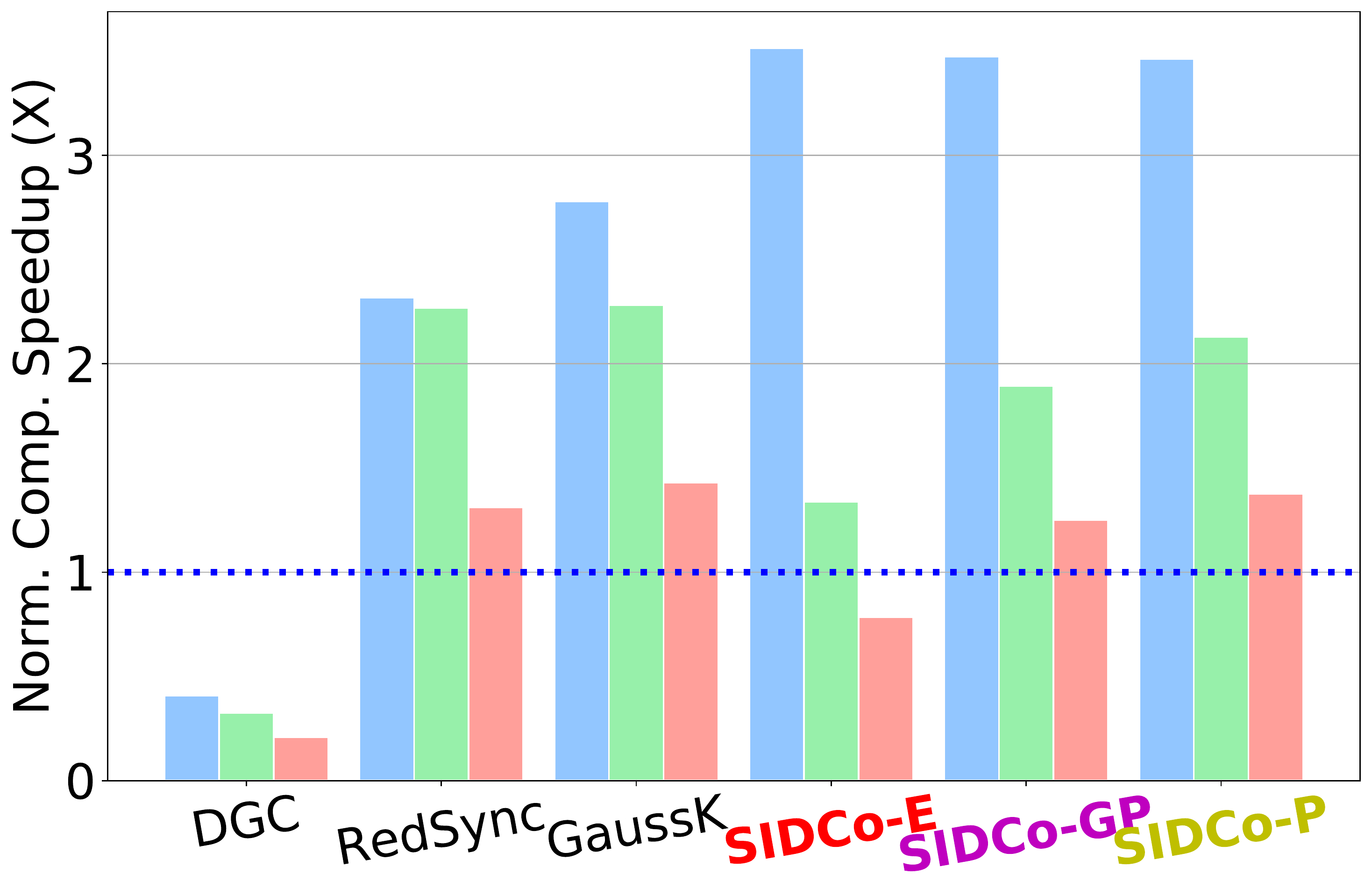}
    \caption{LSTM on CPU}
     \label{fig:lstm-cpu-speedupall}
    \end{subfigure}
    \caption{Compression speed-up over $\topk$ of compressing gradient vector of different models using various compressors and ratios on GPU (a,b,c,d) and CPU (e,f,g,h).}
    \label{fig:microbench-models-speedup}
\end{figure*}

 \begin{figure*}[!t]
  \centering
  \begin{subfigure}[ht]{0.37\linewidth}
  \includegraphics[width=1\linewidth]{Figures/experiments/legend2.pdf}
 \end{subfigure}
  \\
  \begin{subfigure}{0.235\linewidth}
    \includegraphics[width=1\textwidth]{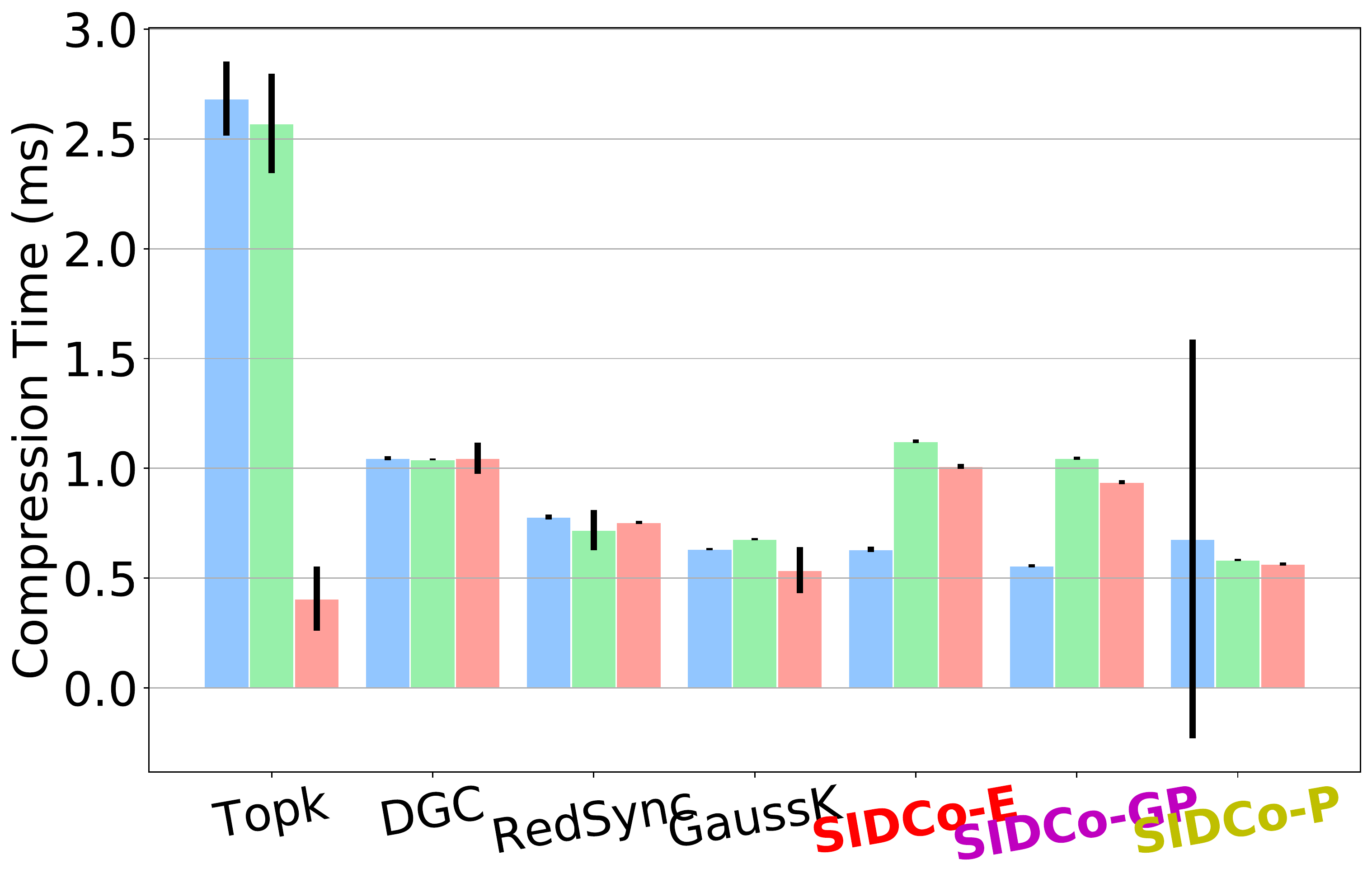}
    \caption{ResNet20 on GPU}
     \label{fig:resnet20-cuda-timeall}
    \end{subfigure}
    \hfill
      \begin{subfigure}{0.235\linewidth}
  \includegraphics[width=1\textwidth]{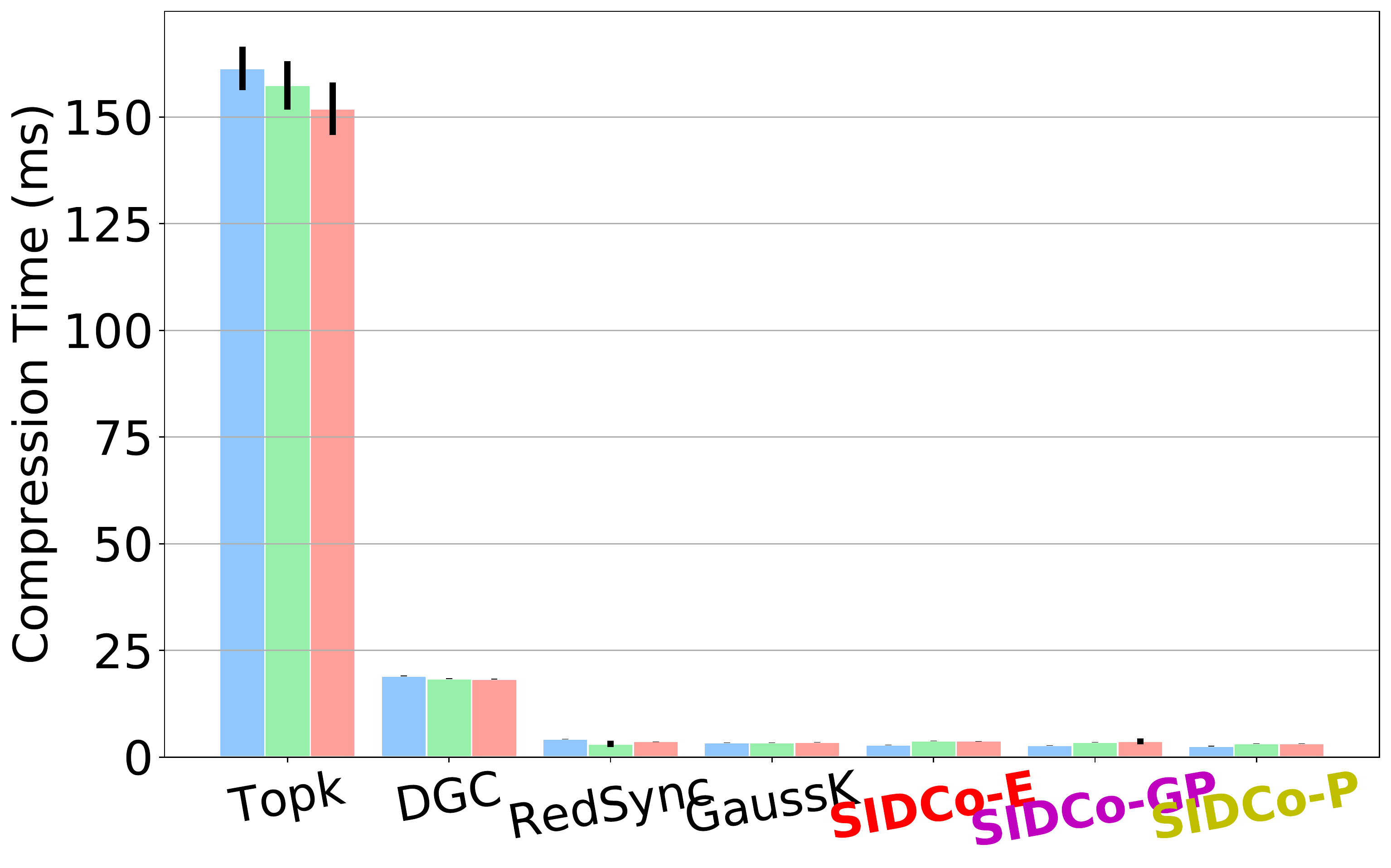}
      \caption{VGG16 on GPU}
        \label{fig:vgg16-cuda-timeall}
     \end{subfigure}
        \hfill
  \begin{subfigure}{0.235\linewidth}
    	\includegraphics[width=1\textwidth]{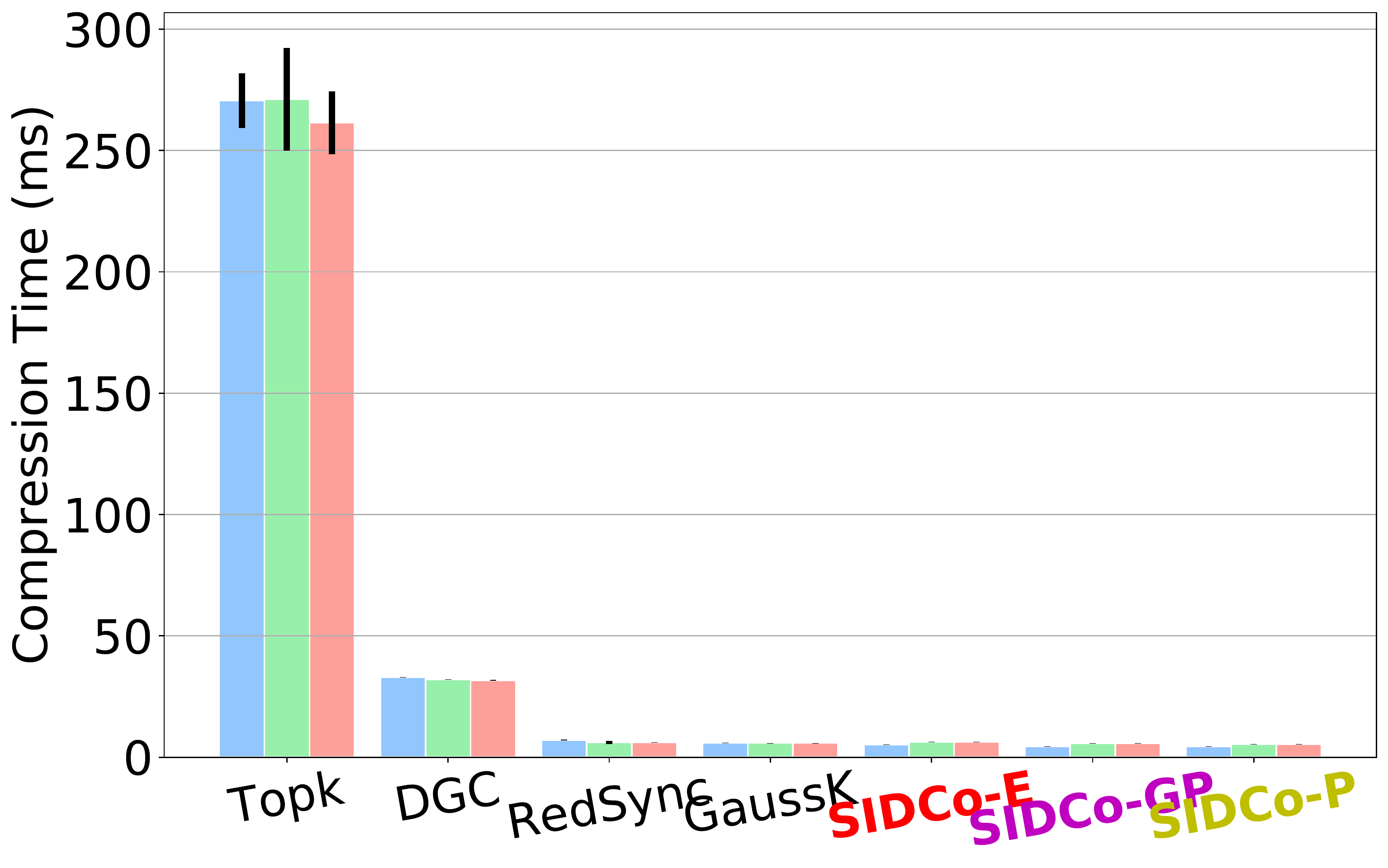}
    	\caption{ResNet50 on GPU}
     \label{fig:resnet50-cuda-timeall}
    \end{subfigure}
    \hfill
      \begin{subfigure}{0.235\linewidth}
    	\includegraphics[width=1\textwidth]{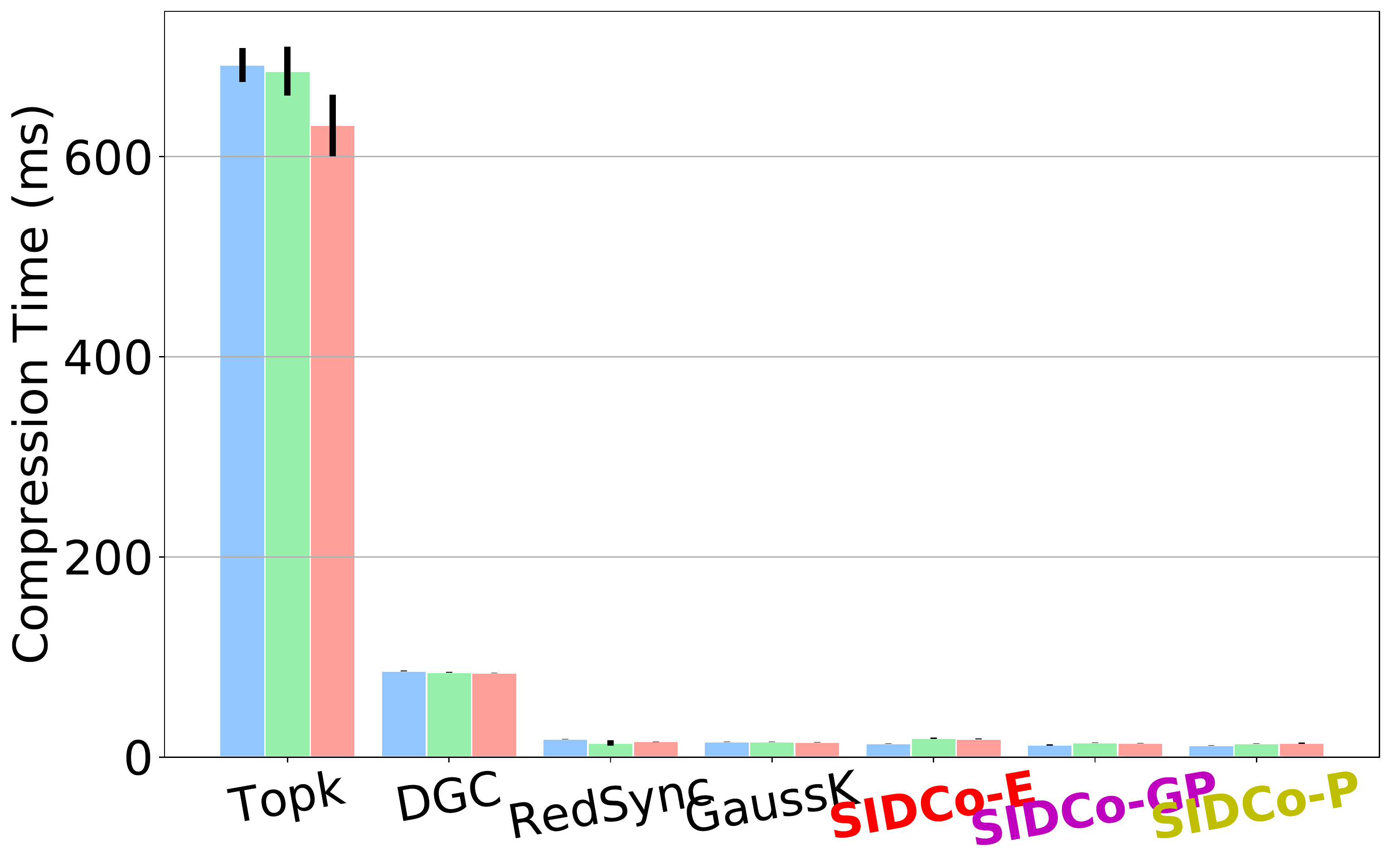}
    	\caption{LSTM on GPU}
     \label{fig:lstm-cuda-timeall}
    \end{subfigure}
    \\
    \begin{subfigure}{0.235\linewidth}
    \includegraphics[width=1\textwidth]{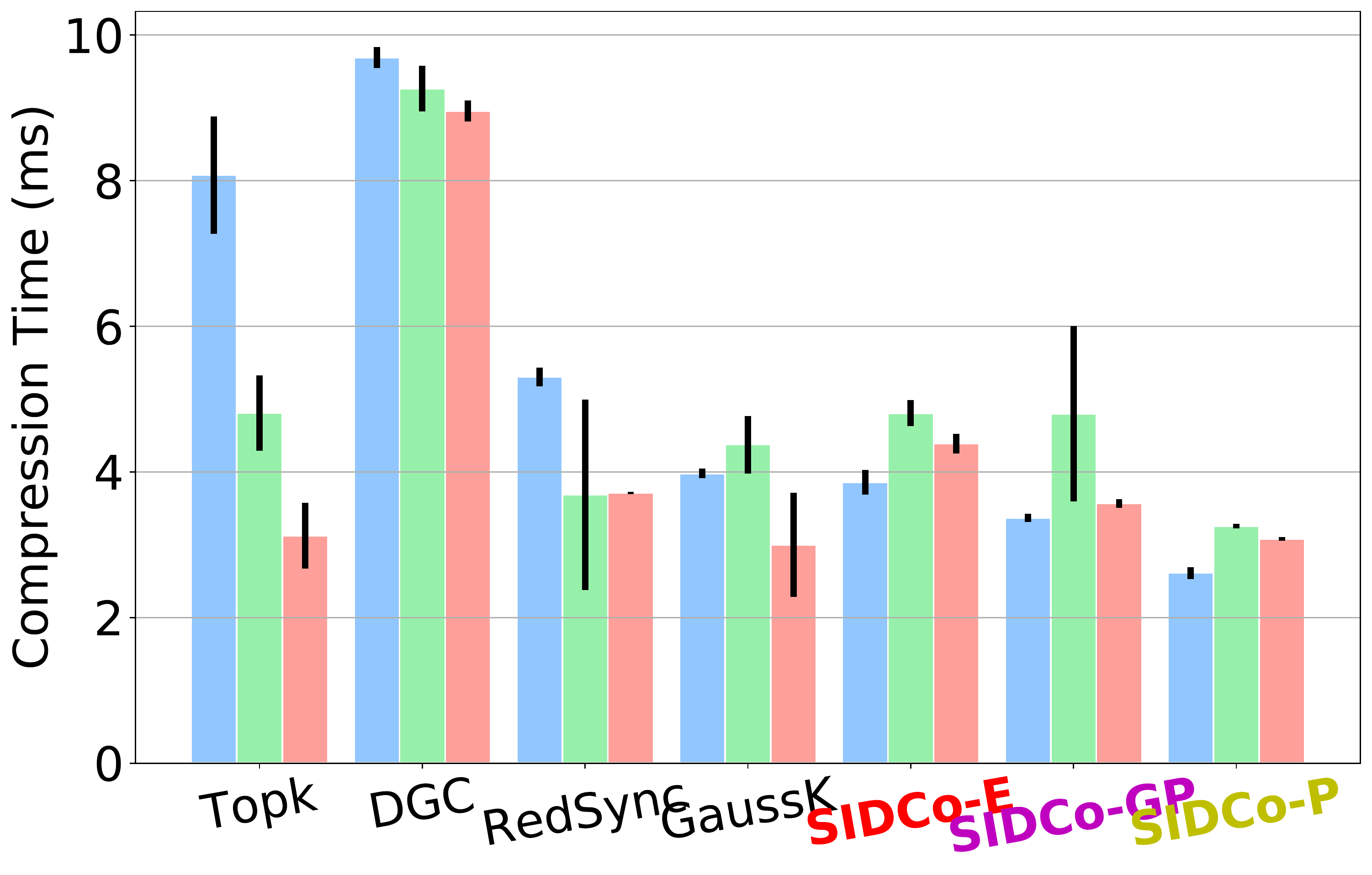}
    \caption{ResNet20 on CPU}
     \label{fig:resnet20-cpu-timeall}
    \end{subfigure}
    \hfill
      \begin{subfigure}{0.235\linewidth}
  \includegraphics[width=1\textwidth]{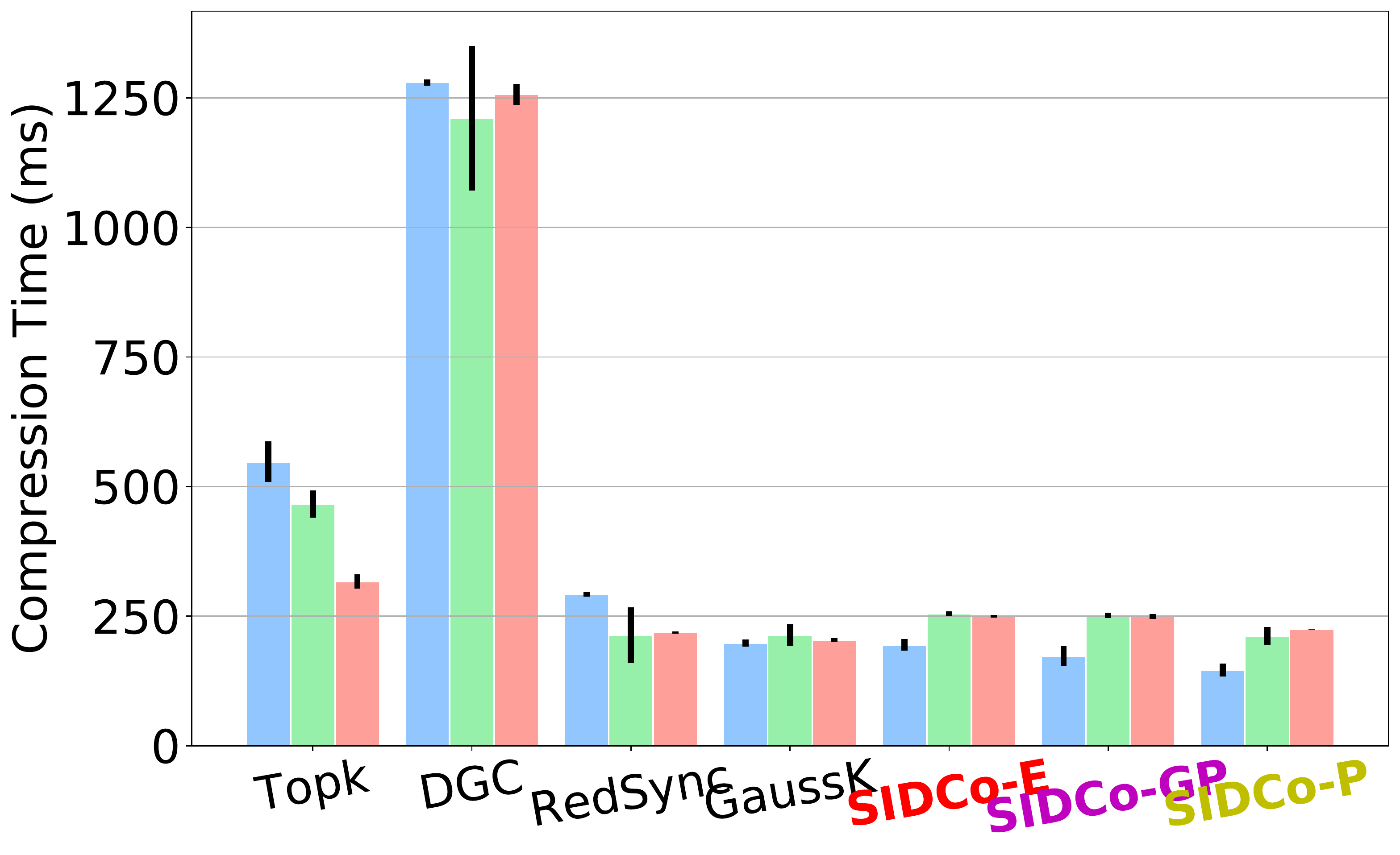}
      \caption{VGG16 on CPU}
        \label{fig:vgg16-cpu-timeall}
     \end{subfigure}
        \hfill
  \begin{subfigure}{0.235\linewidth}
    	\includegraphics[width=1\textwidth]{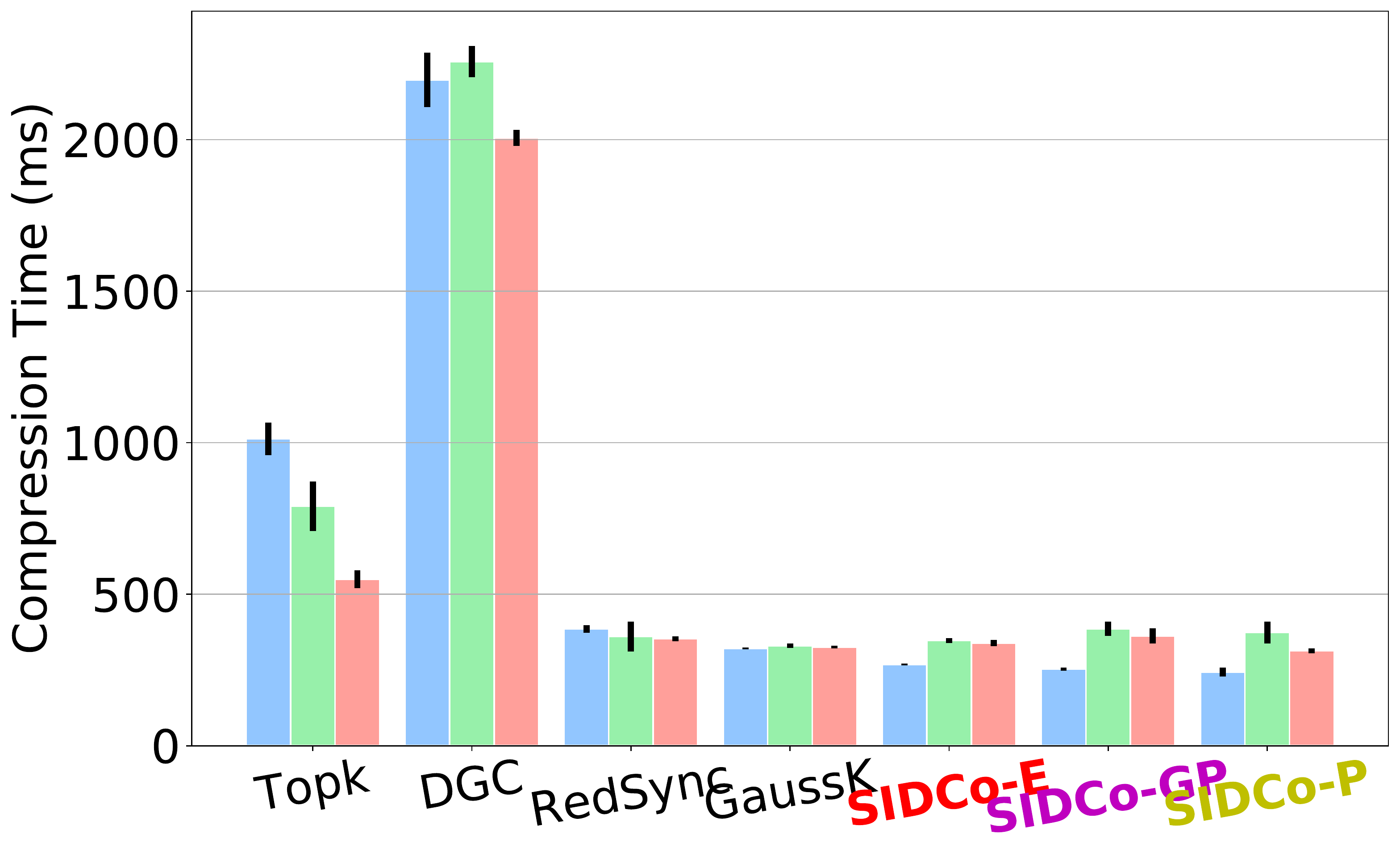}
    	\caption{ResNet50 on CPU}
     \label{fig:resnet50-cpu-timeall}
    \end{subfigure}
    \hfill
      \begin{subfigure}{0.235\linewidth}
    	\includegraphics[width=1\textwidth]{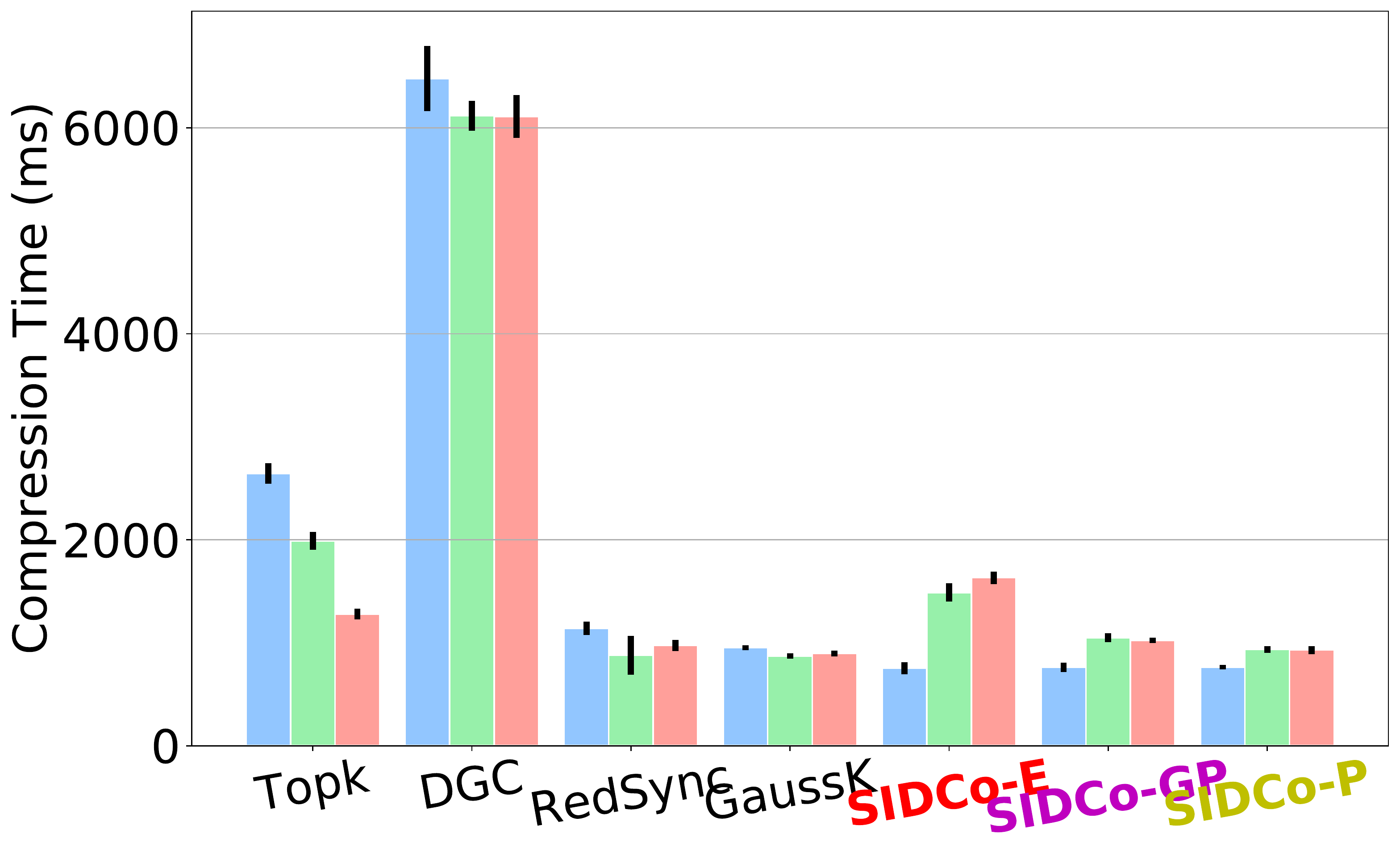}
    	\caption{LSTM on CPU}
     \label{fig:lstm-cpu-timeall}
    \end{subfigure}
    \caption{Compression latency of different models using various compressors and ratios on GPU (a,b,c,d) and CPU (e,f,g,h).}
    \label{fig:microbench-models-time}
\end{figure*}

\subsection{Compression Complexity using Synthetic Gradients Vectors of Different Sizes}
Here, we run the micro-benchmark using synthetic gradient vectors initialized based on input size of (0.26, 2.6, 26, 260) Million elements which is equivalent to $\approx\!(1, 11, 114, 1140)$ MBytes of gradient data sent in each iteration, respectively. We aim to measure the performance of each compressor in terms of the speed-up over $\topk$ and latency for wide range of gradient sizes. The results match the former observations on DNN models of different sizes. In particular, \cref{fig:synth-microbench-speedup} shows the speed-up over $\topk$ on GPU and CPU for each size of the synthetic gradient vectors. We again can observe that on GPU, all methods are faster than $\topk$ and all threshold estimation methods achieve higher speed-ups over DGC and nearly same speed-ups among each other which is attributed to the slow performance of $\topk$ (or sorting) operations on GPU. On the CPU, in contrary, we observe that DGC is the slowest method and $\topk$ excels over it which is attributed to slow performance of random sampling on CPU. Threshold estimation methods maintains same speed-ups on both GPU and CPU (but with relatively different compression times on CPU and GPU).
 \begin{figure*}[!t]
  \centering
  \begin{subfigure}{0.24\linewidth}
    \includegraphics[width=1\textwidth]{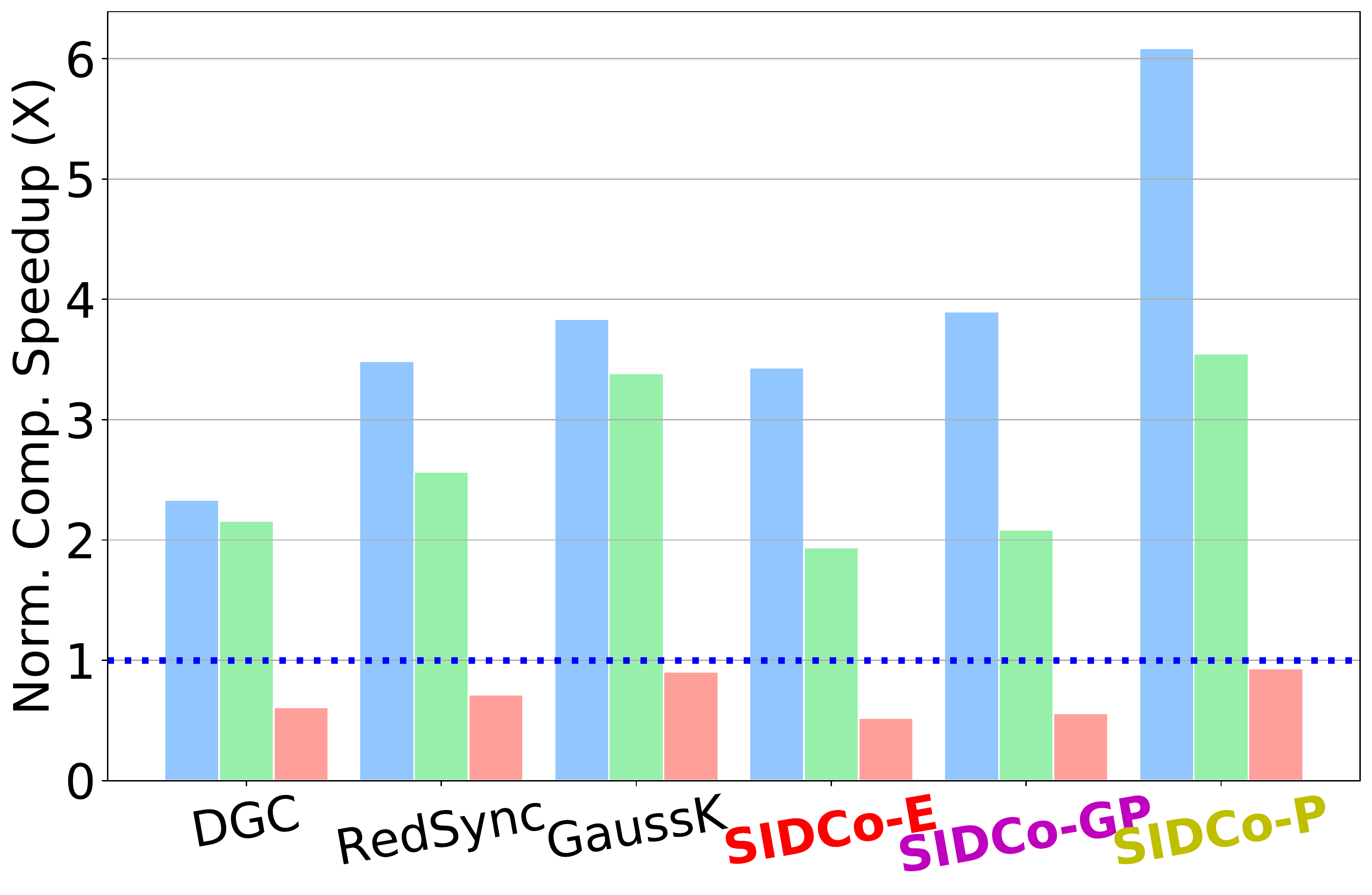}
    \caption{0.26 Mil Elem Tensor on GPU}
     \label{fig:randn0.26-cuda-speedup}
    \end{subfigure}
    \hfill
      \begin{subfigure}{0.24\linewidth}
   \includegraphics[width=1\textwidth]{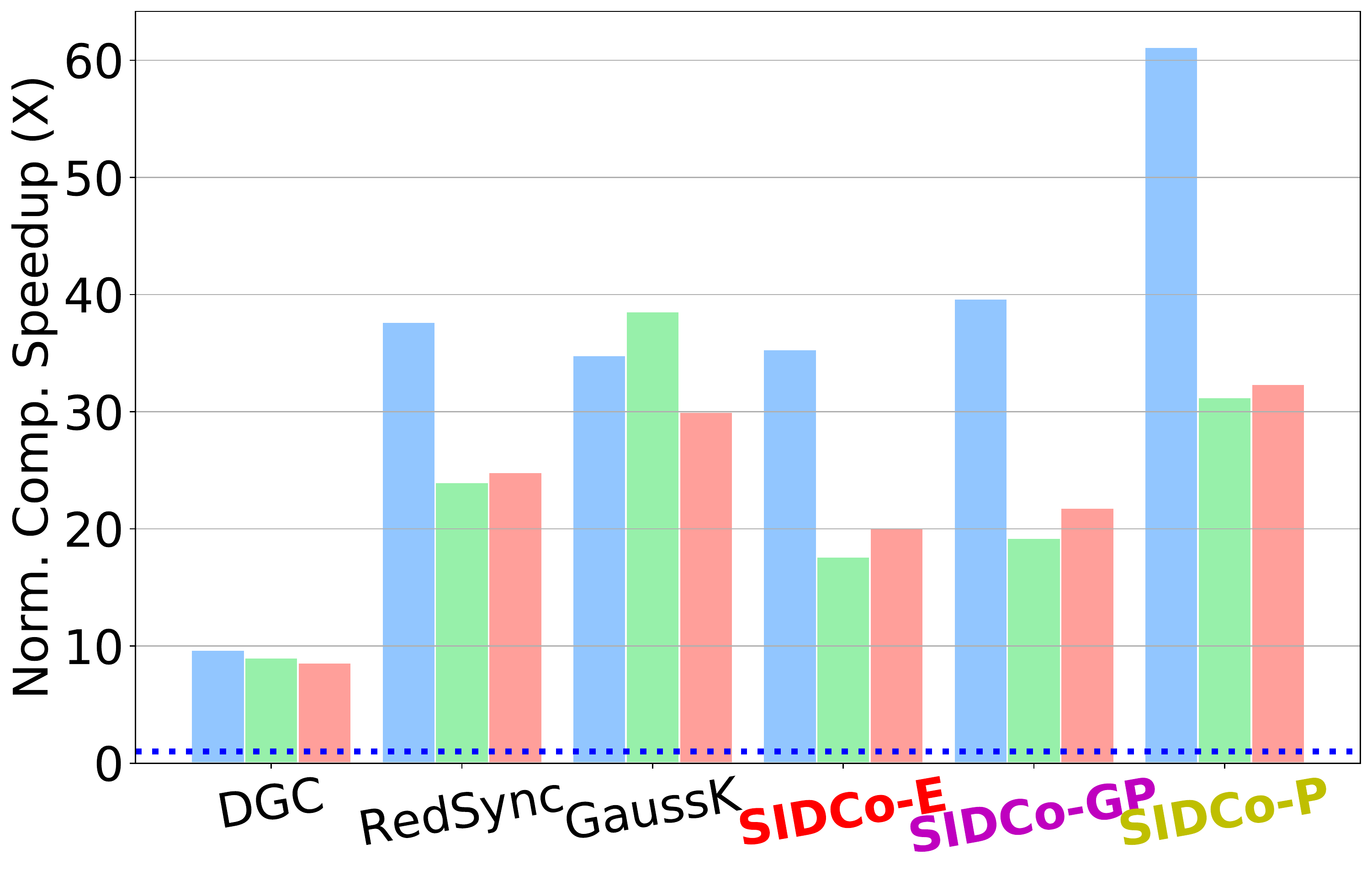}
       \caption{2.6 Mil Elem Tensor on GPU}
        \label{fig:rand2.6-cuda-speedup}
     \end{subfigure}
        \hfill
   \begin{subfigure}{0.24\linewidth}
    	\includegraphics[width=1\textwidth]{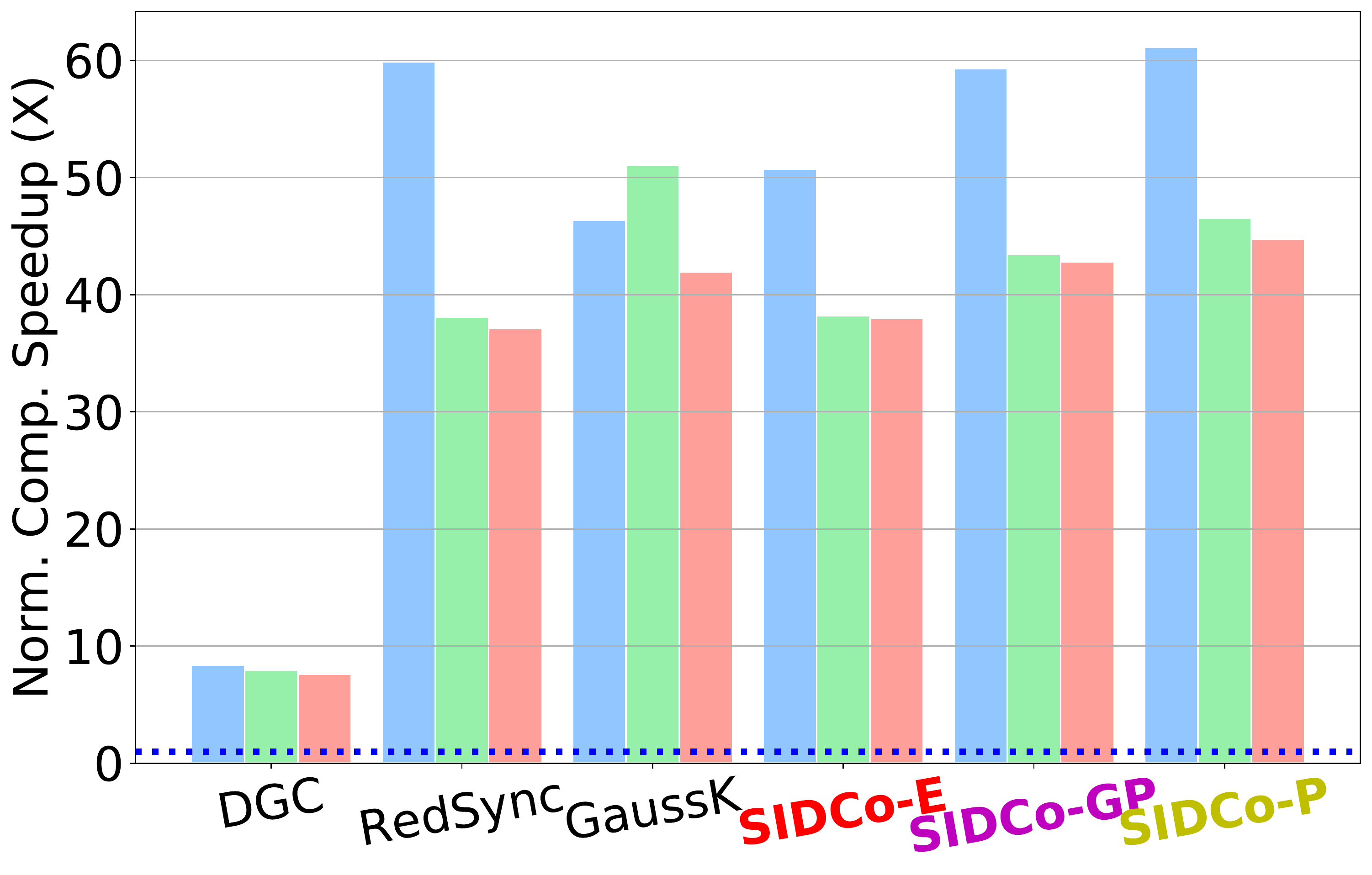}
    	\caption{26 Mil Elem Tensor on GPU}
     \label{fig:rand26M-cuda-speedup}
    \end{subfigure}
    \hfill
      \begin{subfigure}{0.24\linewidth}
    	\includegraphics[width=1\textwidth]{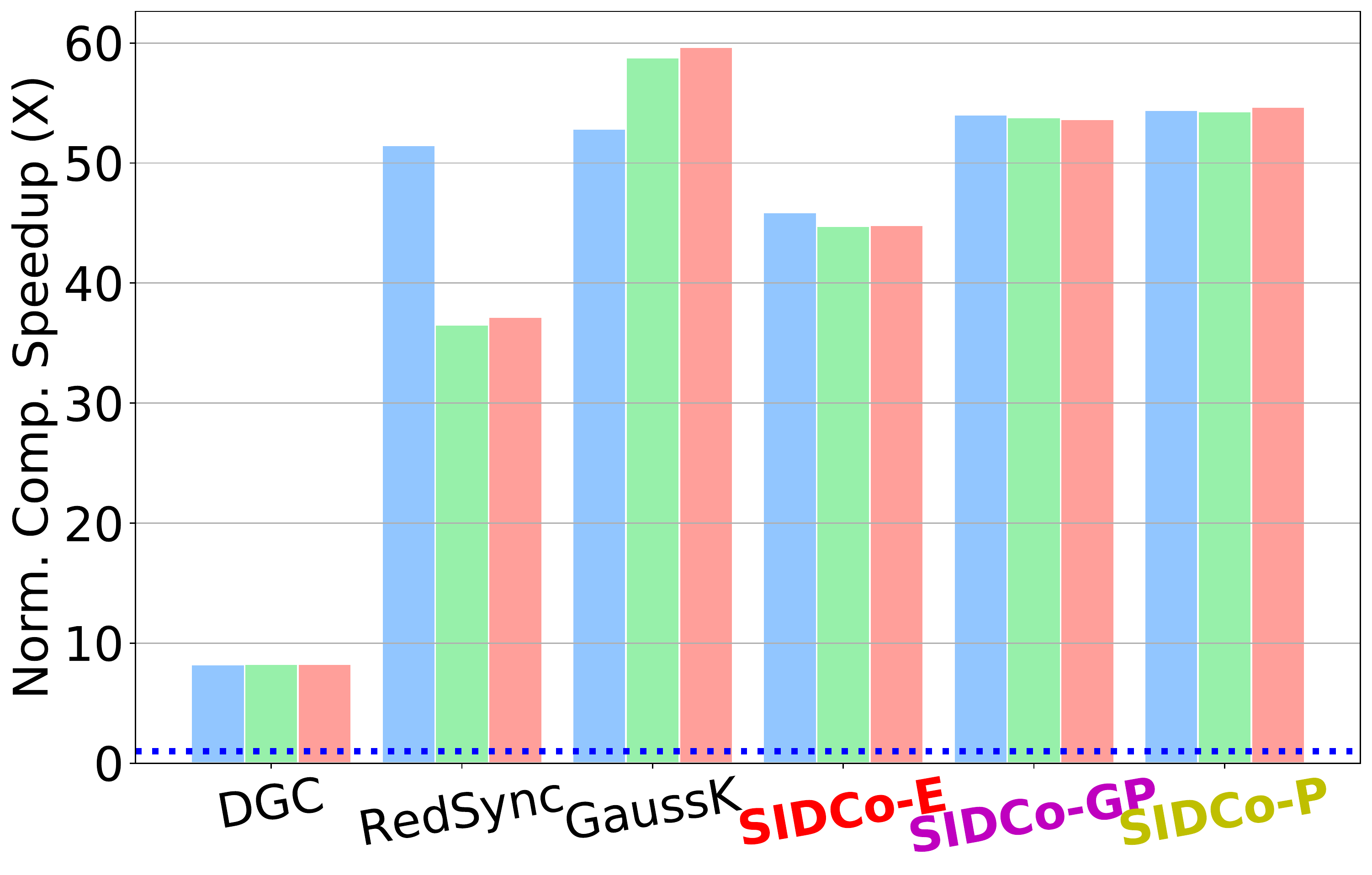}
    	\caption{260 Mil Elem Tensor on GPU}
     \label{fig:rand260M-cuda-speedup}
    \end{subfigure}
    \\
      \begin{subfigure}{0.24\linewidth}
    \includegraphics[width=1\textwidth]{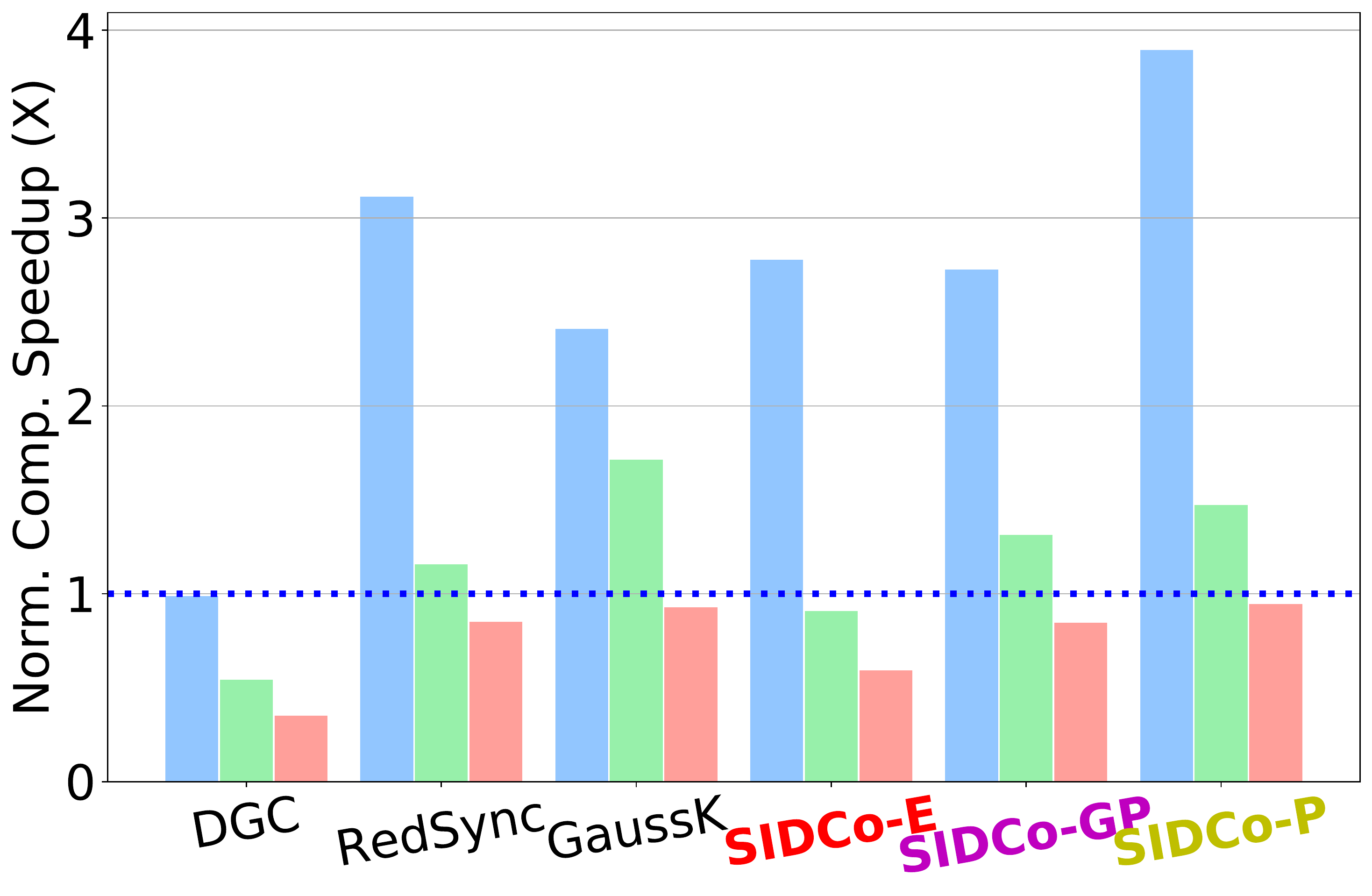}
    
    \caption{0.26 Mil Elem Tensor on CPU}
     \label{fig:randn0.26-cpu-speedup}
    \end{subfigure}
    \hfill
      \begin{subfigure}{0.24\linewidth}
   \includegraphics[width=1\textwidth]{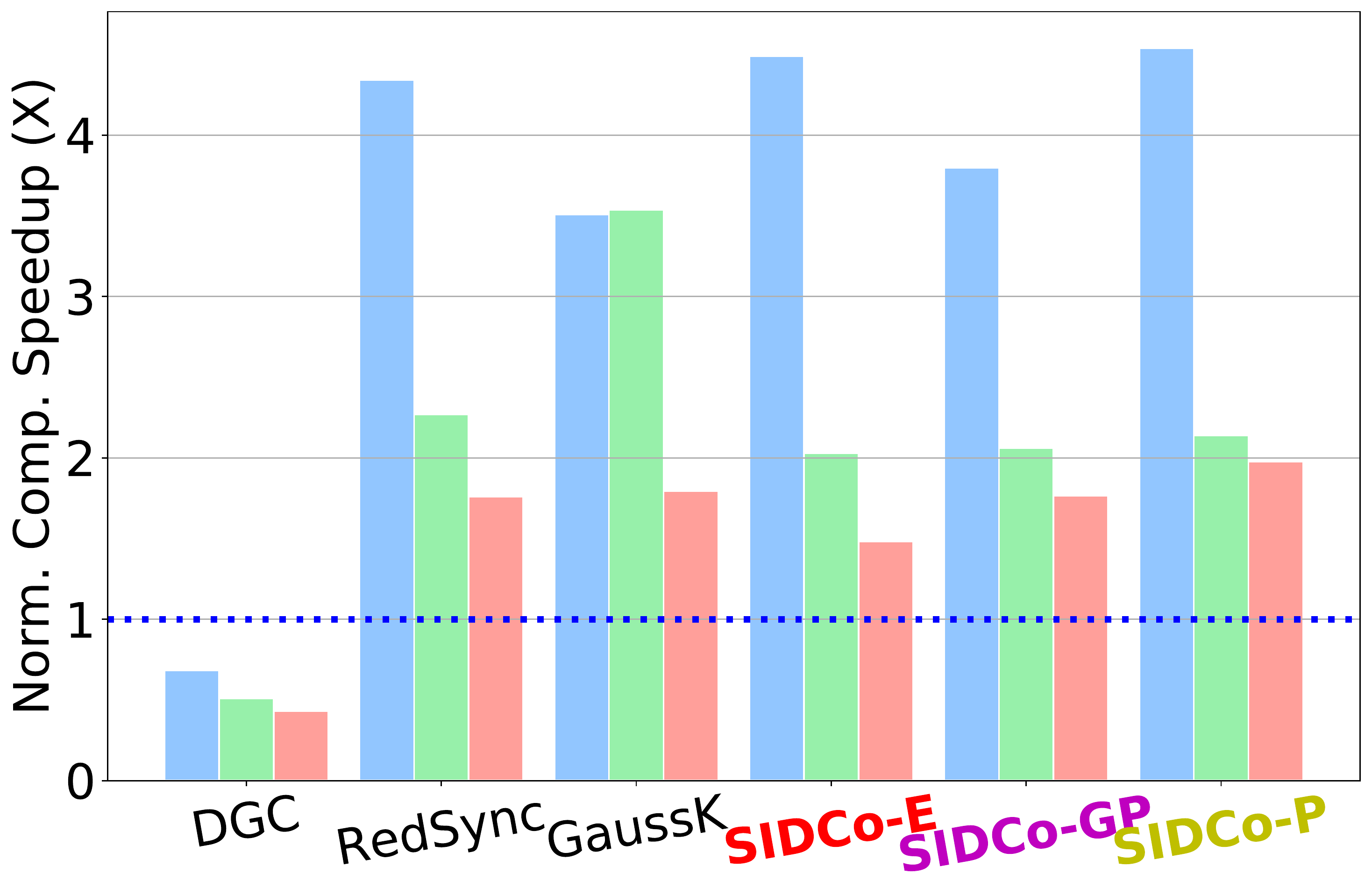}
       \caption{2.6 Mil Elem Tensor on CPU}
        \label{fig:rand2.6-cpu-speedup}
     \end{subfigure}
        \hfill
    \hfill
  \begin{subfigure}{0.24\linewidth}
	\includegraphics[width=1\textwidth]{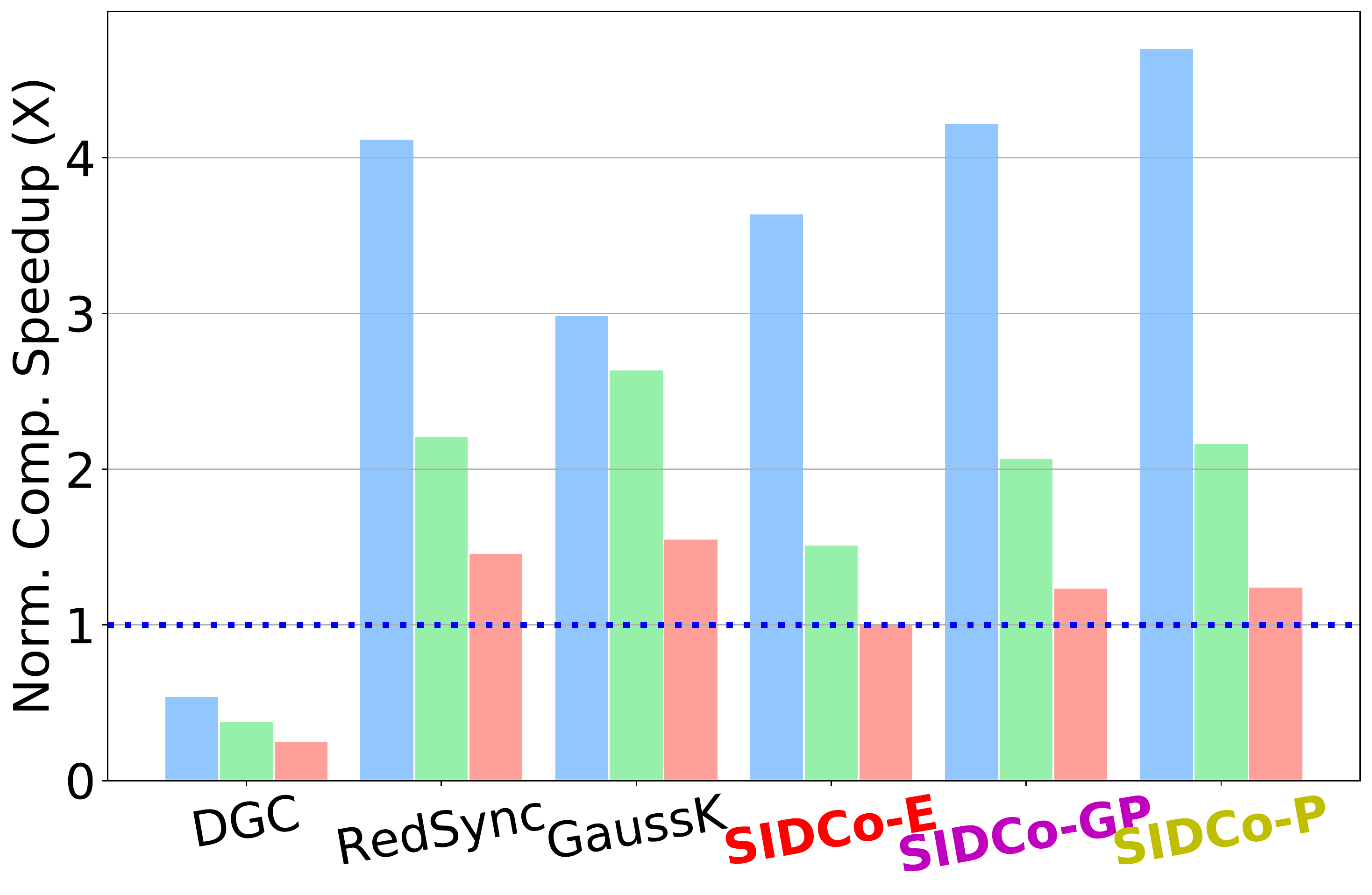}
    	\caption{26 Mil Elem Tensor on CPU}
     \label{fig:rand26M-cpu-speedup}
    \end{subfigure}
    \hfill
      \begin{subfigure}{0.24\linewidth}
    	\includegraphics[width=1\textwidth]{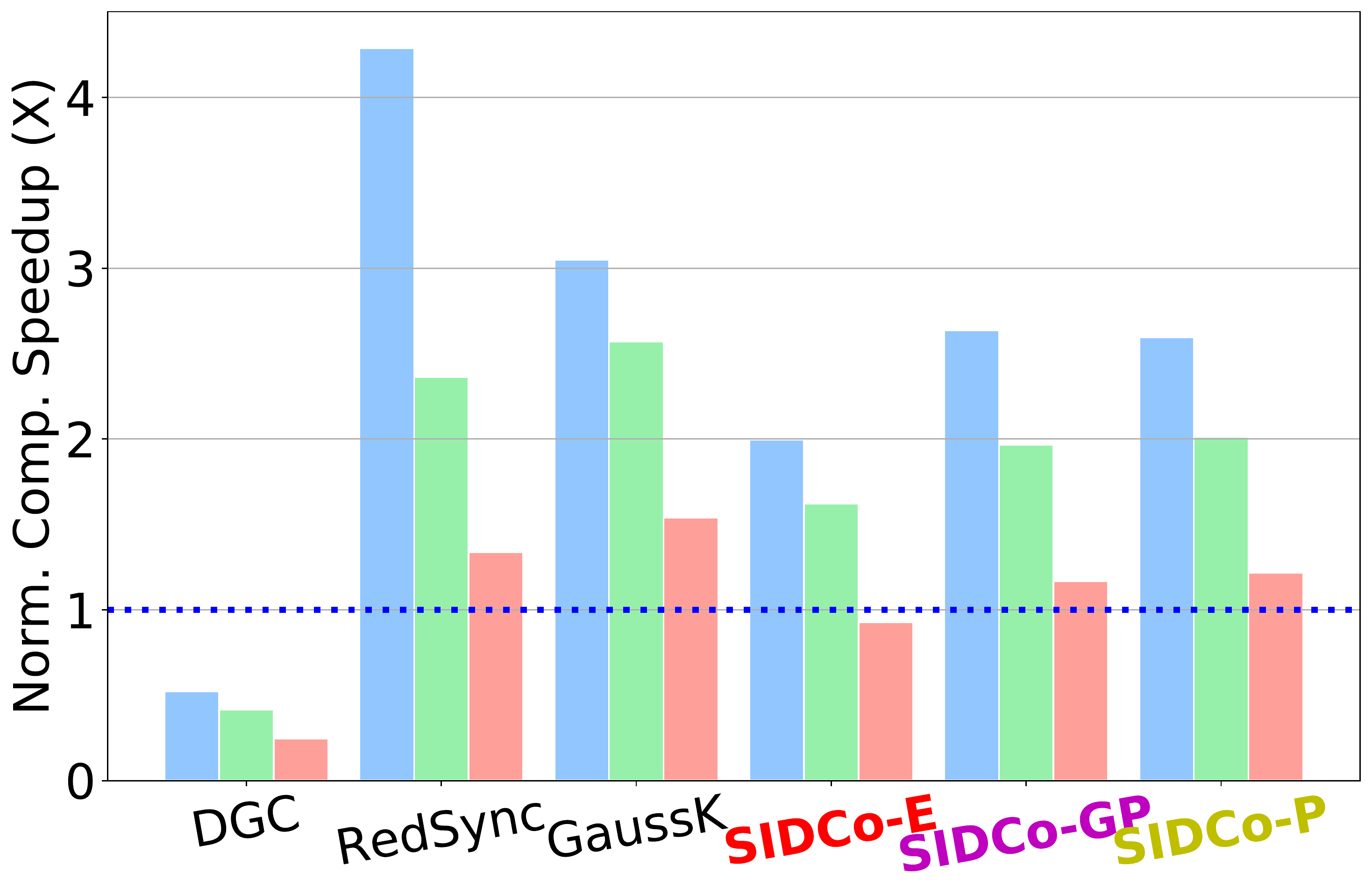}
    	\caption{260 Mil Elem Tensor on CPU}
     \label{fig:rand260M-cpu-speedup}
    \end{subfigure}
    
    \caption{Compression speedups over $\topk$ of synthetic tensors using various compressors and ratios on GPU (a,b,c,d) and CPU (e,f,g,h).}
    \label{fig:synth-microbench-speedup}
\end{figure*}

 \begin{figure*}[!t]
  \centering
  \begin{subfigure}{0.24\linewidth}
    \includegraphics[width=1\textwidth]{Figures/experiments/all-microbench/randnormal/allratio/randn_262144_cuda_compression_microbench_speedup.pdf}
    \caption{0.26 Mil Elem Tensor on GPU}
     \label{fig:randn0.26-cuda-time}
    \end{subfigure}
    \hfill
      \begin{subfigure}{0.24\linewidth}
   \includegraphics[width=1\textwidth]{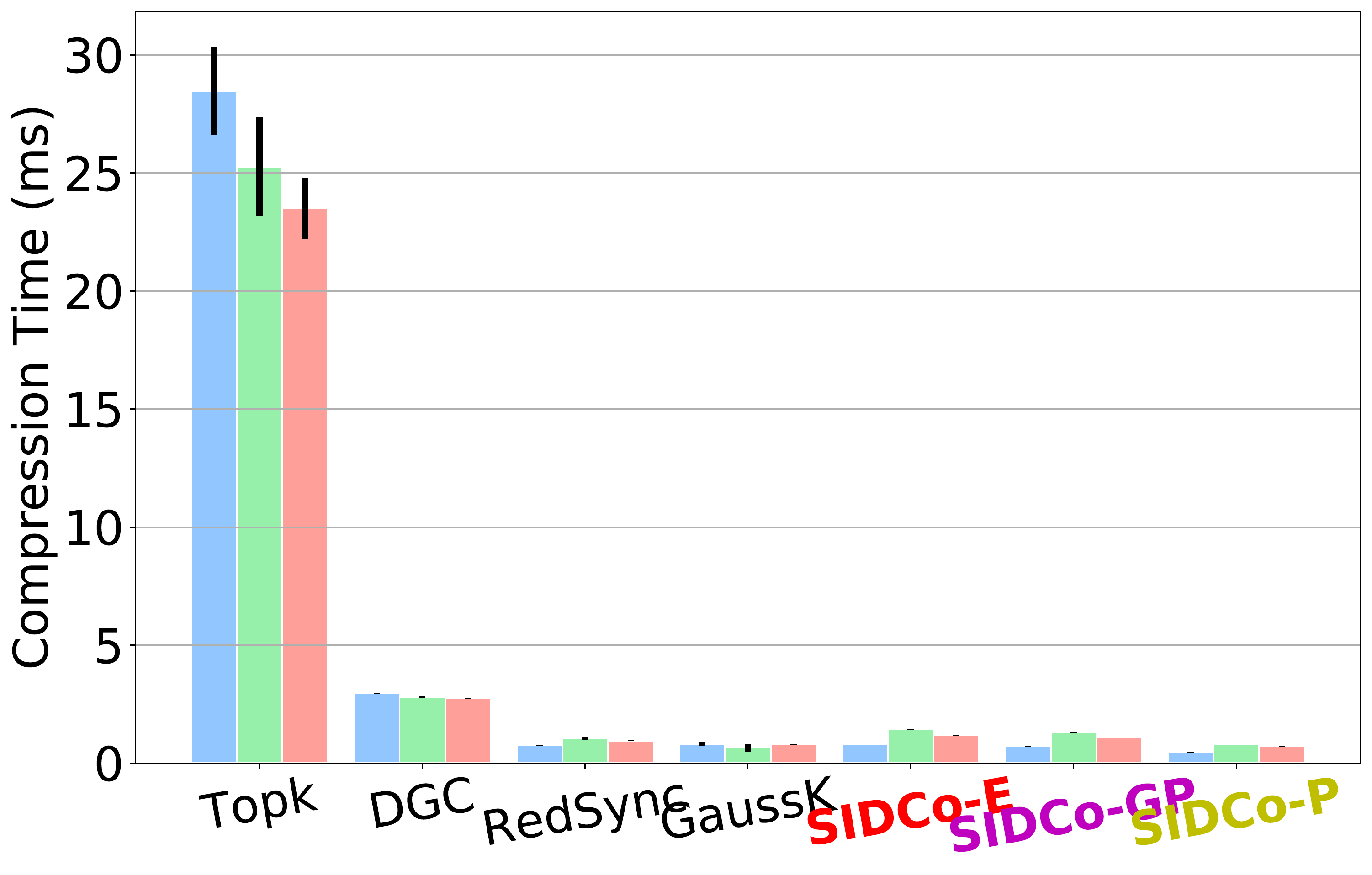}
       \caption{2.6 Mil Elem Tensor on GPU}
        \label{fig:rand2.6-cuda-time}
     \end{subfigure}
        \hfill
   \begin{subfigure}{0.24\linewidth}
    	\includegraphics[width=1\textwidth]{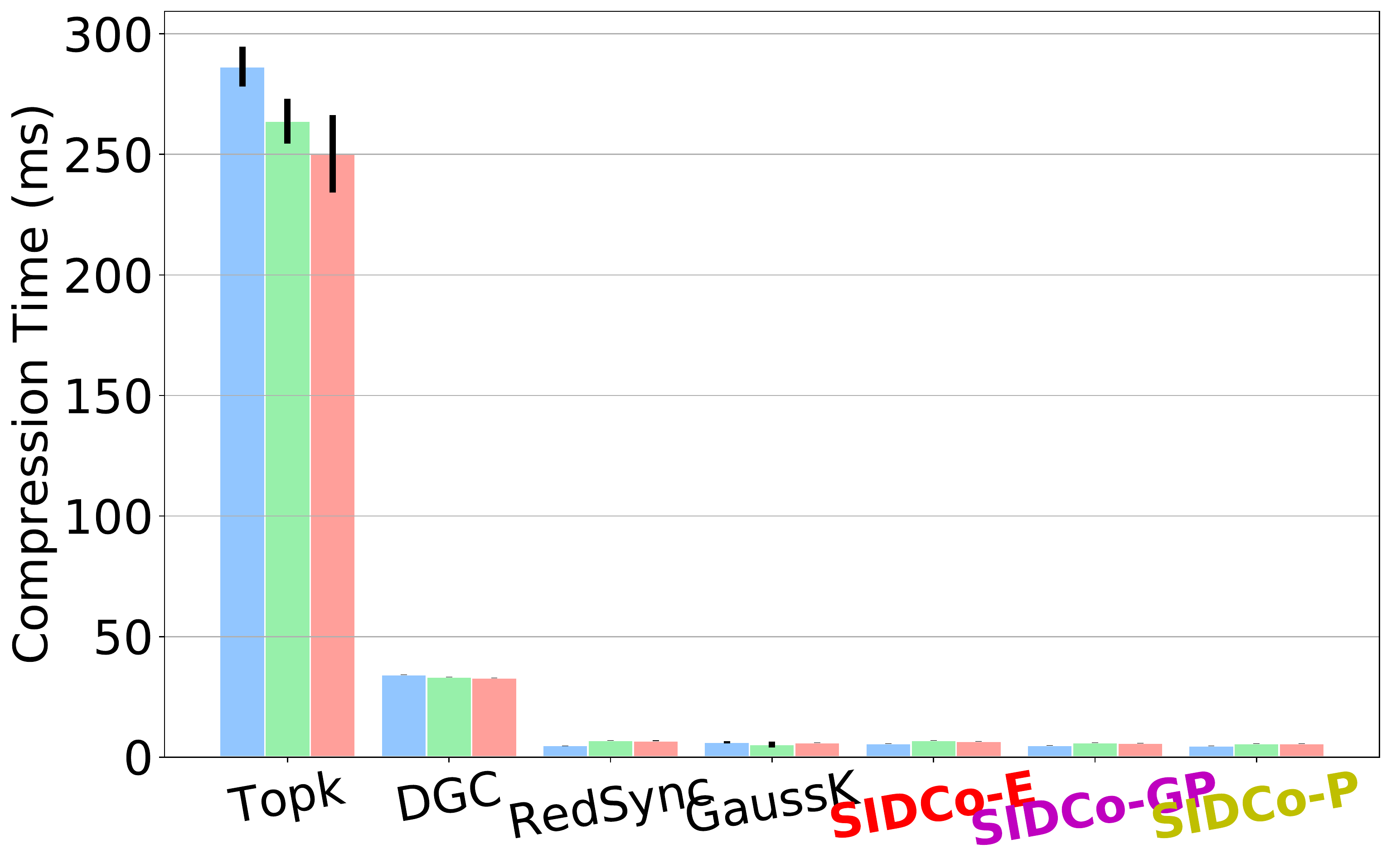}
    	\caption{26 Mil Elem Tensor on GPU}
     \label{fig:rand26M-cuda-time}
    \end{subfigure}
    \hfill
      \begin{subfigure}{0.24\linewidth}
    	\includegraphics[width=1\textwidth]{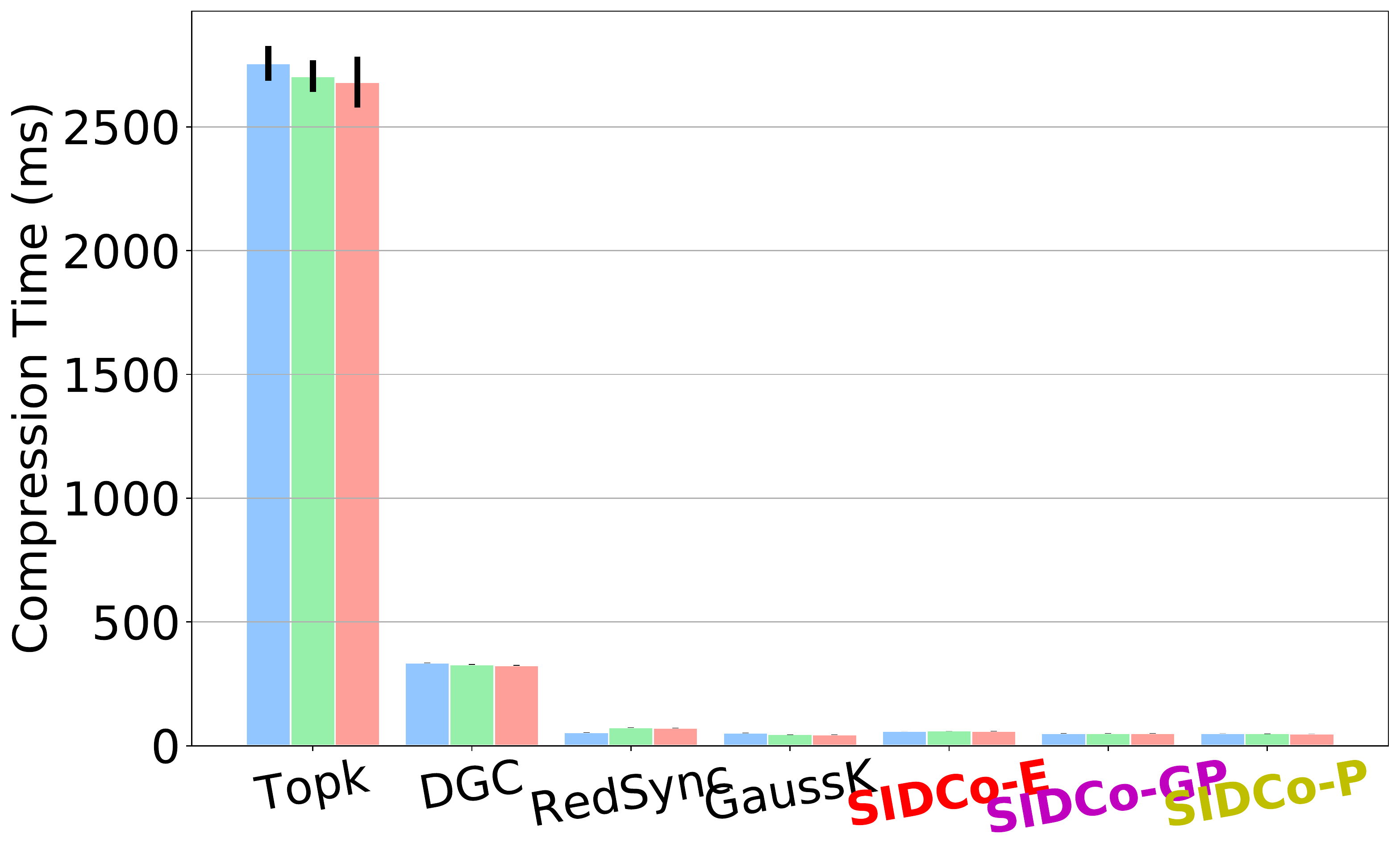}
    	\caption{260 Mil Elem Tensor on GPU}
     \label{fig:rand260M-cuda-time}
    \end{subfigure}
    \\
      \begin{subfigure}{0.24\linewidth}
    \includegraphics[width=1\textwidth]{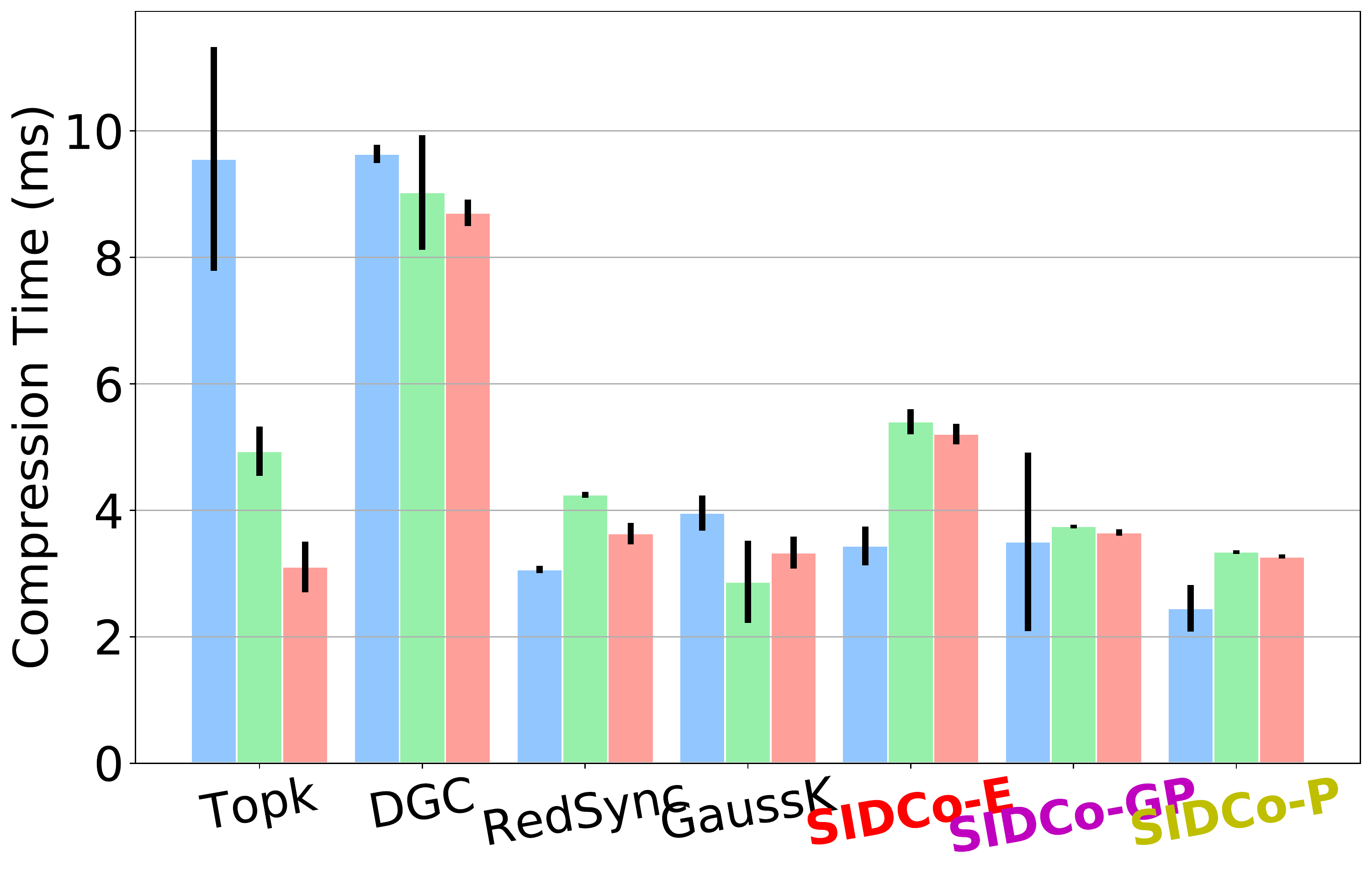}
    
    \caption{0.26 Mil Elem Tensor on CPU}
     \label{fig:randn0.26-cpu-time}
    \end{subfigure}
    \hfill
      \begin{subfigure}{0.24\linewidth}
   \includegraphics[width=1\textwidth]{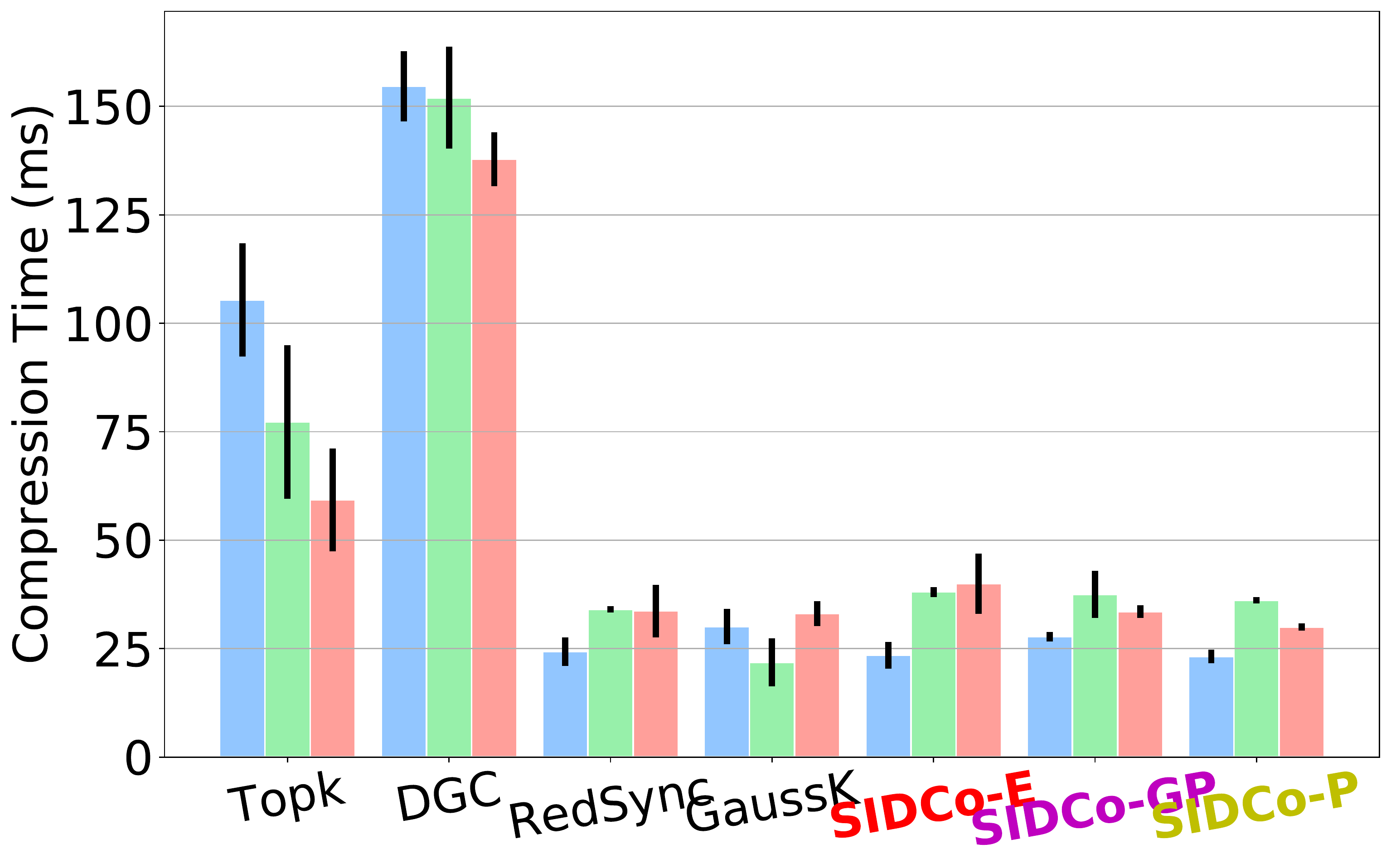}
       \caption{2.6 Mil Elem Tensor on CPU}
        \label{fig:rand2.6-cpu-time}
     \end{subfigure}
        \hfill
    \hfill
  \begin{subfigure}{0.24\linewidth}
	\includegraphics[width=1\textwidth]{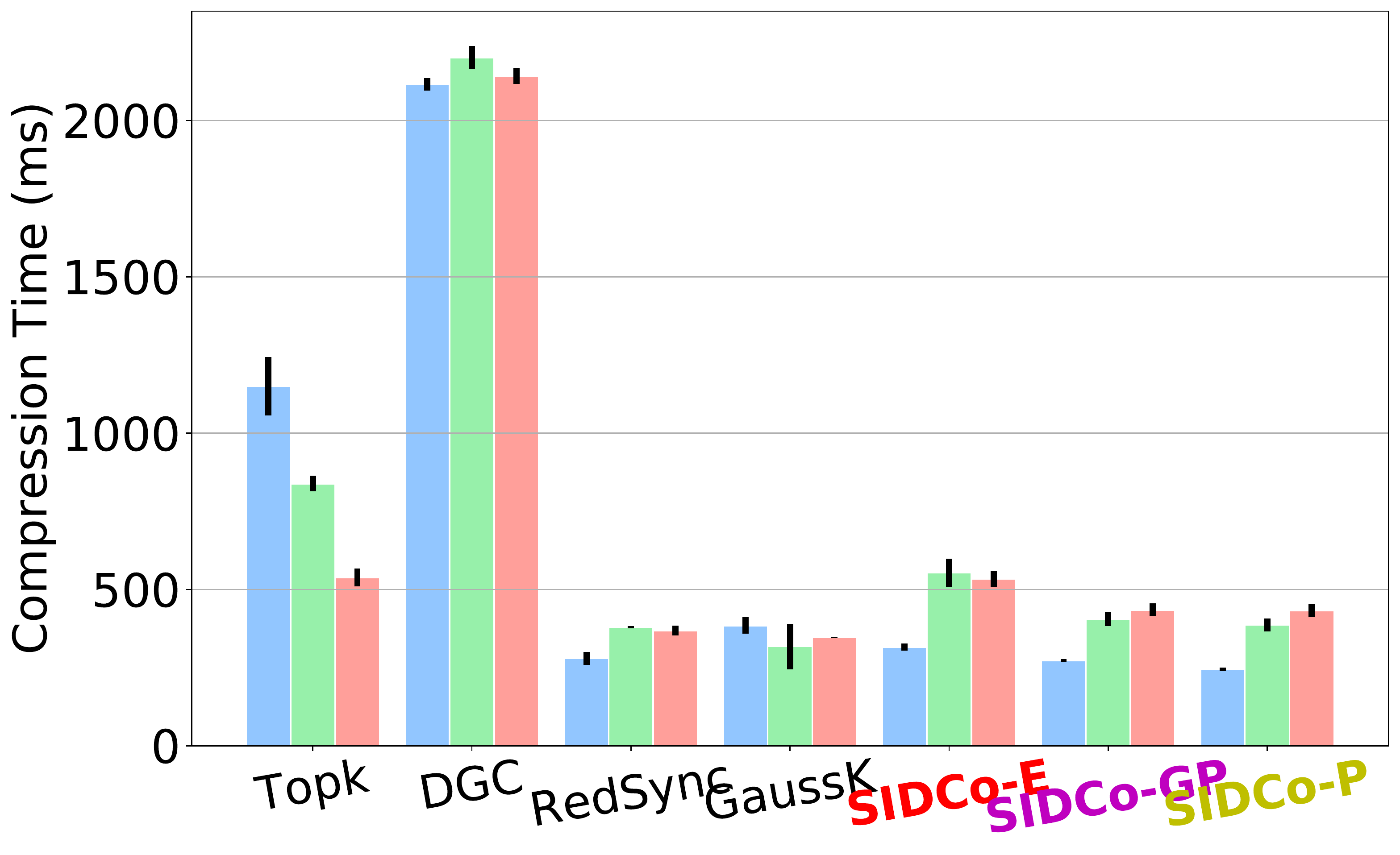}
    	\caption{26 Mil Elem Tensor on CPU}
     \label{fig:rand26M-cpu-time}
    \end{subfigure}
    \hfill
      \begin{subfigure}{0.24\linewidth}
    	\includegraphics[width=1\textwidth]{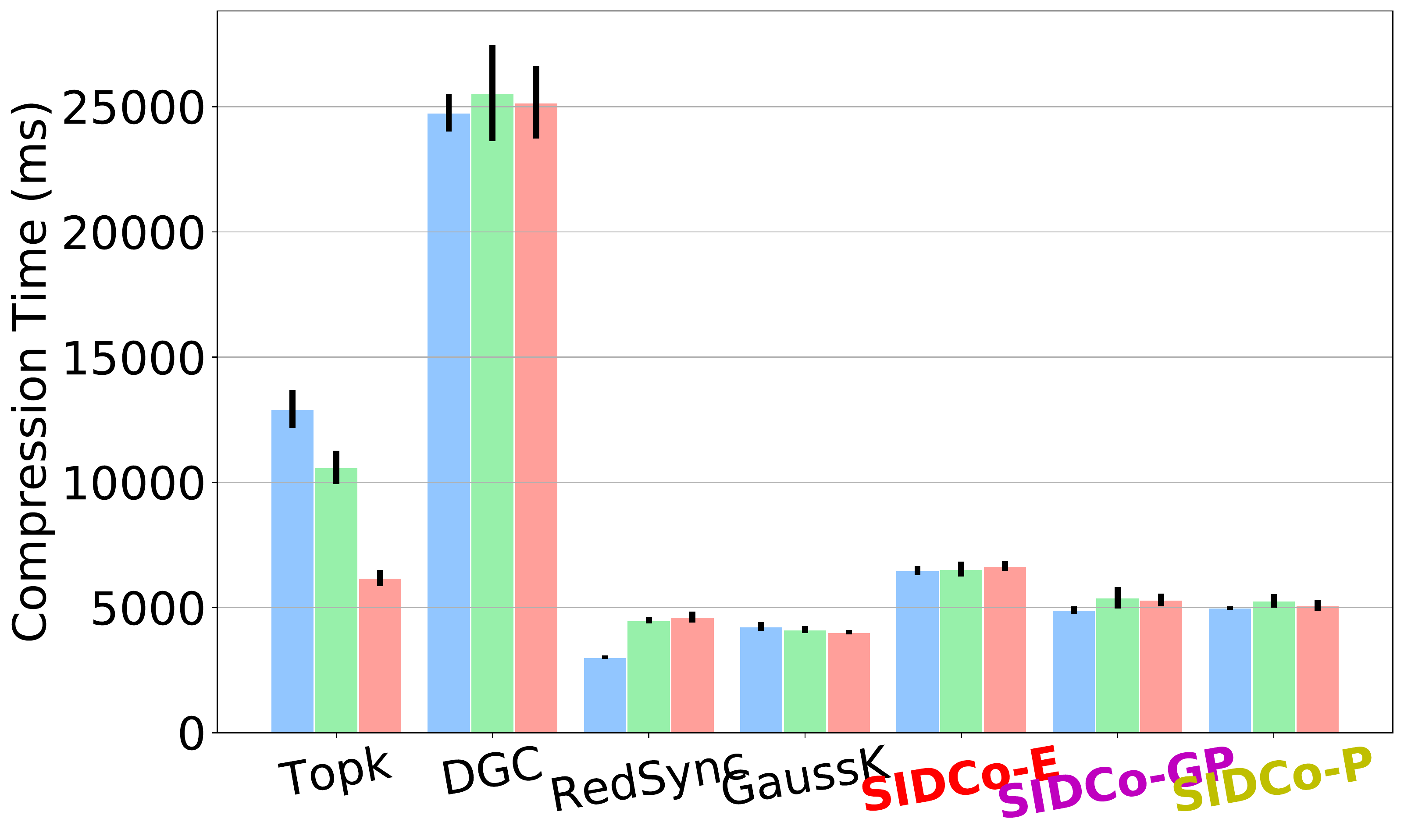}
    	\caption{260 Mil Elem Tensor on CPU}
     \label{fig:rand260M-cpu-time}
    \end{subfigure}
    \caption{Compression latency of synthetic tensors using various compressors and ratios on GPU (a,b,c,d) and CPU (e,f,g,h).}
    \label{fig:microbench-synth-time}
\end{figure*}

\section{Results of all SIDs}
\label{apdx:expalldist}
Here, in \cref{fig:all}, we include the results for the other two \ac{SPD}s discussed in \cref{apdx:GPDthreshold}, i.e., double Gamma and Generalized Pareto. Note that the two multi-stage \ac{SPD} added here are the 1-stage double Gamma followed by $M-1$ stage of Generalized Pareto and multi-stage Generalized Pareto which are refereed to as \scheme\!-GP and \scheme\!-P, respectively. The results and observations are, in general, match the ones we made earlier in \cref{sec:experiments} for \scheme\!-E. However, we observe that, in some cases, \scheme\!-E achieves slightly better speed-ups compared to \scheme\!-GP and \scheme\!-P. This is because of better and slightly lower overhead estimation of the exponential-based threshold which requires only the calculation of the mean of the gradient vector (\cref{algo:algo1}). Specifically, in these cases, \scheme\!-GP which achieves on average the target compression ratio but it tends to have slightly higher variance in terms of the estimation quality (e.g., \cref{fig:ptb-comp-8-all} and \cref{fig:ptb-comp-8-all}). Hence, while variance might be a problem, if it is within the pre-defined tolerance range from the target ratio ($\epsilon_L$,$\epsilon_H$), the impact on the performance would be negligible.

\begin{figure*}[!h]
\captionsetup[subfigure]{justification=centering}
\centering
 \begin{subfigure}[ht]{0.75\linewidth}
  \includegraphics[width=1\linewidth]{Figures/experiments/legend3.pdf}
 \end{subfigure}
  \\
  \begin{subfigure}[ht]{0.32\linewidth}
    \includegraphics[width=\linewidth]{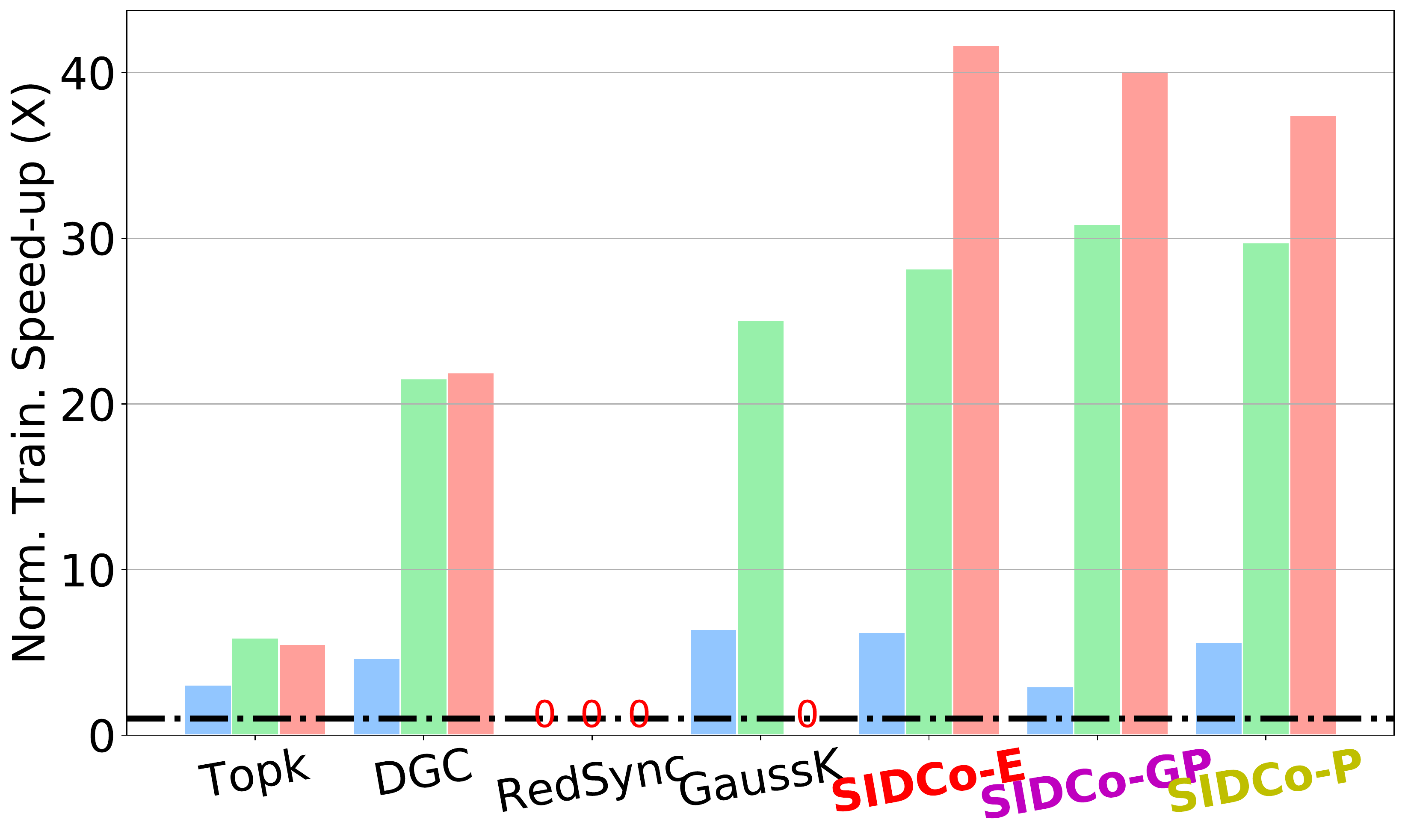}
	\caption{LSTM-PTB (Speed-up).}
	\label{fig:ptb-speedup-8-all}
     \end{subfigure}
     \hfill
	\begin{subfigure}[ht]{0.32\linewidth}
  \includegraphics[width=\linewidth]{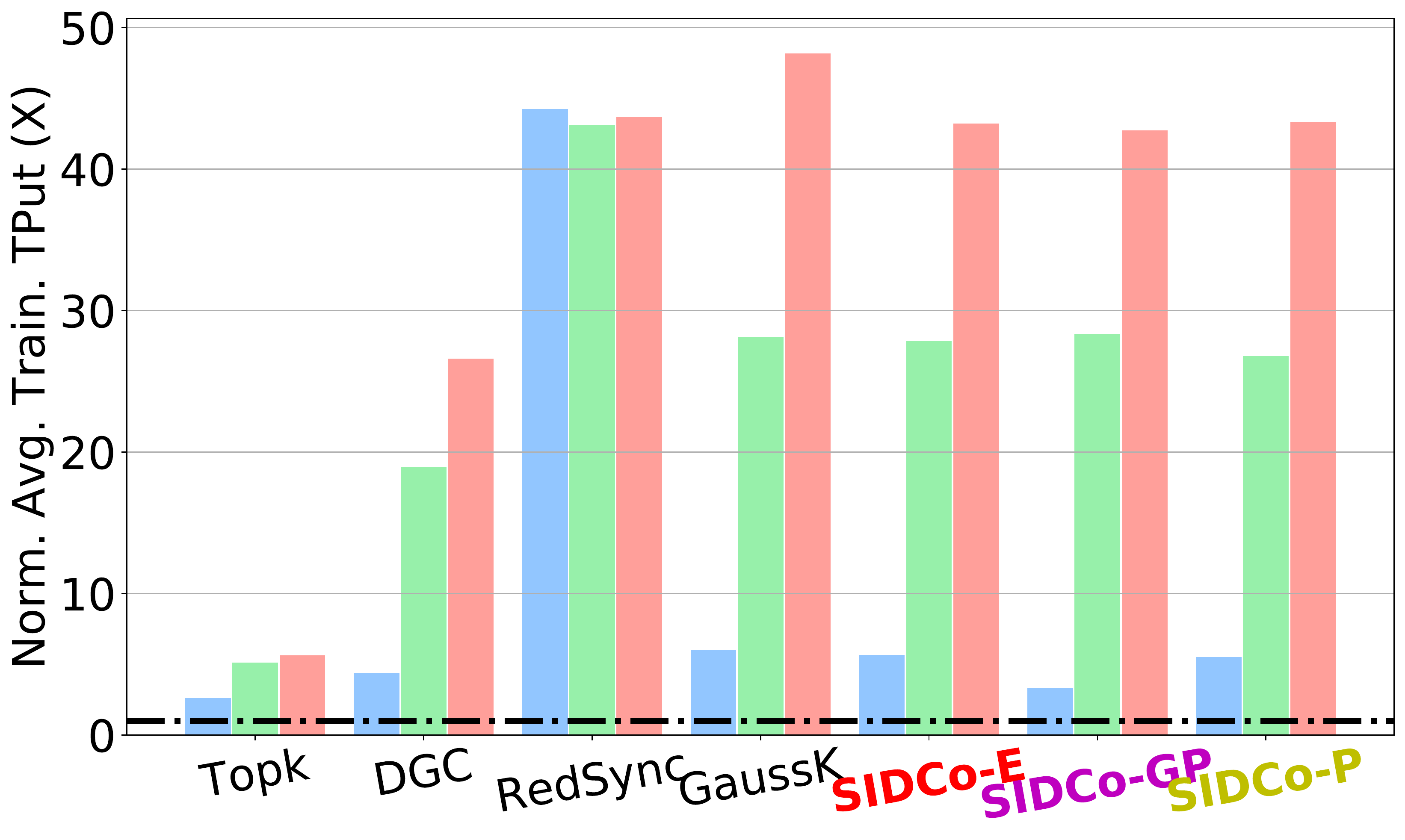}
	\caption{LSTM-PTB (Throughput).}
	\label{fig:ptb-throughput-8-all}
    \end{subfigure}
    \hfill
    \begin{subfigure}[ht]{0.32\linewidth}
    \includegraphics[width=\linewidth]{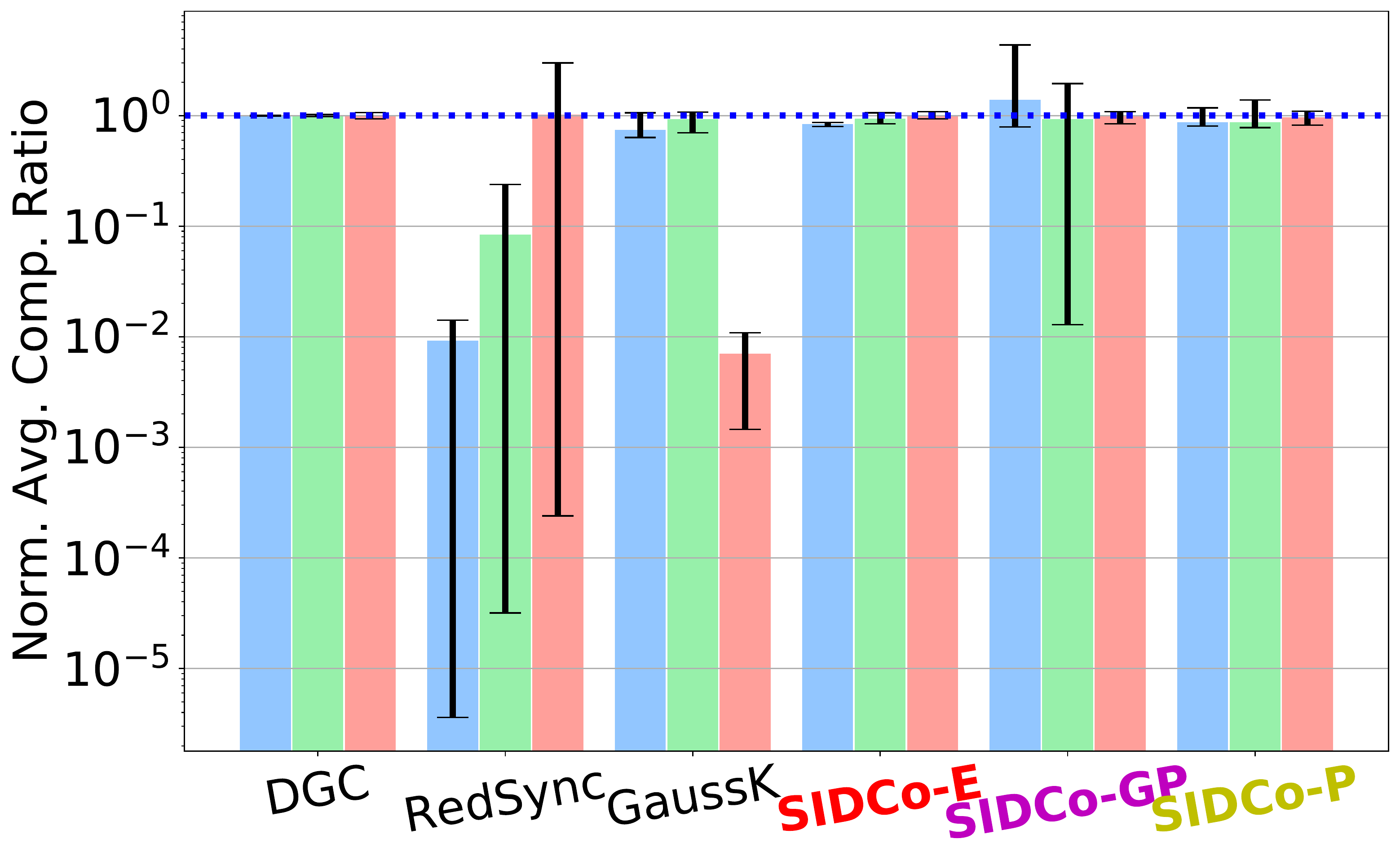}
	\caption{LSTM-PTB (Est. Quality).}
	\label{fig:ptb-comp-8-all}
     \end{subfigure}
     \\
      \begin{subfigure}[ht]{0.32\linewidth}
    \includegraphics[width=\linewidth]{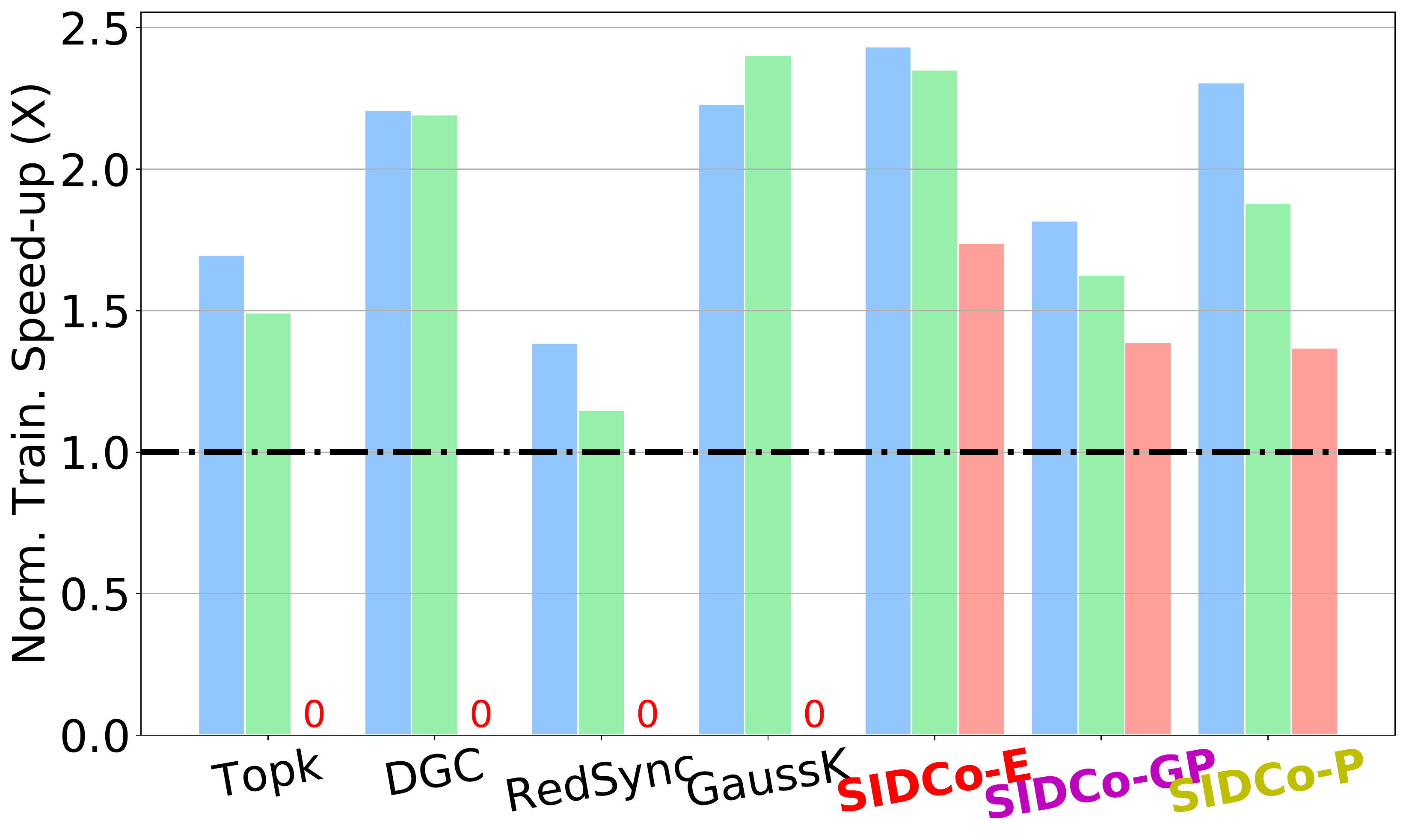}
	\caption{LSTM-AN4 (Speed-up).}
	\label{fig:an4-speedup-8-all}
     \end{subfigure}
     \hfill
	\begin{subfigure}[ht]{0.32\linewidth}
  \includegraphics[width=\linewidth]{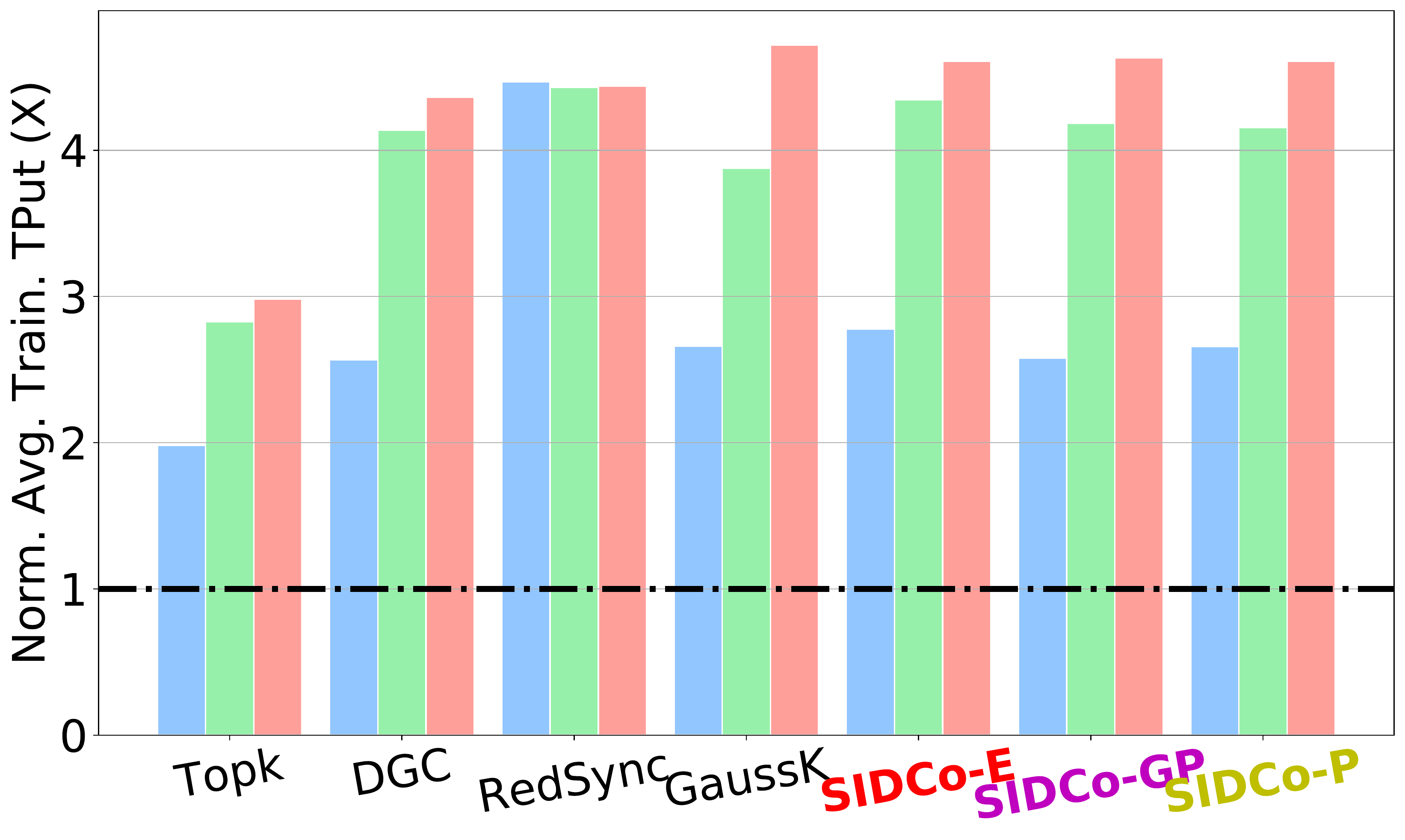}
	\caption{LSTM-AN4 (Throughput).}
	\label{fig:an4-throughput-8-all}
    \end{subfigure}
    \hfill
    \begin{subfigure}[ht]{0.32\linewidth}
    \includegraphics[width=\linewidth]{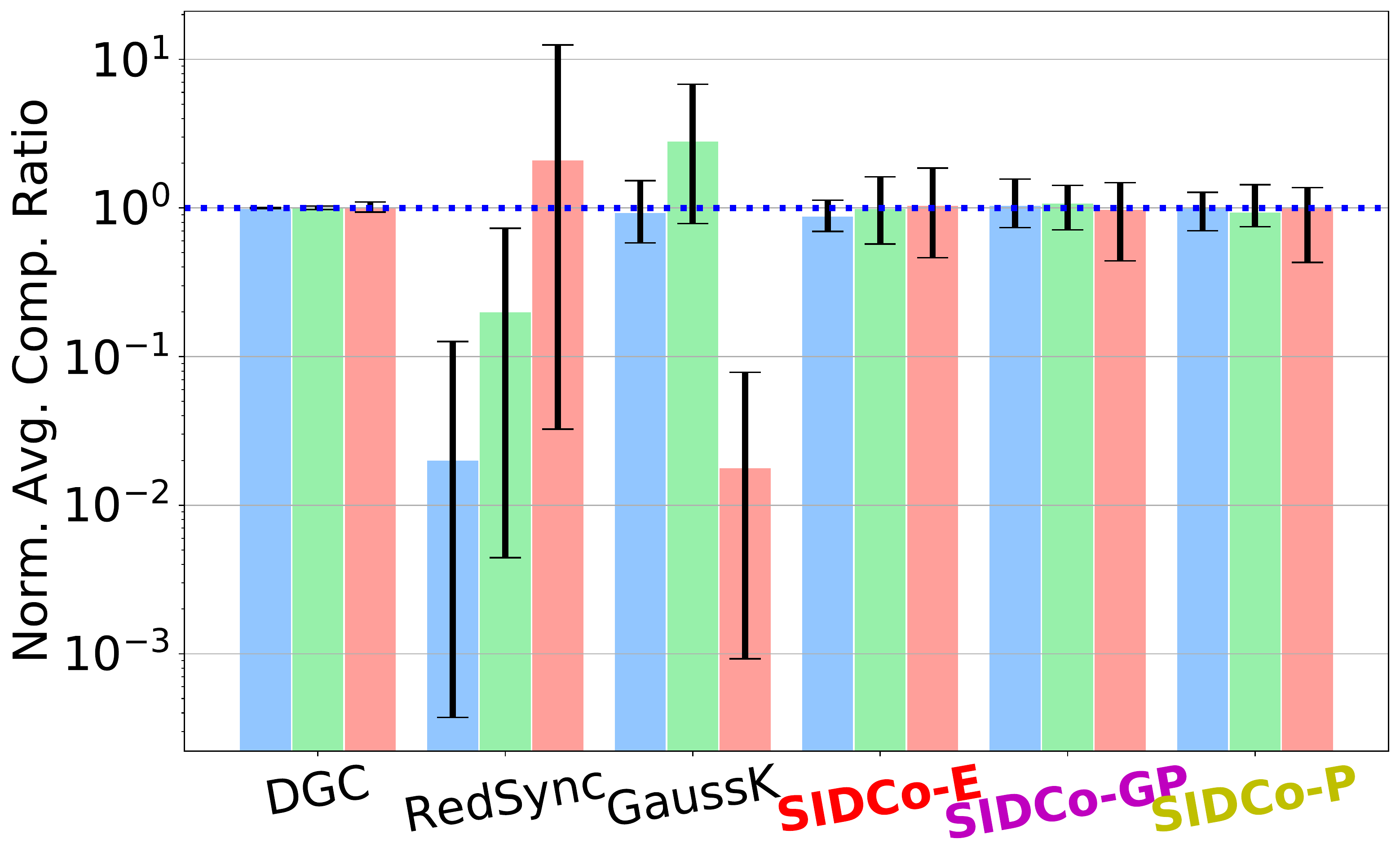}
	\caption{LSTM-AN4 (Est. Quality).}
	\label{fig:an4-comp-8-all}
     \end{subfigure}
     \\
  \begin{subfigure}[ht]{0.32\linewidth}
    \includegraphics[width=\linewidth]{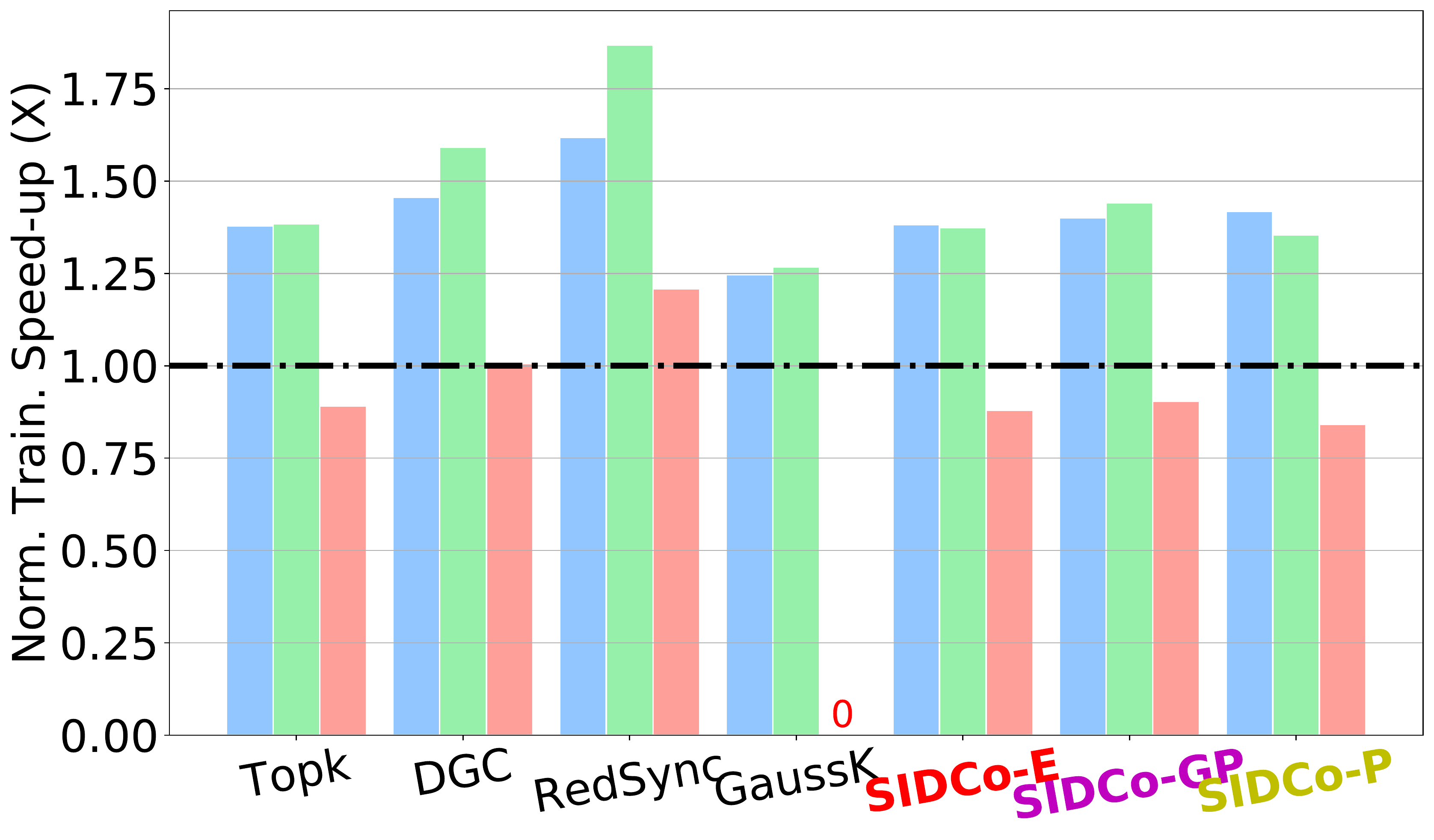}
	\caption{ResNet20-CIFAR10 (Speedup).}
	\label{fig:resnet20-speedup-8-all}
     \end{subfigure}
     \hfill
     \begin{subfigure}[ht]{0.32\linewidth}
  \includegraphics[width=\linewidth]{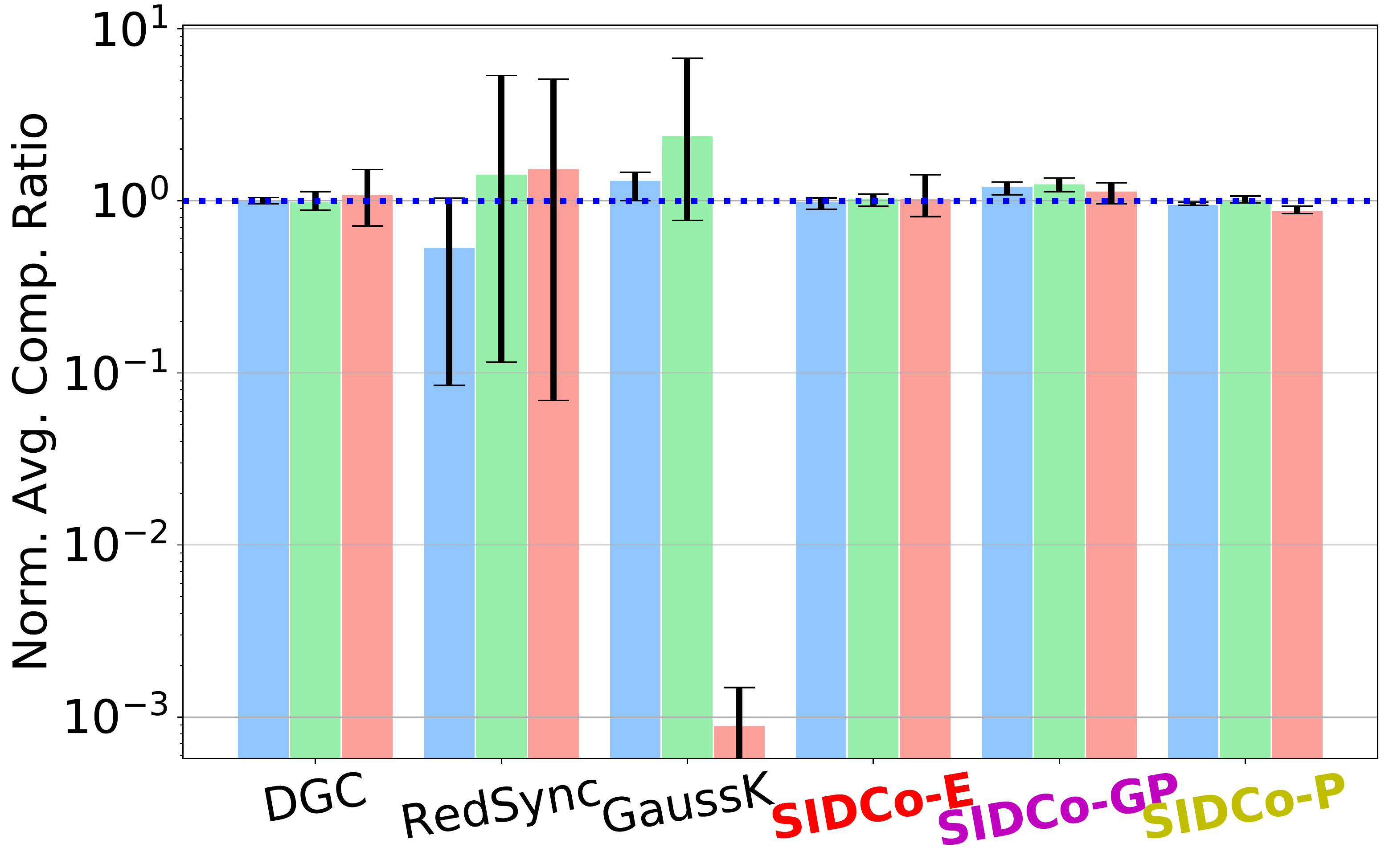}
	\caption{ResNet20-CIFAR10 (Est. Quality).}
	\label{fig:resnet20-good-8-all}
    \end{subfigure}
    \hfill
        \begin{subfigure}[ht]{0.32\linewidth}
    \includegraphics[width=\linewidth]{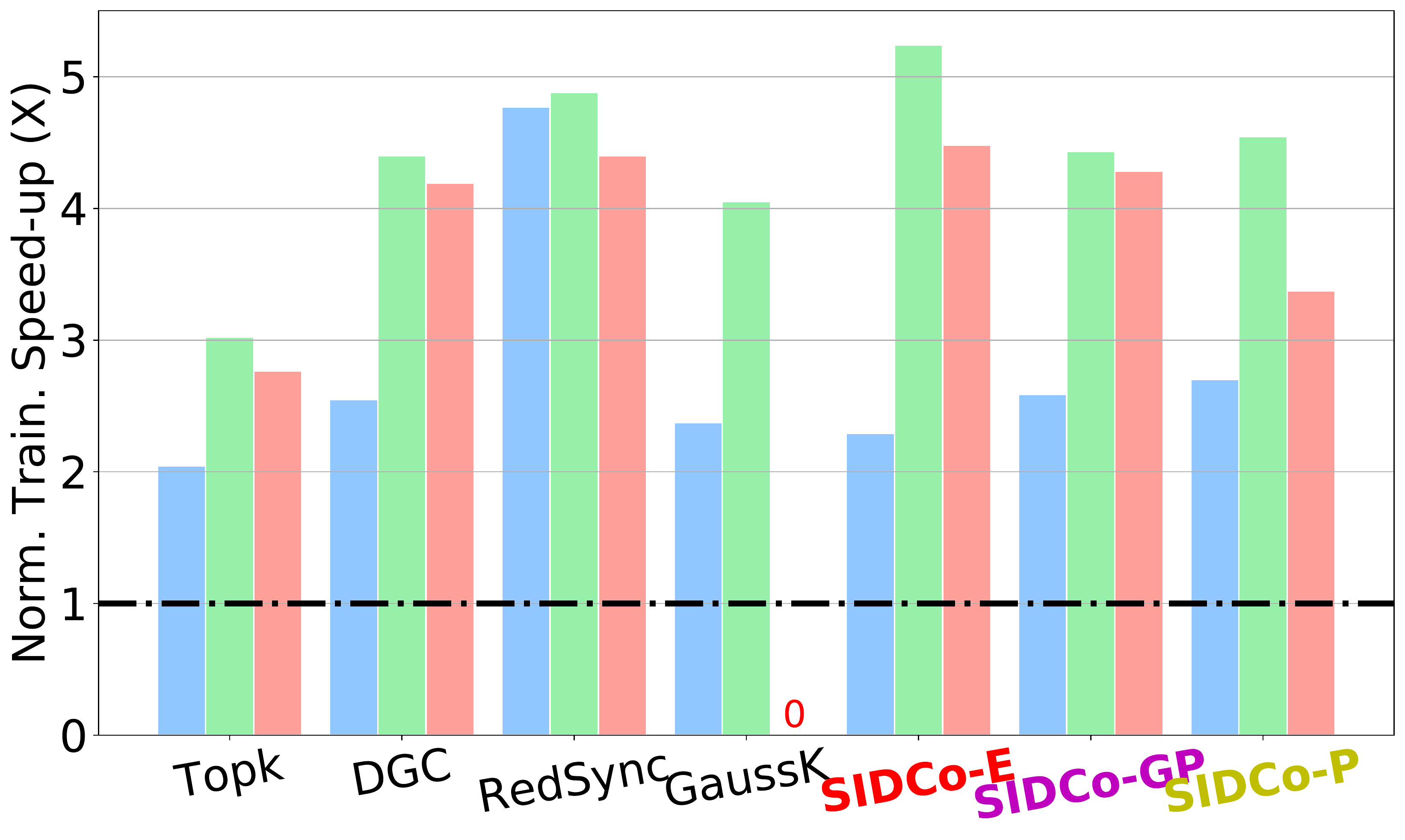}
	\caption{VGG16-CIFAR10 (Speedup).}
	\label{fig:vgg16-speedup-8-all}
     \end{subfigure}
     \\
  \begin{subfigure}[ht]{0.32\linewidth}
    \includegraphics[width=\linewidth]{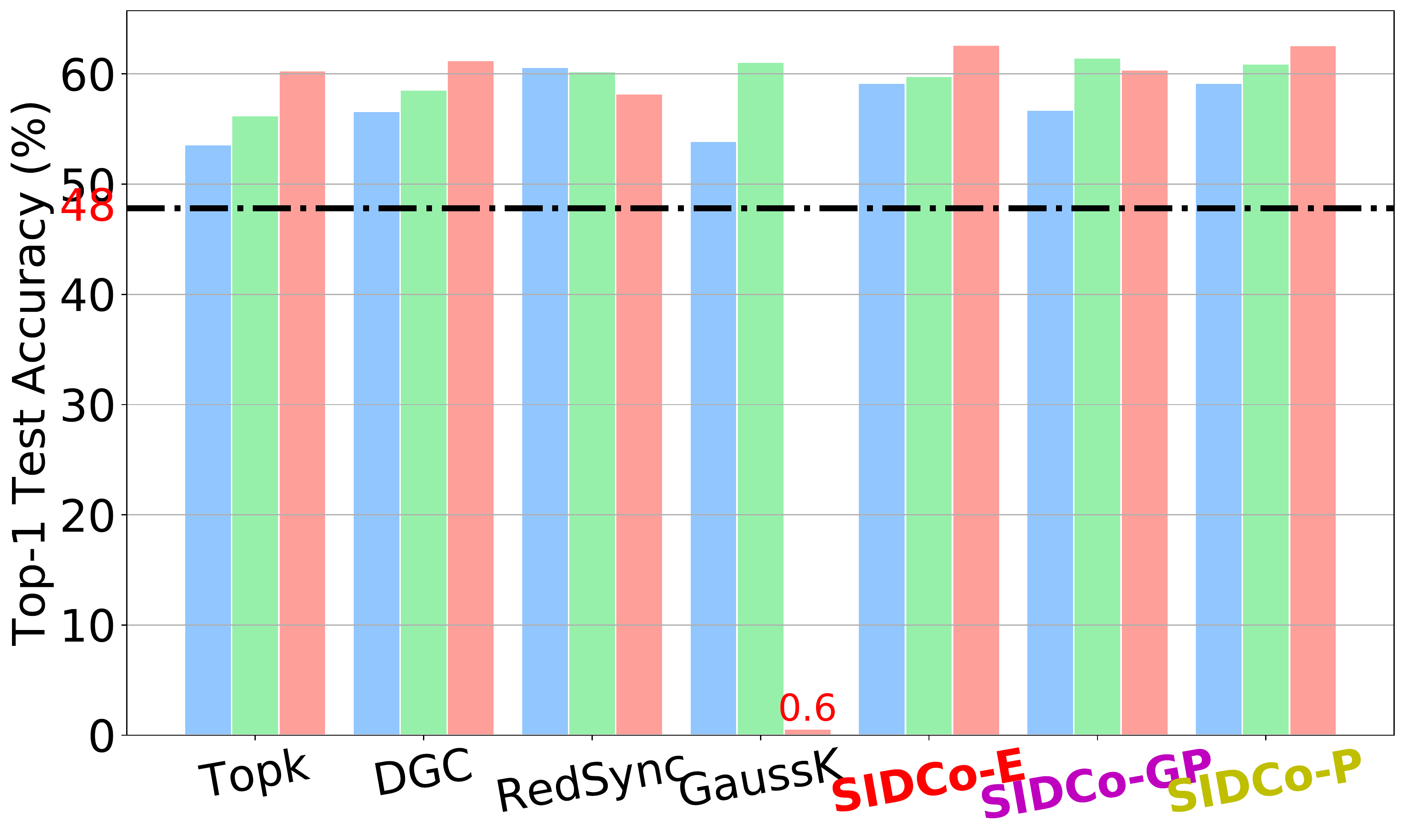}
	\caption{ResNet50-ImageNet (Accuracy).}
	\label{fig:resnet50-speedup-8-all}
     \end{subfigure}
     \hfill
     \begin{subfigure}[ht]{0.32\linewidth}
  \includegraphics[width=\linewidth]{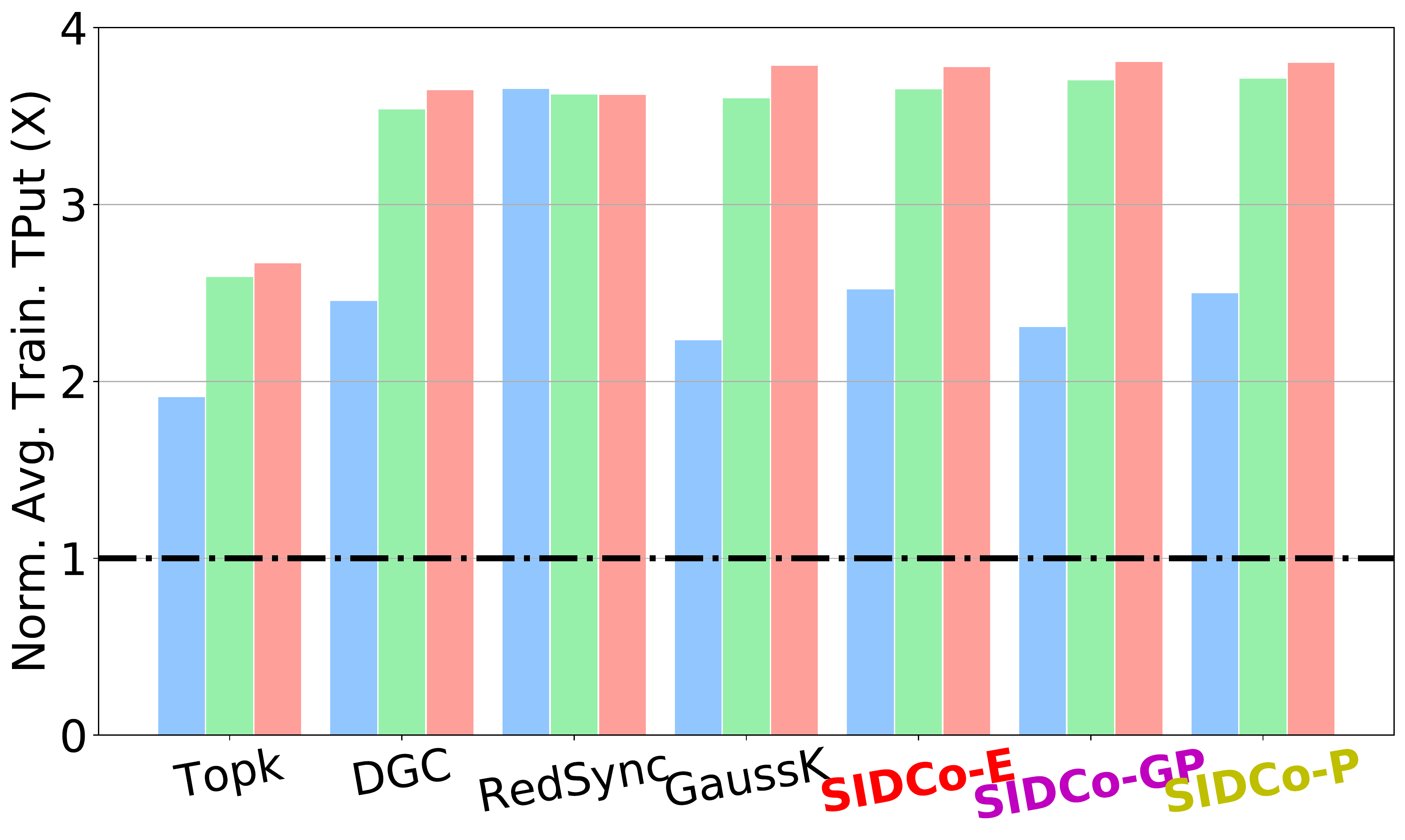}
	\caption{ResNet50-ImageNet (Throughput).}
	\label{fig:resnet50-tput-8-all}
    \end{subfigure}
    \hfill
        \begin{subfigure}[ht]{0.32\linewidth}
    \includegraphics[width=\linewidth]{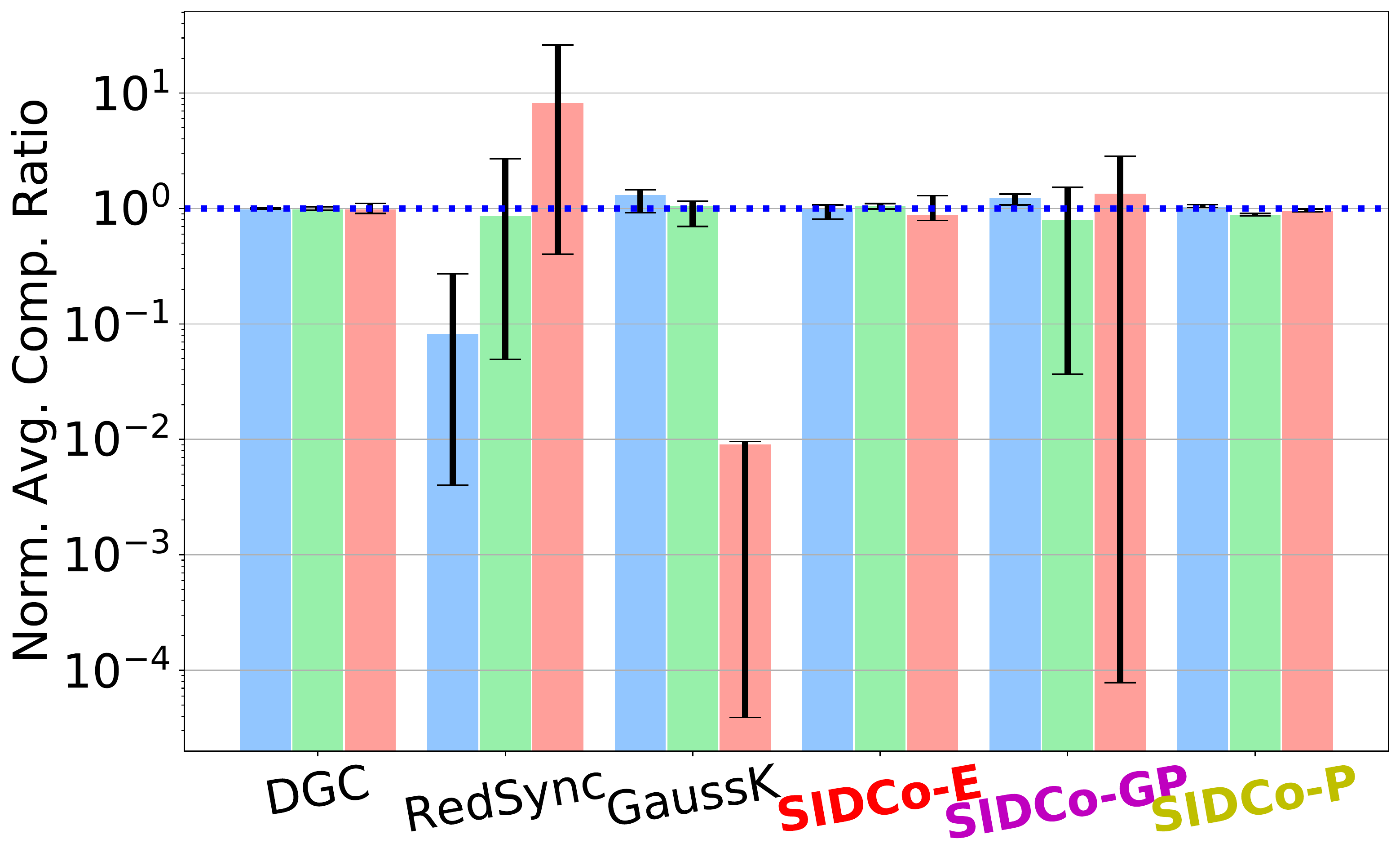}
	\caption{ResNet50-ImageNet (Est. Quality).}
	\label{fig:resnet50-comp-8-all}
     \end{subfigure}
     \\
     \begin{subfigure}[ht]{0.32\linewidth}
    \includegraphics[width=\linewidth]{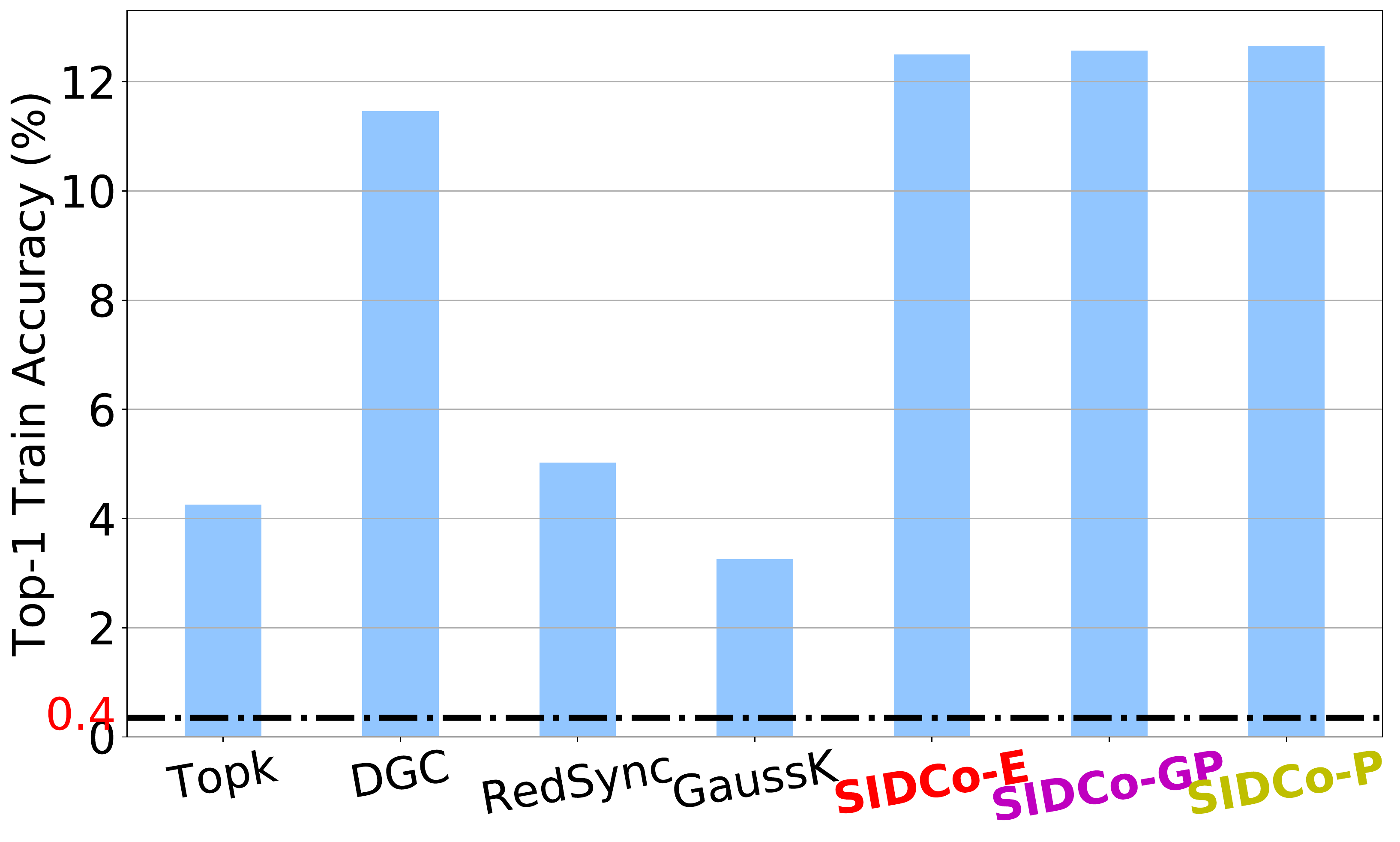}
	\caption{VGG19 on ImageNet (Accuracy).}
	\label{fig:vgg19-speedup-8-all}
     \end{subfigure}
     \hfill
     \begin{subfigure}[ht]{0.32\linewidth}
  \includegraphics[width=\linewidth]{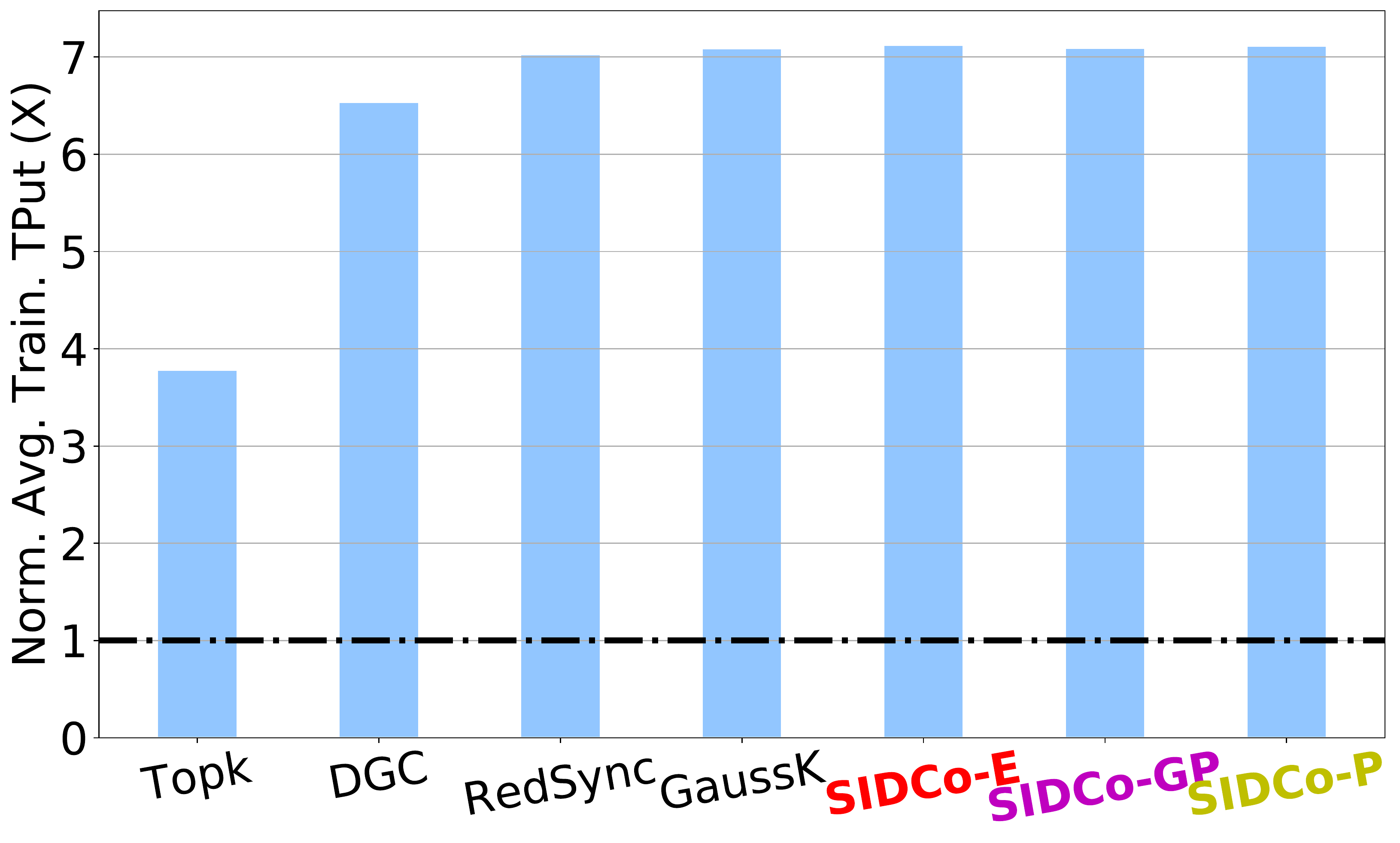}
	\caption{VGG19 on ImageNet (Throughput).}
	\label{fig:vgg19-tput-8-all}
    \end{subfigure}
    \hfill
    \begin{subfigure}[ht]{0.32\linewidth}
    \includegraphics[width=\linewidth]{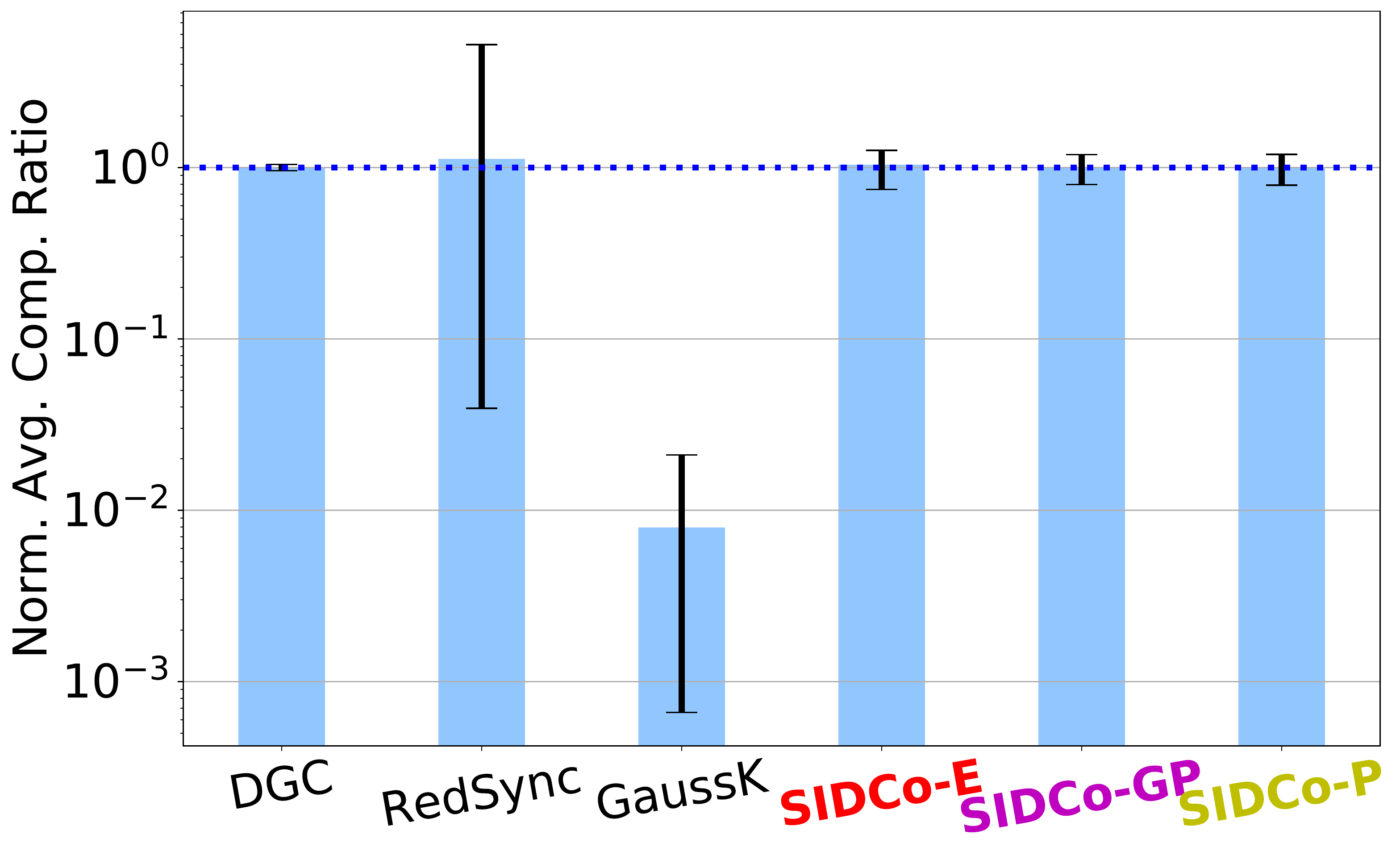}
	\caption{VGG19 on ImageNet (Est. Quality).}
	\label{fig:vgg19-comp-8-all}
     \end{subfigure}
\caption{Performance of using 8 nodes for training LSTM on PTB [(a),(b),(c)], a LSTM on AN4  [(d),(e),(f)], CIFAR10 on ResNet20 [(g),(h)] and VGG16 [(i)], and training ResNet50 [(j), (k), (l)] and VGG19 [(m), (n), (o)] on ImageNet dataset.}
\label{fig:all}
\end{figure*}

\end{document}